\documentclass[11pt]{article}

\usepackage[a4paper,margin=1.05in]{geometry}
\usepackage[T1]{fontenc}
\usepackage[utf8]{inputenc}
\usepackage{lmodern}
\usepackage{microtype}
\usepackage{amsmath,amssymb,amsfonts,amsthm,mathtools}
\usepackage{bbm}
\usepackage{enumitem}
\usepackage{graphicx}
\usepackage{kbordermatrix}
\usepackage{float}

\usepackage{placeins}   %
\usepackage{booktabs}
\usepackage{array}
\usepackage{xcolor}
\definecolor{Accent}{RGB}{32,74,135}
\definecolor{AccentB}{RGB}{140,81,10}
\newif\ifshowtodos
\showtodosfalse
\ifshowtodos
  \usepackage[colorinlistoftodos,textsize=small]{todonotes}
\else
  \usepackage[disable]{todonotes}
\fi

\usepackage{mdframed}
\usepackage{tikz-cd}
\usetikzlibrary{arrows.meta,positioning,backgrounds,fit}
\usepackage{bussproofs}
\usepackage{stmaryrd}
\usepackage[colorlinks=true,linkcolor=blue!50!black,citecolor=blue!50!black,urlcolor=blue!50!black]{hyperref}
\usepackage[nameinlink,capitalize]{cleveref}

\usepackage{aliascnt}

\theoremstyle{plain}
\newtheorem{theorem}{Theorem}[section]
\newtheorem*{informaltheorem}{Theorem}

\newaliascnt{proposition}{theorem}
\newtheorem{proposition}[proposition]{Proposition}
\aliascntresetthe{proposition}

\newaliascnt{lemma}{theorem}
\newtheorem{lemma}[lemma]{Lemma}
\aliascntresetthe{lemma}

\newaliascnt{corollary}{theorem}
\newtheorem{corollary}[corollary]{Corollary}
\aliascntresetthe{corollary}

\theoremstyle{definition}
\newaliascnt{definition}{theorem}
\newtheorem{definition}[definition]{Definition}
\aliascntresetthe{definition}

\newaliascnt{example}{theorem}
\newtheorem{example}[example]{Example}
\aliascntresetthe{example}

\theoremstyle{definition}
\newaliascnt{remark}{theorem}
\newtheorem{remark}[remark]{Remark}
\aliascntresetthe{remark}

\crefname{theorem}{Theorem}{Theorems}
\Crefname{theorem}{Theorem}{Theorems}
\crefname{proposition}{Proposition}{Propositions}
\Crefname{proposition}{Proposition}{Propositions}
\crefname{lemma}{Lemma}{Lemmas}
\Crefname{lemma}{Lemma}{Lemmas}
\crefname{corollary}{Corollary}{Corollaries}
\Crefname{corollary}{Corollary}{Corollaries}
\crefname{definition}{Definition}{Definitions}
\Crefname{definition}{Definition}{Definitions}
\crefname{example}{Example}{Examples}
\Crefname{example}{Example}{Examples}
\crefname{remark}{Remark}{Remarks}
\Crefname{remark}{Remark}{Remarks}

\newcommand{\DFA}{\mathrm{DFA}}
\newcommand{\code}{\texttt{code}}
\newcommand{\step}{\mathrm{step}}
\newcommand{\acc}{\mathrm{acc}}
\newcommand{\rej}{\mathrm{rej}}

\newcommand{\qacc}{q_{\mathrm{acc}}}
\newcommand{\qrej}{q_{\mathrm{rej}}}
\newcommand{\blank}{\square}
\newcommand{\E}{\mathbb E}
\newcommand{\Cov}{\operatorname{Cov}}
\newcommand{\Var}{\operatorname{Var}}
\newcommand{\simplex}{\Delta}
\newcommand{\rank}{\operatorname{rank}}
\newcommand{\id}{\mathrm{id}}
\newcommand{\one}{\mathbbm{1}}
\newcommand{\PSV}{\operatorname{PSV}}
\newcommand{\PSR}{\operatorname{PSR}}
\newcommand{\asym}{\operatorname{Asym}}

\makeatletter
\DeclareRobustCommand{\rvdots}{%
  \vbox{
    \baselineskip4\p@ \lineskiplimit\z@
    \kern-\p@
    \hbox{.}\hbox{.}\hbox{.}
  }}
\makeatother

\newcommand{\lto}{\longrightarrow}
\newcommand{\umappanel}[3][0.32]{%
  \begin{minipage}[t]{#1\textwidth}\centering
    \includegraphics[width=\linewidth]{figures/#3.png}\par\vspace{2pt}
    {\scriptsize #2}
  \end{minipage}}

\usetikzlibrary{calc,shapes.geometric,decorations.pathreplacing}
\usepackage{string-diagrams}
\tikzset{branch/.style={circle, fill, inner sep=0pt, minimum size=4.5pt}}
\usetikzlibrary{decorations.markings}
\tikzset{marrow/.style={
  decoration={markings, mark=at position 0.55 with {\pgfsetlinewidth{0.4pt}\arrow{Stealth}}},
  postaction={decorate}}}

\newcommand{\resolve}{\texttt{resolve}}
\newcommand{\Cfg}{\mathrm{Cfg}}

\title{Interpretability for Turing Machines}
\author{%
  Billy Snikkers$^\wedge$\thanks{$^\wedge$Equal contribution. Resolution. \texttt{billy@resolution.org}}
  \and Rumi Salazar$^\wedge$\thanks{$^\wedge$Equal contribution. University of Melbourne. \texttt{ruminawi.salazar@gmail.com}}
  \and Daniel Murfet\thanks{Resolution. \texttt{murfet@resolution.org}}
  \and Will Troiani\thanks{Resolution. \texttt{will@resolution.org}}%
}
\date{September 2026}

\begin{document}
\maketitle

\begin{abstract}
We show that susceptibilities, an interpretability technique developed for neural networks, can identify the presence of algorithmic structure in Turing machines by probing the local loss landscape of a learning problem for \emph{noisy} Turing machines introduced in \cite{murfet2025pas}. We prove that symmetries and path separation in the algorithm implemented by a Turing machine induce permutation symmetries and low-rank blocks in its susceptibility matrix. We study this empirically on a set of deterministic finite automata (DFAs) and demonstrate that algorithmic features can be recovered by principal component analysis and clustering methods in susceptibility space.
\end{abstract}

\begin{figure}[H]
\centering
\includegraphics[width=0.88\linewidth]{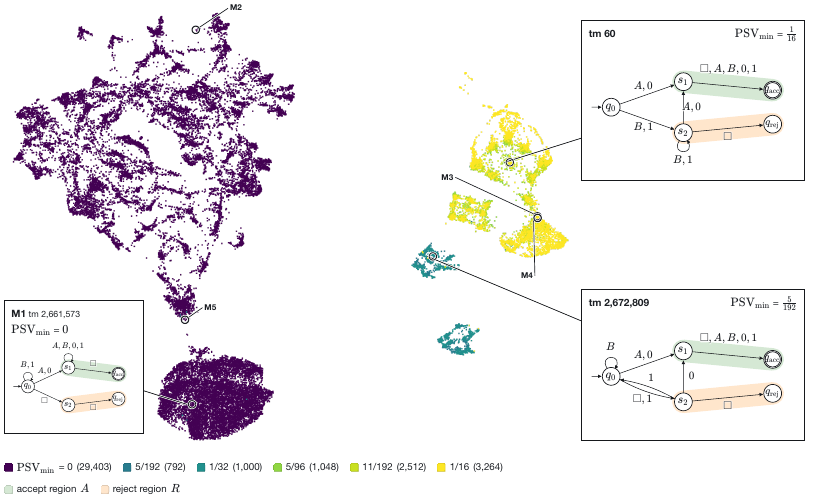}
\caption{\textbf{Turing machines in susceptibility
space.}  A UMAP embedding of the susceptibility matrices of
$38{,}019$ deterministic finite automata (DFAs) that agree with the language $\mathcal L_{\texttt{A}, \texttt{0}}$ of strings containing an $\texttt{A}$ or a $\texttt{0}$ on a finite set of test strings (\cref{sec:expt_task,sec:expt_setup}).
Each DFA is coloured by
$\PSV_{\min}$, which is zero when accepted and rejected inputs pass through disjoint sets of
intermediate states, and grows as they overlap.
The fully path-separated machines form the large purple cluster on the
left. Three example machines are shown of varying path-separability: in $M_1$ every accepted
input passes through $s_1$ only and every rejected input through $s_2$
only, while in the two machines on the right, accepted and rejected
inputs share intermediate states.}
\label{fig:psv_umap_intro}\label{fig:umap_main}
\end{figure}

\tableofcontents

\section{Introduction}\label{sec:introduction}

We study how the internal computational structure of Turing machines is encoded in the local geometry of a loss function which is zero at machines with a specified input-output behaviour, and how this structure may be discovered empirically using susceptibilities. That is, we study interpretability for Turing machines in the spirit of \cite{baker2025studyingsmalllanguagemodels}.

We consider Turing machines over a fixed tape alphabet $\Sigma$ and set of states $Q$, differing only in their transition functions. This allows us to identify the space of Turing machines with the following set
\begin{equation}
  W^\code := \prod_{(\sigma, q) \in \Sigma \times Q} \Sigma \times Q \times \{L, S, R\}.
\end{equation}
The code for a machine $M$ will be denoted $[M] \in W^\code$. If the entry indexed by $(\sigma, q)$ is $(\sigma',q',d)$ then on reading $\sigma$ in state $q$ the machine $M$ will write $\sigma'$, transition to state $q'$ and move in direction $d$ ($L$ for left, $R$ for right and $S$ for stay). We consider $W^\code$ as embedded into a larger space of \emph{noisy Turing machines}
\begin{equation}\label{eq:noisy_tm_space}
  W = \prod_{(\sigma,q)\in \Sigma\times Q}
      \Delta\Sigma\times\Delta Q\times\Delta\{L,S,R\},
\end{equation}
where $\Delta T$ is the set of probability distributions over a finite set $T$. A point of $W$ nearby $[M]$ specifies, for each $(\sigma, q)$, distributions over the symbol $\sigma'$ to write, state $q'$ to transition to and direction to move. We interpret these as \emph{error channels} by which a universal Turing machine simulating $M$ can misread the entries on the description tape of $M$.

Given such a point $w \in W$ and an input string $x$, the simulation in the presence of these errors produces a probability distribution over outputs or final states $y$ (in this paper we consider $y$ as a state) when we simulate the machine for a fixed number of steps $T$
\[
  p(y|x,w)=\Delta\step^{T}(x,w)_y.
\]
Here $\Delta\step$ denotes the function which propagates uncertainty about the description tape, as described above, to uncertainty about the configuration of the simulated machine after one step. We call this a \emph{smooth relaxation} of the step function of the universal Turing machine and recall the background from \cite{clift2020derivatives,clift2021geometryprogramsynthesis,murfet2025pas} in \cref{sec:smooth_cycle}.

The Turing machines we consider are \emph{deciders}: each has an accept state $\qacc$ and a reject state $\qrej$, and it decides membership of a language by ending its run on each input in one of the two. Fix a distribution $q(x)$ over input strings with finite support $I \subseteq \Sigma^*$, and a function $y = y(x)$ specifying the desired halting state on input $x$. Associated to this specification is the \emph{loss function}
\[
  H(w)=\E_{x\sim q}\bigl[h_x(w)^2\bigr],
  \qquad
  h_x(w)=1-p(y(x)|x,w)\,.
\]
We call a Turing machine $M$ a \emph{classical solution} if $M(x) = y(x)$ for all $x \in I$, by which we mean that after $T$ steps of execution on input $x$ the machine $M$ is in state $y(x)$. Then $H([M]) = 0$. With this notation we can restate the main goal of the paper: to study how the internal computational structure of a classical solution $M$ is encoded in the local geometry of $H$ near $[M]$.

It has been argued \cite{murfet2025pas} that the structure of a learned program is reflected in the local geometry of the loss that guides the synthesis of that program. The relevant meaning of local geometry is the following: given a finite dataset $D_n$ drawn from some true distribution, the posterior $p(w|D_n)$ concentrates, as $n$ grows, near the minima of the loss, and its asymptotic behaviour is governed by the geometry of the loss near those minima \cite{watanabe2009algebraic}. This perspective on Bayesian learning has been called \emph{structural Bayesianism} \cite{murfet2025pas}. The interpretability programme built on it \cite{baker2025studyingsmalllanguagemodels, gordon2025lang3} defends the hypothesis that the internal structure of a learned parameter is implicitly encoded in the loss near it. The problem of interpretability is then to make this structure explicit from the germ of the loss at the parameter under study. The present paper establishes instances of this hypothesis in a setting where internal structure is exactly known: the truth-model-prior triple of \cite{murfet2025pas} (restated in \cref{sec:bayes_setup}), whose parameters are noisy Turing machines.

In the regular case the geometric information in this germ reduces to its low-order part, since the value and the Hessian determine the germ up to diffeomorphism. In singular models this reduction fails and higher-order information becomes relevant. Our method of choice for probing it is the theory of susceptibilities, long studied in statistical mechanics \cite{yeomans1992statistical}, reformulated for neural networks in \cite{baker2025studyingsmalllanguagemodels}, and adapted here to noisy Turing machines. Each square on the description tape is a \emph{component} of the machine's parameter. Perturbing the input distribution towards an input $x$ gives a perturbed loss $H^\varepsilon$ and a Gibbs distribution $p_{\beta}^{\varepsilon} \propto \exp\{-\beta H^\varepsilon\}\,\varphi$ over $W$, and the \emph{susceptibility} (\cref{def:susceptibility}) of a machine $M$ at the input $x$ and a component $C$ is the linear response of a \emph{component observable} $\phi_C$ to this perturbation
\begin{equation}\label{eq:sus_intro}
  \chi_x^C([M]) := \frac{1}{\beta}\left.\frac{\partial}{\partial \varepsilon}\,
  \E_{p_{\beta}^{\varepsilon}} \bigl[\phi_C\bigr]\right|_{\varepsilon=0},
\end{equation}
where $[M]$ is the code of $M$ (a point of the space \eqref{eq:noisy_tm_space}) and $\varphi$ is a prior on that space.
Arranging the values $\chi_x^C([M])$ over all inputs and components gives the \emph{susceptibility matrix} $\chi([M])$.

This paper makes two main contributions. Both are theorems relating a property of the transitions defining $M$ to a corresponding property of $\chi([M])$. The first (\cref{sec:intro_rank_bound}) concerns how $M$ uses its intermediate states to decide. If the non-initial states are \emph{not} used to distinguish accepted and rejected strings (each acts as a pass-through to a constant final state), then certain blocks of $\chi([M])$ have rank $\leq 2$. The second (\cref{sec:intro_symmetry}) concerns symmetries: when $M$ is unchanged by a relabelling of the symbols and states in its transitions, $\chi([M])$ is fixed by the corresponding permutation of its rows and columns.

\subsection{Path separation and low-rank blocks}\label{sec:intro_rank_bound}

We pose a decision problem: with tape alphabet $\Sigma := \{\blank, \texttt{A}, \texttt{B}, \texttt{0}, \texttt{1}\}$, decide the language
\begin{equation}\label{eq:LA0}
  \mathcal L_{\texttt{A},\texttt{0}}
  :=
  \{x\in\{\texttt{A},\texttt{B}\}^*\cup\{\texttt{0},\texttt{1}\}^* : x \text{ contains an } \texttt{A} \text{ or a } \texttt{0}\}.
\end{equation}
Fixing a set of states $Q = \{q_0, s_1, s_2, \qacc, \qrej\}$, with $q_0$ the initial state, $\qacc$ the accept state, $\qrej$ the reject state, we study the deciders $M$, which may or may not make use of the auxiliary states $s_1, s_2$ to decide if $x \in \mathcal L_{\texttt{A},\texttt{0}}$. A ``usage'' of $s_i \in Q$ means an occurrence of $s_i$ in the sequence (or \emph{path}) of states visited by $M$ on an input $x$. We consider $s_i$ non-trivially used if there are $x,x'$ which both use it but ultimately end in different final states ($\qacc$ and $\qrej$), and thus we are interested in the overlap between the non-initial states seen during \emph{accepting runs} (ending in $\qacc$) and \emph{rejecting runs} (ending in $\qrej$).
\Cref{fig:sep_machines_intro} shows the transition graphs of two deciders $M_1$ and $M_4$ of $\mathcal L_{\texttt{A},\texttt{0}}$, both \emph{deterministic finite automata} (DFAs): read-only and always moving right, so that a transition is determined by its next state, $(\sigma, q) \mapsto q'$.  The machine $M_1$ has no such overlap (only accepting runs see $s_1$ and only rejecting runs see $s_2$). We call such machines \emph{path-separable} (\cref{def:path_separation_partition}). The machine $M_4$ is not path-separable; it uses both $s_1$ and $s_2$ for detecting an $\texttt{A}$ or $\texttt{0}$, so we cannot partition them by whether they see only accepted or only rejected strings.
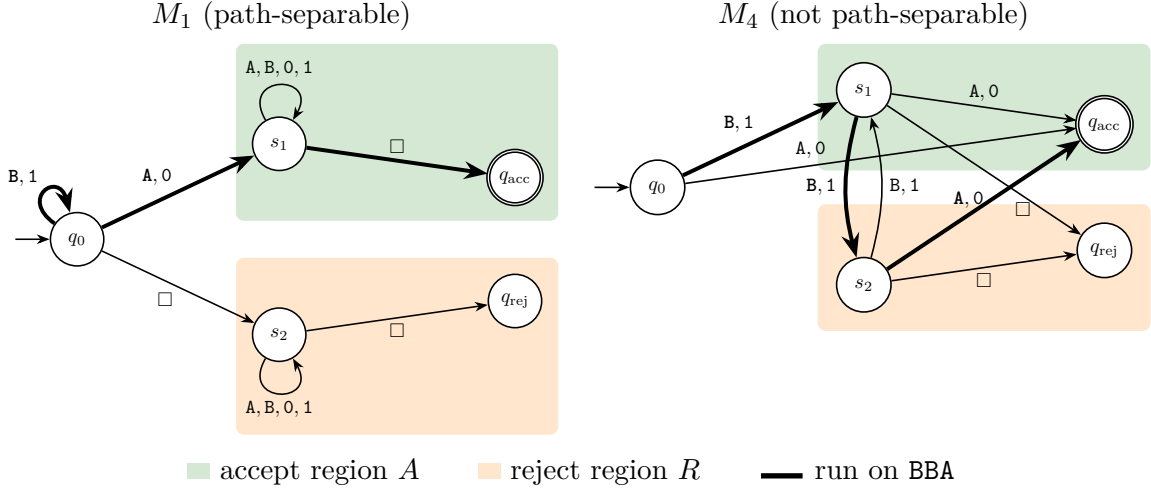
\begin{figure}[H]
\centering
\definecolor{accreg}{RGB}{213,232,212}%
\definecolor{rejreg}{RGB}{255,230,204}%
\begin{minipage}[t]{0.5\linewidth}\centering
$M_1$\ (path-separable)\\[4pt]
\resizebox{0.94\linewidth}{!}{%
\begin{tikzpicture}[>=Stealth,
    st/.style={circle, draw, thick, fill=white, minimum size=0.95cm, inner sep=0pt},
    edge/.style={->, thick},
    used/.style={->, line width=2.0pt},
    every node/.append style={font=\small}]
  \node[st] (q0) at (0,0) {$q_0$};
  \node[st] (s1) at (3.6,1.7) {$s_1$};
  \node[st] (s2) at (3.6,-1.7) {$s_2$};
  \node[st, double] (qa) at (7.8,1.1) {$\qacc$};
  \node[st] (qr) at (7.8,-1.1) {$\qrej$};
  \coordinate (s1top) at (3.6,3.2);
  \coordinate (s2bot) at (3.6,-3.2);
  \begin{scope}[on background layer]
    \node[fill=accreg, rounded corners, fit=(s1)(qa)(s1top), inner sep=8pt] {};
    \node[fill=rejreg, rounded corners, fit=(s2)(qr)(s2bot), inner sep=8pt] {};
  \end{scope}
  \draw[edge] (q0) -- node[below left=-2pt] {$\blank$} (s2);
  \draw[edge] (s1) to[out=120,in=60,looseness=5] node[above] {$\texttt{A},\texttt{B},\texttt{0},\texttt{1}$} (s1);
  \draw[edge] (s2) to[out=240,in=300,looseness=5] node[below] {$\texttt{A},\texttt{B},\texttt{0},\texttt{1}$} (s2);
  \draw[edge] (s2) -- node[below=-1pt] {$\blank$} (qr);
  \draw[used] (q0) to[out=145,in=105,looseness=6] node[above left=-2pt] {$\textbf{\texttt{B}},\texttt{1}$} (q0);
  \draw[used] (q0) -- node[above left=-2pt] {$\textbf{\texttt{A}},\texttt{0}$} (s1);
  \draw[used] (s1) -- node[above=-1pt] {$\boldsymbol{\blank}$} (qa);
  \draw[->, thick] (-1.1,0) -- (q0);
\end{tikzpicture}}
\end{minipage}%
\begin{minipage}[t]{0.5\linewidth}\centering
$M_4$\ (not path-separable)\\[4pt]
\resizebox{0.94\linewidth}{!}{%
\begin{tikzpicture}[>=Stealth,
    st/.style={circle, draw, thick, fill=white, minimum size=0.95cm, inner sep=0pt},
    edge/.style={->, thick},
    used/.style={->, line width=2.0pt},
    every node/.append style={font=\small}]
  \node[st] (q0) at (0,0) {$q_0$};
  \node[st] (s1) at (3.6,1.7) {$s_1$};
  \node[st] (s2) at (3.6,-1.7) {$s_2$};
  \node[st, double] (qa) at (7.8,1.1) {$\qacc$};
  \node[st] (qr) at (7.8,-1.1) {$\qrej$};
  \begin{scope}[on background layer]
    \node[fill=accreg, rounded corners, fit=(s1)(qa), inner sep=9pt] {};
    \node[fill=rejreg, rounded corners, fit=(s2)(qr), inner sep=9pt] {};
  \end{scope}
  \draw[edge] (q0) -- node[above, pos=0.32] {$\texttt{A},\texttt{0}$} (qa);
  \draw[edge] (s2) to[bend right=15] node[right] {$\texttt{B},\texttt{1}$} (s1);
  \draw[edge] (s1) -- node[above=-1pt] {$\texttt{A},\texttt{0}$} (qa);
  \draw[edge] (s1) -- node[pos=0.7, below=-1pt] {$\blank$} (qr);
  \draw[edge] (s2) -- node[below=-1pt] {$\blank$} (qr);
  \draw[used] (q0) -- node[above left=-2pt] {$\textbf{\texttt{B}},\texttt{1}$} (s1);
  \draw[used] (s1) to[bend right=15] node[left] {$\textbf{\texttt{B}},\texttt{1}$} (s2);
  \draw[used] (s2) -- node[pos=0.42, above=-1pt] {$\textbf{\texttt{A}},\texttt{0}$} (qa);
  \draw[->, thick] (-1.1,0) -- (q0);
\end{tikzpicture}}
\end{minipage}\\[8pt]
\tikz{\fill[accreg] (0,0) rectangle (0.3,0.2);}~accept region $A$\qquad
\tikz{\fill[rejreg] (0,0) rectangle (0.3,0.2);}~reject region $R$\qquad
\tikz{\draw[line width=2pt] (0,0) -- (0.55,0);}~run on $\texttt{BBA}$
\caption{Two structurally different deciders of $\mathcal L_{\texttt{A},\texttt{0}}$ which are deterministic finite automata (DFAs; read-only and always move right), depicted by their transition graph on the set of states. An edge $q \to q'$ labelled $\sigma$ denotes the transition $(\sigma, q) \mapsto (\sigma, q', R)$. \textbf{Left:} the machine $M_1$ is path-separable (only accepted $x$ visit $s_1$ and only rejected visit $s_2$).  The shaded regions partition the non-initial states into an accept side $\{s_1, \qacc\}$ and a reject side $\{s_2, \qrej\}$, and runs of $M_1$ never enter the region of the opposite outcome.  \textbf{Right:} in $M_4$ the bold run of the accepted input $\texttt{BBA}$ passes through both regions.}
\label{fig:sep_machines_intro}
\end{figure}

Fix a partition $Q \setminus \{q_0\} = A \sqcup R$ of the non-initial states with
$\qacc \in A$ and $\qrej \in R$.  The inputs are
partitioned, $I = I_{\acc} \sqcup I_{\rej}$, into accepted and rejected inputs.
The component observables, indexed by $C = (\sigma, q)$, measure the loss over subspaces $\Delta Q$,
which vary the choice $q'$ in the transition $(\sigma, q) \mapsto q'$, and we partition them
as $\mathcal C_A \sqcup \mathcal C_R \sqcup \mathcal C_0$, according to $q \in A$, $q \in R$ or $q = q_0$. The susceptibility matrix therefore decomposes into blocks:
\begin{equation}\label{eq:sus_blocks}
  \chi([M]) = \kbordermatrix{
    & \mathcal C_A & \mathcal C_R & \mathcal C_0 \\
    I_{\acc} & X_{\acc, A} & X_{\acc, R} & X_{\acc, 0} \\
    I_{\rej} & X_{\rej, A} & X_{\rej, R} & X_{\rej, 0}
  }.
\end{equation}
\begin{informaltheorem}[Informal]
Let $M$ be a classical solution and let $A \sqcup R$ be a \emph{path separation partition} of $M$ (\cref{def:path_separation_partition}).  Then
\[
  \rank(X_{\acc,R})\le 2,
  \qquad
  \rank(X_{\rej,A})\le 2.
\]
See \cref{thm:rank_bound} for the formal statement (including conditions on $\Delta\step$).
\end{informaltheorem}
In \cref{sec:expt_rank_stats} we compute these ranks numerically, on estimated susceptibilities of the classical DFA solutions of \cref{sec:expt_setup}, and find the estimates agree with the bound.  They also exhibit a converse: a DFA in this set is path-separable if and only if some partition $Q \setminus \{q_0\} = A \sqcup R$ has both $\rank(X_{\acc, R}) \le 2$ and $\rank(X_{\rej, A}) \le 2$.  \Cref{fig:psv_umap_intro} shows path separability organising the same solution set into clusters in susceptibility space, with $M_1$ (resp. $M_4$) of \cref{fig:sep_machines_intro} in the left path-separable cluster (resp. right non-path-separable cluster). Each DFA is coloured by $\PSV_{\min}(M)$, the minimum over partitions $A \sqcup R$ of the \emph{path separation violation} (\cref{def:psv}), which counts how often runs enter the wrong side. It is zero exactly for the path-separable machines, e.g.\ $\PSV_{\min}(M_1) = 0$ and $\PSV_{\min}(M_4) \neq 0$.

\subsection{Recoding symmetries and permutation symmetries}\label{sec:intro_symmetry}

Our second contribution concerns how symmetries of the data are
reflected in symmetries of the susceptibilities.  The language $\mathcal L_{\texttt{A},\texttt{0}}$ of
\eqref{eq:LA0} admits a symmetry in that the decision function $f: \{\texttt{A},\texttt{B}\}^* \cup \{\texttt{0},\texttt{1}\}^* \to \{\qacc, \qrej\}$, accepting strings in $\mathcal L_{\texttt{A},\texttt{0}}$, is
invariant under precomposition with the involution
$\texttt{A} \leftrightarrow \texttt{0},\, \texttt{B} \leftrightarrow \texttt{1}$ (denoted $\theta: \Sigma \to \Sigma$).  Within the set of Turing machines
accepting this language, there are those which reflect this symmetry, and
those which do not. For instance, the DFAs $M_1, M_4$
(\cref{fig:sep_machines_intro}) and $M_3$ (\cref{fig:recoding_action})
all reflect it. Exchanging $\texttt{A} \leftrightarrow \texttt{0}$ and
$\texttt{B} \leftrightarrow \texttt{1}$ throughout the transitions of $M_1$ and $M_4$ returns the
same transition graph, while for $M_3$ it returns the transition graph
with the roles of $s_1$ and $s_2$ interchanged.  Thus $M_1$ and $M_4$ carry the
recoding symmetry $(\theta, \id_Q)$, while $M_3$ carries
$(\theta, s_1{\leftrightarrow}s_2)$ (\cref{def:recoding_symmetry});
\cref{fig:recoding_action} depicts the exchanges acting on $M_3$.

\begin{figure}[H]
\centering
$
\underset{\textstyle M_3}{\vcenter{\hbox{\resizebox{0.39\linewidth}{!}{%
\begin{tikzpicture}[
    >=Stealth, thick,
    state/.style={circle, draw, thick, minimum size=1.1cm, inner sep=0pt},
    accept/.style={state},
    dispatch/.style={->, thick},
    trans/.style={->, thick},
    selfloop/.style={->, thick},
    every node/.append style={font=\small},
  ]
  \node[state] (q0) at (0,0) {$q_0$};
  \node[state] (s1) at (3.8,1.8) {$s_1$};
  \node[state]  (s2) at (3.8,-1.8) {$s_2$};
  \node[accept, double] (qa) at (8.5,1.2) {$\qacc$};
  \node[accept] (qr) at (8.5,-1.2) {$\qrej$};
  \draw[->, thick] (-1.2,0) -- (q0);
  \draw[dispatch] (q0) -- node[above left=-2pt] {$\texttt{B}$} (s1);
  \draw[dispatch] (q0) -- node[below left=-2pt] {$\texttt{1}$} (s2);
  \draw[trans] (q0) -- node[above=-1pt,pos=0.3] {$\texttt{A},\texttt{0}$} (qa);
  \draw[selfloop] (s1) to[out=120,in=60,looseness=5] node[above] {$\texttt{B},\texttt{1}$} (s1);
  \draw[trans] (s1) -- node[above=-1pt] {$\texttt{A}$} (qa);
  \draw[trans] (s1) -- node[pos=0.18,below left=-3pt] {$\blank$} (qr);
  \draw[selfloop] (s2) to[out=240,in=300,looseness=5] node[below] {$\texttt{B},\texttt{1}$} (s2);
  \draw[trans] (s2) -- node[below=-1pt] {$\blank$} (qr);
  \draw[trans] (s2) -- node[pos=0.18,above left=-3pt] {$\texttt{0}$} (qa);
\end{tikzpicture}}}}}
\quad
\begin{array}{c}
\xleftrightarrow{\ (\theta,\ \id_Q)\ } \\[8pt]
\xleftrightarrow{\ (\id_\Sigma,\ s_1 \leftrightarrow s_2)\ }
\\[14pt]
\begin{tikzpicture}[baseline=-0.5ex]
  \path[use as bounding box] (0,-0.3) rectangle (1.92,0.45);
  \draw[->, thick, >=Stealth]
    (0,0.14) .. controls (2.55,0.5) and (2.55,-0.5) .. (0,-0.14);
  \node[above=0.5pt] at (0.95,0.2)
    {\footnotesize $(\theta,\ s_1 \leftrightarrow s_2)$};
\end{tikzpicture}
\end{array}
\quad
\underset{\textstyle \overline{M_3}}{\vcenter{\hbox{\resizebox{0.39\linewidth}{!}{%
\begin{tikzpicture}[
    >=Stealth, thick,
    state/.style={circle, draw, thick, minimum size=1.1cm, inner sep=0pt},
    accept/.style={state},
    dispatch/.style={->, thick},
    trans/.style={->, thick},
    selfloop/.style={->, thick},
    every node/.append style={font=\small},
  ]
  \node[state] (q0) at (0,0) {$q_0$};
  \node[state] (s1) at (3.8,1.8) {$s_1$};
  \node[state]  (s2) at (3.8,-1.8) {$s_2$};
  \node[accept, double] (qa) at (8.5,1.2) {$\qacc$};
  \node[accept] (qr) at (8.5,-1.2) {$\qrej$};
  \draw[->, thick] (-1.2,0) -- (q0);
  \draw[dispatch] (q0) -- node[above left=-2pt] {$\texttt{1}$} (s1);
  \draw[dispatch] (q0) -- node[below left=-2pt] {$\texttt{B}$} (s2);
  \draw[trans] (q0) -- node[above=-1pt,pos=0.3] {$\texttt{A},\texttt{0}$} (qa);
  \draw[selfloop] (s1) to[out=120,in=60,looseness=5] node[above] {$\texttt{B},\texttt{1}$} (s1);
  \draw[trans] (s1) -- node[above=-1pt] {$\texttt{0}$} (qa);
  \draw[trans] (s1) -- node[pos=0.18,below left=-3pt] {$\blank$} (qr);
  \draw[selfloop] (s2) to[out=240,in=300,looseness=5] node[below] {$\texttt{B},\texttt{1}$} (s2);
  \draw[trans] (s2) -- node[below=-1pt] {$\blank$} (qr);
  \draw[trans] (s2) -- node[pos=0.18,above left=-3pt] {$\texttt{A}$} (qa);
\end{tikzpicture}}}}}
$
\caption{Exchanging the symbols $\texttt{A} \leftrightarrow \texttt{0}$ and
$\texttt{B} \leftrightarrow \texttt{1}$ throughout the transition table of $M_3$ (left)
produces the machine $\overline{M_3}$ (right); exchanging the states
$s_1 \leftrightarrow s_2$ instead produces the same machine.
Performing both exchanges therefore returns $M_3$ unchanged: this is
the symmetry $(\theta, s_1{\leftrightarrow}s_2)$ of $M_3$, made precise in
\cref{sec:recoding_symmetries}.}
\label{fig:recoding_action}
\end{figure}
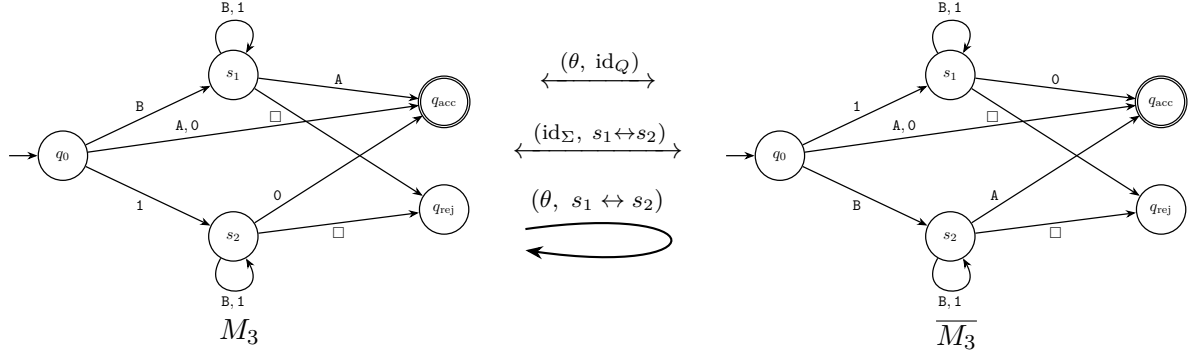

The alphabet partition $\Sigma \setminus \{\blank\} = \{\texttt{A},\texttt{B}\} \sqcup \{\texttt{0},\texttt{1}\}$
partitions the inputs by their alphabet, $I = I_1 \sqcup I_2$.  It also
partitions the components: a component lies in $\mathcal C_1$,
$\mathcal C_2$ or $\mathcal C_\blank$ according to whether its symbol
$\sigma$ lies in $\{\texttt{A},\texttt{B}\}$, in $\{\texttt{0},\texttt{1}\}$ or is $\blank$
(\cref{def:alphabet_partition}).
The susceptibility matrix therefore decomposes into blocks:
\[
  \chi([M]) = \kbordermatrix{
    & \mathcal C_1 & \mathcal C_2 & \mathcal C_\blank \\
    I_1 & X_{1, 1} & X_{1, 2} & X_{1, \blank} \\
    I_2 & X_{2, 1} & X_{2, 2} & X_{2, \blank}
  }.
\]
\begin{informaltheorem}[Informal]
Let $M$ be a classical solution and suppose it has a recoding symmetry $(a, s)$.  Then, under invariance hypotheses on $\Delta \step$, the input distribution $q(x)$ and the prior $\varphi(w)$, the susceptibility matrix of $M$ carries the corresponding permutation symmetry: for every input $x$ and every component $(\sigma, q)$,
\[
  \chi_x^{(\sigma, q)}([M]) = \chi_{a(x)}^{(a(\sigma),\, s(q))}([M]).
\]
See \cref{thm:sus_symmetry} for the formal statement.
\end{informaltheorem}
The symmetries of $M_1$ and $M_3$ both have $a = \theta$.  Since $x \mapsto \theta(x)$ exchanges $I_1$ with $I_2$ while $(\sigma, q)
\mapsto (\theta(\sigma), s(q))$ exchanges $\mathcal C_1$ with $\mathcal C_2$
and preserves
$\mathcal C_\blank$, this identifies $X_{1,1}$ with $X_{2,2}$, $X_{1,2}$ with
$X_{2,1}$, and $X_{1,\blank}$ with $X_{2,\blank}$, each up to this recoding
of rows and columns.

This is the fixed-point
case of an action: recoding \emph{any} machine permutes the entries of its
susceptibility matrix in the same way, so each recoding acts on the
set of machines in susceptibility space.
\Cref{fig:symmetry_pca_intro} shows this action in a principal component
projection, where it appears as a reflection. Recoded machines land at
mirrored positions, and the fixed-points $M_1$ and $M_3$ each sit
on the axis of their symmetry.

The colouring of \cref{fig:symmetry_pca_intro} provides visual evidence that the
reflection is the recoding, and not some other symmetry.
Each DFA $M$ is coloured by its \emph{asymmetry} $\asym(M; g)$ at the recoding
$g$ (\cref{def:asymmetry}), which measures differences in visited states at each time step with and without the recoding, averaged over $x \in I, t \in [0,T-1]$. It vanishes when $g$ is a symmetry of $M$ (\cref{prop:asym_zero}), and increases as $M$'s behaviour is ``more affected'' by $g$.  The recodings here are involutions, and therefore preserve asymmetry: $\asym(g \cdot M; g) =
\asym(M; g)$.  So if the reflection is the recoding, the colouring must be
mirror-symmetric about the axis, with the machines with $\asym=0$ on the
axis itself, and this is what we see in the figure. The colouring is a
proxy for a stronger pointwise statement: that the mirror image of
each $M$ lies close to $g \cdot M$. We depict individual recoding-is-reflection examples in the figure for $M_1, M_3, M_5$, and we extend this analysis to every $M$ in
\cref{sec:expt_symmetry} (\cref{fig:symmetry_dist}).

\begin{figure}[H]
\centering
\includegraphics[width=\linewidth]{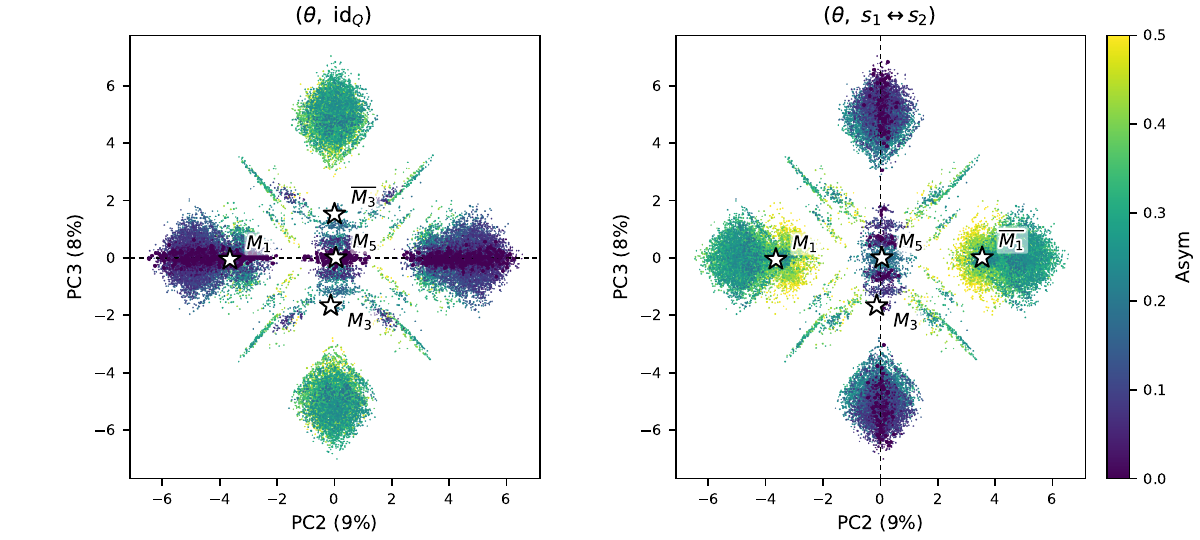}
\caption{\textbf{Recoding acts on the PCA by reflections.}
The $\mathrm{PC}_2 \times \mathrm{PC}_3$ plane of the PCA of the set of susceptibility matrices $\{\widehat{\psi}([M])\}$ of the solution DFAs (\cref{sec:expt_population_group}).
\textbf{Left:} The recoding $(\theta, \id_Q)$ swaps $\texttt{A} \leftrightarrow \texttt{0}$ and $\texttt{B} \leftrightarrow \texttt{1}$ in a machine's transitions, and this action becomes the reflection negating $\mathrm{PC}_3$, e.g. $M_3 \mapsto \overline{M_3}$. $M_1$ is invariant to $(\theta, \id_Q)$ and sits fixed on the horizontal axis. \textbf{Right:} $(\theta, s_1{\leftrightarrow}s_2)$ does the same but also swaps the states $s_1 \leftrightarrow s_2$, and becomes the reflection negating $\mathrm{PC}_2$, e.g. $M_1 \mapsto \overline{M_1}$. $M_3$ is invariant to $(\theta, s_1{\leftrightarrow}s_2)$ and sits fixed on the vertical axis. DFAs are coloured by their asymmetry $\asym(M; g)$ at each of the two recodings $g$, which measures changes in $M$'s behaviour on inputs (\cref{def:asymmetry}), and is $0$ on fixed points. Stars mark the DFAs $M_1$, $M_3$, and $M_5$, and the reflections $\overline{M_1}$, $\overline{M_3}$. The only machine fixed by both symmetries is $M_5$, which sits near the origin.}
\label{fig:symmetry_pca}\label{fig:symmetry_pca_intro}
\end{figure}

\subsection{Further contributions}\label{sec:intro_minor}

Alongside these main contributions
(\cref{thm:rank_bound,thm:sus_symmetry} and their associated
experiments), the paper makes two minor
contributions: (1) a circuit calculus for the noisy computations used
in \cite{murfet2025pas} as well as the present paper
(\cref{sec:circuits_body}), and (2) a minimal example of a singular setting
in which susceptibilities probe structure that the Hessian cannot
(\cref{app:absorbing_example}).

\subsection*{Acknowledgements}

We thank Chris Elliott for helpful comments and feedback, Rohan Hitchcock for assisting with sampling methods, and Max Adam for initial experimentation that inspired some of this work.  The majority of this work was carried out while three of the authors were at Timaeus.  This project is funded by the Advanced Research + Invention Agency (ARIA). Rumi Salazar was supported by the Melbourne Research Scholarship and by the Commonwealth through an Australian Government Research Training Program Scholarship [DOI: \url{https://doi.org/10.82133/C42F-K220}] during the completion of this project.

\section{Noisy Turing machines}\label{sec:setup}

In this section we recall the theoretical setting of noisy Turing machines within singular learning theory and susceptibilities.

\subsection{Pseudo-UTMs}\label{sec:pseudo_utm}

\begin{definition}\label{def:configuration}
Let $M$ be a Turing machine with alphabet $\Sigma$ (with blank symbol $\blank$) and set of states $Q$.  Write
\[
  \Sigma^{\mathbb Z, \blank} := \bigl\{\, g \colon \mathbb Z \to \Sigma \mid g(i) = \blank \text{ for all but finitely many } i \,\bigr\}
\]
where it is understood that the head is at index $0$.  A \emph{configuration of $M$} is a point of
\[
  \Cfg := \Sigma^{\mathbb Z, \blank} \times Q.
\]
\end{definition}

Turing machines over the fixed alphabet $\Sigma$ and set of states $Q$ differ
only in their transition functions, identifying the space of Turing machines
with the set
\begin{equation}\label{eq:W_code}
  W^\code := \prod_{(\sigma, q) \in \Sigma \times Q} \Sigma \times Q \times \{L, S, R\}.
\end{equation}
The code for a machine $M$ will be denoted $[M] \in W^\code$.  If the entry
indexed by $(\sigma, q)$ is $(\sigma', q', d)$ then on reading $\sigma$ in
state $q$ the machine $M$ will write $\sigma'$, transition to state $q'$ and
move in direction $d$ ($L$ for left, $R$ for right and $S$ for stay).

\begin{definition}\label{def:step}
The \emph{configuration map} $\step_{[M]} \colon \Cfg \to \Cfg$ of a machine $M$ acts on a configuration $(\tau, q) \in \Sigma^{\mathbb Z, \blank} \times Q$ through the transition on the head symbol $\tau(0)$ and state $q$.  Write $(\sigma', q', d) := [M](\tau(0), q)$ and let $\widetilde\tau$ be $\tau$ with the head square overwritten,
\[
  \widetilde\tau(i) := \begin{cases} \sigma', & i = 0, \\ \tau(i), & i \neq 0. \end{cases}
\]
For $s \in \mathbb Z$ let $\mathrm{Shift}_s \colon \Sigma^{\mathbb Z, \blank} \lto \Sigma^{\mathbb Z, \blank}$ be the function which on $g \in \Sigma^{\mathbb Z, \blank}$ and $i \in \mathbb Z$ is given by $\mathrm{Shift}_s(g)(i) := g(i + s)$.  With the head-relative shifts $s_L := -1$, $s_S := 0$, $s_R := +1$ and $d \in \{L, S, R\}$,
\[
  \step_{[M]}(\tau, q) := \bigl(\mathrm{Shift}_{s_d}(\widetilde\tau),\ q'\bigr).
\]
\Cref{fig:tm_config} illustrates the first step on an input configuration.
\end{definition}

\begin{figure}[H]
\centering
\includegraphics[width=\textwidth]{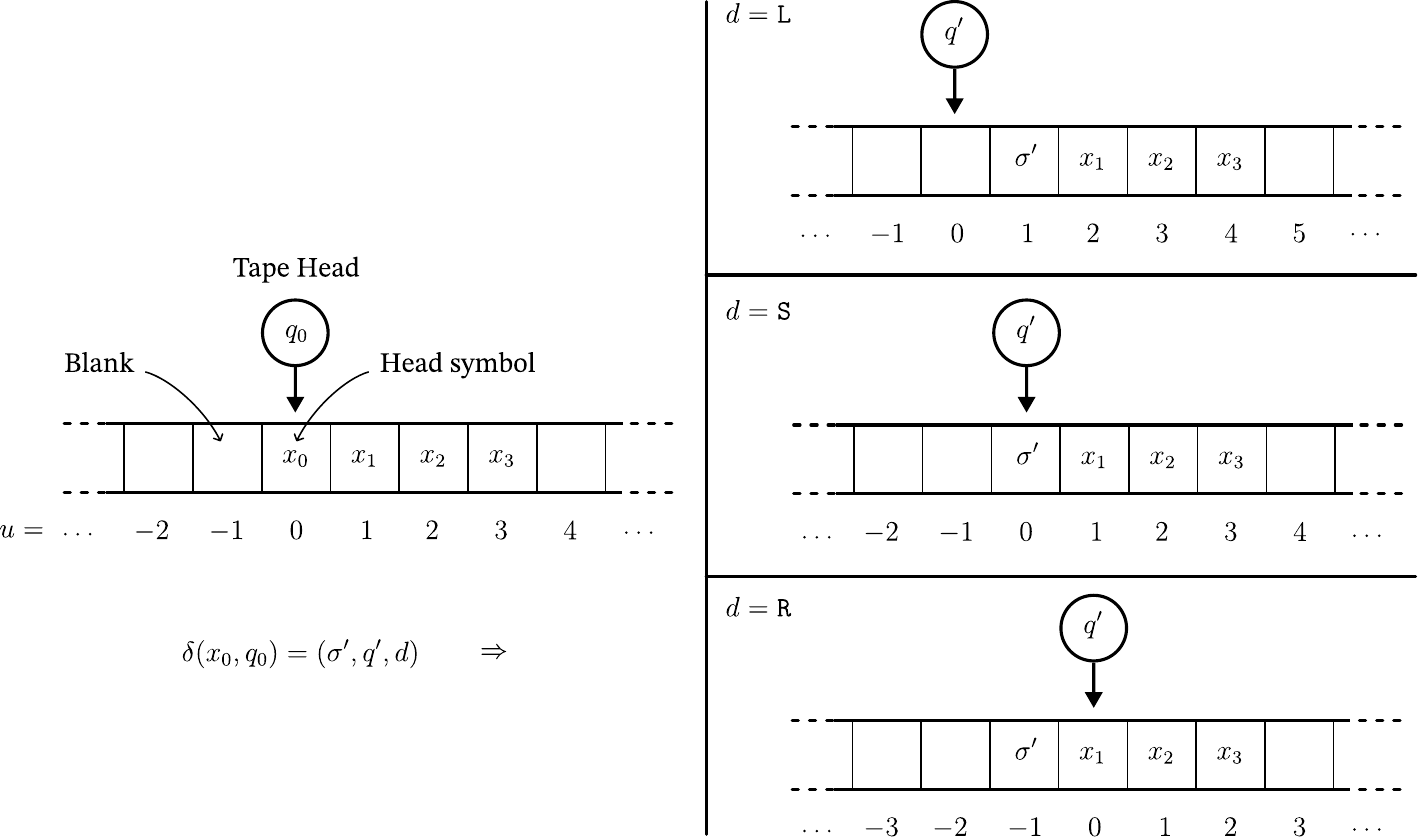}
\caption{The initial configuration (left) of a Turing machine with input $x = x_0 x_1 x_2 x_3 \in \Sigma^*$, and the next configuration (right) after the first transition $[M](x_0, q_0) = (\sigma', q', d)$, split into the three cases of the direction $d$. Note that the head is always over tape position $0$, so we do not additionally track a position $p \in \mathbb Z$ with the configuration data. It is natural to think of shifting the tape in the opposite direction, rather than shifting the head.}
\label{fig:tm_config}
\end{figure}

Fix an \emph{initial state} $q_0 \in Q$.  An input $x = x_0 \cdots x_{n-1}
\in \Sigma^*$ determines the \emph{initial configuration}
$(\tau_x, q_0) \in \Cfg$, with $\tau_x(i) := x_i$ for $0 \le i < n$ and
$\tau_x(i) := \blank$ otherwise.  For $T \in \mathbb N$ the
\emph{classical execution} is
\[
  \step^T \colon \Sigma^* \times W^\code \lto Q,
  \qquad
  \step^T(x, [M]) := \pi_Q\, \step_{[M]}^{\,T}(\tau_x, q_0),
\]
the state of $M$ after $T$ steps on input $x$, where
$\pi_Q \colon \Cfg \lto Q$ is the projection onto the state of the machine.

\begin{definition}\label{def:pseudo_utm}
A \emph{pseudo-UTM} for $(\Sigma, Q)$ consists of the following data:
\begin{itemize}[topsep=2pt,itemsep=2pt]
  \item a multi-tape Turing machine $\mathcal U$ with tape alphabet
    $\Sigma_{\mathrm{UTM}} \supseteq \Sigma \sqcup Q \sqcup \{L, S, R\}$
    and set of states $Q_{\mathrm{UTM}}$;
  \item a set $\{\text{working}, \text{state}\} \sqcup \mathrm{Aux}$ of tapes:
    a \emph{working tape} with alphabet $\Sigma$, a \emph{state tape} with
    alphabet $Q \sqcup \{\blank\}$, and \emph{auxiliary tapes}
    $\mathrm{Aux}$, whose squares are indexed by
    $\mathrm{Aux} \times \mathbb Z$;
  \item a \emph{code layout}
    $\ell \colon \Sigma \times Q \lto (\mathrm{Aux} \times \mathbb Z)^{3}$,
    sending each tuple $(\sigma, q)$ to the three auxiliary squares that
    carry its write-symbol, next-state, and move entries;
  \item \emph{fixed values}: the contents of every other auxiliary square,
    the initial UTM state, and the initial head positions, independent of
    the code and input.
\end{itemize}
This data determines the configuration space
\[
  \Cfg_{\mathcal U} = \underbrace{\Sigma^{\mathbb Z, \blank}}_{\text{working tape}} \times \underbrace{(Q \sqcup \{\blank\})^{\mathbb Z, \blank}}_{\text{state tape}} \times \underbrace{\bigl(\Sigma_{\mathrm{UTM}}^{\mathbb Z, \blank}\bigr)^{\mathrm{Aux}}}_{\text{auxiliary tapes}} \times \underbrace{Q_{\mathrm{UTM}}}_{\text{UTM state}}
\]
and the \emph{loading map}
\[
  \langle -\,;\,-\rangle \colon W^\code \times \Cfg \lto \Cfg_{\mathcal U},
\]
which sends a code $[M]$ and a configuration $\xi = (\tau, q)$ to the
configuration with the entries of $[M]$ on $\operatorname{Im}(\ell)$,
$\tau$ on the working tape, $q$ at index $0$ of the state tape ($\blank$
elsewhere), and the fixed values everywhere else (\cref{fig:loading}).  A
transition of $\mathcal U$ reads the squares under its heads and the UTM
state, overwrites them, and moves the tape heads, as in \cref{def:step}.
We require:
\begin{itemize}[topsep=2pt,itemsep=2pt]
  \item The $3\,|\Sigma|\,|Q|$ squares in the image of $\ell$ are pairwise
    distinct
  \item All but finitely many of the fixed auxiliary squares are blank
  \item $\mathcal U$ writes only symbols of $\Sigma$ to the working tape
    and of $Q \sqcup \{\blank\}$ to the state tape, so that a single
    transition is a map
    $\step_{\mathcal U} \colon \Cfg_{\mathcal U} \lto \Cfg_{\mathcal U}$
  \item The loading map is injective
  \item $\mathcal U$ has a \emph{period} $P > 0$:
    \begin{equation}\label{eq:simulates}
      (\step_{\mathcal U})^{P}\, \bigl\langle\, [M] \,;\, \xi \,\bigr\rangle
      =
      \bigl\langle\, [M] \,;\, \step_{[M]}(\xi) \,\bigr\rangle
      \qquad \text{for all } [M] \in W^\code,\ \xi \in \Cfg,
    \end{equation}
    that is, running $\mathcal U$ for $P$-steps simulates one step of $M$, and all auxiliary tapes are returned to the loading map configuration by the end of the period.
\end{itemize}
\end{definition}

In \cref{sec:lookup_utm} we define an explicit pseudo-UTM with two
auxiliary tapes, there called the \emph{description tape} and the
\emph{staging tape}.  To illustrate the general shape of the definition,
\cref{fig:loading} shows the loading map of that machine.  See that section
for the details.

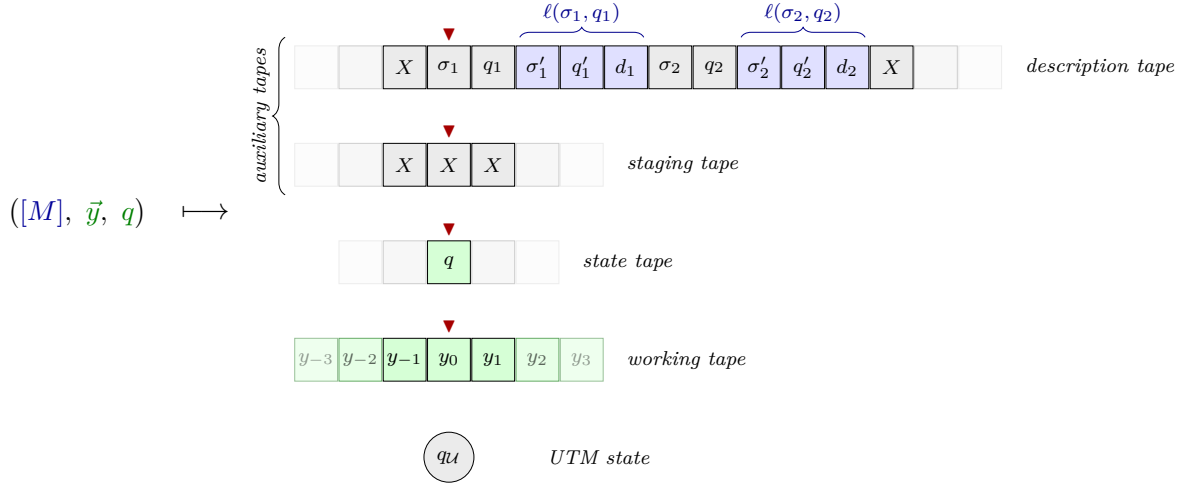
\begin{figure}[H]
\centering
\resizebox{\textwidth}{!}{%
\begin{tikzpicture}[
  x=0.62cm, y=1.05cm,
  cell/.style={draw, minimum width=0.6cm, minimum height=0.6cm, inner sep=1pt, font=\scriptsize},
  scaf/.style={cell, fill=black!8},
  codecell/.style={cell, fill=blue!12},
  cfgcell/.style={cell, fill=green!16},
  fade1/.style={cell, draw=black!35, fill=black!5, opacity=0.55},
  fade2/.style={cell, draw=black!25, fill=black!4, opacity=0.28},
  gfade1/.style={cell, draw=green!45!black, fill=green!16, opacity=0.6},
  gfade2/.style={cell, draw=green!40!black, fill=green!16, opacity=0.4},
  lbl/.style={font=\scriptsize\itshape},
  hd/.style={-{Triangle[length=5pt,width=5pt]}, line width=1pt, red!65!black}
]
  \def\yD{4.5}
  \node[fade2] at (-3,\yD){}; \node[fade1] at (-2,\yD){};
  \node[scaf] at (-1,\yD){$X$};
  \node[scaf] at (0,\yD){$\sigma_1$}; \node[scaf] at (1,\yD){$q_1$};
  \node[codecell] at (2,\yD){$\sigma'_1$}; \node[codecell] at (3,\yD){$q'_1$}; \node[codecell] at (4,\yD){$d_1$};
  \node[scaf] at (5,\yD){$\sigma_2$}; \node[scaf] at (6,\yD){$q_2$};
  \node[codecell] at (7,\yD){$\sigma'_2$}; \node[codecell] at (8,\yD){$q'_2$}; \node[codecell] at (9,\yD){$d_2$};
  \node[scaf] at (10,\yD){$X$};
  \node[fade1] at (11,\yD){}; \node[fade2] at (12,\yD){};
  \node[lbl,anchor=west] at (12.8,\yD){description tape};
  \draw[decorate, decoration={brace, amplitude=4pt}, blue!55!black] (1.6,\yD+0.4) -- (4.4,\yD+0.4);
  \node[font=\scriptsize, blue!55!black] at (3,\yD+0.72) {$\ell(\sigma_1,q_1)$};
  \draw[decorate, decoration={brace, amplitude=4pt}, blue!55!black] (6.6,\yD+0.4) -- (9.4,\yD+0.4);
  \node[font=\scriptsize, blue!55!black] at (8,\yD+0.72) {$\ell(\sigma_2,q_2)$};
  \def\yS{3.2}
  \node[fade2] at (-3,\yS){}; \node[fade1] at (-2,\yS){};
  \node[scaf] at (-1,\yS){$X$}; \node[scaf] at (0,\yS){$X$}; \node[scaf] at (1,\yS){$X$};
  \node[fade1] at (2,\yS){}; \node[fade2] at (3,\yS){};
  \node[lbl,anchor=west] at (3.8,\yS){staging tape};
  \def\yQ{1.9}
  \node[fade2] at (-2,\yQ){}; \node[fade1] at (-1,\yQ){};
  \node[cfgcell] at (0,\yQ){$q$};
  \node[fade1] at (1,\yQ){}; \node[fade2] at (2,\yQ){};
  \node[lbl,anchor=west] at (2.8,\yQ){state tape};
  \def\yW{0.6}
  \node[gfade2] at (-3,\yW){$y_{-3}$}; \node[gfade1] at (-2,\yW){$y_{-2}$};
  \node[cfgcell] at (-1,\yW){$y_{-1}$}; \node[cfgcell] at (0,\yW){$y_0$}; \node[cfgcell] at (1,\yW){$y_1$};
  \node[gfade1] at (2,\yW){$y_2$}; \node[gfade2] at (3,\yW){$y_3$};
  \node[lbl,anchor=west] at (3.8,\yW){working tape};
  \def\yU{-0.7}
  \node[draw, circle, fill=black!8, font=\scriptsize, minimum size=0.7cm, inner sep=1pt] at (0,\yU) {$q_{\mathcal U}$};
  \node[lbl,anchor=west] at (2.0,\yU){UTM state};
  \node[red!65!black] at (0,\yD+0.44) {$\scriptstyle\blacktriangledown$};
  \node[red!65!black] at (0,\yS+0.44) {$\scriptstyle\blacktriangledown$};
  \node[red!65!black] at (0,\yQ+0.44) {$\scriptstyle\blacktriangledown$};
  \node[red!65!black] at (0,\yW+0.44) {$\scriptstyle\blacktriangledown$};
  \draw[decorate, decoration={brace, amplitude=5pt, mirror}] (-3.7,\yD+0.4) -- (-3.7,\yS-0.4);
  \node[font=\scriptsize\itshape, rotate=90] at (-4.25,{(\yD+\yS)/2}){auxiliary tapes};
  \node[font=\normalsize] at (-8.4,{(\yD+\yW)/2})
    {$\bigl({\color{blue!60!black}[M]},\ {\color{green!50!black}\vec y},\ {\color{green!50!black}q}\bigr)$};
  \node[font=\large] at (-5.5,{(\yD+\yW)/2}) {$\longmapsto$};
\end{tikzpicture}%
}
\caption{The loading map $\langle -\,;\,-\rangle$ for the lookup pseudo-UTM (\cref{sec:lookup_utm}).  The figure implicitly defines its code layout $\ell$.  The code ${\color{blue!60!black}[M]}$ (blue) occupies the squares $\ell(\sigma_j, q_j)$ of the description tape, and the configuration $({\color{green!50!black}\vec y}, {\color{green!50!black}q})$ (green) the working and state tapes.  Every other square (grey) is fixed by $\mathcal U$, independent of the code and the configuration.  The red triangles mark the initial positions of the four UTM heads, and the circle is the initial UTM state $q_{\mathcal U}$, fixed by $\mathcal U$.}
\label{fig:loading}
\end{figure}

\subsection{Smooth relaxation of a pseudo-UTM}\label{sec:smooth_cycle}

The \emph{parameter space} is the set of noisy Turing machines, originally
defined in \cite{clift2021geometryprogramsynthesis},
\begin{equation}\label{eq:W}
  W := \prod_{(\sigma, q) \in \Sigma \times Q} \Delta\Sigma \times \Delta Q \times \Delta\{L, S, R\},
\end{equation}
indexed by transition tuples, with the classical codes $W^\code$ at its vertices.

\begin{definition}\label{def:noisy_configuration}
Write $(\Delta\Sigma)^{\mathbb Z, \blank}$ for $\Sigma^{\mathbb Z, \blank}$ (\cref{def:configuration}) with $\Sigma$ and $\blank$ replaced by $\Delta\Sigma$ and the point mass $e_\blank$.  A \emph{noisy configuration} of a Turing machine is a point of
\[
  \Delta\Cfg := (\Delta\Sigma)^{\mathbb Z, \blank} \times \Delta Q.
\]
For a pseudo-UTM $\mathcal U$ (\cref{def:pseudo_utm}) we define
\[
  \Delta\Cfg_{\mathcal U} := (\Delta\Sigma)^{\mathbb Z, \blank} \times \bigl(\Delta(Q \sqcup \{\blank\})\bigr)^{\mathbb Z, \blank} \times \bigl((\Delta\Sigma_{\mathrm{UTM}})^{\mathbb Z, \blank}\bigr)^{\mathrm{Aux}} \times \Delta Q_{\mathrm{UTM}}.
\]
The inclusions $\Sigma, Q, \{L, S, R\} \hookrightarrow \Sigma_{\mathrm{UTM}}$ and $Q \hookrightarrow Q \sqcup \{\blank\}$ induce inclusions of the corresponding simplices.  The \emph{noisy loading map}
\[
  \langle -\,;\,-\rangle \colon W \times \Delta\Cfg \lto \Delta\Cfg_{\mathcal U}
\]
sends $(w, \xi)$ to the noisy configuration in which, for each $(\sigma, q) \in \Sigma \times Q$, the three squares $\ell(\sigma, q)$ carry the three components of $w_{(\sigma, q)}$, read through these inclusions; the working tape carries the tape component of $\xi$, and the state tape its state component at index $0$ ($e_\blank$ elsewhere); and every other square, and the UTM state, carry the point mass at the value fixed by $\mathcal U$ in \cref{def:pseudo_utm}.
\end{definition}

The noisy loading map is injective, since the working and state tapes recover $\xi$ and the squares of $\operatorname{Im}(\ell)$ recover $w$.  On the vertices it restricts to the loading map of \cref{def:pseudo_utm}.

\begin{definition}\label{def:utm_relaxation}
A \emph{smooth relaxation} of $\step_{\mathcal U}$ is a smooth map
\[
  \Delta\step_{\mathcal U} \colon \Delta\Cfg_{\mathcal U} \lto \Delta\Cfg_{\mathcal U}
\]
restricting to $\step_{\mathcal U}$ on the configurations (the vertices of $\Delta\Cfg_{\mathcal U}$).
\end{definition}

Both $\Delta\Cfg$ and $\Delta\Cfg_{\mathcal U}$ are increasing unions of the finite-dimensional products obtained by fixing all but finitely many tape squares to the point mass $e_\blank$.  A map between them is \emph{smooth} if it carries each such product smoothly into another.

\begin{definition}\label{def:periodic_relaxation}
Let $\mathcal U$ be a pseudo-UTM with period $P$ (\cref{def:pseudo_utm}).  A smooth relaxation $\Delta\step_{\mathcal U}$ of its step (\cref{def:utm_relaxation}) is \emph{periodic} if one period returns each configuration $\langle w ; \xi\rangle$ to a configuration $\langle w ; \xi'\rangle$ with the same code $w$, the new working- and state-tape part $\xi'$ depending smoothly on $(w, \xi)$.  Since $\langle - ; - \rangle$ is injective, $\xi'$ is unique, defining the smooth \emph{cycle map} $F \colon W \times \Delta\Cfg \to \Delta\Cfg$, $F(w, \xi) := \xi'$, for which the following diagram commutes:
\begin{equation}\label{eq:cycle_map}
\begin{tikzcd}[column sep=huge, row sep=large]
  W \times \Delta\Cfg \arrow[r, "\langle -\,;\,-\rangle"] \arrow[d, "{(\pi_W,\, F)}"'] & \Delta\Cfg_{\mathcal U} \arrow[d, "(\Delta\step_{\mathcal U})^{P}"] \\
  W \times \Delta\Cfg \arrow[r, "\langle -\,;\,-\rangle"'] & \Delta\Cfg_{\mathcal U}
\end{tikzcd}
\end{equation}
Here $(\pi_W, F) \colon (w, \xi) \mapsto (w, F(w, \xi))$.
The period $P$ is part of the pseudo-UTM \eqref{eq:simulates}; periodicity
is a condition on the smooth relaxation.  Write $F_w := F(w, -)$ for the cycle
map at a fixed code.
\end{definition}

\begin{definition}\label{def:model_utm}
Fix a periodic smooth relaxation, with cycle map $F$.  For $x \in \Sigma^*$ write
$e_{\tau_x} := (e_{\tau_x(i)})_{i \in \mathbb Z} \in (\Delta\Sigma)^{\mathbb Z, \blank}$
for the point-mass tape of the initial configuration, and
$\pi_{\Delta Q} \colon \Delta\Cfg \lto \Delta Q$ for the projection.  For
$T \in \mathbb N$ define
\[
  \Delta\step^T \colon \Sigma^* \times W \lto \Delta Q,
  \qquad
  \Delta\step^T(x, w) := \pi_{\Delta Q}\, F_w^{\,T}\bigl(e_{\tau_x},\, e_{q_0}\bigr).
\]
\end{definition}

\begin{remark}\label{rem:classical_restriction}
At a code $[M]$ and a classical configuration the relaxed step is the classical step, so \eqref{eq:cycle_map} reduces to \eqref{eq:simulates} and $F_{[M]} = \step_{[M]}$.  Hence $\Delta\step^T$ (\cref{def:model_utm}) is a smooth relaxation of the classical execution $\step^T$, the diagram
\begin{equation}\label{eq:smooth_relax_diagram}
\begin{tikzcd}[column sep=huge]
  \Sigma^* \times W \arrow[r, "\Delta\step^T"] & \Delta Q \\
  \Sigma^* \times W^\code \arrow[u, hook] \arrow[r, "\step^T"'] & Q \arrow[u, hook]
\end{tikzcd}
\end{equation}
commuting, with vertical maps the canonical inclusions.  The precise smooth relaxation we have in mind is given in \cref{app:model_details}.
\end{remark}

\subsection{The noisy Turing machine truth-model-prior triple}\label{sec:bayes_setup}

\begin{definition}\label{def:stat_model}
A \emph{truth-model-prior triple} \cite{watanabe2009algebraic} is $(q, p, \varphi)$ in which $q(x, y)$ is a probability density on $\mathcal X \times \mathcal Y$ (the \emph{truth}), $\{p(y|x, w) : w \in W\}$ is a family of conditional densities indexed by a compact parameter space $W \subseteq \mathbb R^d$ (the \emph{model}), and $\varphi$ is a probability density on $W$ (the \emph{prior}).
\end{definition}

Such a triple is often given alongside a loss function $L(w)$, minimised at values of $w$ for which $p$ predicts $q$ well.  We instantiate this for the noisy Turing machine.  The resulting triple was first defined in \cite{murfet2025pas}.  The data is a finite alphabet $\Sigma$ and set of states $Q$, a finite input set $I \subseteq \Sigma^*$, an input distribution $q$ on $I$ with full support ($q(x) > 0$ for every $x \in I$), a target map $y \colon I \to Q$, and a number of time steps $T \in \mathbb N$.  The parameter space is $W$ \eqref{eq:W}.

The model of \cref{def:stat_model} is
\begin{equation}\label{eq:abstract_model}
  p(y|x, w) := \Delta\step^T(x, w)_y,
\end{equation}
the relaxed execution $\Delta\step^T$ of \cref{def:model_utm}; the truth is $q(x, y) := q(x)\,\one[y = y(x)]$, deterministic on the target; the prior is a choice of density $\varphi$ on $W$, fixed where it is used (\cref{def:susceptibility}).

For a Turing machine $M$ with code $[M] \in W^\code$ and $T \in \mathbb N$, we write $M(x) := \step^T(x, [M]) \in Q$ for the final state of $M$ after executing for $T$ steps with input $x$.
\begin{definition}\label{def:classical_solution}
Given the data $(I, y, T)$, consisting of $I \subseteq \Sigma^*$ a finite set of \emph{inputs}, $y : I \to Q$ a function called the \emph{target map}, and a number of time steps $T \in \mathbb N$, we call $M$ a \emph{classical solution} if $M(x) = y(x)$ for all $x \in I$.
\end{definition}

\begin{remark}
We do \emph{not} consider a $T$-independent meaning of \emph{classical solution}. The set of classical solutions typically varies both with $I$ and $T$, so where they are not explicitly mentioned, they should still be understood as fixed. For classical solutions $M$ to exist, $T$ must be chosen sufficiently large so that $M$ has enough time to distinguish all $x \in I$ (usually $T \geq |x|$).
\end{remark}

By commutativity of \eqref{eq:smooth_relax_diagram}, $M$ is a classical solution if and only if $\Delta\step^T(x, [M])_{y(x)} = 1$ for all $x \in I$.

Fix a classical solution $M$ and a constant $\mu \in (0, 1)$.  Following \cite{murfet2025pas}, we bound the probabilities away from zero. This makes the loss defined below finite and smooth on all of $W$, with classical solutions among its critical points (\cref{rem:mu_shift}).  Project the truth and the model to the binary simplex $\Delta\{\mathrm{correct}, \mathrm{incorrect}\}$, where the truth is the point mass at $\mathrm{correct}$ and the model predicts $\mathrm{correct}$ with probability $p(y(x)|x, w)$, and mix each with the uniform distribution, with weight $\mu$:
\[
  q_\mu^{\mathbb Z_2}(\mathrm{correct} | x) := 1 - \tfrac{\mu}{2},
  \qquad
  p_\mu^{\mathbb Z_2}(\mathrm{correct} | x, w) := (1 - \mu)\, p(y(x) | x, w) + \tfrac{\mu}{2}.
\]
The shifted per-input cross-entropy and population loss are
\begin{equation}\label{eq:shifted_losses}
  \ell_x^\mu(w) := -\sum_z q_\mu^{\mathbb Z_2}(z | x)\, \log p_\mu^{\mathbb Z_2}(z | x, w), \qquad L^\mu(w) := \E_{x \sim q}[\ell_x^\mu(w)],
\end{equation}
and in the remainder of the paper we drop the $\mu$-superscript, treating $\mu$ as a fixed small positive parameter.  For the experiments, sampling at small $\mu$ is indistinguishable from sampling at $\mu = 0$: as $\mu \to 0$ the Gibbs distribution $p_\beta \propto e^{-\beta L}\varphi$ of \cref{sec:experiments} converges to its $\mu = 0$ counterpart, and the susceptibilities we estimate converge with it (\cref{prop:mu_zero_limit}).  The experiments in fact take $\mu = 0$ (\cref{rem:mu_shift}, \cref{sec:empirical_sus}).

\begin{definition}\label{def:error_probability}
The \emph{error probability} is
\begin{equation}\label{eq:hx}
\begin{aligned}
  h_x(w) &:= p\bigl(y \neq y(x)|x, w\bigr) \\
         &= 1 - p\bigl(y(x)|x, w\bigr).
\end{aligned}
\end{equation}
\end{definition}

The losses are expressed through it: $p_\mu^{\mathbb Z_2}(\mathrm{correct}|x, w) = 1 - \mu/2 - (1 - \mu)\,h_x(w)$, and the per-input loss depends on $w$ only through $h_x$.

\begin{remark}\label{rem:mu_shift}
Explicitly, $\ell_x = \kappa_\mu \circ h_x$, where
\[
  \kappa_\mu(h) := -\bigl(1 - \tfrac{\mu}{2}\bigr)\log\bigl(1 - \tfrac{\mu}{2} - (1 - \mu)\,h\bigr) - \tfrac{\mu}{2}\log\bigl(\tfrac{\mu}{2} + (1 - \mu)\,h\bigr)
\]
is smooth on $[0, 1]$ with $\kappa_\mu'(0) = 0$ and $\kappa_\mu''(0) > 0$.  Hence
\[
  \ell_x = \kappa_\mu(0) + \tfrac12 \kappa_\mu''(0)\, h_x^2 + O(h_x^3)
\]
near $h_x = 0$.  In particular $\nabla \ell_x = \kappa_\mu'(h_x)\, \nabla h_x$ vanishes wherever $h_x = 0$, so every parameter at which the model is exactly correct on every input, in particular the code $[M]$ of a classical solution, is a critical point of $L$, and near such a parameter the loss agrees with the squared error to leading order. Indeed $L - L([M])$ and $H$ are \emph{comparable} near $[M]$, each bounded above and below by a constant multiple of the other \cite[Lemma~2.11 and Appendix~B]{murfet2025pas}.  Without the $\mu$-shift ($\mu = 0$), $\kappa_0(h) = -\log(1 - h)$ has $\kappa_0'(0) = 1$, and neither smoothness nor criticality holds. However, neither property is used by the theorems of \cref{sec:theorems}, which is why the experiments can take $\mu = 0$ (\cref{sec:empirical_sus}, \cref{rem:log_loss_mu_zero}).
\end{remark}

\begin{definition}\label{def:squared_error_loss}
The \emph{squared-error loss} is
\begin{align*}
  H(w) &:= \sum_{x \in I} q(x)\, h_x(w)^2 \\
       &= \E_{x \sim q}\bigl[h_x(w)^2\bigr],
\end{align*}
with zero set $W_0 := \{w \in W : H(w) = 0\}$, the exact solutions.
\end{definition}

Since $M$ is a classical solution, $h_x([M]) = 0$ for all $x$, so $H([M]) = 0$ and $[M] \in W_0$.  For the smooth relaxation of \cref{app:model_details} the loss $H$ is a polynomial, and the germ $(H, [M])$ is the algebraic singularity of \cite{murfet2025pas}.  We use $H$ as the potential for the theory of \cref{sec:theorems}. The experiments (\cref{sec:empirical_sus}) use the log-loss instead, to which both theorems extend for every $\mu \ge 0$ (\cref{app:reparam}).  Where we work with $H$, the per-input squared error $h_x^2$ and its mean $H$ play the roles of $\ell_x$ and $L$.

\subsection{Components, observables, and susceptibilities}

\begin{definition}\label{def:component}
The set of \emph{components} is $\mathcal C := \Sigma \times Q \times \{\Sigma, Q, D\}$, indexing triples $C = (\sigma, q, f)$ where $f$ labels the simplex factor (symbol, state, or direction).  The parameter space decomposes as
\[
  W = \prod_{C \in \mathcal C} W_C, \qquad
  W_{(\sigma, q, \Sigma)} = \Delta\Sigma, \quad
  W_{(\sigma, q, Q)} = \Delta Q, \quad
  W_{(\sigma, q, D)} = \Delta\{L, S, R\}.
\]
We write $w_{(\sigma, q, f)}$ for the corresponding entry of a code $w \in W$. The three components $(\sigma, q, \Sigma)$, $(\sigma, q, Q)$, $(\sigma, q, D)$ are the entries of the transition tuple $(\sigma, q)$.  For $C \in \mathcal C$, write $W = U_C \times W_C$ with $U_C := \prod_{C' \neq C} W_{C'}$, and $w = (u, v)$ with $[M] = (u^*, v^*)$, so that $u^*$ is the projection of $[M]$ onto $U_C$.  The \emph{slice through $[M]$ along $C$} is the locus $\{w \in W : u = u^*\}$.
\end{definition}

\begin{definition}\label{def:phi_C}
The \emph{component observable} of $C \in \mathcal C$ is the distribution on $W$
\[
  \phi_C(w) := \delta(u - u^*) \bigl(H(w) - H([M])\bigr).
\]
\end{definition}

Since $H([M]) = 0$, $\phi_C(w) = \delta(u - u^*)\, H(w)$: the observable is supported on the slice through $[M]$ along $C$, and integrals against it are integrals over the slice,
\[
  \int_W \phi_C(w) g(w) \rho(w)\, dw = \int_{W_C} H(u^*, v) g(u^*, v) \rho(u^*, v)\, dv
\]
for a function $g$ and a density $\rho$ on $W$.

\begin{definition}\label{def:susceptibility}
Fix an inverse temperature $\beta > 0$ and a prior $\varphi$ on $W$.  Following \cite{baker2025studyingsmalllanguagemodels}, we define the \emph{Gibbs distribution}
\[
  p_{\beta}(w) := \frac{1}{Z_{\beta}}\exp\bigl\{-\beta H(w)\bigr\}\varphi(w), \qquad
  Z_{\beta} := \int_W \exp\bigl\{-\beta H(w)\bigr\}\varphi(w)\, dw.
\]
Fix $x \in I$ and perturb the input distribution towards $x$, $q_\varepsilon := (1 - \varepsilon)\,q + \varepsilon\,\delta_x$, giving the perturbed loss $H^\varepsilon := \sum_{x'} q_\varepsilon(x')\, h_{x'}^2 = (1 - \varepsilon) H + \varepsilon\, h_x^2$ and the Gibbs distribution $p_{\beta}^\varepsilon$ built from $H^\varepsilon$.  For $C \in \mathcal C$, the \emph{per-input component susceptibility of $M$ at component $C$ and input $x$} is the linear response of the component observable to this perturbation,
\begin{equation}\label{eq:sus}
  \chi_x^C([M]) := \frac{1}{\beta}\left.\frac{\partial}{\partial \varepsilon}\,\E_{p_{\beta}^\varepsilon}\bigl[ \phi_C \bigr]\right|_{\varepsilon=0}.
\end{equation}
The \emph{susceptibility matrix} is $\chi := (\chi_x^C([M]))_{x \in I,\, C \in \mathcal C}$.
\end{definition}

\begin{remark}\label{rem:fdt}
By \cite[Lemma~1]{baker2025studyingsmalllanguagemodels}, the derivative equals a covariance under the unperturbed Gibbs distribution,
\begin{equation}\label{eq:fdt}
  \chi_x^C([M]) = -\Cov_{p_{\beta}}\bigl[\phi_C,\, h_x^2 - H\bigr],
\end{equation}
and this is the primary form we work with for both the proofs and experiments of this paper.
See \cite{elliott2026susceptibilities} for a systematic account of susceptibilities as linear response in Bayesian learning, and of the covariance form as the fluctuation--dissipation counterpart of the derivative.
\end{remark}

\Cref{app:absorbing_example} works out a minimal example in which susceptibilities detect an error direction to which the Hessian is blind.

\section{Structures in Turing machines}\label{sec:structures}

As emphasised in \cite{murfet2025pas}, we anticipate the ``internal structure'' of programs to reflect patterns of correlations between degeneracies of the Turing machine code.  To this end, in lieu of a general definition we continue working with explicit examples of such structure, introducing two further ones in this section: \emph{path separation partitions} and \emph{recoding symmetries}.

\subsection{Path separation partitions}\label{sec:path_separation}

Loosely speaking, a Turing machine may reach a point in its computation where the outcome of the decision has been made, but superfluous computation continues until eventual termination at the result made inevitable at that point.  A \emph{path separation partition} formalises the special case in which that point is the initial state itself: the states of $Q \setminus \{q_0\}$ are partitioned as $A \sqcup R$, and beyond $q_0$ an accepting run visits only states in $A$, bouncing superfluously between them before reaching the accepting state $\qacc$ once and for all, while a rejecting run does likewise in $R$ before reaching $\qrej$.  Accepting inputs never visit $R$, and rejecting inputs never visit $A$.

Throughout this subsection $Q$ contains distinguished accept and reject states $\qacc, \qrej$, and we write
\[
  I_{\acc} := \{x \in I : y(x) = \qacc\}, \qquad I_{\rej} := \{x \in I : y(x) = \qrej\}
\]
for the accept- and reject-inputs.

\begin{definition}\label{def:path_separation_partition}
A \emph{path separation partition} of $M$ is a partition
$Q \setminus \{q_0\} = A \sqcup R$ with $\qacc \in A$, $\qrej \in R$,
such that for every $x \in I_{\acc}$ the classical run of $M$ on $x$ visits only states in $\{q_0\} \cup A$, and for every $x \in I_{\rej}$ it visits only states in $\{q_0\} \cup R$.
We call $M$ \emph{path-separable} if it admits a path separation
partition.
\end{definition}

\begin{definition}\label{def:psv}
Let $q_t(x)$ be the state of $M$ at step $t$ on input $x$, and let $\mathrm{Wrong}(x) := R$ if $x \in I_{\acc}$ and $\mathrm{Wrong}(x) := A$ if $x \in I_{\rej}$.  The \emph{path separation violation} of $M$ at a partition $Q \setminus \{q_0\} = A \sqcup R$ with $\qacc \in A$, $\qrej \in R$ is the $q$-weighted fraction of intermediate steps at which the run lies in the wrong region,
\[
  \PSV(M;\, A, R) := \frac{1}{T-1} \sum_{x \in I} q(x) \sum_{t=1}^{T-1} \one\!\bigl[ q_t(x) \in \mathrm{Wrong}(x) \bigr].
\]
The \emph{minimal path separation violation} of $M$ is the minimum
\[
  \PSV_{\min}(M) := \min_{A \sqcup R} \PSV(M;\, A, R)
\]
over such partitions.
\end{definition}

\begin{proposition}\label{prop:psv_zero}
Let $M$ be a classical solution.  Then $A \sqcup R$ is a path separation partition of $M$ if and only if $\PSV(M;\, A, R) = 0$.  In particular, $M$ is path-separable if and only if $\PSV_{\min}(M) = 0$.
\end{proposition}

\begin{proof}
The summands of $\PSV$ are non-negative and the input distribution has full support, so $\PSV(M; A, R) = 0$ if and only if $q_t(x) \notin \mathrm{Wrong}(x)$ for every $x \in I$ and every $1 \le t \le T - 1$.  Since $Q \setminus \{q_0\} = A \sqcup R$, a state is not in $\mathrm{Wrong}(x)$ if and only if it is $q_0$ or lies in the correct set $A, R$ depending on whether $x \in I_{\acc}$ or $x \in I_{\rej}$.  For $t = 0, T$: $q_0(x) = q_0$, and $q_T(x) = M(x) = y(x)$ because $M$ is a classical solution, which lies in $A$ for $x \in I_{\acc}$ and in $R$ for $x \in I_{\rej}$.  Hence $\PSV = 0$ holds if and only if every state visited by the run on $x$ lies in $\{q_0\} \cup A$ (accepting $x$) resp.\ $\{q_0\} \cup R$ (rejecting $x$), which is \cref{def:path_separation_partition}.
\end{proof}

\subsection{Recoding symmetries}\label{sec:recoding_symmetries}

In \cite{murfet2025pas} we considered the language $\mathcal{L}_{\texttt{A}}$ of strings containing an $\texttt{A}$ from the alphabet $\{\texttt{A}, \texttt{B}\}$.  A Turing machine $M_{\texttt{A}}$ for this language is only superficially different from a Turing machine $M_{\texttt{0}}$ for the language of strings containing a $\texttt{0}$ from the alphabet $\{\texttt{0}, \texttt{1}\}$.  These isomorphic machines can be placed in parallel to construct a machine $M_{\texttt{A}, \texttt{0}}$ deciding the language $\mathcal{L}_{\texttt{A}, \texttt{0}}$ of strings from $\{\texttt{A}, \texttt{B}\}^* \cup \{\texttt{0}, \texttt{1}\}^*$ containing either an $\texttt{A}$ or a $\texttt{0}$, by introducing a controller which first determines whether the input lies in $\{\texttt{A}, \texttt{B}\}^*$ or $\{\texttt{0}, \texttt{1}\}^*$ and then runs $M_{\texttt{A}}$ or $M_{\texttt{0}}$ accordingly.  Such isomorphic subroutines are the program structure we consider in this section, and we present the described $M_{\texttt{A}, \texttt{0}}$ machine as $M_3$ in \cref{tab:examples} and depict the symmetry in \cref{fig:recoding_action}.

\begin{definition}\label{def:recoding}
Given an alphabet $\Sigma$, set of states $Q$ and a set of inputs $I \subset \Sigma^\ast$, a \emph{recoding} of the data ($\Sigma$, $Q$, $I$) is a pair $g = (a, s)$ of bijections:
\begin{itemize}
    \item an \emph{alphabet map} $a: \Sigma \to \Sigma$ fixing the blank, $a(\blank) = \blank$, whose symbolwise extension $\Sigma^* \to \Sigma^*$, $\sigma_0 \dots \sigma_{n-1} \mapsto a(\sigma_0) \dots a(\sigma_{n-1})$, restricts to a map $I \to I$
    \item a state map $s\colon Q \to Q$ fixing $q_0, \qacc, \qrej$.
\end{itemize}
We typically infer the data $(\Sigma, Q, I)$ from context.  The condition that $I$ is closed under $a$ is what lets a recoding act on inputs, and the reindexing $x \mapsto g \cdot x$ in \cref{thm:sus_symmetry} depends on it.
\end{definition}

\begin{definition}\label{def:recoding_action}
The set of recodings form a subgroup of $\mathrm{Sym}(\Sigma) \times \mathrm{Sym}(Q)$, which acts on $W^\code$ by conjugation,
\[
  g \cdot [M] := (a \times s \times \id) \circ [M] \circ (a \times s)^{-1}.
\]
\end{definition}

\begin{definition}\label{def:recoding_symmetry}
The pair $g$ is a \emph{recoding symmetry} of $M$ (briefly, a \emph{symmetry} of $M$) if it fixes the code, $g \cdot [M] = [M]$. This says that the diagram
\[
  \begin{tikzcd}[column sep=large, row sep=large]
    \Sigma \times Q \arrow[r, "a \times s"] \arrow[d, "{[M]}"'] & \Sigma \times Q \arrow[d, "{[M]}"] \\
    \Sigma \times Q \times D \arrow[r, "a \times s \times \id"'] & \Sigma \times Q \times D
  \end{tikzcd}
\]
commutes.  The symmetries of $M$ form a group under composition, denoted
$\operatorname{Aut}(M)$.
\end{definition}

\begin{definition}\label{def:asymmetry}
Let $g = (a, s)$ be a recoding (\cref{def:recoding}), writing $a$ also for its symbolwise extension $a\colon I \lto I$.  The \emph{asymmetry} of $M$ at $(a, s)$ is
\[
  \asym(M;\, a, s) := \frac{1}{T-1} \sum_{x \in I} q(x) \sum_{t=1}^{T-1} \one\!\bigl[ q_t(a(x)) \neq s(q_t(x)) \bigr],
\]
where $q_t(x)$ is the state of $M$ at step $t$ on input $x$ (\cref{def:psv}).  We work with $\asym(M;\, a, s)$ at a fixed pair $(a, s)$ rather than minimising over $s$.
\end{definition}

\begin{proposition}\label{prop:asym_zero}
If $(a, s)$ is a symmetry of $M$, then $$\asym(M;\, a, s) = 0.$$
\end{proposition}

\begin{proof}
By induction on $T$.
\end{proof}

\begin{corollary}\label{cor:target_compat}
If $M$ is a classical solution and $(a, s)$ is a symmetry of $M$, then $y(a(x)) = s(y(x))$ for all $x \in I$.
\end{corollary}

\begin{proof}
The induction of \cref{prop:asym_zero} holds at the halting step, $q_T(a(x)) = s(q_T(x))$, and since $M$ is a classical solution $q_T(x) = M(x) = y(x)$ and $q_T(a(x)) = y(a(x))$.
\end{proof}

\subsection{Examples}\label{sec:five_machines}

The framework of \cref{sec:path_separation,sec:recoding_symmetries} applies to general Turing machines.  In this subsection we present some examples in a particular family of $5$-state DFAs, all deciding the same finite-alphabet language.

\begin{definition}\label{def:DFA}
A \emph{deterministic finite automaton} ($\DFA$) is a Turing machine $M$ such that all transitions $[M](\sigma, q) = (\sigma', q', d)$ satisfy $\sigma' = \sigma$ (it is ``read-only'') and $d = R$ (it always moves right). Thus the data of a DFA $M$ is determined by its state-function $[M]_Q : \Sigma \times Q \to Q$.
\end{definition}

\begin{remark}
Since the tape is read-only and the head sees each symbol in timestep order, the sequence of configurations $(\tau_i, q_i) \in \Sigma^{\mathbb Z, \blank} \times Q$ encountered by executing a DFA $M$ on input $x = x_0 \dots x_{n-1}$ has no $M$-dependence in the tape configuration $\tau_i$, and is uniquely described by the recurrence $q_i := [M]_Q(x_{i-1}, q_{i-1})$ (with $q_0$ the initial state).
\end{remark}

\begin{definition}\label{def:LA0_language}
Fix $\Sigma := \{\blank, \texttt{A}, \texttt{B}, \texttt{0}, \texttt{1}\}$.  The \emph{language} $\mathcal L_{\texttt{A}, \texttt{0}} \subseteq \{\texttt{A}, \texttt{B}\}^* \cup \{\texttt{0}, \texttt{1}\}^*$ is
\[
  \mathcal L_{\texttt{A}, \texttt{0}} := \bigl\{\, x \in \{\texttt{A}, \texttt{B}\}^* \cup \{\texttt{0}, \texttt{1}\}^* \,:\, x \text{ contains an } \texttt{A} \text{ or a } \texttt{0}\,\bigr\}.
\]
\end{definition}

\begin{definition}\label{def:LA0_decider}
Fix $Q := \{q_0, s_1, s_2, \qacc, \qrej\}$. An \emph{$\mathcal L_{\texttt{A}, \texttt{0}}$-DFA} is a DFA with alphabet $\Sigma$ and states $Q$ that \emph{classically decides} $\mathcal L_{\texttt{A}, \texttt{0}}$: its run on an input $x \in \{\texttt{A}, \texttt{B}\}^* \cup \{\texttt{0}, \texttt{1}\}^*$ eventually reaches the state $y(x)$, where the target map $y$ encodes the indicator function
\[
  y(x) := \begin{cases} \qacc, & x \in \mathcal L_{\texttt{A}, \texttt{0}}, \\ \qrej, & x \notin \mathcal L_{\texttt{A}, \texttt{0}}. \end{cases}
\]
\end{definition}

\begin{definition}\label{def:LA0_symmetry_data}
The \emph{alphabet involution} $\theta: \Sigma \to \Sigma$ is $\theta(\texttt{A}) = \texttt{0}, \theta(\texttt{B}) = \texttt{1}$.  We consider two recodings (\cref{def:recoding}), both with alphabet map $\theta$:
\begin{itemize}[topsep=2pt,itemsep=2pt]
  \item $(\theta, \id_Q)$, with the identity state map;
  \item $(\theta, s_1 \leftrightarrow s_2)$, with the \emph{subroutine swap}.
\end{itemize}
\end{definition}

\begin{remark}\label{rem:klein_group}
The two recodings $(\theta, s_1 \leftrightarrow s_2)$ and $(\theta, \id_Q)$ generate a subgroup of $\mathrm{Sym}(\Sigma) \times \mathrm{Sym}(Q)$ isomorphic to the Klein four-group $\mathbb Z/2 \times \mathbb Z/2$, the other two recodings being the identity $(\id_\Sigma, \id_Q)$ and the product $(\id_\Sigma, s_1 \leftrightarrow s_2)$. We focus on the generators in the experiments section (\cref{sec:experiments}), as they fix non-trivial subsets of the set of solution DFAs we consider, whereas the product $(\id_\Sigma, s_1 \leftrightarrow s_2)$ fixes only $M_5$.
\end{remark}

\begin{definition}\label{def:alphabet_partition}
The \emph{alphabet partition} for $\mathcal L_{\texttt{A},\texttt{0}}$ is $\Sigma \setminus \{\blank\} = \Sigma_1 \sqcup \Sigma_2$ with $\Sigma_1 := \{\texttt{A}, \texttt{B}\}$ and $\Sigma_2 := \{\texttt{0}, \texttt{1}\}$, exchanged by the alphabet involution $\theta$.  It induces a partition of the inputs into \emph{input classes} $I_i := I \cap \Sigma_i^*$ (so $I = I_1 \sqcup I_2$, exchanged by $\theta$) and of the components into \emph{component classes}
\[
  \mathcal C_i := \{(\sigma, q, f) \in \mathcal C : \sigma \in \Sigma_i\}, \qquad i = 1, 2, \qquad
  \mathcal C_\blank := \{(\sigma, q, f) \in \mathcal C : \sigma = \blank\},
\]
exchanged ($\mathcal C_1 \leftrightarrow \mathcal C_2$) and preserved ($\mathcal C_\blank$) by the action of either recoding of \cref{def:LA0_symmetry_data}.
\end{definition}

\cref{tab:examples} presents five examples with varying path separation partition and recoding symmetry.

\begin{table}[p]
\centering
\setlength{\tabcolsep}{5pt}
\renewcommand{\arraystretch}{1.25}
\begin{tabular*}{\textwidth}{@{\extracolsep{\fill}}>{\raggedright\arraybackslash}m{1.8cm} >{\centering\arraybackslash}m{6.4cm} >{\raggedright\arraybackslash\small}m{5.3cm}@{}}
\toprule
 & diagram & measurements \\
\midrule
$M_1$ &
\resizebox{\linewidth}{!}{%
\begin{tikzpicture}[
    >=Stealth, thick,
    state/.style={circle, draw, thick, minimum size=1.1cm, inner sep=0pt},
    accept/.style={state},
    dispatch/.style={->, thick},
    trans/.style={->, thick},
    selfloop/.style={->, thick},
    every node/.append style={font=\small},
  ]
  \node[state] (q0) at (0,0) {$q_0$};
  \node[state] (s1) at (3.8,1.8) {$s_1$};
  \node[state]  (s2) at (3.8,-1.8) {$s_2$};
  \node[accept, double] (qa) at (8.5,1.2) {$\qacc$};
  \node[accept] (qr) at (8.5,-1.2) {$\qrej$};
  \draw[->, thick] (-1.2,0) -- (q0);
  \draw[selfloop] (q0) to[out=145,in=105,looseness=6] node[above left=-2pt] {$\texttt{B},\texttt{1}$} (q0);
  \draw[dispatch] (q0) -- node[above left=-2pt] {$\texttt{A},\texttt{0}$} (s1);
  \draw[dispatch] (q0) -- node[below left=-2pt] {$\blank$} (s2);
  \draw[selfloop] (s1) to[out=120,in=60,looseness=5] node[above] {$\texttt{A},\texttt{B},\texttt{0},\texttt{1}$} (s1);
  \draw[trans] (s1) -- node[above=-1pt] {$\blank$} (qa);
  \draw[selfloop] (s2) to[out=240,in=300,looseness=5] node[below] {$\texttt{A},\texttt{B},\texttt{0},\texttt{1}$} (s2);
  \draw[trans] (s2) -- node[below=-1pt] {$\blank$} (qr);
\end{tikzpicture}} &
$\PSV(M_1;\, \{\qacc,s_1\},\{\qrej,s_2\})=0$;\ $\asym(M_1;\, \theta,\, \id_Q)=0$ \\
\midrule
$M_2$ &
\resizebox{\linewidth}{!}{%
\begin{tikzpicture}[
    >=Stealth, thick,
    state/.style={circle, draw, thick, minimum size=1.1cm, inner sep=0pt},
    accept/.style={state},
    dispatch/.style={->, thick},
    trans/.style={->, thick},
    selfloop/.style={->, thick},
    every node/.append style={font=\small},
  ]
  \node[state] (q0) at (0,0) {$q_0$};
  \node[state] (s1) at (3.8,1.8) {$s_1$};
  \node[state]  (s2) at (3.8,-1.8) {$s_2$};
  \node[accept, double] (qa) at (8.5,1.2) {$\qacc$};
  \node[accept] (qr) at (8.5,-1.2) {$\qrej$};
  \draw[->, thick] (-1.2,0) -- (q0);
  \draw[selfloop] (q0) to[out=145,in=105,looseness=6] node[above left=-2pt] {$\texttt{B},\texttt{1}$} (q0);
  \draw[dispatch] (q0) -- node[above left=-2pt] {$\texttt{A},\texttt{0}$} (s1);
  \draw[trans] (q0) -- node[below=-1pt,pos=0.85] {$\blank$} (qr);
  \draw[selfloop] (s1) to[out=120,in=60,looseness=5] node[above] {$\texttt{A},\texttt{B},\texttt{0},\texttt{1}$} (s1);
  \draw[dispatch] (s1) -- node[left] {$\blank$} (s2);
  \draw[selfloop] (s2) to[out=240,in=300,looseness=5] node[below] {$\texttt{A},\texttt{B},\texttt{0},\texttt{1}$} (s2);
  \draw[trans] (s2) -- node[pos=0.85,above left=-3pt] {$\blank$} (qa);
\end{tikzpicture}} &
$\PSV(M_2;\, \{\qacc,s_1,s_2\},\{\qrej\})=0$;\ $\PSV(M_2;\, \{\qacc\},\{\qrej,s_1,s_2\})\neq 0$;\ $\asym(M_2;\, \theta,\, \id_Q)=0$ \\
\midrule
$M_3$ &
\resizebox{\linewidth}{!}{%
\begin{tikzpicture}[
    >=Stealth, thick,
    state/.style={circle, draw, thick, minimum size=1.1cm, inner sep=0pt},
    accept/.style={state},
    dispatch/.style={->, thick},
    trans/.style={->, thick},
    selfloop/.style={->, thick},
    every node/.append style={font=\small},
  ]
  \node[state] (q0) at (0,0) {$q_0$};
  \node[state] (s1) at (3.8,1.8) {$s_1$};
  \node[state]  (s2) at (3.8,-1.8) {$s_2$};
  \node[accept, double] (qa) at (8.5,1.2) {$\qacc$};
  \node[accept] (qr) at (8.5,-1.2) {$\qrej$};
  \draw[->, thick] (-1.2,0) -- (q0);
  \draw[dispatch] (q0) -- node[above left=-2pt] {$\texttt{B}$} (s1);
  \draw[dispatch] (q0) -- node[below left=-2pt] {$\texttt{1}$} (s2);
  \draw[trans] (q0) -- node[above=-1pt,pos=0.3] {$\texttt{A},\texttt{0}$} (qa);
  \draw[selfloop] (s1) to[out=120,in=60,looseness=5] node[above] {$\texttt{B},\texttt{1}$} (s1);
  \draw[trans] (s1) -- node[above=-1pt] {$\texttt{A}$} (qa);
  \draw[trans] (s1) -- node[pos=0.18,below left=-3pt] {$\blank$} (qr);
  \draw[selfloop] (s2) to[out=240,in=300,looseness=5] node[below] {$\texttt{B},\texttt{1}$} (s2);
  \draw[trans] (s2) -- node[below=-1pt] {$\blank$} (qr);
  \draw[trans] (s2) -- node[pos=0.18,above left=-3pt] {$\texttt{0}$} (qa);
\end{tikzpicture}} &
$\asym(M_3;\, \theta,\, s_1{\leftrightarrow}s_2)=0$;\ $\asym(M_3;\, \theta,\, \id_Q)\neq 0$ \\
\midrule
$M_4$ &
\resizebox{\linewidth}{!}{%
\begin{tikzpicture}[
    >=Stealth, thick,
    state/.style={circle, draw, thick, minimum size=1.1cm, inner sep=0pt},
    accept/.style={state},
    dispatch/.style={->, thick},
    trans/.style={->, thick},
    selfloop/.style={->, thick},
    every node/.append style={font=\small},
  ]
  \node[state] (q0) at (0,0) {$q_0$};
  \node[state] (s1) at (3.8,1.8) {$s_1$};
  \node[state]  (s2) at (3.8,-1.8) {$s_2$};
  \node[accept, double] (qa) at (8.5,1.2) {$\qacc$};
  \node[accept] (qr) at (8.5,-1.2) {$\qrej$};
  \draw[->, thick] (-1.2,0) -- (q0);
  \draw[dispatch] (q0) -- node[above left=-2pt] {$\texttt{B},\texttt{1}$} (s1);
  \draw[trans] (q0) -- node[above,pos=0.32] {$\texttt{A},\texttt{0}$} (qa);
  \draw[dispatch] (s1) to[bend right=15] node[left] {$\texttt{B},\texttt{1}$} (s2);
  \draw[dispatch] (s2) to[bend right=15] node[right] {$\texttt{B},\texttt{1}$} (s1);
  \draw[trans] (s1) -- node[above=-1pt] {$\texttt{A},\texttt{0}$} (qa);
  \draw[trans] (s1) -- node[pos=0.7,below=-1pt] {$\blank$} (qr);
  \draw[trans] (s2) -- node[below=-1pt] {$\blank$} (qr);
  \draw[trans] (s2) -- node[pos=0.42,above=-1pt] {$\texttt{A},\texttt{0}$} (qa);
\end{tikzpicture}} &
$\asym(M_4;\, \theta,\, \id_Q)=0$;\ $\asym(M_4;\, \theta,\, s_1{\leftrightarrow}s_2)\neq 0$ \\
\midrule
$M_5$ &
\resizebox{\linewidth}{!}{%
\begin{tikzpicture}[
    >=Stealth, thick,
    state/.style={circle, draw, thick, minimum size=1.1cm, inner sep=0pt},
    accept/.style={state},
    dispatch/.style={->, thick},
    trans/.style={->, thick},
    selfloop/.style={->, thick},
    every node/.append style={font=\small},
  ]
  \node[state] (q0) at (0,0) {$q_0$};
  \node[state] (s1) at (3.8,1.8) {$s_1$};
  \node[state]  (s2) at (3.8,-1.8) {$s_2$};
  \node[accept, double] (qa) at (8.5,1.2) {$\qacc$};
  \node[accept] (qr) at (8.5,-1.2) {$\qrej$};
  \draw[->, thick] (-1.2,0) -- (q0);
  \draw[selfloop] (q0) to[out=145,in=105,looseness=6] node[above left=-2pt] {$\texttt{B},\texttt{1}$} (q0);
  \draw[trans] (q0) -- node[above=-1pt,pos=0.5] {$\texttt{A},\texttt{0}$} (qa);
  \draw[trans] (q0) -- node[below=-1pt,pos=0.5] {$\blank$} (qr);
\end{tikzpicture}} &
$\PSV(M_5;\, A,R)=0$ for every partition;\newline $\asym(M_5;\, \theta,\id_Q)=0$;\newline $\asym(M_5;\, \theta,s_1 \leftrightarrow s_2)=0$ \\
\bottomrule
\end{tabular*}
\caption{The five running-example $\mathcal L_{\texttt{A},\texttt{0}}$-DFAs: state diagram with relevant $\PSV$ and $\asym$ values. Omitted transitions are understood to be a self-loop $(\sigma, q) \mapsto q$.}
\label{tab:examples}
\end{table}

\section{Theorems}\label{sec:theorems}

Each algorithmic structure of \cref{sec:structures} constrains the susceptibility matrix: path separation partitions bound the rank of the off-diagonal blocks (\cref{thm:rank_bound}), and a recoding symmetry makes the matrix equivariant (\cref{thm:sus_symmetry}).  The latter also endows the singularity germ of the loss with a group action (\cref{cor:germ_action}).  Each statement assumes a compatibility property of the smoothly-relaxed UTM (unmatched-component invariance, recoding equivariance), introduced where it is used.

\subsection{Rank bounds from path separation}
\label{sec:rank_bound}

Fix a partition $Q \setminus \{q_0\} = A \sqcup R$ with $\qacc \in A$, $\qrej \in R$, and partition the components by their state:
\[
  \mathcal C_A := \{(\sigma, q, f) \in \mathcal C : q \in A\}, \quad
  \mathcal C_R := \{(\sigma, q, f) \in \mathcal C : q \in R\}, \quad
  \mathcal C_0 := \{(\sigma, q, f) \in \mathcal C : q = q_0\}.
\]
Partition the susceptibility matrix into the corresponding blocks:
\[
  \chi = \kbordermatrix{
    & \mathcal C_A & \mathcal C_R & \mathcal C_0 \\
    I_{\acc} & X_{\acc, A} & X_{\acc, R} & X_{\acc, 0} \\
    I_{\rej} & X_{\rej, A} & X_{\rej, R} & X_{\rej, 0}
  }.
\]
The blocks $X_{\acc, R}$ and $X_{\rej, A}$, pairing each input class
with the opposite region, are the \emph{off-diagonal blocks}.
Let $\mathcal U$ be a pseudo-UTM with a periodic smooth relaxation and cycle map $F_w$.
For a classical code $[M] \in W^\code$, a component
$C \in \mathcal C$, and a point $v \in W_C$, write
\[
  [M]_{C \leftarrow v}
\]
for the noisy code obtained from $[M]$ by replacing only the $C$-coordinate by
$v$.  Abusing notation, for a deterministic simulated tape
$\tau \in \Sigma^{\mathbb Z,\blank}$ and state $q \in Q$ we write $\tau$ in
place of $(e_{\tau(i)})_{i \in \mathbb Z}$ and $q$ in place of $e_q$, so that
$(\tau, q)$ also denotes a vertex of $\Delta\Cfg$.

\begin{definition}\label{def:unmatched_component_invariance}
A periodic smooth relaxation of a pseudo-UTM $\mathcal U$, with cycle
map $F_w$, is \emph{unmatched-component invariant} if, for every
classical code $[M] \in W^\code$, component $C \in \mathcal C$
belonging to the tuple $(\sigma_C, q_C)$, every $v \in W_C$ and deterministic configuration $(\tau, q) \in \Cfg \subseteq \Delta \Cfg$,
\[
  (\tau(0), q) \neq (\sigma_C, q_C)
  \quad\Longrightarrow\quad
  F_{[M]_{C \leftarrow v}}(\tau, q)
  =
  F_{[M]}(\tau, q).
\]
\end{definition}
In words, when the (smoothly-relaxed) UTM is simulating the deterministic transition $[M](\tau(0), q) = (\sigma', q', d)$, noise can be injected arbitrarily into components outside of the squares for $\sigma', q', d,$ without affecting the result.

\begin{remark}\label{rem:mismatched_induction}
Let $(\tau_t,q_t)$ be the classical configuration of
$M$ on input $x$ after $t$ simulated steps.  Assume the cycle map $F_w$ is unmatched-component invariant.  It follows by
induction on $T$ that, for all $T \geq 1$, if
\[
  (\tau_t(0), q_t) \neq (\sigma_C, q_C)
  \qquad \text{for all } 0 \le t < T,
\]
then $F^{\,T}_{[M]_{C \leftarrow v}}(\tau_0, q_0) = (\tau_T, q_T)$.
\end{remark}

\begin{lemma}\label{lem:mismatched_slice}
Assume the model $p(y|x, w)$ is given with respect to a pseudo-UTM
whose cycle map is unmatched-component invariant (\cref{def:unmatched_component_invariance}), and let $M$ be a classical solution
with path separation partition $A \sqcup R$.  Let either
\begin{enumerate}[label=(\roman*),topsep=2pt,itemsep=1pt]
  \item $x \in I_{\acc}$ and $C \in \mathcal C_R$, or
  \item $x \in I_{\rej}$ and $C \in \mathcal C_A$.
\end{enumerate}
Then, writing $W = U_C \times W_C$ with $[M] = (u^*,v^*)$,
\[
  h_x(u^*,v) = 0
  \qquad
  \forall v \in W_C.
\]
\end{lemma}

\begin{proof}
Consider case~(i); case~(ii) is symmetric.  Fix $x \in I_{\acc}$, $C \in \mathcal C_R$, and
$v \in W_C$, and note that $[M]_{C \leftarrow v} = (u^*,v)$.
Since $A \sqcup R$ is a path separation partition of $M$,
$q_t \not\in R$ for all $t$.  Writing $C = (\sigma_C, q_C, f_C)$, the condition $C \in \mathcal C_R$
means $q_C \in R$, and so by \cref{rem:mismatched_induction},
\[
  F_{[M]_{C \leftarrow v}}^{\,T}(\tau_0, q_0)
  =
  (\tau_T, q_T),
\]
and so $\Delta\step^T(x, [M]_{C \leftarrow v}) = q_T$.  Since
$M$ is a classical solution, $q_T = M(x) = y(x)$, so
\[
  p(y(x)|x,[M]_{C \leftarrow v})=1.
\]
By the definition of the error probability,
\[
  h_x(u^*,v)
  =
  h_x([M]_{C \leftarrow v})
  =
  1-p(y(x)|x,[M]_{C \leftarrow v})
  =
  0,
\]
proving the lemma.
\end{proof}

\begin{remark}\label{rem:srp_slice}
The smooth relaxations of the staged pseudo-UTM \cite[\S 5.1]{murfet2025pas} and of the lookup
pseudo-UTM of \cref{sec:lookup_utm} are both unmatched-component invariant
(\cref{lem:staged_lookup_unmatched_component_invariance}).
\end{remark}

\begin{theorem}\label{thm:rank_bound}
Assume the model $p(y|x, w)$ is given with respect to a pseudo-UTM
whose cycle map is unmatched-component invariant.  Let $M$ be a classical solution
with path separation partition $A \sqcup R$.  Then
\[
  \rank(X_{\acc, R}) \le 2, \qquad \rank(X_{\rej, A}) \le 2.
\]
\end{theorem}

\begin{proof}
Consider $X_{\acc, R}$; the argument for $X_{\rej, A}$ is symmetric.  Write $\langle - \rangle$ for the expectation under the Gibbs distribution $p_{\beta}$.  For $(x, C) \in I_{\acc} \times \mathcal C_R$ we have:
\begin{align*}
  \chi_x^C([M])
  &= -\Cov_{p_{\beta}}\bigl[\phi_C,\, h_x^2 - H\bigr]
  && \text{(\cref{def:susceptibility})} \\
  &= \langle \phi_C \rangle\, \langle h_x^2 - H \rangle - \langle \phi_C\, (h_x^2 - H) \rangle \\
  &= \langle \phi_C \rangle\, \langle h_x^2 \rangle - \langle \phi_C \rangle\, \langle H \rangle - \langle \phi_C\, h_x^2 \rangle + \langle \phi_C\, H \rangle \\
  &= \langle \phi_C \rangle\, \langle h_x^2 \rangle - \langle \phi_C \rangle\, \langle H \rangle + \langle \phi_C\, H \rangle
  && \text{(\cref{lem:mismatched_slice}(i), \cref{def:phi_C})} \\
  &= \langle \phi_C \rangle\, \langle h_x^2 \rangle + \Cov_{p_{\beta}}\bigl[\phi_C,\, H\bigr].
\end{align*}
Each column of $X_{\acc, R}$ is therefore a linear combination of the two vectors $(\langle h_x^2 \rangle)_{x \in I_{\acc}}$ and $(1)_{x \in I_{\acc}}$, so $\rank(X_{\acc, R}) \le 2$.
\end{proof}

\subsection{Susceptibility symmetry from recoding symmetry}\label{sec:sus_symmetry}

In \cref{def:recoding_action}, we define the action of a recoding $g = (a: \Sigma \to \Sigma, s: Q \to Q)$ on codes $[M] \in W^\code$. We now extend this action notation $g \cdot (-)$ to all objects involving $\Sigma$ and $Q$ in the expected way. For a map $f \colon T \to T'$ of finite sets, the \emph{pushforward}
$f_* \colon \Delta T \to \Delta T'$ is the linear extension on the vertices $f_*(\sum_t \lambda_t \cdot e_t) := \sum_t \lambda_t \cdot e_{f(t)}$.

\begin{definition}\label{def:recoding_action_ext}
Let $g = (a,s)$ be a recoding. Define the following actions:
\begin{itemize}[topsep=2pt,itemsep=2pt]
\item On noisy Turing machines $w \in W$ (as a function $w:\Sigma \times Q \to \Delta\Sigma \times \Delta Q \times \Delta D$):
  \[
    g \cdot w := (a_* \times s_* \times \id) \circ w \circ (a \times s)^{-1},
  \]
\item On inputs $x = \sigma_1 \cdots \sigma_k \in I$:
  \[
    g \cdot x := a(\sigma_1) \cdots a(\sigma_k),
  \]
\item On states $r \in Q$: \[g \cdot r := s(r),\]
\item On components $C = (\sigma, q, f) \in \mathcal C$ (\cref{def:component}): \[g \cdot (\sigma, q, f) := (a(\sigma),\, s(q),\, f),\]
\item On noisy configurations $((\boldsymbol\sigma_i)_{i\in\mathbb Z},\mathbf q) \in \Delta\Cfg$ (\cref{def:noisy_configuration}):
  \[
    g \cdot \bigl((\boldsymbol\sigma_i)_{i\in\mathbb Z},\mathbf q\bigr)
    :=
    \bigl((a_*\boldsymbol\sigma_i)_{i\in\mathbb Z},\, s_*\mathbf q\bigr).
  \]
\end{itemize}
\end{definition}

\begin{remark}\label{rem:action_compat}
Reading the formula of \cref{def:recoding_action_ext} one factor at a time: the action of $g = (a, s)$ carries the simplex factor $W_{C}$ onto $W_{g \cdot C}$, acting within it by the pushforward of the corresponding value map ($a$ on $\Delta\Sigma$ factors, $s$ on $\Delta Q$ factors, the identity on $\Delta D$ factors).  In particular $w \mapsto g \cdot w$ is a measure-preserving bijection of $W$.  On the vertices $W^\code \subseteq W$ the pushforward of a point mass is the point mass at the image, $a_* e_\sigma = e_{a(\sigma)}$ and $s_* e_q = e_{s(q)}$, so the action restricts to the conjugation action of \cref{def:recoding_action}.  In particular a symmetry $g$ of $M$ fixes $[M]$, computed in $W$.
\end{remark}

\begin{definition}\label{def:recoding_equivariance}
A periodic smooth relaxation of a
pseudo-UTM $\mathcal U$, with cycle map $F_w$, is \emph{recoding
equivariant} if, for every recoding $g$, $w\in W$, and
noisy configuration $z\in\Delta\Cfg$,
\[
  g \cdot F_{w}(z)
  =
  F_{g \cdot w}\bigl(g \cdot z\bigr).
\]
In words, running one relaxed period and then recoding the
configuration agrees with recoding the code and the configuration first and then running one period.
\end{definition}

\begin{remark}\label{rem:equivariance_induction}
Assume the cycle map $F_w$ is recoding equivariant.  It
follows by induction on $T$ that, for all $T \geq 1$, every recoding
$g$, every noisy code
$w \in W$, and every noisy configuration $z \in \Delta\Cfg$,
\[
  g \cdot F_{w}^{\,T}(z)
  =
  F_{g \cdot w}^{\,T}\bigl(g \cdot z\bigr).
\]
\end{remark}

\begin{remark}\label{rem:equivariance_machines}
The lookup pseudo-UTM of \cref{sec:lookup_utm} is recoding
equivariant (\cref{cor:lookup_recoding_equivariance}).  The staged pseudo-UTM of
\cite[\S 5.1]{murfet2025pas} violates the condition because its smooth relaxation
depends on the order of the tuples on the description tape, which recodings permute.
\end{remark}

\begin{lemma}\label{lem:model_equivariance}
Let $g = (a, s)$ be a recoding for which the target map is
equivariant, $y(g \cdot x) = g \cdot y(x)$ for all $x \in I$.  Assume the
model $p(y|x, w)$ is given with respect to a pseudo-UTM whose cycle map
is recoding equivariant
(\cref{def:recoding_equivariance}).  Then
$h_{g \cdot x}(g \cdot w) = h_x(w)$ for every $x \in I$ and $w \in W$.
\end{lemma}

\begin{proof}
With $g = (a, s)$, the claim reads $h_{a(x)}\bigl((a, s) \cdot w\bigr) = h_x(w)$.  Fix $x \in I$ and $w \in W$.  Let
$z_x \in \Delta\Cfg$ be the
configuration of the simulated machine with input $x$ on the working tape and simulated
state $e_{q_0}$.  Since $a(\blank) = \blank$ and $s(q_0) = q_0$, we have
$(a, s) \cdot z_x = z_{a(x)}$, and so by \cref{rem:equivariance_induction},
\[
  (a, s) \cdot F_{w}^{\,T}(z_x) = F_{(a, s) \cdot w}^{\,T}(z_{a(x)}).
\]
The cycle-map iterate read off at the state tape is the model
(\cref{def:model_utm}).  Since the configuration action is by $s_*$ on the state distribution,
\[
  s_*\, \pi_{\Delta Q} F_{w}^{\,T}(z_x)
  = \pi_{\Delta Q} F_{(a, s) \cdot w}^{\,T}(z_{a(x)}).
\]
Using $(s_* \mu)(s(r)) = \mu(r)$ (valid for any bijection
$s$) and $y(a(x)) = s(y(x))$,
\begin{align*}
  p(y(a(x))|a(x), (a, s) \cdot w)
  &= \pi_{\Delta Q}F_{(a, s) \cdot w}^{\,T}(z_{a(x)})(y(a(x))) \\
  &= \bigl(s_*\, \pi_{\Delta Q}F_{w}^{\,T}(z_x)\bigr)(s(y(x))) \\
  &= \pi_{\Delta Q}F_{w}^{\,T}(z_x)(y(x))
  = p(y(x)|x, w).
\end{align*}
Hence $h_{a(x)}\bigl((a, s) \cdot w\bigr) = 1 - p(y(a(x))|a(x), (a, s) \cdot w) = 1 - p(y(x)|x, w)
= h_x(w)$, as required.
\end{proof}

\begin{proposition}\label{prop:H_invariance}
Let $M$ be a classical solution and let $g = (a, s)$ be a symmetry of $M$ (\cref{def:recoding_symmetry}), assume the model $p(y|x, w)$ is given with respect to a pseudo-UTM whose cycle map is recoding equivariant (\cref{def:recoding_equivariance}), and suppose the input distribution is invariant, $q(g \cdot x) = q(x)$ for all $x \in I$.  Then
\[
  H(g \cdot w) = H(w) \qquad (w \in W).
\]
\end{proposition}

\begin{proof}
$M$ being a classical solution and $g$ a symmetry of $M$, \cref{cor:target_compat} gives the target equivariance $y(g \cdot x) = g \cdot y(x)$, so \cref{lem:model_equivariance} applies and $h_x\bigl((a, s) \cdot w\bigr) = h_{a^{-1}(x)}(w)$.  Reindexing $x \mapsto a(x)$ and using that $q$ is invariant,
\begin{align*}
  H\bigl((a, s) \cdot w\bigr)
  &= \sum_x q(x)\, h_x\bigl((a, s) \cdot w\bigr)^2 \\
  &= \sum_x q(x)\, h_{a^{-1}(x)}(w)^2 \\
  &= \sum_x q(a(x))\, h_{x}(w)^2 \\
  &= H(w). \qedhere
\end{align*}
\end{proof}

\begin{theorem}\label{thm:sus_symmetry}
Let $M$ be a classical solution and let $g = (a, s)$ be a symmetry of $M$ (\cref{def:recoding_symmetry}).  Assume the model $p(y|x, w)$ is given with respect to a pseudo-UTM whose cycle map is recoding equivariant (\cref{def:recoding_equivariance}).  Suppose the input distribution is invariant, $q(g \cdot x) = q(x)$ for all $x \in I$, and the prior is invariant, $\varphi(g \cdot w) = \varphi(w)$ for all $w \in W$.  Then for every $x \in I$ and $C \in \mathcal C$,
\[
  \chi_{g \cdot x}^{g \cdot C}([M]) = \chi_x^C([M]).
\]
\end{theorem}

\begin{proof}
Since $M$ is a classical solution and $g$ a symmetry of $M$, \cref{cor:target_compat} and \cref{lem:model_equivariance} give $h_{g \cdot x}(g \cdot w) = h_x(w)$ for all $x$ and $w$, and \cref{prop:H_invariance} gives $H(g \cdot w) = H(w)$ for all $w$.

By \cref{rem:action_compat}, $w \mapsto g \cdot w$ is a measure-preserving bijection of $W$ fixing $[M]$.  Write $P := p_{\beta}$ for the Gibbs distribution of \cref{def:susceptibility}.  Since $\varphi$ and $H$ are invariant, the change of variables $w = g \cdot w'$ leaves $P$ invariant, so $\Cov_P[f_1, f_2] = \Cov_P[f_1(g \cdot -),\, f_2(g \cdot -)]$ for all integrable $f_1, f_2$.

Write $W = U_C \times W_C$ with $w = (u, v)$ and $[M] = (u^*, v^*)$, and likewise $W = U_{g \cdot C} \times W_{g \cdot C}$ with $w = (u', v')$ and $[M] = (u'^{\,*}, v'^{\,*})$.  Since the action carries each factor $W_{C'}$ onto $W_{g \cdot C'}$ by a measure-preserving pushforward (\cref{rem:action_compat}), it restricts to a measure-preserving bijection $U_C \to U_{g \cdot C}$, and since it fixes $[M]$ this bijection carries $u^*$ to $u'^{\,*}$.  Hence $\delta(u'(g \cdot w) - u'^{\,*}) = \delta(u - u^*)$.  With the invariance of $H$ this gives $\phi_{g \cdot C}(g \cdot w) = \phi_{C}(w)$, and $h_{g \cdot x}(g \cdot w) = h_x(w)$ gives $(h_{g \cdot x}^2 - H)(g \cdot w) = (h_{x}^2 - H)(w)$.  Therefore
\begin{align*}
  \chi_{g \cdot x}^{g \cdot C}([M])
    &= -\Cov_P[\phi_{g \cdot C},\, h_{g \cdot x}^2 - H] \\
    &= -\Cov_P[\phi_{g \cdot C}(g \cdot -),\, (h_{g \cdot x}^2 - H)(g \cdot -)] \\
    &= -\Cov_P[\phi_{C},\, h_{x}^2 - H] \\
    &= \chi_{x}^{C}([M]). \qedhere
\end{align*}
\end{proof}

\begin{corollary}\label{cor:germ_action}
Under the hypotheses of \cref{prop:H_invariance}, the action restricts to bijections of $W_0$ fixing $[M]$, endowing the singularity germ $(H, [M])$ of \cref{def:squared_error_loss} with an action of the cyclic group $\langle g \rangle \le \operatorname{Aut}(M)$.  More generally, if every element of a subgroup $G \le \operatorname{Aut}(M)$ satisfies the hypotheses, the germ carries a $G$-action.
\end{corollary}

\begin{proof}
Since $w \mapsto g \cdot w$ is a bijection of $W$ fixing $[M]$ (\cref{rem:action_compat}) and $H$ is invariant (\cref{prop:H_invariance}), it maps $W_0$ to $W_0$, and the dot is an action (\cref{def:recoding_action_ext}).  Hence if $g$ has order $n$ in $\operatorname{Aut}(M)$, the germ carries a well-defined action of $\langle g \rangle \cong \mathbb Z/n$.  If every element of a subgroup $G \le \operatorname{Aut}(M)$ satisfies the hypotheses, the invariance of $H$ holds for each element and the germ carries a $G$-action.
\end{proof}

\section{Experiments}\label{sec:experiments}

We estimate the susceptibility matrix $\chi$ of DFAs that decide $\mathcal L_{\texttt{A},\texttt{0}}$ and measure the ranks and symmetries described in \cref{sec:theorems}, with the following objectives:

\begin{itemize}
    \item To verify that proven theoretical properties of the true susceptibility are reflected in estimated susceptibilities.
    \item To observe whether the theorems generalise further in the empirical context (e.g. to the converse of \cref{thm:rank_bound}).
\end{itemize}

We study the example DFAs of \cref{tab:examples} in \cref{sec:expt_pointwise}, and then extend to the whole set of solution DFAs in \cref{sec:expt_population_group}. In \cref{sec:expt_reading} we attempt to explain clusters of DFAs that emerge from susceptibility-space.

\subsection{The DFA task and parameter space}\label{sec:expt_task}

Recall the language
\[
  \mathcal L_{\texttt{A}, \texttt{0}} = \bigl\{\, x \in \{\texttt{A}, \texttt{B}\}^* \cup \{\texttt{0}, \texttt{1}\}^* \,:\, x \text{ contains an } \texttt{A} \text{ or a } \texttt{0}\,\bigr\}
\]
of \cref{def:LA0_language}.  Fix the input set of nonempty strings of
length at most three,
$I := \{x \in \{\texttt{A},\texttt{B}\}^* \cup \{\texttt{0},\texttt{1}\}^* : 1 \le |x| \le 3\}$, and a fixed
number of time steps $T = 5$.  The input distribution $q(x)$ on $I$ selects the alphabet class,
then a length in $\{1,2,3\}$, then a string of that class and length, each
uniformly, resulting in $q(x) = \tfrac{1}{12}, \tfrac{1}{24}, \tfrac{1}{48},$
depending on $|x|$.
We consider $\mathcal L_{\texttt{A},\texttt{0}}$-DFAs (\cref{def:LA0_decider}) with
the extra condition that for $q \in \{\qacc, \qrej\}$, the transition $[M](\sigma, q) = (\sigma', q', d)$ satisfies $q' = q$. We preserve these restrictions on transitions in the presence of noise, leaving only noise in the choice $q'$ when $q \in \{q_0, s_1, s_2\}$, and
hence define
\begin{equation}
W^{\DFA} := \prod_{(\sigma, q) \in \Sigma \times \{q_0, s_1, s_2\}} \Delta Q \quad \lhook\joinrel\longrightarrow \quad W
\end{equation}
This is the subspace of $W$ that fixes $W_{(\sigma, q, \Sigma)}$ to $e_\sigma \in\Delta \Sigma$, $W_{(\sigma, q, D)}$ to $e_R \in \Delta D$, and, when $q \in \{\qacc, \qrej\}$ is a terminal state, fixes $W_{(\sigma, q, Q)}$ to $e_q \in \Delta Q$.
The definitions of $p_{\beta}(w), \chi_x^C([M])$ from \cref{def:susceptibility} and the empirical counterparts in \cref{sec:empirical_sus}  and both theorems of
\cref{sec:theorems} follow verbatim, with $W^{\DFA}$ in place of $W$\footnote{The
equivariance of \cref{thm:sus_symmetry} transfers because the action of $(a, s)$
preserves terminal states.}. All susceptibilities are computed with respect to $W^\DFA$, and the components we choose for observables are exactly those $(\sigma, q)$-indexed subspaces $\Delta Q$ that make up the product $W^{\DFA}$. That is, we study susceptibilities with respect to component observables $\phi_{(\sigma, q, Q)}$, or simply $\phi_{(\sigma,q)}$, with $q \in \{q_0, s_1, s_2\}$. We continue writing $\mathcal C$ for the set of observables, which now denotes the subset $\Sigma \times \{q_0, s_1, s_2\}$ with $f = Q$ fixed.

\subsection{From theoretical to empirical susceptibilities}\label{sec:empirical_sus}
Our theoretical treatment concerned $\chi_x^C([M])$, built from the squared-error potential $H$.  The experiments use the \emph{log-loss} susceptibility, replacing $H$ and $h_x^2$ throughout \cref{def:phi_C,def:susceptibility} by the cross-entropy $L$ and $\ell_x$, with $\ell_x(w) := -\log p(y(x)|x, w)$ (the $\mu = 0$ case of \eqref{eq:shifted_losses}) and $L := \E_{x \sim q}[\ell_x]$.  Since $[M]$ is a classical solution $L([M]) = 0$, so $\phi_C^{L} = \delta(u - u^*)\, L$ and, by \cref{rem:fdt},
\[
  \chi_x^{C, L}([M]) = -\Cov_{p_{\beta}^{L}}\bigl[\phi_C^{L},\, \ell_x - L\bigr],
  \qquad
  p_{\beta}^{L} \propto \exp\{-\beta L\}\,\varphi.
\]
Both theorems of \cref{sec:theorems} hold verbatim for $\chi_x^{C, L}$, even at $\mu = 0$ (\cref{prop:signatures_general,rem:log_loss_mu_zero}).  Moreover, the renormalised susceptibility \eqref{eq:renorm_sus} that we estimate below converges, as $\mu \to 0$, to its value at $\mu = 0$ (\cref{prop:mu_zero_limit}).

\paragraph{Smooth relaxation.}  The model $p(y | x, w) = \Delta\step^T(x, w)_y$ (\cref{def:model_utm}) is fixed by the choice of pseudo-UTM whose step is relaxed (\cref{def:utm_relaxation}).  We use two, with cycle maps $F^{\mathrm{staged}}$ and $F^{\mathrm{lookup}}$ (\cref{def:circuit_cycle_map}) given in closed form by \cref{prop:cycle_factor}: the \emph{staged UTM} of \cite[Appendix~G]{clift2021geometryprogramsynthesis} (used in \cite[\S 5.1]{murfet2025pas}; $\nu = \nu^{\mathrm{staged}}$), and the \emph{lookup UTM} of \cref{sec:lookup_utm} ($\nu = \nu^{\mathrm{lookup}}$).  The cycle map of the lookup UTM is the naive probabilistic extension of the eval cycle circuit (\cref{cor:lookup_srp}), so on $W^\DFA$ its model is the matrix product of \cref{prop:dfa_matrix}.  The staged smooth relaxation depends on the order of the description tuples through $\nu^{\mathrm{staged}}$ \eqref{eq:decision_weights} and does not reduce to this form.

\paragraph{Localisation.}  To keep sampling near $[M]$ we localise the prior, following \cite{baker2025studyingsmalllanguagemodels} (see \cite[\S 6]{elliott2026susceptibilities}).  Recall from \cref{def:component} that $W = \prod_{C \in \mathcal C} W_C$ with $W_C = \Delta Z_C$ and $[M]$ a product of vertices $e_{[M]_C}$.  For a \emph{localisation strength} $\gamma \ge 0$ and a \emph{boundary parameter} $\alpha > 0$, the \emph{localising prior}, or \emph{localiser}, is the product of Dirichlet densities
\begin{equation}\label{eq:dirichlet_localiser}
  \varphi_{[M], \gamma, \alpha}(w) := \prod_{C \in \mathcal C} \operatorname{Dir}\bigl(w_C;\; \gamma\, e_{[M]_C} + \alpha\, \mathbf 1_{Z_C}\bigr),
\end{equation}
giving the \emph{local Gibbs distribution} \cite{bissiri2016general}
\begin{equation}\label{eq:local_gibbs}
  p\bigl(w;\, [M], \beta, \gamma, \alpha\bigr) \propto \exp\bigl\{ -\beta\, L(w) \bigr\} \varphi_{[M], \gamma, \alpha}(w).
\end{equation}
We use Dirichlet distributions at each simplex because the sampler we use, the stochastic gradient Riemannian Langevin dynamics (SGRLD) sampler of Patterson and Teh \cite{patterson2013sgrld}, is derived for Dirichlet priors on simplices (\cref{app:grld}).
Larger $\gamma$ confines sampling to a smaller neighbourhood of $[M]$, and, at our $\gamma \ge 1$, $\alpha < 1$ ensures that the log-gradients all point towards $[M]$ rather than an interior mode. These facts make this $(\gamma, \alpha)$-indexed family of distributions a natural choice for our $W$, and it is discussed further in \cref{rem:dirichlet_bayes}. The experiments use $(\gamma, \alpha) = (1, 0.01)$.

\paragraph{Estimating susceptibilities.}  Expanding the covariance $\chi_x^{C, L} = -\Cov_p[\phi_C^{L},\, \ell_x - L]$ gives
\[
  \chi_x^{C, L}
  = -\E_{p}\bigl[\phi_C^{L}\,(\ell_x - L)\bigr]
    + \E_{p}\bigl[\phi_C^{L}\bigr]\, \E_{p}[\ell_x - L].
\]
Since $\phi_C^{L} = \delta(u - u^*)\, L$ is supported on the slice $u = u^*$, the two $\phi_C^{L}$-expectations are integrals over that slice (\cref{def:phi_C}).  Write $Z$ for the partition function of the local Gibbs distribution \eqref{eq:local_gibbs}, and let
\[
  p_C(v) := \frac{1}{Z_C}\, e^{-\beta\, L(u^*, v)}\, \varphi_{[M], \gamma, \alpha}(u^*, v),
  \qquad
  Z_C := \int_{W_C} e^{-\beta\, L(u^*, v)}\, \varphi_{[M], \gamma, \alpha}(u^*, v)\, dv,
\]
be the \emph{restricted Gibbs distribution} on the slice and the \emph{component partition function}.
Then \cref{def:phi_C} gives $\E_{p}[\phi_C^{L}\, g] = \tfrac{Z_C}{Z}\, \E_{p_C}[L\, g]$ for a function $g$, so that
\begin{equation}\label{eq:sus_two_normalisers}
  \chi_x^{C, L}
  = \frac{Z_C}{Z}\Bigl( -\E_{p_C}\bigl[L\,(\ell_x - L)\bigr]
    + \E_{p_C}[L]\; \E_{p}[\ell_x - L] \Bigr).
\end{equation}
We estimate expectations by sampling both distributions with \emph{gradient Riemannian Langevin dynamics} (GRLD), a full-gradient variant of the SGRLD sampler of Patterson and Teh \cite{patterson2013sgrld} (\cref{app:grld}).  In practice we do not compute the ratio $Z_C / Z$,\footnote{For $\alpha < 1$ this ratio is in fact infinite, and the localised $\chi_x^{C, L}$ is itself ill-defined.  The divergence is a single constant in each column, and the renormalised susceptibility below is free of it, being defined by $p_C$- and $p$-expectations alone.  The standardised $\psi$ is unchanged by rescaling a column in any case, so it agrees with what $\chi_x^{C, L}$ would give whenever the latter is defined.  \Cref{rem:slice_infinity} does the bookkeeping.} and instead take as our targets the \emph{renormalised susceptibility}
\begin{equation}\label{eq:renorm_sus}
  \tilde\chi_x^C
  := -\E_{p_C}\!\bigl[L\,(\ell_x - L)\bigr]
     + \E_{p_C}[L]\; \E_{p}[\ell_x - L]
  = \frac{Z}{Z_C}\, \chi_x^{C, L}
\end{equation}
and, following \cite{baker2025studyingsmalllanguagemodels, gordon2025lang3}, the \emph{standardised susceptibility}
\begin{equation}\label{eq:standardised_sus}
  \psi_x^C := \frac{\tilde\chi_x^C - \overline{\tilde\chi^C}}{\operatorname{std}_x[\,\tilde\chi_x^C\,]},
  \qquad
  \overline{\tilde\chi^C} := \frac{1}{|I|}\sum_{x \in I} \tilde\chi_x^C.
\end{equation}
This standardisation cancels the constant factor $Z / Z_C$, avoiding the need to compute it \cite[\S 6.3]{elliott2026susceptibilities}.  For a column with zero variance the divisor in \eqref{eq:standardised_sus} is replaced by $1$, so the column is merely centred.  In particular a column that is identically zero stays zero.
From $r$ samples $\{v_t\}_{t=1}^{r}$ taken from $p_C$ and $r'$ samples $\{w_s\}_{s=1}^{r'}$ taken from $p$, our estimator of \eqref{eq:renorm_sus} replaces each expectation by the corresponding sample mean,
\begin{equation}\label{eq:two_chain_estimator}
\begin{aligned}
  \widehat{\tilde\chi}^{\,C}_x = & -\frac{1}{r} \sum_{t=1}^{r} L(u^*, v_t)\, \bigl[ \ell_x(u^*, v_t) - L(u^*, v_t) \bigr] \\
  & + \Bigl( \frac{1}{r} \sum_{t=1}^{r} L(u^*, v_t) \Bigr) \cdot \Bigl( \frac{1}{r'} \sum_{s=1}^{r'} \bigl[ \ell_x(w_s) - L(w_s) \bigr] \Bigr).
\end{aligned}
\end{equation}

Applying \eqref{eq:standardised_sus} to $\widehat{\tilde\chi}$ gives the estimate $\widehat{\psi}$ of $\psi$.  \Cref{lem:standardisation_equivariance,prop:standardisation_rank} verify that the conclusions of \cref{sec:theorems} survive renormalisation and standardisation.

\subsection{Block structure in estimated susceptibilities}\label{sec:expt_pointwise}

We compute the standardised susceptibilities $\widehat{\psi}([M]) \in \mathbb R^{I \times \mathcal C}$ at inverse temperature $\beta=30$ for each of the example DFAs $M_1, M_2, M_3, M_4, M_5$ of \cref{tab:examples}, and examine their block structure against the theorems of \cref{sec:theorems}. We use the smooth relaxation of the lookup UTM for the likelihood $p(y|x,w) := \Delta \step^T_{\mathrm{lookup}}(x,w)_y$, which satisfies all UTM hypotheses of \cref{sec:theorems}.

\begin{figure}[tbp]
\centering
\includegraphics[width=\linewidth]{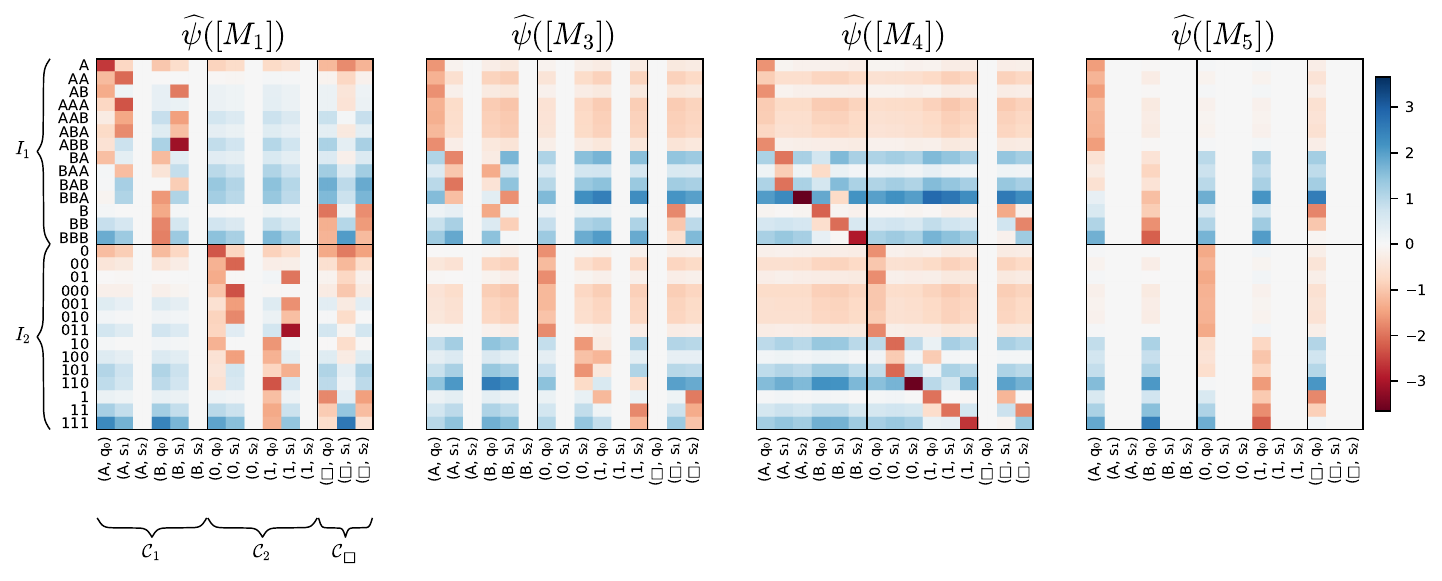}
\caption{\textbf{Susceptibilities of the DFAs $M_1, M_3, M_4, M_5$}. Rows are inputs $x \in I \subseteq \Sigma^*$ and columns are components $(\sigma, q)$, denoting the subspace $\Delta Q$ of distributions over the ``next-state'' behaviour of $[M]$ when seeing symbol $\sigma$ in state $q$. The value at row $x$, column $(\sigma, q)$ is the linear-response of the component observable $\phi_{(\sigma, q)}$ (averaged over $W^\DFA$) to a perturbation of the data distribution on $I$ that increases the probability of $x$ (\cref{def:susceptibility}).
Rows are grouped by the partition $x \in \{\texttt{A},\texttt{B}\}^* \sqcup \{\texttt{0},\texttt{1}\}^*$, and columns by the partition $\sigma \in \{\texttt{A},\texttt{B}\} \sqcup \{\texttt{0},\texttt{1}\} \sqcup \{\blank \}$, where $\sigma$ is the symbol being read by $M$. The susceptibilities are computed by sampling a neighbourhood of $M_i$ on the lookup UTM at inverse temperature $\beta = 30$.}
\label{fig:examples_matrix}
\end{figure}

\cref{fig:examples_matrix} shows the estimates $\widehat{\psi}([M_i])$, with rows grouped into input classes $I_1 = \{\texttt{A},\texttt{B}\}^{\leq 3}, I_2 := \{\texttt{0},\texttt{1}\}^{\leq 3}$ and columns into the component classes
\[
\begin{aligned}
  \mathcal C_1 &= \{(\sigma, q) \mid \sigma \in \{\texttt{A},\texttt{B}\}, q \in \{q_0, s_1, s_2\}\}, \\
  \mathcal C_2 &= \{(\sigma, q) \mid \sigma \in \{\texttt{0},\texttt{1}\}, q \in \{q_0, s_1, s_2\}\}, \\
  \mathcal C_\blank &= \{(\blank, q) \mid q \in \{q_0, s_1, s_2\}\},
\end{aligned}
\]
as in \cref{def:alphabet_partition}.
These groupings induce a block partition
\begin{equation}
  \widehat{\psi}([M]) = \kbordermatrix{
    & \mathcal C_1 & \mathcal C_2 & \mathcal C_\blank \\
    I_1 & X_{1, 1} & X_{1, 2} & X_{1, \blank} \\
    I_2 & X_{2, 1} & X_{2, 2} & X_{2, \blank}
  },
\end{equation}
and it is apparent in \cref{fig:examples_matrix} that the matrices for $M_1, M_4, M_5$ have near equality with the permuted matrix
\begin{equation}
  \begin{bmatrix}
    X_{2, 2} & X_{2, 1} & X_{2, \blank} \\
    X_{1, 2} & X_{1, 1} & X_{1, \blank}
  \end{bmatrix}
\end{equation}
which exchanges the row-blocks $(X_{1, j} \leftrightarrow X_{2, j})$ and the column-blocks $(X_{i, 1} \leftrightarrow X_{i, 2})$. This empirically measured (approximate) symmetry on $\widehat{\psi}([M_1]), \widehat{\psi}([M_4]), \widehat{\psi}([M_5])$ is a shadow of a symmetry intrinsic to $M_1, M_4, M_5$, namely that their transition function $\Sigma \times Q \to \Sigma \times Q \times D$ treats $\texttt{A}$ identically to $\texttt{0}$, and $\texttt{B}$ identically to $\texttt{1}$.
Recall the recodings $(\theta,\id_Q)$ and $(\theta, s_1 \leftrightarrow s_2)$ (\cref{def:LA0_symmetry_data}), where $\theta : \Sigma \to \Sigma$ is the alphabet involution, $\theta(\texttt{A}) = \texttt{0}, \theta(\texttt{B}) = \texttt{1}, \theta(\blank) = \blank$. A recoding $g = (a,s)$ acts on both inputs, $x \mapsto a(x_0) \dots a(x_{n-1})$, and components, $(\sigma, q) \mapsto (a(\sigma), s(q))$, (\cref{def:recoding_action_ext}), so it permutes the index set $I \times \mathcal C$ and acts on susceptibility space by the operator
\[
  (P_g X)^{C}_x := X^{\,g^{-1} \cdot C}_{\,g^{-1} \cdot x},
\]
making the assignment $g \mapsto P_g$ a representation of the Klein four-group of \cref{rem:klein_group} on susceptibility space (the recodings considered here are involutions, so $g^{-1} \cdot C = g \cdot C$). In this notation \cref{thm:sus_symmetry} reads: If $g \cdot M = M$ then $P_g\, \chi([M]) = \chi([M])$. Due to sampling noise, we want to measure the approximate equality $P_g \widehat{\psi}([M]) = \widehat{\psi}([M])$, and to this end we define the \emph{symmetry defect}.

\paragraph{Symmetry defect.} For each recoding $g$ of \cref{def:LA0_symmetry_data}, $P_g$ is a symmetric orthogonal involution, so $\mathbb R^{I \times \mathcal C}$ splits orthogonally into its $\pm 1$ eigenspaces,
\[
  \mathbb R^{I \times \mathcal C} = V^+_g \oplus V^-_g,
  \qquad
  V^\pm_g := \ker(P_g \mp \id),
\]
with orthogonal projections $P_\pm = \tfrac12(\id \pm P_g)$. The true susceptibility of a machine with the symmetry $g$ lies in the invariant subspace $V^+_g$, so we define the \emph{symmetry defect} of an estimate $\widehat{\psi}$ at $g$ as the relative squared size of its $V^-_g$-part,
\begin{equation}\label{eq:sym_defect}
  \mathrm{sd}_g(\widehat{\psi}) := \frac{\|P_- \widehat{\psi}\|_F^2}{\|\widehat{\psi}\|_F^2} = \frac{\bigl\|\tfrac12(\widehat{\psi} - P_g\widehat{\psi})\bigr\|_F^2}{\|\widehat{\psi}\|_F^2} \in [0, 1],
\end{equation}
its relative squared distance from $V^+_g$.  The defect is near zero when $\widehat{\psi} \approx P_g\widehat{\psi}$, i.e.\ when the paired blocks of \cref{fig:examples_matrix} are visually similar. \Cref{tab:sym_defect} reports $\mathrm{sd}_g(\widehat{\psi}([M]))$ at both recodings for the examples DFAs. In \cref{sec:expt_symmetry}, we extend this analysis to all classical solution DFAs.

\begin{table}[tbp]
\centering
\begin{tabular}{lcccc}
\toprule
 & \multicolumn{2}{c}{$(\theta, \id_Q)$} & \multicolumn{2}{c}{$(\theta, s_1{\leftrightarrow}s_2)$} \\
\cmidrule(lr){2-3}\cmidrule(lr){4-5}
 & $\asym$ & $\mathrm{sd}$ & $\asym$ & $\mathrm{sd}$ \\
\midrule
$M_1$ & $0$       & \textbf{0.040} & $19/48$ & $0.297$ \\
$M_2$ & $0$       & \textbf{0.012} & $1/2$   & $0.192$ \\
$M_3$ & $17/96$   & $0.253$        & $0$     & \textbf{0.016} \\
$M_4$ & $0$       & \textbf{0.011} & $17/96$ & $0.217$ \\
$M_5$ & $0$       & \textbf{0.005} & $0$     & \textbf{0.005} \\
\bottomrule
\end{tabular}
\caption{\textbf{Machine asymmetry vs symmetry defect}. The left column group is the symmetry that swaps the symbols $\texttt{A} \leftrightarrow \texttt{0}$ and $\texttt{B} \leftrightarrow \texttt{1}$, and the right column group is the symmetry that swaps those symbols \emph{and} the subroutine states $s_1 \leftrightarrow s_2$. A machine carrying the symmetry has $0$ in the $\asym$ column (\cref{prop:asym_zero}), and for the five machines shown the converse also holds, by direct inspection of their transition tables.  Where $\asym = 0$, the symmetry defect (sd) of the estimate is small (bold). }
\label{tab:sym_defect}
\end{table}

\paragraph{Inputs in the space of susceptibility vectors.} The rows $\widehat{\psi}_x([M]), x \in I$,  are estimated susceptibility vectors living in the tangent space $\mathbb R^{15}$ (they are linear responses of $\E_{p_\beta}[\phi_C]$ to a perturbation of the distribution $q(x)$). We look at the first two principal components of this embedding in \cref{fig:examples_pca}. We see the alphabet involution $x \leftrightarrow \theta(x)$ as a symmetry roughly aligned with the reflection $\mathrm{PC}_1 \mapsto -\mathrm{PC}_1$ in $M_5$ and $\mathrm{PC}_2 \mapsto -\mathrm{PC}_2$ in $M_3, M_4$. Furthermore, one can see in these examples that one of the two leading PCs appears to be correlated with the length $|x|$\footnote{For a more conclusive empirical analysis of the organisation of points in susceptibility space we would take a larger set of inputs $I$, but we put this aside for this paper.}.

\begin{figure}[tbp]
\centering
\includegraphics[width=\linewidth]{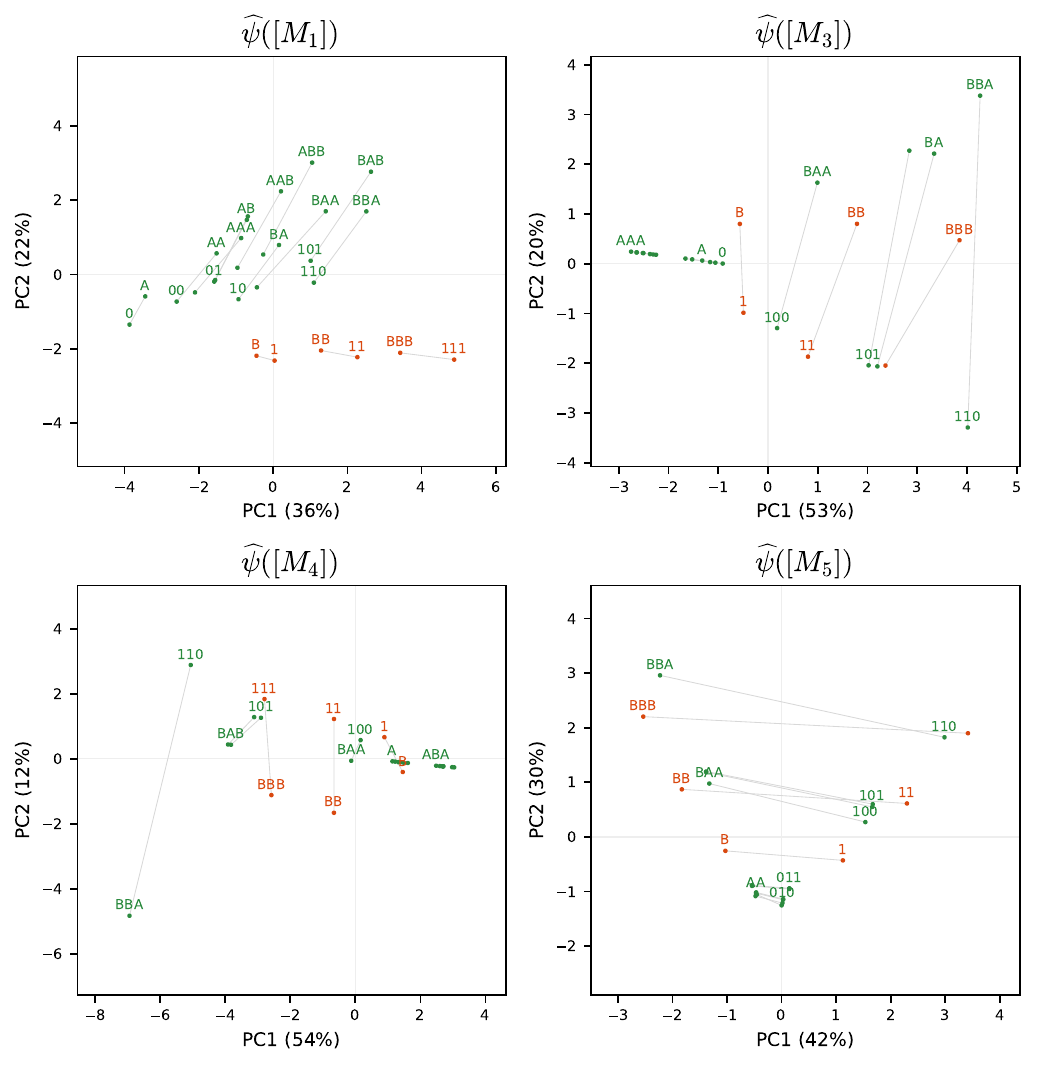}
\caption{\textbf{First two principal components of the inputs in susceptibility space.} Each point $x \in I$ is an input to the DFA $M_i$, and embeds into a space of linear response-vectors $\widehat{\psi}([M_i])_x \in \mathbb{R}^{15}$, valued across the 15 component observables we measure during sampling. The leading principal components show projections of this space along which there is the greatest variance between inputs.  Each panel is a separate PCA of its machine's $28$ rows, so components are not comparable across panels. Points are coloured as accepted inputs (green) and rejected inputs (orange). We observe each $M_i$ has a component correlated with the length $|x|$ (e.g. $\mathrm{PC}_1$ of $M_3$), and in some cases the involution $\texttt{A} \leftrightarrow \texttt{0}, \texttt{B} \leftrightarrow \texttt{1}$ is a reflection across this component (edges indicate the involution).}
\label{fig:examples_pca}
\end{figure}

\paragraph{Estimated block ranks and separation.}
Fix a pair $(A,R)$ forming a partition $Q = A \sqcup R \sqcup \{q_0\}$ of the set of states, with $\qacc \in A$ and $\qrej \in R$. This induces a partition on components $\mathcal C = \mathcal C_A \sqcup \mathcal C_R \sqcup \mathcal C_{0}$, where
\[
  \mathcal C_A := \{(\sigma, q) \in \mathcal C : q \in A\}, \quad
  \mathcal C_R := \{(\sigma, q) \in \mathcal C : q \in R\}, \quad
  \mathcal C_0 := \{(\sigma, q) \in \mathcal C : q = q_0\},
\]
and a block partition of $\widehat{\psi}([M_i])$,
\begin{equation}\label{eq:rank_blocks}
  \kbordermatrix{
    & \mathcal C_A & \mathcal C_R & \mathcal C_0 \\
    I_{\acc} & X_{\acc, A} & X_{\acc, R} & X_{\acc, 0} \\
    I_{\rej} & X_{\rej, A} & X_{\rej, R} & X_{\rej, 0}
  },
\end{equation}
where $I_{\acc} \sqcup I_{\rej}$ partitions $I$ into accepted and rejected inputs. We inspect ranks of the blocks $X_{\acc, R}, X_{\rej, A}$, which \cref{thm:rank_bound} proves to be $\leq 2$ when $(A, R)$ is a path separation partition of the DFA (\cref{def:path_separation_partition}). \Cref{fig:rank_story} illustrates the contrast between the low-rank blocks of $M_1$ and the high-rank blocks of $M_4$. These matrices are estimated from samples and therefore subject to noise. Since almost any perturbation of a low-rank matrix yields a full-rank one, the exact rank is uninformative, and so we instead compute the \emph{numerical rank}: the number of singular values exceeding a tolerance $\tau = \max(m, n)\,\varepsilon\,\sigma_1$, where $m \times n$ is the shape of the block, $\sigma_1$ is its largest singular value, and $\varepsilon$ is \texttt{float32} machine precision. \cref{tab:psv_rank} shows the computed numerical rank for each DFA $[M_i]$, at each such partition $(A,R)$.

\begin{figure}[tbp]
\centering
\definecolor{accreg}{RGB}{213,232,212}%
\definecolor{rejreg}{RGB}{255,230,204}%
\begin{minipage}[t]{0.5\linewidth}\centering
$M_1$\ (path-separable)\\[4pt]
\resizebox{0.94\linewidth}{!}{%
\begin{tikzpicture}[>=Stealth,
    st/.style={circle, draw, thick, fill=white, minimum size=0.95cm, inner sep=0pt},
    edge/.style={->, thick},
    used/.style={->, line width=2.0pt},
    every node/.append style={font=\small}]
  \node[st] (q0) at (0,0) {$q_0$};
  \node[st] (s1) at (3.6,1.7) {$s_1$};
  \node[st] (s2) at (3.6,-1.7) {$s_2$};
  \node[st, double] (qa) at (7.8,1.1) {$\qacc$};
  \node[st] (qr) at (7.8,-1.1) {$\qrej$};
  \coordinate (s1top) at (3.6,3.2);
  \coordinate (s2bot) at (3.6,-3.2);
  \begin{scope}[on background layer]
    \node[fill=accreg, rounded corners, fit=(s1)(qa)(s1top), inner sep=8pt] {};
    \node[fill=rejreg, rounded corners, fit=(s2)(qr)(s2bot), inner sep=8pt] {};
  \end{scope}
  \draw[edge] (q0) -- node[below left=-2pt] {$\blank$} (s2);
  \draw[edge] (s1) to[out=120,in=60,looseness=5] node[above] {$\texttt{A},\texttt{B},\texttt{0},\texttt{1}$} (s1);
  \draw[edge] (s2) to[out=240,in=300,looseness=5] node[below] {$\texttt{A},\texttt{B},\texttt{0},\texttt{1}$} (s2);
  \draw[edge] (s2) -- node[below=-1pt] {$\blank$} (qr);
  \draw[used] (q0) to[out=145,in=105,looseness=6] node[above left=-2pt] {$\textbf{\texttt{B}},\texttt{1}$} (q0);
  \draw[used] (q0) -- node[above left=-2pt] {$\textbf{\texttt{A}},\texttt{0}$} (s1);
  \draw[used] (s1) -- node[above=-1pt] {$\boldsymbol{\blank}$} (qa);
  \draw[->, thick] (-1.1,0) -- (q0);
\end{tikzpicture}}
\end{minipage}%
\begin{minipage}[t]{0.5\linewidth}\centering
$M_4$\ (not path-separable)\\[4pt]
\resizebox{0.94\linewidth}{!}{%
\begin{tikzpicture}[>=Stealth,
    st/.style={circle, draw, thick, fill=white, minimum size=0.95cm, inner sep=0pt},
    edge/.style={->, thick},
    used/.style={->, line width=2.0pt},
    every node/.append style={font=\small}]
  \node[st] (q0) at (0,0) {$q_0$};
  \node[st] (s1) at (3.6,1.7) {$s_1$};
  \node[st] (s2) at (3.6,-1.7) {$s_2$};
  \node[st, double] (qa) at (7.8,1.1) {$\qacc$};
  \node[st] (qr) at (7.8,-1.1) {$\qrej$};
  \begin{scope}[on background layer]
    \node[fill=accreg, rounded corners, fit=(s1)(qa), inner sep=9pt] {};
    \node[fill=rejreg, rounded corners, fit=(s2)(qr), inner sep=9pt] {};
  \end{scope}
  \draw[edge] (q0) -- node[above, pos=0.32] {$\texttt{A},\texttt{0}$} (qa);
  \draw[edge] (s2) to[bend right=15] node[right] {$\texttt{B},\texttt{1}$} (s1);
  \draw[edge] (s1) -- node[above=-1pt] {$\texttt{A},\texttt{0}$} (qa);
  \draw[edge] (s1) -- node[pos=0.7, below=-1pt] {$\blank$} (qr);
  \draw[edge] (s2) -- node[below=-1pt] {$\blank$} (qr);
  \draw[used] (q0) -- node[above left=-2pt] {$\textbf{\texttt{B}},\texttt{1}$} (s1);
  \draw[used] (s1) to[bend right=15] node[left] {$\textbf{\texttt{B}},\texttt{1}$} (s2);
  \draw[used] (s2) -- node[pos=0.42, above=-1pt] {$\textbf{\texttt{A}},\texttt{0}$} (qa);
  \draw[->, thick] (-1.1,0) -- (q0);
\end{tikzpicture}}
\end{minipage}\\[8pt]
\tikz{\fill[accreg] (0,0) rectangle (0.3,0.2);}~accept region $A$\qquad
\tikz{\fill[rejreg] (0,0) rectangle (0.3,0.2);}~reject region $R$\qquad
\tikz{\draw[line width=2pt] (0,0) -- (0.55,0);}~run on $\texttt{BBA}$
\includegraphics[width=\linewidth]{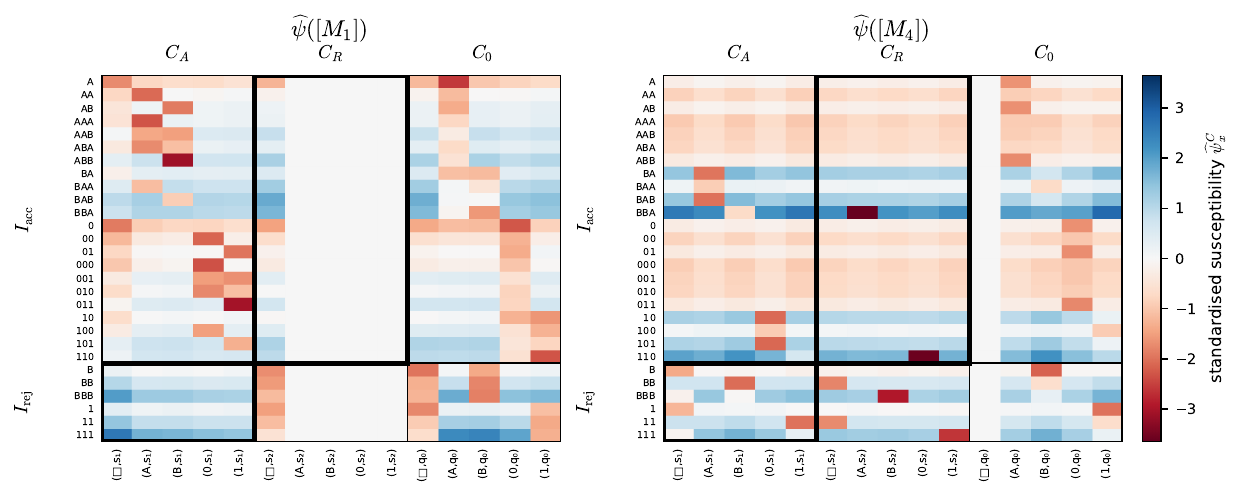}
\caption{\textbf{Blocks with low rank in path-separable DFAs vs high rank in non-path-separable DFAs}. An illustration of the rank bound of \cref{thm:rank_bound} on susceptibility matrices for $M_1$ and $M_4$. Since $M_1$ never visits region $A$ on rejected inputs, variations in $s_1$-transitions have no effect on the rejected input's loss, therefore the block $X_{\rej, A}$ (rows $I_{\rej}$, columns $\mathcal C_A$), has rank $2 \leq 2$, and since it never visits region $R$ on accepted inputs, the block $X_{\acc, R}$ (rows $I_{\acc}$, columns $\mathcal C_R$) has rank $1 \leq 2$ (both bounded by $2$ by \cref{thm:rank_bound}). $M_4$ does \emph{not} keep $I_{\rej}$ and $I_{\acc}$ separated along the regions $A, R$, and the corresponding blocks have ranks exceeding the bound: $\mathrm{rank}(X_{\rej, A}) = 5 > 2$ and $\mathrm{rank}(X_{\acc, R}) = 4 > 2$. Top: the transition diagrams of $M_1$ and $M_4$, with the regions shaded and the run of $\texttt{BBA}$ bolded (repeated from \cref{fig:sep_machines_intro}), juxtaposed above their susceptibility matrices.}
\label{fig:rank_story}
\end{figure}
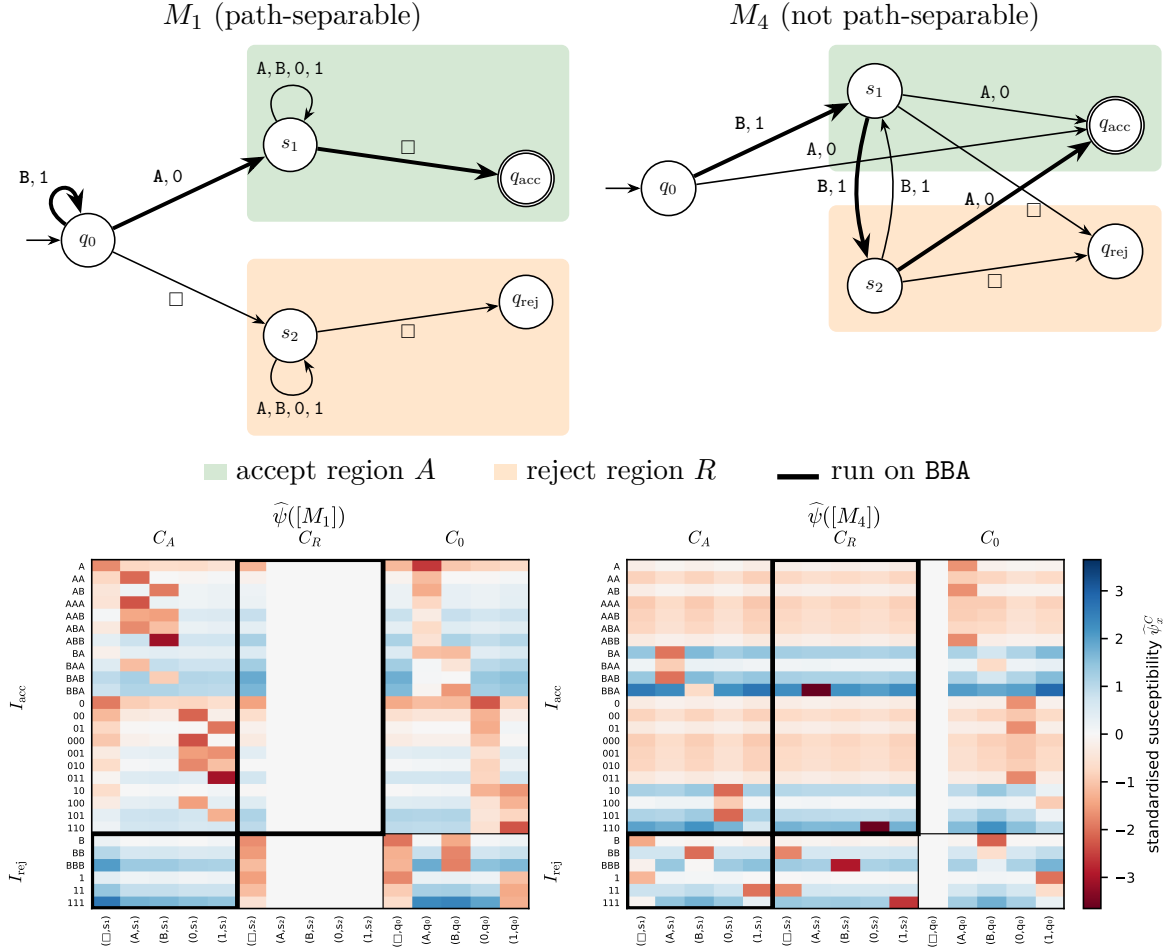

\begin{table}[tbp]
\centering
\footnotesize
\setlength{\tabcolsep}{3pt}
\begin{tabular}{l*{4}{ccc}}
\toprule
$A$ & \multicolumn{3}{c}{$\{\qacc,s_1\}$} & \multicolumn{3}{c}{$\{\qacc,s_2\}$} & \multicolumn{3}{c}{$\{\qacc,s_1,s_2\}$} & \multicolumn{3}{c}{$\{\qacc\}$} \\
\cmidrule(lr){2-4}\cmidrule(lr){5-7}\cmidrule(lr){8-10}\cmidrule(lr){11-13}
 & $\PSV$ & $X_{\acc,R}$ & $X_{\rej,A}$ & $\PSV$ & $X_{\acc,R}$ & $X_{\rej,A}$ & $\PSV$ & $X_{\acc,R}$ & $X_{\rej,A}$ & $\PSV$ & $X_{\acc,R}$ & $X_{\rej,A}$ \\
\midrule
$M_1$ & \textbf{0} & \textbf{1} & \textbf{2} & $19/48$ & $5$ & $1$ & $7/96$ & $-$ & $3$ & $31/96$ & $6$ & $-$ \\
$M_2$ & $17/96$ & $1$ & $2$ & $31/96$ & $5$ & $1$ & \textbf{0} & $-$ & \textbf{2} & $1/2$ & $6$ & $-$ \\
$M_3$ & $17/192$ & $3$ & $3$ & $17/192$ & $3$ & $3$ & $11/96$ & $-$ & $6$ & $1/16$ & $6$ & $-$ \\
$M_4$ & $3/32$ & $4$ & $5$ & $1/12$ & $5$ & $5$ & $11/96$ & $-$ & $6$ & $1/16$ & $6$ & $-$ \\
$M_5$ & \textbf{0} & \textbf{0} & \textbf{0} & \textbf{0} & \textbf{0} & \textbf{0} & \textbf{0} & $-$ & \textbf{0} & \textbf{0} & \textbf{0} & $-$ \\
\bottomrule
\end{tabular}
\caption{\textbf{Path separation violations and off-diagonal numerical ranks}. Each row is one of the example DFAs $M_1, \dots, M_5$ of \cref{tab:examples}.  Each column group is a partition $Q = A \sqcup R \sqcup \{q_0\}$ with
$\qacc \in A$, $\qrej \in R$, named by $A$ (which determines $R$).  Each group gives the path separation violation $\PSV = \PSV(M; A, R)$
(\cref{def:psv}), which is zero when all accepted inputs only visit states $q \in A$ and all rejected inputs only visit states $q \in R$ (other than $q_0$). The next two columns in each group give the numerical ranks of the two off-diagonal blocks $X_{\acc, R}, X_{\rej, A}$ of $\widehat{\psi}([M])$, with ``$-$'' when the block is empty (no component observables).  Both ranks are at most $2$ wherever $\PSV = 0$ (bold), as \cref{thm:rank_bound} requires.}
\label{tab:psv_rank}
\end{table}

\FloatBarrier
\subsection{The solution set in susceptibility space}\label{sec:expt_population_group}\label{sec:expt_setup}

We extend the analysis to a solution set of DFAs for $\mathcal L_{\texttt{A}, \texttt{0}}$. Throughout the remaining sections we use the smooth relaxation of the \emph{staged} UTM\footnote{We have verified the following quantitative results of these sections for the lookup UTM at the same sampler settings: the biconditional of \cref{sec:expt_rank_stats} holds with no exceptions and an almost identical rank distribution, the symmetry-defect medians of \cref{fig:symmetry_dist} split the same way ($0.015$ for $\asym = 0$ against $0.183$ at $(\theta, \id_Q)$, and $0.016$ against $0.191$ at $(\theta, s_1{\leftrightarrow}s_2)$), and the $\mathrm{PC}_1$ correlations of \cref{fig:pca_clustering} are $-0.95$ and $+0.86$.
We switch to the staged UTM here as we have spent more time understanding its clustering behaviour than the lookup UTM.} for the likelihood function $p(y|x,w) := \Delta\step^T_{\mathrm{staged}}(x,w)_y$, which does \emph{not} satisfy recoding equivariance (\cref{def:recoding_equivariance}). Despite this theoretical gap, we find the same predicted permutation symmetries in the estimated susceptibilities, indicating some degree of stability to the choice of UTM in the empirical setting. We continue using inverse temperature $\beta=30$ and the localiser parameters $(\gamma, \alpha) = (1, 0.01)$ for sampling.

\paragraph{Setup.}

The data $\mathcal L_{\texttt{A}, \texttt{0}}, I \subseteq \Sigma^*$ and the fixed number of time steps $T = 5$ determine the set of classical solution DFAs
\begin{equation}\label{def:classical_solutions_L_A0}
	W^{\mathrm{sol}}_{\mathrm{all}} := W^\DFA \cap \{[M] \in W^\code \mid \forall x \in I, \step^T(x, [M]) = y(x) \},
\end{equation}
where $y(x) = \qacc$ if $x \in \mathcal L_{\texttt{A}, \texttt{0}}$, and $y(x) = \qrej$ otherwise.

\begin{remark}
	$\step^T(x, [M])$ and its smooth relaxation $\Delta \step^T(x, w)$ execute the (noisy) Turing machine for exactly $T=5$ steps, irrespective of the input $x$ or the (noisy) Turing machine being executed. Machines can still enter a terminal state ($\qacc$ or $\qrej$) at an earlier time step $t < T$, but will continue to run for the $T - t$ remaining time steps. Classically this makes no difference as we constrain terminal states to always self-loop, but is an important simplification for noisy execution, where the distribution over states at each time step can be a mix of terminal and non-terminal states. Furthermore, there is no constraint on the final-state of $M$ on strings $x \notin I$, including the empty string, mixed-alphabet strings (like $\texttt{0B1A}$), and strings $x$ with $|x| > 3$. This means solutions in $W^{\mathrm{sol}}_{\mathrm{all}}$ do not necessarily decide $\mathcal L_{\texttt{A}, \texttt{0}}$ in the typical sense\footnote{No DFA can decide all of $\mathcal L_{\texttt{A}, \texttt{0}}$ in 5 steps.}.
\end{remark}

With $|\Sigma| \cdot |\{q_0, s_1, s_2\}| = 5 \cdot 3 = 15$ free state transitions, there are $5^{15} = 30{,}517{,}578{,}125$ candidate DFAs, and exactly $|W^{\mathrm{sol}}_{\mathrm{all}}| = 18{,}980{,}499$ are solutions in the above sense (relative to $I$ and $T$). For a classical solution $M$ and an input $x = x_0 \dots x_{n-1} \in I$, the
run of $M$ on $x$ visits states $q_0(x) = q_0, q_1(x), \dots, q_T(x)$
(\cref{def:psv}), and the tape-head sees symbols $x_0, \dots, x_{T-1}$, with $x_t = \blank$ for $t \geq n$.
Write
\[
  U(M) := \bigl\{(x_t,\, q_t(x)) \mid x \in I,\ 0 \le t < T \bigr\}
\]
for the set of head-symbol-and-state pairs observed during the runs of $M$ on
$I$. A run on $I$ only ever consults transition entries at pairs of $U(M)$: two classical solutions that agree at every pair of $U(M)$ have identical runs on $I$, step by step, and hence the same behaviour and the same used pairs, while their entries at the remaining pairs are unconstrained. Setting every unused entry to preserve the current state therefore changes no run, and each solution is equivalent in this sense to a unique such representative. We work with these canonical representatives:
\begin{equation}\label{def:canonical_solutions}
  W^{\mathrm{sol}} := \bigl\{[M] \in W^{\mathrm{sol}}_{\mathrm{all}} \mid
    (\sigma,q) \notin U(M) \Rightarrow [M]_Q(\sigma, q) = q\bigr\}.
\end{equation}

This leaves $|W^{\mathrm{sol}}| = 38{,}019$ \emph{canonical solutions}, still far more than the handful of qualitatively distinct algorithms one would expect for the task $(\mathcal L_{\texttt{A},\texttt{0}}, I, T)$. The surplus reflects further algorithmic symmetries and redundancies of the task\footnote{There are perhaps more symmetries to study in $W^\mathrm{sol}_{\mathrm{all}}$, but computing and interpreting susceptibilities for all $18{,}980{,}499$ solutions would be impractical.}. To expose this organising structure we embed each machine into \emph{susceptibility space} $\mathbb R^{I \times \mathcal C} \cong \mathbb R^{28 \times 15}$ via the renormalised susceptibility and study its geometry.
\begin{equation}
\begin{tikzcd}[row sep=small]
    W^{\mathrm{sol}} \arrow[r] & \mathbb R^{I \times \mathcal C} \\
    {[M]} \arrow[r, mapsto] & \chi([M])
\end{tikzcd}
\end{equation}
Susceptibility space carries the Frobenius inner product $\langle X, Y\rangle_F = \operatorname{tr}(X^\top Y)=\sum_{x,C} X^C_x\,Y^C_x$, and we probe the solution set through the eigenvectors (\emph{eigenmatrices} henceforth) of its covariance operator
$$
\Gamma = \frac{1}{|W^{\mathrm{sol}}|}\sum_M \bigl(\chi([M])-\bar\chi\bigr)\otimes\bigl(\chi([M])-\bar\chi\bigr),
\qquad \bar\chi = \frac{1}{|W^{\mathrm{sol}}|}\sum_{M}\chi([M]),
$$
where $(u\otimes v)\,w:=\langle v,w\rangle_F\,u$. Its quadratic form is the variance of the solution set along a direction, $\langle u,\Gamma u\rangle_F=\operatorname{Var}_M\langle\chi([M]),u\rangle_F$ for $u\in\mathbb R^{I\times\mathcal C}$, so the leading eigenmatrices are the directions along which the canonical solutions vary most.

In experiments we flatten each estimated matrix $\widehat{\psi}([M])$ to a length-$420$ vector and run principal component analysis (PCA) on the resulting $38{,}019\times420$ data matrix.

\paragraph{Principal components.}

Computing standardised susceptibilities $\{\widehat{\psi}([M]) \mid [M] \in W^{\mathrm{sol}}\}$ at $\beta=30$ on the smooth relaxation of the staged UTM, we observe organisational structure along the leading principal components. \Cref{fig:pca_clustering} colours the $\mathrm{PC}_1 \times \mathrm{PC}_2$ plane of the pointwise PCA by $\PSV_{\min}$ and by \textit{expected halting time} \[\sum_{x \in I} q(x) \min \bigl\{1 \le t \le T \mid q_t(x) \in \{\qacc, \qrej\} \bigr\},\] the expectation under the input distribution $q(x)$ of the earliest time that the DFA reaches a terminal state (the set is nonempty for a classical solution, since $q_T(x) = y(x)$). Both are carried almost entirely by $\mathrm{PC}_1$ (correlations $-0.97$ and $0.84$ respectively). The path-separable machines ($\PSV_{\min} = 0$) form a clearly distinct cluster (sub-organised by expected halting time) and we observe this to be the most consistently-present partition of the space across all experiments we performed.

\begin{figure}[tbp]
\centering
\includegraphics[width=\linewidth]{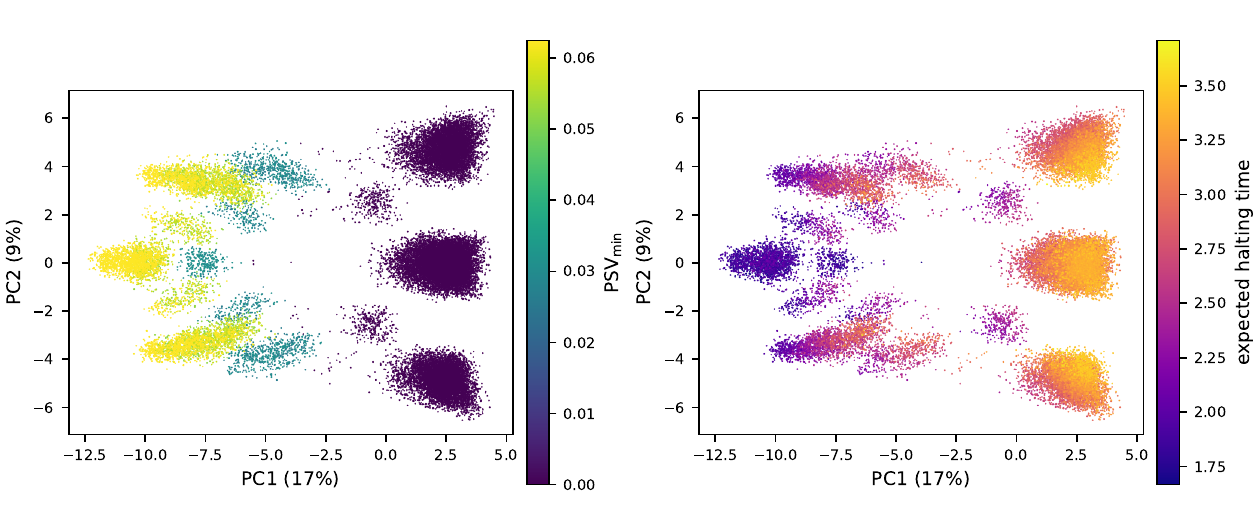}
\caption{\textbf{PC1 correlates with path-separability and halting time.}  The first two principal components of the set of susceptibility matrices $\{\widehat{\psi}([M]) \mid [M] \in W^{\mathrm{sol}}\}$. DFAs in the \textbf{left plot} are coloured by $\PSV_{\min}$, which is zero for path-separable DFAs, whose rejected and accepted inputs visit disjoint sets of intermediate states, and increases with violations of this condition. The DFAs in the \textbf{right plot} are coloured by \emph{expected halting time}, the time $M$ takes to enter the reject or accept state on input $x$ (subject to the input distribution $q(x)$).
}
\label{fig:pca_clustering}
\end{figure}

\paragraph{Recodings reflect or preserve principal components.}
\Cref{fig:symmetry_pca} shows the
$\mathrm{PC}_2 \times \mathrm{PC}_3$ plane, coloured by
$\asym(M;g)$ at each recoding $g$
(\cref{sec:recoding_symmetries}).  For a pseudo-UTM whose cycle map is
recoding equivariant (\cref{def:recoding_equivariance}), the
susceptibility matrix of $g \cdot M$ is the $P_g$-image of the
susceptibility matrix of $M$ (\cref{lem:localised_equivariance}).  Since
$W^{\mathrm{sol}}$ is closed under recoding (\cref{lem:sol_closure}), it follows that
$P_g \Gamma P_g^{-1} = \Gamma$ for the covariance operator $\Gamma$ of
\cref{sec:expt_setup}.  The two
recodings considered are involutions, so $P_g^{-1} = P_g$.  Every
eigenmatrix of $\Gamma$ with simple eigenvalue then lies in $V^+_g$ or
$V^-_g$, and each recoding either preserves or negates each principal
component.
The staged UTM behind these estimates is not recoding equivariant, so
these statements hold only approximately.  Measured directly,
$\|P_g \Gamma P_g - \Gamma\|_F / \|\Gamma\|_F = 0.012$ at
$(\theta, \id_Q)$ and $0.082$ at $(\theta, s_1{\leftrightarrow}s_2)$,
against $0.010$ and $0.009$ for the equivariant lookup UTM.
The symmetry defects \eqref{eq:sym_defect}
of the leading eigenmatrices
$(v_1, v_2, v_3)$ are
$(0.000,\, 0.000,\, 1.000)$ at $(\theta, \id_Q)$ and
$(0.000,\, 0.998,\, 0.001)$ at $(\theta, s_1{\leftrightarrow}s_2)$, and
the three eigenvalues are simple, with $\lambda_1/\lambda_2 = 1.82$ and
$\lambda_2/\lambda_3 = 1.10$.  A defect of $0$ places the eigenmatrix
in $V^+_g$, and a defect of $1$ in $V^-_g$.  Hence
$(\theta, \id_Q)$ negates $\mathrm{PC}_3$,
$(\theta, s_1{\leftrightarrow}s_2)$ negates $\mathrm{PC}_2$, their
product negates both, and all fix $\mathrm{PC}_1$.  The same pattern
continues down the spectrum (\cref{app:sd_spectrum}).  In the figure,
machines fixed by a recoding lie on the axis of its reflection
(dashed), and the running examples are starred.  $M_1$, fixed by
$(\theta, \id_Q)$, and $M_3$, fixed by
$(\theta, s_1{\leftrightarrow}s_2)$, sit on their respective axes.
$M_5$, the only machine among the $38{,}019$ fixed by both, sits at
the origin.  The recoded machines
$\overline{M_3} := (\theta, \id_Q) \cdot M_3$ and
$\overline{M_1} := (\theta, s_1{\leftrightarrow}s_2) \cdot M_1$ land at
the reflected positions.

\subsection{Symmetry statistics}\label{sec:expt_symmetry}

Recall the symmetry defect $\mathrm{sd}_g(\widehat{\psi}([M]))$ of \eqref{eq:sym_defect}. We find empirically, in \cref{fig:symmetry_dist}, that machines with $\asym(M; g) = 0$ have a lower symmetry defect at $g$ than those with $\asym(M; g) \neq 0$. We also measure directly the claim of \cref{fig:symmetry_pca} that reflecting a point $M$ in the $\mathrm{PC}_2 \times \mathrm{PC}_3$ plane across the $\mathrm{PC}_3$-axis (resp. $\mathrm{PC}_2$-axis) puts it near its recoded machine $\overline{M} := (\theta, \id_Q) \cdot M$ (resp. $(\theta, s_1{\leftrightarrow}s_2) \cdot M$). We take distances between the reflected $M$ and $\overline{M}$ with the $\mathrm{PC}_2 \times \mathrm{PC}_3$ subspace metric, and inspect the distribution of distances against the median random-pair distance $1.42$. The third plot of \cref{fig:symmetry_dist} shows these distance distributions concentrated at relatively small distances, with medians $0.06$ for $(\theta, \id_Q)$ and $0.07$ for $(\theta, s_1{\leftrightarrow}s_2)$.

\begin{figure}[H]
\centering
\includegraphics[width=\linewidth]{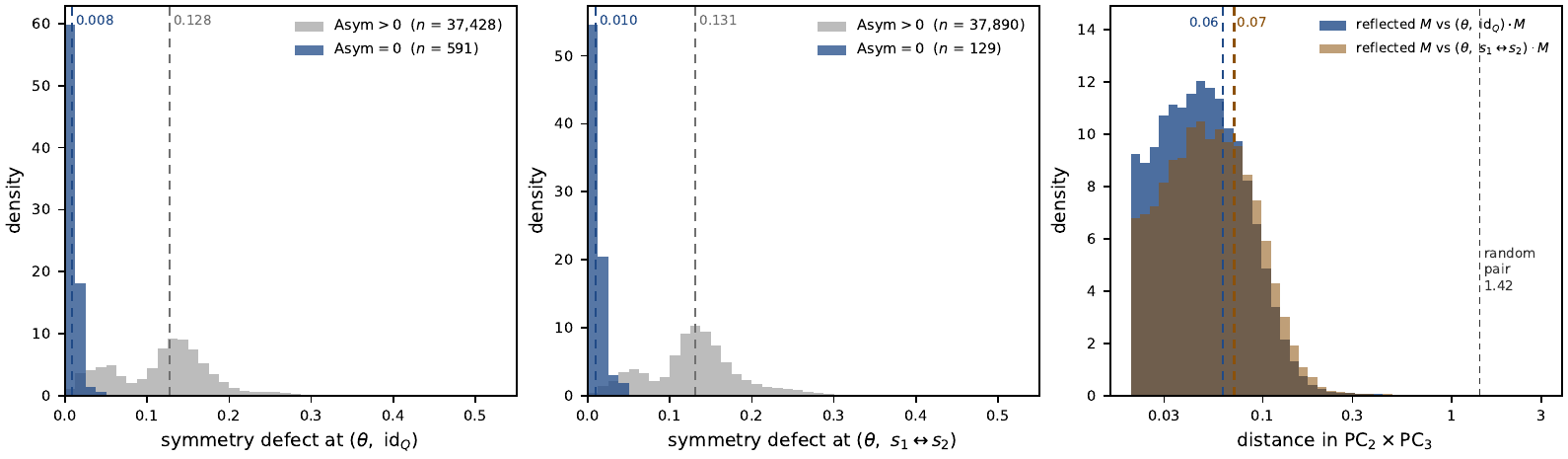}
\caption{\textbf{Algorithmic symmetry is measurable from estimated susceptibilities.}  Left, middle: distributions of the symmetry defect
$\mathrm{sd}_g(\widehat{\psi})$ \eqref{eq:sym_defect}
at the two symmetries $(\theta,\id_Q)$ and $(\theta, s_1{\leftrightarrow}s_2)$, split between $\asym = 0$ and $\asym > 0$. Dashed lines mark each group's median, and we see the median $\asym = 0$ machine has much lower defect than the median $\asym > 0$ machine.
Right: distributions of the distance, in the $\mathrm{PC}_2 \times \mathrm{PC}_3$ plane of
\cref{fig:symmetry_pca}, from a machine reflected across its symmetry's
fixed axis (negating $\mathrm{PC}_3$ at $(\theta, \id_Q)$ and
$\mathrm{PC}_2$ at $(\theta, s_1{\leftrightarrow}s_2)$) to the recoded machine
$g \cdot [M]$. Dashed
lines mark the median of each distribution and a baseline median-distance between
random pairs of machines.}
\label{fig:symmetry_dist}
\end{figure}

\subsection{Rank statistics}\label{sec:expt_rank_stats}

We test the rank bound of \cref{thm:rank_bound} on the susceptibility matrices of every DFA in $W^\mathrm{sol}$.
Recall that a partition $Q \setminus \{q_0\} = A \sqcup R$ with $\qacc \in A$ and $\qrej \in R$ determines two off-diagonal blocks $X_{\acc,R}$ and $X_{\rej,A}$ of $\chi$. In the present setting there are four such partitions:
\[
  \{\qacc, s_1\} \sqcup \{\qrej, s_2\},
  \quad
  \{\qacc, s_2\} \sqcup \{\qrej, s_1\},
  \quad
  \{\qacc, s_1, s_2\} \sqcup \{\qrej\},
  \quad
  \{\qacc\} \sqcup \{\qrej, s_1, s_2\}.
\]

\begin{definition}\label{def:psr}
The
\emph{path separation rank} of a matrix $X \in \mathbb R^{I \times \mathcal C}$, with off-diagonal blocks $X_{\acc, R}$, $X_{\rej, A}$ as in \eqref{eq:rank_blocks}, is
\[
  \PSR(X)
  :=
  \min_{(A, R)}
  \max\bigl\{\rank(X_{\acc, R}), \rank(X_{\rej, A})\bigr\},
\]
the minimum ranging over the four partitions. A block may be empty, e.g. $X_{\rej, A}$ with $A = \{\qacc\}$ which has no component observables, in which case we define $\rank = 0$.
\end{definition}

\begin{corollary}\label{cor:psr_bound}
Under the hypotheses of \cref{thm:rank_bound}: if $M$ is
path-separable (equivalently, if $\PSV_{\min}(M) = 0$,
\cref{prop:psv_zero}), then
$\PSR(\chi([M])) \le 2$.
\end{corollary}

\begin{proof}
At a path separation partition,
\cref{thm:rank_bound} bounds the rank of both off-diagonal
blocks by $2$.
\end{proof}

For each of the four partitions
$A \sqcup R$ and each machine we take the numerical ranks
(\cref{sec:expt_pointwise}) of the off-diagonal blocks $X_{\acc,R}$ and
$X_{\rej,A}$ of the standardised estimate $\widehat{\psi}$.
\Cref{thm:rank_bound} is satisfied without exception: every
machine with $\PSV(M;\, A, R) = 0$ has both off-diagonal blocks of rank at
most $2$, at all four partitions.  The empirical converse holds in the
non-empty blocks (\cref{fig:cluster_ranks}): at the partition
$s_1 \in A$, $s_2 \in R$ the block $X_{\acc,R}$ exceeds the bound for
$99.8\%$ of the machines with $\PSV(M;\, A, R) > 0$, while $X_{\rej,A}$
is uninformative there,
having rank $2$ for most machines whether or not $\PSV$ vanishes.  At the partition
$s_1, s_2 \in A$, a path separation partition for all but $8{,}666$
machines, the single
non-empty block $X_{\rej,A}$ exceeds the bound for all but two of
those $8{,}666$.  In terms of the
path separation rank, the solution set satisfies the exact
biconditional
\[
  \PSV_{\min}(M) = 0 \iff \PSR\bigl(\widehat{\psi}([M])\bigr) \le 2,
\]
with no exceptions among the $38{,}019$ machines
(\cref{fig:cluster_ranks_min}): the forward implication
is \cref{cor:psr_bound}, transferred to the estimates by
\cref{prop:standardisation_rank}; the converse an empirical fact of this
solution set.  Note that the two sides of the biconditional are computed
from disjoint data: $\PSV_{\min}$ from the classical trajectories,
the ranks from the sampled matrix.  The numerical-rank tolerance is
\texttt{float32} roundoff, far below the sampling error, so the computed
ranks are those of the estimated blocks, and the two directions of the
biconditional relate to the true blocks differently.  Where the bound
holds, noise plays no part: by \cref{prop:standardisation_rank} the
estimated blocks at a path separation partition satisfy it exactly, at
any number of samples, and across all such blocks the third singular
value sits at roundoff, $\sigma_3/\sigma_1 \le 7 \times 10^{-8}$.  Where
the bound is exceeded, the claim is weaker.  A rank above the bound is
a fact about the estimated block, and noise could in principle produce
it even when the true block satisfies the bound.  The data rule this
out in two ways.  In the exceeding blocks the third singular value is
large, not marginal: its median is $0.52\sigma_1$, and in $99\%$ of
them it is above $0.1\sigma_1$, far beyond the reach of sampling
noise.  And
noise does not inflate ranks generically in this data: at the first
partition the block $X_{\rej,A}$ stays at rank at most $2$ for most
machines with $\PSV > 0$, rather than the full rank generic noise
would produce (\cref{fig:cluster_ranks}).

\begin{figure}[H]
\centering
\includegraphics[width=\linewidth]{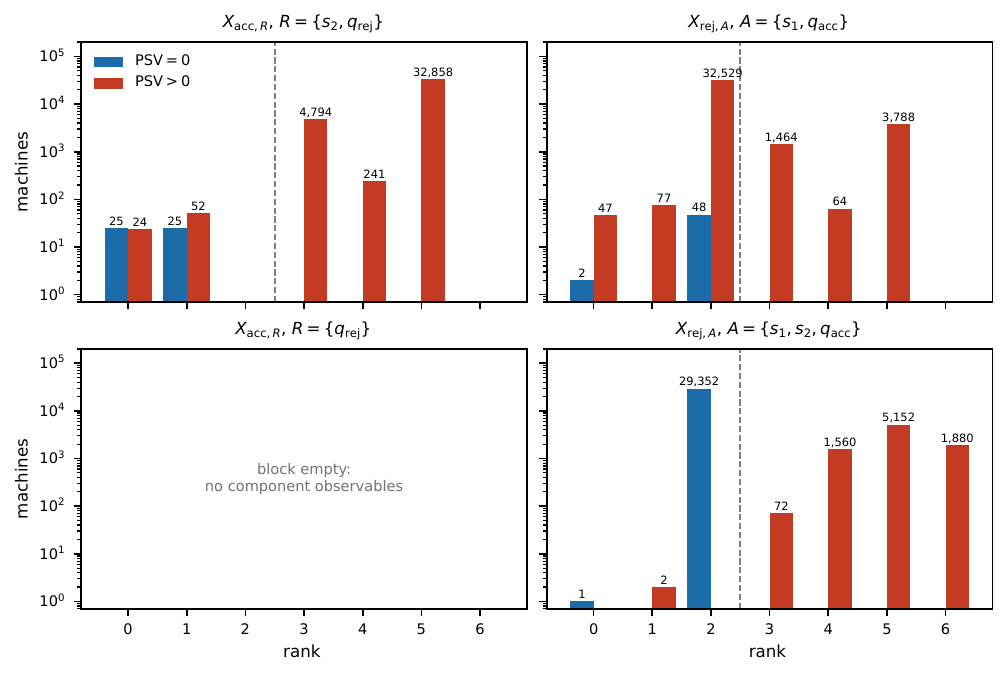}
\caption{\textbf{Machines with $\PSV = 0$ satisfy the rank bound, and
machines with $\PSV > 0$ exceed it.}  Numerical ranks of the off-diagonal blocks
$X_{\acc,R}$ and $X_{\rej,A}$ (\cref{def:psr}) at the
partitions $s_1 \in A$, $s_2 \in R$ (top) and $s_1, s_2 \in A$
(bottom), split by $\PSV(M;\, A, R) = 0$ versus $> 0$.  At
$s_1, s_2 \in A$ the block $X_{\acc,R}$ is empty: $R = \{\qrej\}$ has
no component observables.  The dashed line marks the bound
of \cref{thm:rank_bound}.}
\label{fig:cluster_ranks}
\end{figure}

\begin{figure}[H]
\centering
\includegraphics[width=0.5\linewidth]{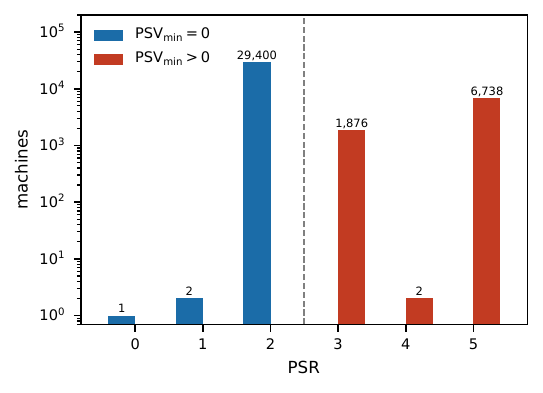}
\caption{\textbf{The path separation rank agrees with path separability.}  Distribution of the path separation rank $\PSR(\widehat{\psi}([M]))$
(\cref{def:psr}): the larger of the two off-diagonal block
ranks, minimised over the four partitions of the
non-initial states, split by $\PSV_{\min}(M) = 0$ versus $> 0$.  No machine
falls on the wrong side of the bound
of \cref{thm:rank_bound} (dashed line).}
\label{fig:cluster_ranks_min}
\end{figure}

\subsection{UMAP embeddings of susceptibilities}\label{sec:expt_population}

\paragraph{Kernels.}  For a susceptibility matrix $X \in \mathbb R^{I \times \mathcal C}$, rows indexed by inputs $x \in I$, its \emph{input kernel} $K(X)$ is the centred, normalised matrix of inner products between input rows,
\begin{equation}\label{eq:input_kernel}
  K(X) := \frac{J\, X X^{\mathsf T} J}{\lVert J\, X X^{\mathsf T} J\rVert_F},
  \qquad
  J := I_{|I|} - \tfrac{1}{|I|}\,\mathbf 1 \mathbf 1^{\mathsf T},
\end{equation}
an $|I| \times |I|$ matrix whose $(x, x')$ entry records how alike inputs $x$ and $x'$ respond across the components.  The \emph{centred kernel alignment} of two matrices $X, Y$ is the inner product of their input kernels,
\begin{equation}\label{eq:cka}
  \mathrm{CKA}(X, Y) := \langle K(X),\, K(Y)\rangle_F \in [0, 1]
  \qquad \text{\cite{kornblith2019similarity}}.
\end{equation}

\paragraph{UMAP.}
To assist further analysis of clustering in the space of Turing machines, we employ a dimensionality reduction technique called UMAP \cite{mcinnes2018umap}.
UMAP optimises a low-dimensional (2D) embedding of the data to reflect local structure in the data, by optimising a graph of pairwise distances. While a useful tool for visualising emergent structure, misleading artifacts in UMAPs are not uncommon, and so we lean on PCA and linear classifiers (like SVM) to prove which clusters are real. Throughout, we use the default UMAP parameters $\texttt{min\_dist=0.1}, \texttt{n\_neighbors=15}$, and use centred kernel alignment \eqref{eq:cka} for the distance metric.

The UMAP embedding of \cref{fig:umap_main} was computed from these susceptibilities (staged UTM, $\beta=30$), each point coloured by $\PSV_{\min}$. Again, as in the PCA plot \cref{fig:pca_clustering}, there is a clear separation of $\PSV_{\min} = 0$ (purple) from the rest of the solution set.  Linear classifiers recover the $\PSV_{\min}$ classes from the susceptibility data alone, so this organisation is not an artifact of the embedding (\cref{sec:expt_cluster_validation}). Colouring the same embedding by the path separation
rank of \cref{sec:expt_rank_stats} recovers the same organisation
(\cref{fig:psv_vs_psr}).

\begin{figure}[H]
\centering
\includegraphics[width=0.92\linewidth]{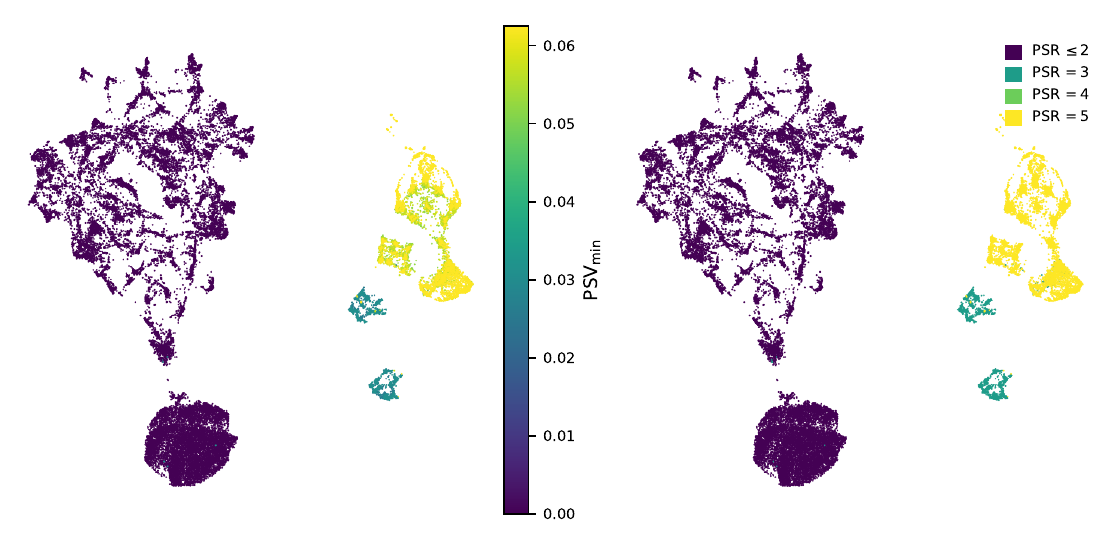}
\caption{\textbf{Path separation violation and path separation rank
agree cluster by cluster.}  The input-kernel embedding of
\cref{fig:umap_main} ($\beta = 30$),
coloured by the minimal path separation violation $\PSV_{\min}$
(\cref{def:psv}, left) and by the path separation rank $\PSR(\widehat{\psi}([M]))$
(\cref{def:psr}, right).}
\label{fig:psv_vs_psr}
\end{figure}

\subsection{Interpreting clusters}\label{sec:expt_reading}

Thus far, the theorems on block ranks and symmetries have explained broad patterns in the organisation of Turing machines in susceptibility space. We complement this by examining the clusters that emerge as we sweep the inverse temperature $\beta \in [1, 1000]$.

Through trial-and-error, we find a decision tree (\cref{fig:transition_tree}) on transitions
(a partition of the solution set $W^\mathrm{sol}$ according to conjunctions of transition rules), such that finer cluster separation is apparent when colouring by the leaves of the tree. At first approximation, major cluster separation (previously
described at $\PSV_{\min}$) is decided by whether $M$ loops at the
initial state on symbols $\texttt{B}$ and $\texttt{1}$. Checking whether $M$ loops on
$\texttt{B}$-only or $\texttt{1}$-only further distinguishes two clusters on the right labelled \textit{twin-B} and \textit{twin-1} (these have $\PSV_{\min} \in \{\tfrac{5}{192}, \tfrac{1}{32}\}$).
Machines that loop on neither symbol make up the remainder of the right major cluster labelled \textit{rest} (these have the highest $\PSV_{\min}$ and the highest $X_{\rej, A}$ and $X_{\acc, R}$ ranks).
On the left half of the tree, we separate the path-separable machines further by
checking whether both subroutine states accept on the blank symbol. Machines with this property form a disc that separates from the main body between $\beta=1$ and $\beta=30$. The finer sub-structure amongst machines that \emph{do not} have $s_1,s_2 \to \qacc$ at the blank symbol is not present in the clustering of $\beta \in \{1, 30\}$, but becomes apparent at $\beta=1000$.

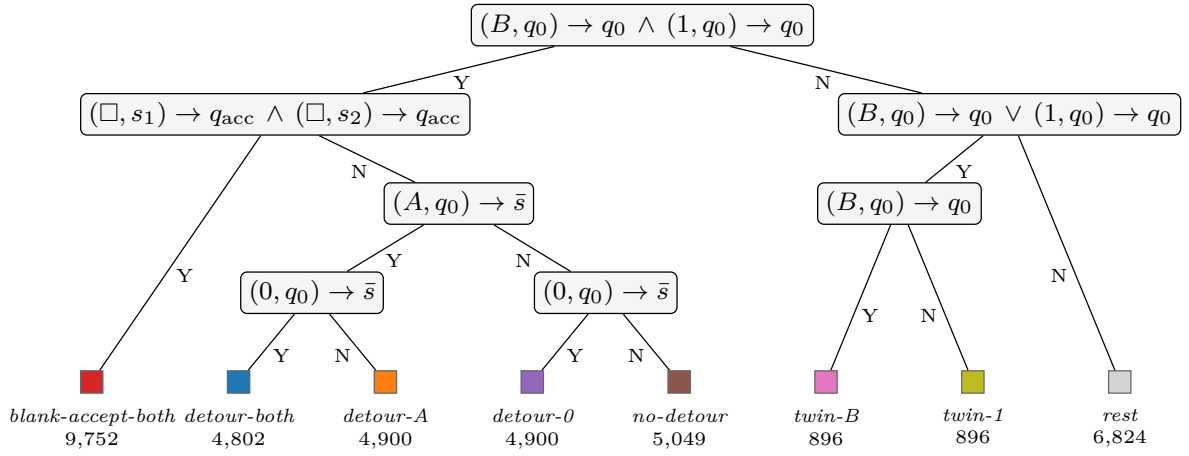
\begin{figure}[H]
\centering
{\centering\resizebox{\textwidth}{!}{%
\begin{tikzpicture}[
  x=1.62cm, y=0.82cm,
  tnode/.style={draw, rounded corners=2pt, align=center, inner sep=2.5pt,
                font=\scriptsize, fill=black!4},
  leafname/.style={font=\tiny\itshape, align=center},
  yes/.style={font=\tiny, pos=0.55, auto=left, inner sep=1pt},
  no/.style={font=\tiny, pos=0.55, auto=right, inner sep=1pt},
  edge/.style={-, thin}]
  \node[tnode] (R)   at (3.75, 0)   {$(B,q_0)\to q_0 \,\wedge\, (1,q_0)\to q_0$};
  \node[tnode] (L2)  at (1.25,-1.2) {$(\blank,s_1)\to\qacc \,\wedge\, (\blank,s_2)\to\qacc$};
  \node[tnode] (R2)  at (6.25,-1.2) {$(B,q_0)\to q_0 \,\vee\, (1,q_0)\to q_0$};
  \node[tnode] (L3)  at (2.5, -2.4) {$(A,q_0)\to \bar s$};
  \node[tnode] (R3)  at (5.5, -2.4) {$(B,q_0)\to q_0$};
  \node[tnode] (L4a) at (1.5, -3.6) {$(0,q_0)\to \bar s$};
  \node[tnode] (L4b) at (3.5, -3.6) {$(0,q_0)\to \bar s$};
  \draw[edge] (R)  -- node[yes] {Y} (L2);
  \draw[edge] (R)  -- node[no]  {N} (R2);
  \draw[edge] (L2) -- node[no]  {N} (L3);
  \draw[edge] (R2) -- node[yes] {Y} (R3);
  \draw[edge] (L3) -- node[yes] {Y} (L4a);
  \draw[edge] (L3) -- node[no]  {N} (L4b);
  \node[draw=black!60, fill={rgb,255:red,214;green,39;blue,40}, minimum size=7pt, inner sep=0pt] (F0) at (0.0,-4.8) {};
  \node[leafname, anchor=north] at (0.0,-5.0) {blank-accept-both\\$9{,}752$};
  \draw[edge] (L2) -- node[yes] {Y} (F0);
  \node[draw=black!60, fill={rgb,255:red,31;green,119;blue,180}, minimum size=7pt, inner sep=0pt] (F1) at (1.0,-4.8) {};
  \node[leafname, anchor=north] at (1.0,-5.0) {detour-both\\$4{,}802$};
  \draw[edge] (L4a) -- node[yes] {Y} (F1);
  \node[draw=black!60, fill={rgb,255:red,255;green,127;blue,14}, minimum size=7pt, inner sep=0pt] (F2) at (2.0,-4.8) {};
  \node[leafname, anchor=north] at (2.0,-5.0) {detour-A\\$4{,}900$};
  \draw[edge] (L4a) -- node[no] {N} (F2);
  \node[draw=black!60, fill={rgb,255:red,148;green,103;blue,189}, minimum size=7pt, inner sep=0pt] (F3) at (3.0,-4.8) {};
  \node[leafname, anchor=north] at (3.0,-5.0) {detour-0\\$4{,}900$};
  \draw[edge] (L4b) -- node[yes] {Y} (F3);
  \node[draw=black!60, fill={rgb,255:red,140;green,86;blue,75}, minimum size=7pt, inner sep=0pt] (F4) at (4.0,-4.8) {};
  \node[leafname, anchor=north] at (4.0,-5.0) {no-detour\\$5{,}049$};
  \draw[edge] (L4b) -- node[no] {N} (F4);
  \node[draw=black!60, fill={rgb,255:red,227;green,119;blue,194}, minimum size=7pt, inner sep=0pt] (F5) at (5.0,-4.8) {};
  \node[leafname, anchor=north] at (5.0,-5.0) {twin-B\\$896$};
  \draw[edge] (R3) -- node[yes] {Y} (F5);
  \node[draw=black!60, fill={rgb,255:red,188;green,189;blue,34}, minimum size=7pt, inner sep=0pt] (F6) at (6.0,-4.8) {};
  \node[leafname, anchor=north] at (6.0,-5.0) {twin-1\\$896$};
  \draw[edge] (R3) -- node[no] {N} (F6);
  \node[draw=black!60, fill={rgb,255:red,211;green,211;blue,211}, minimum size=7pt, inner sep=0pt] (F7) at (7.0,-4.8) {};
  \node[leafname, anchor=north] at (7.0,-5.0) {rest\\$6{,}824$};
  \draw[edge] (R2) -- node[no] {N} (F7);
\end{tikzpicture}%
}\par}
\caption{\textbf{The transition decision tree.}  Each internal node
tests the indicated transition-table entries, edges answer yes (Y) or
no (N), and each leaf carries its colour, name, and machine count.
$\bar s$ denotes the subroutine state that does not accept on blank.}
\label{fig:transition_tree}
\end{figure}

\begin{figure}[H]
\centering
\umappanel{$\beta = 1$}{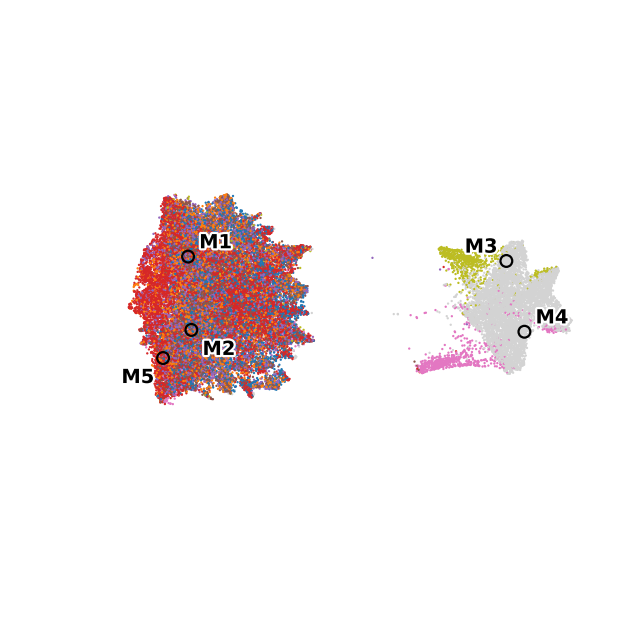}\hfill
\umappanel{$\beta = 30$}{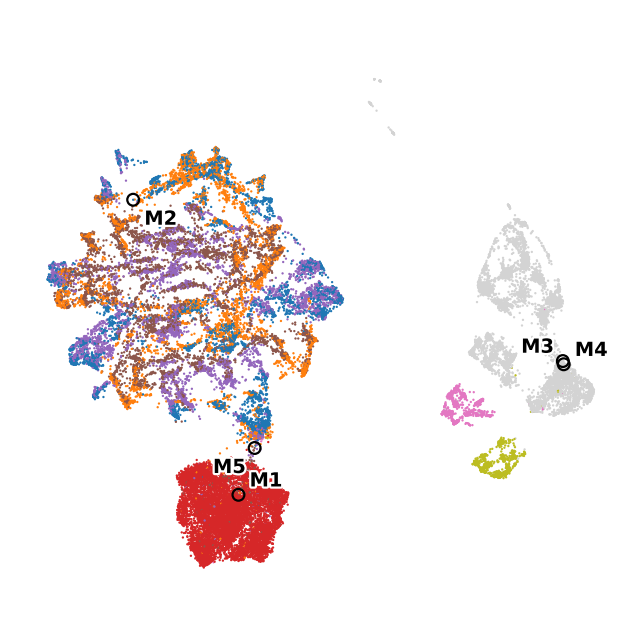}\hfill
\umappanel{$\beta = 1000$}{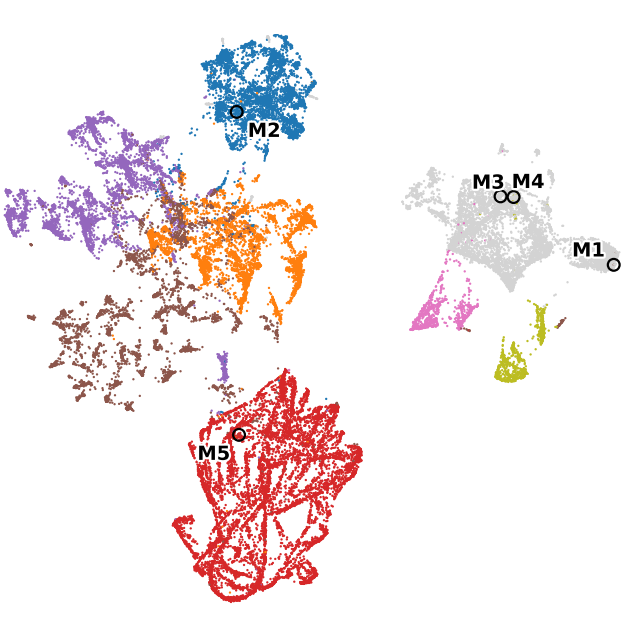}
\caption{\textbf{The input-kernel clusters are conjunctions of
transition-table facts.}  The aligned input-kernel embedding at three values of $\beta$,
coloured by the leaves of the transition decision tree of
\cref{fig:transition_tree}.  Labels mark the running examples
$M_1$--$M_5$.}
\label{fig:read_tree}
\end{figure}

Taken together, the experiments verify the two theorems of
\cref{sec:theorems} across the solution set of DFAs, and show that
further structure of the machines can be read off the susceptibility
geometry between them: the transition rules of
\cref{fig:read_tree}, the halting times of
\cref{fig:read_lat,fig:pca_clustering}, the asymmetries in \cref{fig:symmetry_dist} and the path separation
violations of
\cref{fig:umap_main} are all recovered from the estimated
susceptibilities alone.

\paragraph{Code and data availability.}
Code and data are archived at
\href{https://doi.org/10.5281/zenodo.22205895}{doi:10.5281/zenodo.22205895}:
the experiment runs as Zarr stores, the analysis library that computes the
susceptibilities, ranks, and symmetries, and the implementation of noisy
Turing machines and samplers (GRLD, NUTS, and the smooth relaxation), with a
loader and per-run description in the accompanying \texttt{README}.
The runs cover the susceptibility matrices, embeddings, rank statistics, and
labels for the $38{,}019$ canonical solutions, together with the
temperature-sweep and convergence runs.

\bibliographystyle{alpha}
\bibliography{references}

\newcommand{\etalchar}[1]{$^{#1}$}
\begin{thebibliography}{BWHM26}

\bibitem[BHW16]{bissiri2016general}
Pier~Giovanni Bissiri, Chris~C Holmes, and Stephen~G Walker.
\newblock A general framework for updating belief distributions.
\newblock {\em Journal of the Royal Statistical Society: Series B (Statistical
  Methodology)}, 78(5):1103--1130, 2016.

\bibitem[BWHM26]{baker2025studyingsmalllanguagemodels}
Garrett Baker, George Wang, Jesse Hoogland, and Daniel Murfet.
\newblock Structural inference: Interpreting small language models with
  susceptibilities.
\newblock In {\em Proceedings of The 14th International Conference on Learning
  Representations}, 2026.

\bibitem[CM18]{clift2020derivatives}
James Clift and Daniel Murfet.
\newblock Derivatives of {Turing} machines in {Linear Logic}.
\newblock arXiv preprint arXiv:1805.11813, 2018.

\bibitem[CMW21]{clift2021geometryprogramsynthesis}
James Clift, Daniel Murfet, and James Wallbridge.
\newblock Geometry of program synthesis.
\newblock arXiv preprint arXiv:2103.16080, 2021.

\bibitem[EM26]{elliott2026susceptibilities}
Chris Elliott and Daniel Murfet.
\newblock Susceptibilities and patterning: A primer on linear response in
  {Bayesian} learning.
\newblock arXiv preprint arXiv:2605.07980, 2026.

\bibitem[GBW{\etalchar{+}}26]{gordon2025lang3}
Andrew Gordon, Garrett Baker, George Wang, William Snell, Stan van Wingerden,
  and Daniel Murfet.
\newblock Towards spectroscopy: Susceptibility clusters in language models.
\newblock In {\em Proceedings of The 43rd International Conference on Machine
  Learning}, 2026.

\bibitem[HG14]{hoffman2014nuts}
Matthew~D Hoffman and Andrew Gelman.
\newblock The {No-U-Turn} sampler: Adaptively setting path lengths in
  {Hamiltonian Monte Carlo}.
\newblock {\em Journal of Machine Learning Research}, 15(1):1593--1623, 2014.

\bibitem[JS91]{joyal1991geometry}
Andr\'e Joyal and Ross Street.
\newblock The geometry of tensor calculus, {I}.
\newblock {\em Advances in Mathematics}, 88(1):55--112, 1991.

\bibitem[KNLH19]{kornblith2019similarity}
Simon Kornblith, Mohammad Norouzi, Honglak Lee, and Geoffrey Hinton.
\newblock Similarity of neural network representations revisited.
\newblock In {\em International Conference on Machine Learning}, pages
  3519--3529. PMLR, 2019.

\bibitem[MHM18]{mcinnes2018umap}
Leland McInnes, John Healy, and James Melville.
\newblock {UMAP}: Uniform manifold approximation and projection for dimension
  reduction.
\newblock arXiv preprint arXiv:1802.03426, 2018.

\bibitem[MT25]{murfet2025pas}
Daniel Murfet and Will Troiani.
\newblock Programs as singularities.
\newblock arXiv preprint arXiv:2504.08075, 2025.

\bibitem[PT13]{patterson2013sgrld}
Sam Patterson and Yee~Whye Teh.
\newblock Stochastic gradient {Riemannian Langevin} dynamics on the probability
  simplex.
\newblock In C.J. Burges, L.~Bottou, M.~Welling, Z.~Ghahramani, and K.Q.
  Weinberger, editors, {\em Advances in Neural Information Processing Systems},
  volume~26. Curran Associates, Inc., 2013.

\bibitem[Sel11]{selinger2011survey}
Peter Selinger.
\newblock A survey of graphical languages for monoidal categories.
\newblock In Bob Coecke, editor, {\em New Structures for Physics}, volume 813
  of {\em Lecture Notes in Physics}, pages 289--355. Springer, 2011.

\bibitem[Wat09]{watanabe2009algebraic}
Sumio Watanabe.
\newblock {\em Algebraic geometry and statistical learning theory}.
\newblock Cambridge University Press, 2009.

\bibitem[Xu21]{xu2021smoothrelaxationpreservingturing}
Adrian~K. Xu.
\newblock Smooth relaxation preserving {Turing} machines.
\newblock arXiv preprint arXiv:2106.00956, 2021.

\bibitem[Yeo92]{yeomans1992statistical}
Julia~M. Yeomans.
\newblock {\em Statistical Mechanics of Phase Transitions}.
\newblock Clarendon Press, Oxford, 1992.

\end{thebibliography}

\appendix
\crefalias{section}{appendix}
\crefalias{subsection}{appendix}
\crefalias{subsubsection}{appendix}

\section{Details of the model}\label{app:model_details}

This section records the concrete finite-set circuits behind the model
\eqref{eq:abstract_model}.  We use two pseudo-UTMs.  The staged machine is
imported from \cite[\S 5.1]{murfet2025pas}.  Its cycle map is not recoding
equivariant, which \cref{thm:sus_symmetry} assumes, so \cref{sec:lookup_utm}
introduces a second machine, the lookup pseudo-UTM, whose cycle map is
recoding equivariant (\cref{cor:lookup_recoding_equivariance}).  \cref{sec:circuits_body} sets up
circuits and their naive probabilistic extensions.  This circuit formalism
produces the same relaxations one gets by encoding the step function of a
UTM in differential linear logic and taking the Sweedler semantics, as is
done in \cite{murfet2025pas}.
\cref{sec:one_period_circuits} presents one period of each machine as a
circuit, together with a reference circuit $\texttt{cycle}_{\mathrm{eval}}$
that applies the transition table in a single step.  \cref{app:cycle_maps} takes the cycle
maps to be the naive probabilistic extensions of these circuits and computes them in closed
form.  \cref{app:cycle_symmetries} proves the invariance and equivariance
properties assumed by the theorems of the main text.

Throughout this section put $K := \Sigma \times Q$, write $D$ for the direction
set $\{L, S, R\}$ of \eqref{eq:W_code}, and identify $D$ with $\{-1,0,+1\}$
through the head-relative shifts $s_L = -1$, $s_S = 0$, $s_R = +1$ of
\cref{def:step}, writing $d$ for both a direction and its shift.  We identify a classical code with
\[
  c \in W^\code = \prod_{k\in K} \Sigma \times Q \times D,
  \qquad c_k=(c_{k,\Sigma},c_{k,Q},c_{k,D}),
\]
and write a noisy code as $w=(w_k)_{k\in K}$ with entries
$w_k=(w_{k,\Sigma},w_{k,Q},w_{k,D})\in\Delta\Sigma\times\Delta Q\times\Delta D$, so
that $w_{k,f} = w_{(\sigma,q,f)}$ for $k=(\sigma,q)$ in the notation of
\cref{def:component}.

\subsection{The lookup pseudo-UTM}\label{sec:lookup_utm}

The lookup pseudo-UTM is a four-tape machine with description, staging, state,
and working tapes.  The description tape lists exactly one tuple for each
of the $N=|K|$ input pairs, in the form
\[
  X\; \sigma_1 q_1 \sigma'_1 q'_1 d_1\;\cdots\;
  \sigma_N q_N \sigma'_N q'_N d_N\;X,
\]
so the inputs $(\sigma_j,q_j)$ are distinct.  The staging tape has three cells,
initially $XXX$, the state tape contains the simulated state, and the working
tape contains the simulated tape.

This is a pseudo-UTM in the sense of \cref{def:pseudo_utm}, with the
following data.  The auxiliary tapes are
$\mathrm{Aux} = \{\mathrm{description}, \mathrm{staging}\}$ and the tape
alphabet is $\Sigma_{\mathrm{UTM}} := \Sigma \sqcup Q \sqcup D \sqcup \{X\}$.
The code layout $\ell$ sends $(\sigma_j, q_j)$ to the three description
squares carrying $\sigma'_j$, $q'_j$, $d_j$ in the display above.  The set of
states is
\begin{align*}
  Q_{\mathrm{UTM}} :=\ &\{\mathrm{compInit},\ \mathrm{copySymbol},\
      \mathrm{updateSymbol},\ \mathrm{updateState},\ \mathrm{updateDir},\
      \mathrm{resetDescr}\}\\
  \sqcup\ &\{\mathrm{compSymbol},\ \mathrm{compState},\
      \neg\mathrm{compState},\ \neg\mathrm{copySymbol},\\
  &\phantom{\{}\ \,\neg\mathrm{copyState},\ \neg\mathrm{copyDir}\} \times K\\
  \sqcup\ &\{\mathrm{copyState}\} \times \Sigma
      \ \sqcup\ \{\mathrm{copyDir}\} \times \Sigma \times Q\\
  \sqcup\ &\{\mathrm{postCompSymbol},\ \mathrm{postCompState},\
      \mathrm{postCopySymbol},\\
  &\phantom{\{}\ \,\mathrm{postCopyState},\
      \mathrm{postCopyDir}\} \times \Sigma \times Q \times D,
\end{align*}
where the second factors are written as subscripts in \cref{fig:oi_utm}, so
$\mathrm{compSymbol}_{\sigma,q}$ denotes $(\mathrm{compSymbol}, (\sigma,q))$.  The
indexed families record values in the control state: the comparison and
no-match states carry the recorded pair, the copy states accumulate the
entries read so far, and the post-match states carry the copied triple.  The
transition function is \cref{fig:oi_utm}.  The fixed values are $X$ on the
delimiter and staging squares, the initial state $\mathrm{compInit}$, and
the initial head positions: the description head at the first tuple, the
staging head at its middle cell, and the state and working heads at index
$0$.

In one period (\cref{fig:oi_utm}) the machine records the current pair
$(\sigma,q)$ from the working and state tapes into the UTM state, scans the description tape,
and at the unique tuple with input $(\sigma,q)$ copies the output
$(\sigma',q',d)$ to the staging tape and also records it in the UTM state.  After this match, the machine walks
over the remaining tuples, reading the description tape only to detect the
closing $X$. The behaviour at every tuple is to rewrite all three staging squares, with $X$ before the match
and the recorded triple (in the UTM state) from the match on, so the staged values never depend
on what was previously staged. Once the end of the description tape is reached (marked by $X$), the update phase
reads the staged values one at a time, restoring each staging square to
$X$: it writes $\sigma'$ to the working tape, writes $q'$ to the state
tape, and moves the working head by $d$.  The description head then rewinds
to the opening $X$.  The update
phase also has no-op branches for a staging square still holding $X$
(\cref{fig:oi_utm}), but they are never taken: the description tape lists a
tuple for every pair, so some tuple always matches.  The schedule
has length
\[
  P=5N+1+3+5N+1=10N+5,
\]
the period of \cref{def:pseudo_utm} for this machine.

The staged pseudo-UTM of \cite[\S 5.1]{murfet2025pas} rereads the pair at
every tuple and overwrites the staged value on a later match, so its naive
probabilistic extension depends on the scan order.  The lookup machine reads
the pair once (into the UTM state), and thus its naive probabilistic extension does not (\cref{cor:lookup_srp}).

\begin{center}
\[
\begin{tikzcd}[row sep = large, column sep = small]
	{} && {\text{compInit}/\text{compSymbol}_{\sigma,q}} \\
	{\text{updateSymbol}} && {\text{compState}_{\sigma,q}} && {\neg\text{compState}_{\sigma,q}} \\
	{\text{updateState}} && {\text{copySymbol}} && {\neg\text{copySymbol}_{\sigma,q}} \\
	{\text{updateDir}} && {\text{copyState}_{\sigma'}} && {\neg\text{copyState}_{\sigma,q}} \\
	{\text{resetDescr}} && {\text{copyDir}_{\sigma',q'}} && {\neg\text{copyDir}_{\sigma,q}} \\
	{\textsl{go to resetDescr}} && {\textsl{go to postCompSymbol}_{\sigma',q',d}} && {\textsl{go to compSymbol}_{\sigma,q}} \\
	{\textsl{go to compInit}} && {} && {}
	\arrow["{\substack{X,b,c,u\\ LLSS}}"', from=1-3, to=2-1]
	\arrow["{\substack{\mathrm{desc}=\sigma\\ RLSS}}", from=1-3, to=2-3]
	\arrow["{\substack{\text{else}\\ RLSS}}", from=1-3, to=2-5]
	\arrow["{\substack{\mathrm{desc}=q\\ RSSS}}", from=2-3, to=3-3]
	\arrow["{\substack{\text{else}\\ RSSS}}", from=2-3, to=3-5]
	\arrow["{\substack{a,b,c,u\\ RSSS}}", from=2-5, to=3-5]
	\arrow["{\substack{\sigma',b,c,u\\ \text{write }\sigma',\sigma',c,u\\ RRSS}}", from=3-3, to=4-3]
	\arrow["{\substack{q',b,c,u\\ \text{write }q',q',c,u\\ RRSS}}", from=4-3, to=5-3]
	\arrow["{\substack{d,b,c,u\\ \text{write }d,d,c,u\\ RLSS}}", from=5-3, to=6-3]
	\arrow["{\substack{a,b,c,u\\ \text{write }a,X,c,u\\ RRSS}}", from=3-5, to=4-5]
	\arrow["{\substack{a,b,c,u\\ \text{write }a,X,c,u\\ RRSS}}", from=4-5, to=5-5]
	\arrow["{\substack{a,b,c,u\\ \text{write }a,X,c,u\\ RLSS}}", from=5-5, to=6-5]
	\arrow["{\substack{a,X,c,u\\ SRSS}}"', shift right=3, from=2-1, to=3-1]
	\arrow["{\substack{a,e,c,u\\ \text{write }a,X,c,e\\ SRSS}}", shift left=3, from=2-1, to=3-1]
	\arrow["{\substack{a,X,c,u\\ SRSS}}"', shift right=3, from=3-1, to=4-1]
	\arrow["{\substack{a,e,c,u\\ \text{write }a,X,e,u\\ SRSS}}", shift left=3, from=3-1, to=4-1]
	\arrow["{\substack{a,X,c,u\\ SLSS}}"', shift right=3, from=4-1, to=5-1]
	\arrow["{\substack{a,e,c,u\\ \text{write }a,X,c,u\\ SLSe}}", shift left=3, from=4-1, to=5-1]
	\arrow["{\substack{a,b,c,u\\ LSSS}}"', from=5-1, to=6-1]
	\arrow["{\substack{X,b,c,u\\ RSSS}}", bend left=80, looseness=1.6, from=5-1, to=7-1]
\end{tikzcd}
\]
\[
\begin{tikzcd}[row sep = large, column sep = huge]
	{\text{postCompSymbol}_{\sigma',q',d}} && {\textsl{go to updateSymbol}} \\
	{\text{postCompState}_{\sigma',q',d}} \\
	{\text{postCopySymbol}_{\sigma',q',d}} \\
	{\text{postCopyState}_{\sigma',q',d}} \\
	{\text{postCopyDir}_{\sigma',q',d}} \\
	{\textsl{go to postCompSymbol}_{\sigma',q',d}}
	\arrow["{\substack{X,b,c,u\\ LLSS}}", from=1-1, to=1-3]
	\arrow["{\substack{a,b,c,u\\ RLSS}}", from=1-1, to=2-1]
	\arrow["{\substack{a,b,c,u\\ RSSS}}", from=2-1, to=3-1]
	\arrow["{\substack{a,b,c,u\\ \text{write }a,\sigma',c,u\\ RRSS}}", from=3-1, to=4-1]
	\arrow["{\substack{a,b,c,u\\ \text{write }a,q',c,u\\ RRSS}}", from=4-1, to=5-1]
	\arrow["{\substack{a,b,c,u\\ \text{write }a,d,c,u\\ RLSS}}", from=5-1, to=6-1]
\end{tikzcd}
\]
\end{center}
\refstepcounter{figure}\label{fig:oi_utm}%
{\small\noindent Figure~\thefigure:\ \textbf{The lookup pseudo-UTM $\mathcal{U}$ (one simulated step).}  An arrow $q \to q'$ reads the four tape symbols (description, staging, state, working) indicated in the first line of its label, writes the symbols in the \emph{write} line (with no write line, it rewrites the read symbols), and moves the four heads per the four-letter move line, with $R, S, L$ standing for Right, Stay, Left.  The letter $e$ in the move line $SLSe$ moves the working head by the direction $e$ just read from the staging tape.  Here $a, c, u, e$ are generic symbols which are not $X$, and $b$ a generic symbol (which may be $X$).  On the comparison arrows only the description-tape test ($\mathrm{desc}=\sigma$, $\mathrm{desc}=q$) is shown, and `$\mathrm{else}$' is the default branch.  The merged node $\text{compInit}/\text{compSymbol}_{\sigma,q}$ performs the first comparison: $\text{compInit}$ reads the pair $(\sigma, q)$ from the working and state tapes once, records it in the UTM state, and tests the description tape against it.  Every later comparison ($\mathrm{desc}=\sigma$, $\mathrm{desc}=q$) tests against the recorded pair.  The left column is the update phase, whose paired arrows are the staged-value branch and the never-taken no-op branch for a staging square still holding $X$.  The second diagram is the post-match walk to the closing $X$, which re-stages the copied triple at every remaining tuple.  The slanted \emph{go to} nodes are jump targets, not UTM states.\par}

\subsection{Circuits}\label{sec:circuits_body}

We formalise the computation of the pseudo-UTMs as circuits.  The semantics
will be in the opposite category $\mathbf{CAlg}_k^{\mathrm{op}}$ of
commutative $k$-algebras over a given field $k$.

For a finite set $S$, let $kS$ be the free vector space with basis
$(e_s)_{s\in S}$, and let $x_s := e_s^*$ form the dual basis of $kS^*$.  Set
\[
  A_S := k[x_s \mid s\in S],
\]
the polynomial algebra on these coordinate functions.

\begin{definition}\label{def:circuit}
A \emph{type} is a finite set.  A \emph{gate} with input types
$S_1,\ldots,S_n$ and output types $R_1,\ldots,R_m$ is a tuple of functions
\[
  f_j\colon S_1\times\cdots\times S_n\lto R_j,
  \qquad j=1,\ldots,m,
\]
or equivalently a function
\[
  f=(f_1,\ldots,f_m)\colon
  S_1\times\cdots\times S_n\lto R_1\times\cdots\times R_m.
\]
A \emph{circuit} is a formal string diagram
\cite{joyal1991geometry, selinger2011survey} whose wires are labelled by
types and whose nodes are gates.
\end{definition}

Each gate $f$ has an associated string-diagram generator,
\[
\begin{tikzpicture}[font=\small, x=1cm, y=1cm, baseline=(F.center), >=Stealth, bundle/.style={line width=1.2pt}]
  \node[draw, circle, minimum size=3.4em] (F) at (0,0) {$f$};
  \coordinate (bot) at (0,-1.3);
  \coordinate (top) at (0,1.3);
  \draw[marrow] (F.220 |- bot) -- (F.220);
  \draw[marrow] (F.320 |- bot) -- (F.320);
  \draw[marrow] (F.140) -- (F.140 |- top);
  \draw[marrow] (F.40)  -- (F.40  |- top);
  \node[below,font=\scriptsize] at (F.220 |- bot) {$S_1$};
  \node[below,font=\scriptsize] at (F.320 |- bot) {$S_n$};
  \node[font=\scriptsize] at (0,-0.95) {$\cdots$};
  \node[above,font=\scriptsize] at (F.140 |- top) {$R_1$};
  \node[above,font=\scriptsize] at (F.40 |- top) {$R_m$};
  \node[font=\scriptsize] at (0,0.95) {$\cdots$};
  \node[draw, circle, minimum size=3.4em] (G) at (4.6,0) {$f$};
  \draw[bundle,marrow] (4.6,-1.3) -- (G.south);
  \draw[bundle,marrow] (G.north) -- (4.6,1.3);
  \node[below,font=\scriptsize] at (4.6,-1.3) {$S_1,\ldots,S_n$};
  \node[above,font=\scriptsize] at (4.6,1.3) {$R_1,\ldots,R_m$};
\end{tikzpicture}
\]
We draw a finite list of parallel wires, a \emph{bundle}, as a single thick
arrow labelled by the list of its types: the right-hand form is the same gate
with its inputs and outputs each gathered into a bundle.

We allow $n = 0$, the empty product being a one-point set: a gate with no
inputs is a constant $r \in R_1 \times \cdots \times R_m$, drawn with only its
output wires.  The \emph{blank
constant} is the gate with no inputs and one output of type $\Sigma$, valued
at the blank symbol $\blank$:
\[
\begin{tikzpicture}[font=\small, x=1cm, y=1cm, baseline=(B.center), >=Stealth]
  \node[draw, circle, fill=white, minimum size=2.4em] (B) at (0,0) {$\blank$};
  \draw[marrow] (B) -- (0,1.15);
  \node[above,font=\scriptsize] at (0,1.15) {$\Sigma$};
\end{tikzpicture}
\]
Dually, we allow $m = 0$: the \emph{counit} of type $R$ is the gate with one
input of type $R$ and no outputs, computing the unique function
$R \to \{\ast\}$.  It discards its input, and is drawn as a small hollow
circle terminating the wire:
\[
\begin{tikzpicture}[font=\small, x=1cm, y=1cm, baseline=(W.center), >=Stealth]
  \node[draw, circle, fill=white, inner sep=0pt, minimum size=6pt] (U) at (0,1.15) {};
  \coordinate (W) at (0,0.55);
  \draw[marrow] (0,0) -- (U);
  \node[below,font=\scriptsize] at (0,0) {$R$};
\end{tikzpicture}
\]

\begin{definition}\label{def:coordinate_interpretation}
The \emph{coordinate interpretation} of the gate $f$ is the morphism in
$\mathbf{CAlg}_k^{\mathrm{op}}$ determined by a morphism of $k$-algebras
\[
  f^*\colon A_{R_1}\otimes_k\cdots\otimes_k A_{R_m}
  \longrightarrow
  A_{S_1}\otimes_k\cdots\otimes_k A_{S_n}.
\]
A tensor product of polynomial algebras is again a polynomial algebra, on
the disjoint union of the variables:
\[
  A_{R_1}\otimes_k\cdots\otimes_k A_{R_m}
  \cong
  k\bigl[z^{(j)}_r \bigm| 1\le j\le m,\ r\in R_j\bigr],
\]
where $z^{(j)}_r$ denotes the coordinate function $x_r$ of the $j$-th
factor, and likewise
$A_{S_1}\otimes_k\cdots\otimes_k A_{S_n}
\cong k\bigl[x^{(i)}_s \bigm| 1\le i\le n,\ s\in S_i\bigr]$.
Under these identifications, $f^*$ is the morphism of polynomial algebras
determined by
\[
  f^*(z^{(j)}_{r}) =
  \sum_{f_j(s_1,\ldots,s_n)=r}\prod_i x^{(i)}_{s_i}.
\]
The \emph{coordinate interpretation} of a circuit is the morphism obtained
by composing the coordinate interpretations of its gates according to the
string diagram. That is, gates in sequence compose, gates side by side tensor, and a bare wire is interpreted as the identity.
\end{definition}

For a constant gate ($n = 0$, valued at $r$) the target algebra is the
empty tensor product, which is $k$.  The empty product of variables is
$1 \in k$, and the formula reads $f^*(z^{(j)}_{r'}) = \one[r' = r_j]$.
For the counit of type $R$ ($m = 0$) the source algebra is $k$, and $f^*$ is
the unit map $k \to A_R$ of the algebra.

The interpretation of a circuit is generally \emph{not} the
interpretation of the function it computes, taken as a single gate.

\begin{example}\label{ex:pnotp}
Let $\mathbb B = \{0, 1\}$ with the usual gates $\neg$ and $\wedge$ for
negation and conjunction respectively, and
let $\Delta \colon \mathbb B \lto \mathbb B^2$, $\Delta(s) = (s, s)$, be
the diagonal.  Consider the circuit
\begin{center}
\begin{tikzpicture}[font=\small, x=1cm, y=1cm, >=Stealth]
  \node[branch] (D) at (0,0) {};
  \node[draw, circle, minimum size=2.2em] (N) at (0.9,1.4) {$\neg$};
  \node[draw, circle, minimum size=2.2em] (A) at (0,2.8) {$\wedge$};
  \draw[marrow] (0,-0.8) -- (D.south);
  \draw[marrow] (D.140) to[out=90,in=-90] (A.220);
  \draw[marrow] (D.40) to[out=90,in=-90] (N.south);
  \draw[marrow] (N.north) to[out=90,in=-90] (A.320);
  \draw[marrow] (A.north) -- ++(0,0.55);
  \node[below,font=\scriptsize] at (0,-0.8) {$\mathbb B$};
  \node[above,font=\scriptsize] at ($(A.north)+(0,0.55)$) {$\mathbb B$};
\end{tikzpicture}
\end{center}
which, read bottom to top, is the composite
$\wedge \circ (\id \times \neg) \circ \Delta$ and computes
$s \mapsto s \wedge \neg s$, the constant $0$.  Its coordinate interpretation is the
composite of the coordinate interpretations of its gates, the underlying
algebra morphisms applying in the opposite order: writing
$A_{\mathbb B} = k[x_0, x_1]$ at the input and $k[z_0, z_1]$ at the output,
\[
  z_1
  \;\overset{\wedge^*}{\longmapsto}\; x^{(1)}_1 x^{(2)}_1
  \;\overset{\id^* \otimes \neg^*}{\longmapsto}\; x^{(1)}_1 x^{(2)}_0
  \;\overset{\Delta^*}{\longmapsto}\; x_1 x_0,
\]
and likewise $z_0 \mapsto x_0^2 + x_0 x_1 + x_1^2$.  The constant function
is itself a gate, with coordinate interpretation $z_1 \mapsto 0$,
$z_0 \mapsto x_0 + x_1$.
\end{example}

Applying the functor of $k$-points
$\operatorname{Hom}_{\mathbf{CAlg}_k}(-, k)$ to the coordinate semantics
gives a polynomial map of affine spaces
\[
  kS_1 \times \cdots \times kS_n
  \longrightarrow
  kR_1 \times \cdots \times kR_m.
\]
A $k$-point of $A_{S_i}$ is a tuple
$a^{(i)} = (a^{(i)}_s)_{s \in S_i} \in k^{S_i} = kS_i$ of values of the
coordinate functions $x_s$.  On $k$-points $f^*$ acts by precomposition,
sending $(a^{(1)}, \ldots, a^{(n)})$ to $(b^{(1)}, \ldots, b^{(m)})$, where
\[
  b^{(j)}_{r} =
  \sum_{f_j(s_1,\ldots,s_n)=r}\ \prod_{i=1}^{n} a^{(i)}_{s_i},
  \qquad j = 1, \ldots, m, \quad r \in R_j.
\]

The formula defining $f^*$, and its direction from outputs to inputs, have
a probabilistic reading: the morphism describes a random variable in terms
of the random variables it depends on.  For this discussion take
$k = \mathbb R$ and $m = 1$, so the gate is a single function
$f\colon S_1\times\cdots\times S_n\lto R$.

\begin{definition}\label{def:rv_type}
A \emph{random variable of type $S$}, written $X : S$, is a random variable
valued in the finite set $S$.  Its \emph{distribution} is the point
$P_X \in \Delta S \subseteq \mathbb{R}S$ whose
coordinates are
\[
  x_s(P_X) = P(X = s), \qquad s \in S.
\]
\end{definition}

\begin{definition}\label{def:independent_sampling}
Let $X_1 : S_1, \ldots, X_n : S_n$ be independent random variables.  The
random variable $f(X_1, \ldots, X_n) : R$ is obtained by drawing one sample
from each $X_i$ and applying $f$ to the results.
\end{definition}

By independence, the distribution of $f(X_1, \ldots, X_n)$ is
\begin{equation}\label{eq:sampling_dist}
  P\big(f(X_1,\ldots,X_n) = r\big)
  = \sum_{f(s_1,\ldots,s_n)=r}\prod_i P(X_i = s_i)
  = f^*(z_r)\big(P_{X_1},\ldots,P_{X_n}\big),
\end{equation}
the polynomial $f^*(z_r)$ evaluated at the coordinates of the input
distributions.  The independence of the $X_i$ appears as the products
$\prod_i x^{(i)}_{s_i}$.  That
the outputs of a gate remain probability distributions is visible
algebraically: $f^*$ sends
$\sum_{r} z_{r} \mapsto \prod_i \sum_{s} x^{(i)}_{s}$, which equals $1$ at
any tuple of probability distributions.

\begin{definition}\label{def:naive_extension}
Take $k = \mathbb R$.  The \emph{naive probabilistic extension} of a gate
$f$ is the restriction of the $\mathbb R$-points of its coordinate
interpretation to the simplices
$\Delta S \subseteq \mathbb{R}S$:
\begin{align*}
  \Delta f:\Delta S_1\times\cdots\times\Delta S_n
  &\longrightarrow
  \Delta R_1\times\cdots\times\Delta R_m,
  \\
  \Delta f(\mu_1,\ldots,\mu_n)_j(r)
  &=\sum_{f_j(s_1,\ldots,s_n)=r}\prod_i \mu_i(s_i).
\end{align*}
\end{definition}

For a constant gate ($n = 0$, valued at $r$) this reads
$\Delta f = (e_{r_1}, \ldots, e_{r_m})$, the point masses at the
components of $r$.  For the counit of type $R$ it is the unique map
$\Delta R \to \{\ast\}$, forgetting the distribution.

Naive probabilistic extensions compose, being restrictions of the composing
coordinate interpretations, so a string diagram has a naive probabilistic
extension, the composite of those of its gates.

\begin{example}\label{ex:pnotp_extension}
Restricted to the simplices, the two coordinate interpretations of \cref{ex:pnotp}
become maps $\Delta \mathbb B \lto \Delta \mathbb B$: the circuit's is
$z_1 \mapsto x_0 x_1$, $z_0 \mapsto x_0^2 + x_0 x_1 + x_1^2$, and the
constant's is $z_1 \mapsto 0$, $z_0 \mapsto x_0 + x_1$.  At the distribution
$(x_0, x_1) = (1 - p,\ p)$ they give
\begin{align*}
  \text{circuit:} \qquad & (z_0, z_1) = \bigl(1 - p(1 - p),\ p(1 - p)\bigr), \\
  \text{constant:} \qquad & (z_0, z_1) = (1,\ 0),
\end{align*}
agreeing exactly when $p(1 - p) = 0$: at the two vertices of
$\Delta \mathbb B$.
\end{example}

In the reading of \cref{def:independent_sampling}, each gate receives its
inputs as independent samples.  In \cref{ex:pnotp} both wires entering
$\wedge$ carry the same random variable $X$, the two outputs of the
diagonal, yet $\wedge$ assigns
$P(X = 1)\,P(\neg X = 1) = p(1 - p)$, forgetting that its inputs are
perfectly anticorrelated.  Hence the name: the naive probabilistic extension propagates the
beliefs of a naive Bayesian observer
\cite{clift2020derivatives, clift2021geometryprogramsynthesis}, who treats
every wire as independent of every other.

When $X_1 : S_1, \ldots, X_n : S_n$ are random variables of the respective types,
we annotate the gate $f$ by $X_1, \ldots, X_n$ and $R = R(X_1, \ldots, X_n)$:
\[
\begin{tikzpicture}[font=\small, x=1cm, y=1cm, baseline=(F.center), >=Stealth]
  \node[draw, circle, minimum size=3.4em] (F) at (0,0) {$f$};
  \coordinate (bot) at (0,-1.3);
  \draw[marrow] (F.220 |- bot) -- (F.220);
  \draw[marrow] (F.320 |- bot) -- (F.320);
  \draw[marrow] (F.90) -- ++(0,0.85);
  \node[below,font=\scriptsize] at (F.220 |- bot) {$X_1$};
  \node[below,font=\scriptsize] at (F.320 |- bot) {$X_n$};
  \node[font=\scriptsize] at (0,-0.95) {$\cdots$};
  \node[above,font=\scriptsize] at ($(F.90)+(0,0.85)$) {$R = R(X_1,\ldots,X_n)$};
\end{tikzpicture}
\]

Two gates recur in the construction below.

\begin{example}\label{ex:copy}
For a type $S$ and $n \ge 2$, the diagonal is the function
\[
  \Delta_n : S \longrightarrow S^n, \qquad \Delta_n(s) = (s, \ldots, s),
\]
with $\Delta := \Delta_2$.  In circuit diagrams a diagonal is drawn as a small
solid dot, its arity read off from the outgoing wires.  Its naive
probabilistic extension sends a distribution
$\mu \in \Delta S$ to $(\mu, \ldots, \mu)$: each output wire carries a copy of
its input, and downstream gates treat the copies as independent.  For
a bundle, the $\texttt{copy}$ circuit applies a diagonal to each wire:
\begin{center}
\begin{tikzpicture}[font=\small, x=1cm, y=1cm, >=Stealth, bundle/.style={line width=1.2pt}]
  \node[draw, fill=white, minimum width=1.6cm, minimum height=0.7cm] (C) at (0,0) {$\texttt{copy}$};
  \draw[bundle,marrow] (0,-1.3) -- (C.south);
  \draw[bundle,marrow] (-0.4,0.35) to[out=90,in=-90] (-1.1,1.3);
  \draw[bundle,marrow] (0.4,0.35) to[out=90,in=-90] (1.1,1.3);
  \node[below,font=\scriptsize] at (0,-1.3) {$S_1,\ldots,S_m$};
  \node[above,font=\scriptsize] at (-1.1,1.3) {$S_1,\ldots,S_m$};
  \node[above,font=\scriptsize] at (1.1,1.3) {$S_1,\ldots,S_m$};
\end{tikzpicture}
\end{center}
\end{example}

\begin{example}\label{ex:addrmux}
For a finite address set $J$ and output type $R$, the address multiplexer is the
function
\[
  \mathrm{mux}_J:R^J\times J\longrightarrow R,
  \qquad \mathrm{mux}_J((v^i)_i,j)=v^j ,
\]
together with the string diagram
\begin{center}
\begin{tikzpicture}[font=\small, x=1cm, y=1cm, >=Stealth]
  \node[draw, circle, minimum size=4em] (M) at (0,0) {$\mathrm{mux}_J$};
  \coordinate (bot) at (0,-1.3);
  \draw[marrow] (M.222 |- bot) -- (M.222);
  \draw[marrow] (M.278 |- bot) -- (M.278);
  \draw[marrow] (M.325 |- bot) -- (M.325);
  \draw[marrow] (M.90) -- ++(0,0.8);
  \node[below,font=\scriptsize] at (M.222 |- bot) {$R$};
  \node[below,font=\scriptsize] at (M.278 |- bot) {$R$};
  \node[below,font=\scriptsize] at (M.325 |- bot) {$J$};
  \node[font=\scriptsize] at ($(M.222 |- bot)!0.5!(M.278 |- bot)+(0,0.15)$) {$\cdots$};
  \node[above,font=\scriptsize] at ($(M.90)+(0,0.8)$) {$R$};
\end{tikzpicture}
\end{center}
Its naive probabilistic extension is
\[
  \Delta\mathrm{mux}_J : (\Delta R)^J \times \Delta J \longrightarrow \Delta R,
  \qquad \Delta\mathrm{mux}_J((\mu_i)_i,\alpha)=\sum_{i\in J}\alpha(i)\mu_i .
\]
Below it occurs with address sets $K$ and $D$, as $\mathrm{mux}_K$ and
$\mathrm{mux}_D$.
\end{example}

\subsection{The cycle circuits}\label{sec:one_period_circuits}

We consider three circuits, $\texttt{cycle}_{\bullet}$ with
$\bullet \in \{\mathrm{eval}, \mathrm{lookup}, \mathrm{staged}\}$, each computing a single step of a Turing machine with its code entering as an input bundle.  For
$\bullet \in \{\mathrm{lookup}, \mathrm{staged}\}$ the circuit presents one period of
the corresponding pseudo-UTM (\cref{sec:lookup_utm};
\cite[\S 5.1]{murfet2025pas}).  The description tape is left unchanged by a period
(\cref{sec:lookup_utm}), so it enters the circuit
as the code inputs of $\resolve_{\bullet}$.  The staging squares are
internal to the circuit.  For $\bullet = \mathrm{eval}$ there is
no pseudo-UTM: $\texttt{cycle}_{\mathrm{eval}}$ is the abstract circuit
applying the transition table in a single step, against which the pseudo-UTMs are
compared (\cref{cor:lookup_srp}).  All three are instances of the circuit depicted in
\cref{fig:cycle_circuit}: they share the fixed circuit $\texttt{update}$
(\cref{sec:update}) and differ only in $\resolve_{\bullet}$.

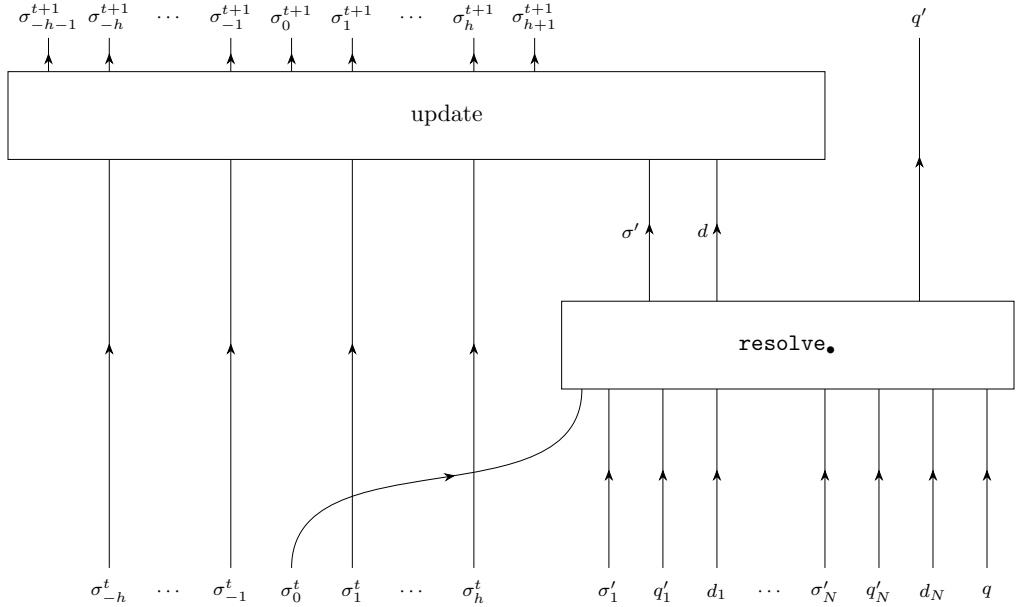
\begin{figure}[H]
\centering
\resizebox{0.85\textwidth}{!}{%
\begin{tikzpicture}[font=\small, x=1cm, y=1cm, >=Stealth, wire/.style={marrow}]
  \draw (-1.5,6.4) rectangle (10.6,7.7);
  \node at (5,7.05) {update};
  \draw (6.7,3.0) rectangle (13.4,4.3);
  \node at (10.05,3.65) {$\resolve_\bullet$};
  \node[font=\scriptsize] at (0,0)    {$\sigma_{-h}^{t}$};
  \node[font=\scriptsize] at (0.9,0)  {$\cdots$};
  \node[font=\scriptsize] at (1.8,0)  {$\sigma_{-1}^{t}$};
  \node[font=\scriptsize] at (2.7,0)  {$\sigma_0^{t}$};
  \node[font=\scriptsize] at (3.6,0)  {$\sigma_1^{t}$};
  \node[font=\scriptsize] at (4.5,0)  {$\cdots$};
  \node[font=\scriptsize] at (5.4,0)  {$\sigma_h^{t}$};
  \node[font=\scriptsize] at (7.4,0)  {$\sigma'_1$};
  \node[font=\scriptsize] at (8.2,0)  {$q'_1$};
  \node[font=\scriptsize] at (9.0,0)  {$d_1$};
  \node[font=\scriptsize] at (9.8,0)  {$\cdots$};
  \node[font=\scriptsize] at (10.6,0) {$\sigma'_N$};
  \node[font=\scriptsize] at (11.4,0) {$q'_N$};
  \node[font=\scriptsize] at (12.2,0) {$d_N$};
  \node[font=\scriptsize] at (13.0,0) {$q$};
  \draw[wire] (0,0.35)   -- (0,6.4);
  \draw[wire] (1.8,0.35) -- (1.8,6.4);
  \draw[wire] (3.6,0.35) -- (3.6,6.4);
  \draw[wire] (5.4,0.35) -- (5.4,6.4);
  \draw[wire] (2.7,0.35) to[out=90,in=270] (7.0,3.0);
  \draw[wire] (7.4,0.35)  -- (7.4,3.0);
  \draw[wire] (8.2,0.35)  -- (8.2,3.0);
  \draw[wire] (9.0,0.35)  -- (9.0,3.0);
  \draw[wire] (10.6,0.35) -- (10.6,3.0);
  \draw[wire] (11.4,0.35) -- (11.4,3.0);
  \draw[wire] (12.2,0.35) -- (12.2,3.0);
  \draw[wire] (13.0,0.35) -- (13.0,3.0);
  \draw[wire] (8,4.3)  -- (8,6.4);
  \draw[wire] (9,4.3)  -- (9,6.4);
  \node[left,font=\scriptsize]  at (8,5.35)  {$\sigma'$};
  \node[left,font=\scriptsize]  at (9,5.35)  {$d$};
  \draw[wire] (12,4.3) -- (12,8.2);
  \draw[wire] (-0.9,7.7) -- (-0.9,8.2);
  \draw[wire] (0,7.7)   -- (0,8.2);
  \draw[wire] (1.8,7.7) -- (1.8,8.2);
  \draw[wire] (2.7,7.7) -- (2.7,8.2);
  \draw[wire] (3.6,7.7) -- (3.6,8.2);
  \draw[wire] (5.4,7.7) -- (5.4,8.2);
  \draw[wire] (6.3,7.7) -- (6.3,8.2);
  \node[font=\scriptsize] at (-0.9,8.5) {$\sigma_{-h-1}^{t+1}$};
  \node[font=\scriptsize] at (0,8.5)   {$\sigma_{-h}^{t+1}$};
  \node[font=\scriptsize] at (0.9,8.5) {$\cdots$};
  \node[font=\scriptsize] at (1.8,8.5) {$\sigma_{-1}^{t+1}$};
  \node[font=\scriptsize] at (2.7,8.5) {$\sigma_0^{t+1}$};
  \node[font=\scriptsize] at (3.6,8.5) {$\sigma_1^{t+1}$};
  \node[font=\scriptsize] at (4.5,8.5) {$\cdots$};
  \node[font=\scriptsize] at (5.4,8.5) {$\sigma_{h}^{t+1}$};
  \node[font=\scriptsize] at (6.3,8.5) {$\sigma_{h+1}^{t+1}$};
  \node[font=\scriptsize] at (12,8.5)  {$q'$};
\end{tikzpicture}%
}
\caption{The cycle circuit $\texttt{cycle}_\bullet$, computing one step of the simulated machine: $\resolve_\bullet$ emits $(\sigma', q', d)$, $\texttt{update}$ applies $\sigma'$ and $d$ to the head-centred tape, and the next state $q'$ passes directly to the output.  Here $\sigma_i^{t}$ and $\sigma_i^{t+1}$ are the head-centred working-tape symbols before and after the step, $q$ and $q'$ the current and next state, $\sigma_j', q_j', d_j$ the transition-table entry (write symbol, next state, direction) of the $j$-th tuple of the code, and $\sigma', q', d$ the write symbol, next state, and direction produced by $\resolve_\bullet$.  The head symbol $\sigma_0^{t}$ enters only $\resolve_\bullet$ (\cref{def:update}).}
\label{fig:cycle_circuit}
\end{figure}

Composing successive periods realises successive steps of the simulated machine.
The code $w$ is read by every period, so it is duplicated by the $\texttt{copy}$ circuit (\cref{ex:copy}).
\cref{fig:cycle_compose} shows the composition for $h=0$ at times $t=0,1,2$, each
period expanded into its $\resolve$ and $\texttt{update}$ gates.

\begin{figure}[H]
\centering
\resizebox{0.78\textwidth}{!}{%
\begin{tikzpicture}[font=\footnotesize, x=1cm, y=1cm, >=Stealth, wire/.style={marrow},
    lab/.style={font=\scriptsize, anchor=east, inner sep=1.5pt}, bundle/.style={line width=2pt}]
  \draw (-1.6,3.6) rectangle (4.0,4.5);  \node at (1.0,4.05) {$\texttt{update}$};
  \draw (-2.6,9.4) rectangle (4.0,10.3); \node at (0.5,9.85) {$\texttt{update}$};
  \draw (-3.6,15.2) rectangle (4.0,16.1);\node at (0.0,15.65){$\texttt{update}$};
  \draw (4.3,1.0) rectangle (7.1,1.9);   \node at (5.7,1.45) {$\resolve_\bullet$};
  \draw (4.3,6.8) rectangle (7.1,7.7);   \node at (5.7,7.25) {$\resolve_\bullet$};
  \draw (4.3,12.6) rectangle (7.1,13.5); \node at (5.7,13.05){$\resolve_\bullet$};
  \node[draw, fill=white, minimum width=1.3cm, minimum height=0.7cm] (D0) at (8.5,1.45) {$\texttt{copy}$};
  \node[draw, fill=white, minimum width=1.3cm, minimum height=0.7cm] (D1) at (8.5,7.25) {$\texttt{copy}$};
  \node[below,font=\scriptsize] at (0,0.4)   {$\sigma_0^{0}$};
  \node[below,font=\scriptsize] at (6.5,0.4) {$q^{0}$};
  \node[below,font=\scriptsize] at (8.5,0.4) {$w$};
  \draw[bundle,marrow] (8.5,0.4) -- (D0.south);
  \draw[bundle,marrow] (D0.north) -- (D1.south);
  \draw[bundle] (D1.north) -- (8.5,13.05);
  \draw[bundle,marrow] (D0.west) -- (7.1,1.45);
  \draw[bundle,marrow] (D1.west) -- (7.1,7.25);
  \draw[bundle,marrow] (8.5,13.05) -- (7.1,13.05);
  \draw[wire] (0,0.4) to[out=90,in=270] (4.9,1.0);
  \draw[wire] (6.5,0.4) -- (6.5,1.0);
  \draw[wire] (4.9,1.9) to[out=90,in=-90] (2.2,3.6);
  \draw[wire] (5.7,1.9) to[out=90,in=-90] (3.6,3.6);
  \draw[wire] (6.5,1.9) -- (6.5,6.8) node[lab,pos=0.5]{$q^{1}$};
  \draw[wire] (0,4.5) to[out=90,in=270] (4.9,6.8);
  \draw[wire] (-1,4.5) -- (-1,9.4);
  \draw[wire] (1,4.5) -- (1,9.4);
  \draw[wire] (4.9,7.7) to[out=90,in=-90] (2.2,9.4);
  \draw[wire] (5.7,7.7) to[out=90,in=-90] (3.6,9.4);
  \draw[wire] (6.5,7.7) -- (6.5,12.6) node[lab,pos=0.5]{$q^{2}$};
  \node[lab] at (0,4.85)  {$\sigma_0^{1}$};
  \node[lab] at (-1,6.95) {$\sigma_{-1}^{1}$};
  \node[lab] at (1,6.95)  {$\sigma_{1}^{1}$};
  \draw[wire] (0,10.3) to[out=90,in=270] (4.9,12.6);
  \draw[wire] (-2,10.3) -- (-2,15.2);
  \draw[wire] (-1,10.3) -- (-1,15.2);
  \draw[wire] (1,10.3) -- (1,15.2);
  \draw[wire] (2,10.3) -- (2,15.2);
  \draw[wire] (4.9,13.5) to[out=90,in=-90] (2.2,15.2);
  \draw[wire] (5.7,13.5) to[out=90,in=-90] (3.6,15.2);
  \draw[wire] (6.5,13.5) -- (6.5,16.7);
  \node[lab] at (0,10.65)  {$\sigma_0^{2}$};
  \node[lab] at (-2,12.75) {$\sigma_{-2}^{2}$};
  \node[lab] at (-1,12.75) {$\sigma_{-1}^{2}$};
  \node[lab] at (1,12.75)  {$\sigma_{1}^{2}$};
  \node[lab] at (2,12.75)  {$\sigma_{2}^{2}$};
  \foreach \x in {-3,-2,-1,0,1,2,3} \draw[wire] (\x,16.1) -- (\x,16.7);
  \node[above,font=\scriptsize] at (-3,16.7) {$\sigma_{-3}^{3}$};
  \node[above,font=\scriptsize] at (-2,16.7) {$\sigma_{-2}^{3}$};
  \node[above,font=\scriptsize] at (-1,16.7) {$\sigma_{-1}^{3}$};
  \node[above,font=\scriptsize] at (0,16.7)  {$\sigma_{0}^{3}$};
  \node[above,font=\scriptsize] at (1,16.7)  {$\sigma_{1}^{3}$};
  \node[above,font=\scriptsize] at (2,16.7)  {$\sigma_{2}^{3}$};
  \node[above,font=\scriptsize] at (3,16.7)  {$\sigma_{3}^{3}$};
  \node[above,font=\scriptsize] at (6.5,16.7){$q^{3}$};
\end{tikzpicture}%
}
\caption{The three-fold composite of the cycle circuit $\texttt{cycle}_\bullet$ (\cref{fig:cycle_circuit}), shown for $h=0$ (meaning all initial tape positions $i$ with $|i| \geq 1$ are assumed to be blank) at times $t=0,1,2$.  Each period is drawn expanded into its $\resolve_\bullet$ and $\texttt{update}$ gates rather than as a single box.  The working tape grows from $\sigma_0^{0}$ to $\sigma_{-3}^{3},\ldots,\sigma_3^{3}$, the state is threaded $q^{0}\to q^{3}$ through the resolves, and the thick arrow is the code bus, duplicated into each period by $\texttt{copy}$ (\cref{ex:copy}).}
\label{fig:cycle_compose}
\end{figure}
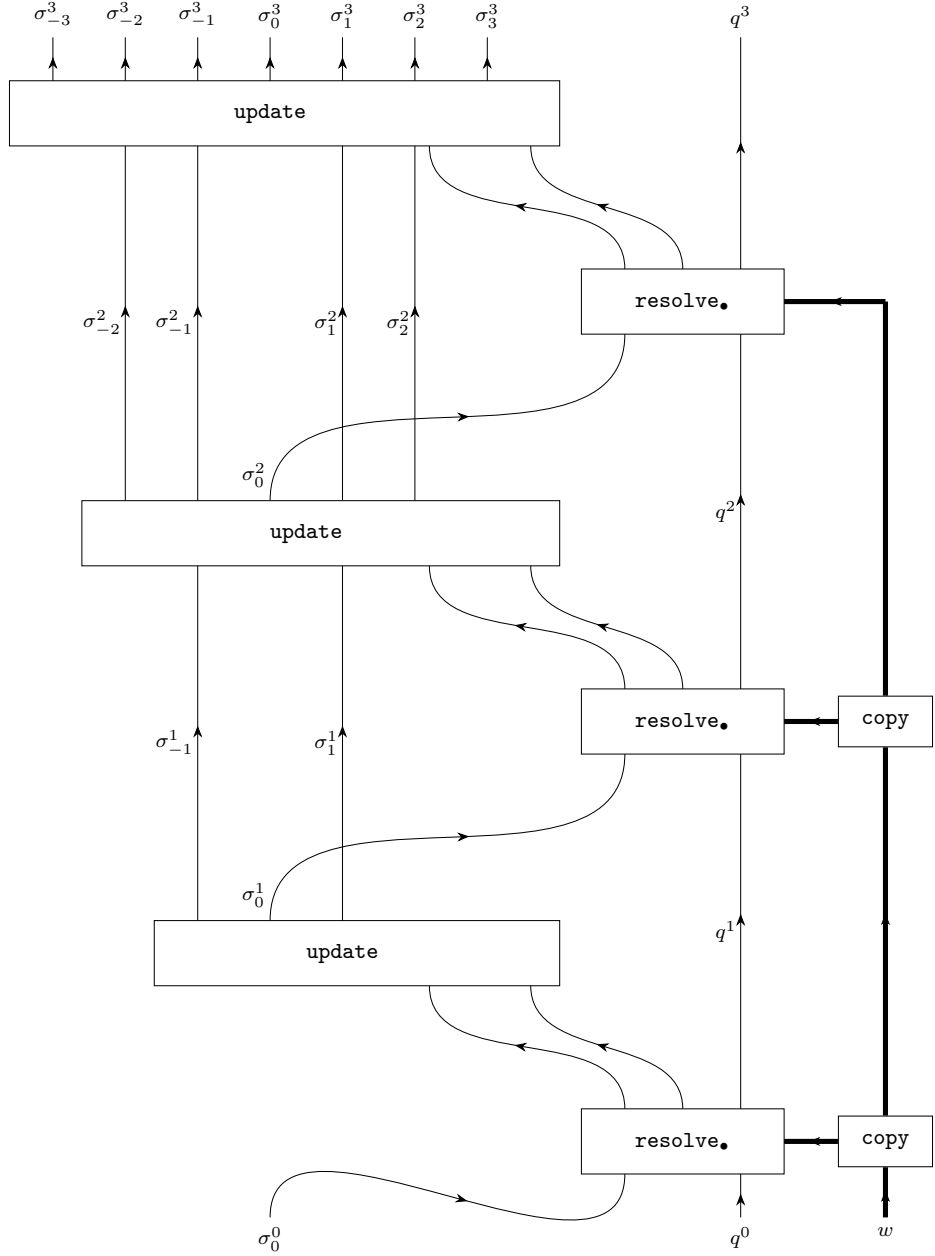

\subsubsection{Resolve}\label{sec:resolve}

\begin{definition}\label{def:resolve}
Write $\Sigma_X := \Sigma \sqcup \{X\}$, $Q_X := Q \sqcup \{X\}$, and
$D_X := D \sqcup \{X\}$ for the \emph{staging types}, the types of the three
staging squares.  The resolve circuits share the type
\[
  \resolve_{\bullet}:W^\code \times\Sigma\times Q
  \longrightarrow \Sigma\times Q\times D,
  \qquad \bullet\in\{\mathrm{eval},\mathrm{lookup},\mathrm{staged}\}.
\]
On classical inputs the three
resolve circuits compute the same lookup,
\[
  \resolve_{\bullet}(c,\sigma,q)=c_{(\sigma,q)},
  \qquad c\in W^\code.
\]
\end{definition}

Let $\boldsymbol\sigma_0\in\Delta\Sigma$ be the head-symbol distribution and
$\mathbf q\in\Delta Q$ the state distribution.  Write $j=1,\ldots,N$ for the
description positions and $(\sigma_j,q_j)$ for the input of the tuple at
position $j$.  The map $j\mapsto(\sigma_j,q_j)$ is a bijection onto $K$
(\cref{sec:lookup_utm}), and we set $w_j := w_{(\sigma_j,q_j)}$.  Define the match
weight
\begin{equation}\label{eq:match_weights}
  \lambda_j:=\boldsymbol\sigma_0(\sigma_j)\mathbf q(q_j).
\end{equation}
Since the inputs exhaust $K$, $\sum_j\lambda_j=1$.

The staged pseudo-UTM implements resolve as the scan of \cref{fig:resolve_st}:
the staged triple is carried through one \emph{overwrite gate}
per tuple with the pair $(\sigma_j, q_j)$ hardwired, and is resolved at the
top:
\begin{align}
  &\mathrm{ovr}_j\colon \Sigma\times Q\times(\Sigma\times Q\times D)\times\Sigma_X\times Q_X\times D_X \lto \Sigma_X\times Q_X\times D_X,
  \notag\\
  &\mathrm{ovr}_j(\sigma, r, v, s)=
  \begin{cases}
    v,& (\sigma,r)=(\sigma_j,q_j),\\
    s,& \text{else},
  \end{cases}
  \label{eq:staged_gates}\\
  &\mathrm{res}_X\colon \Sigma\times Q\times\Sigma_X\times Q_X\times D_X \lto \Sigma\times Q\times D,
  \qquad
  \mathrm{res}_X(\sigma, r, s) = s\bigl[X \mapsto (\sigma, r, S)\bigr].
  \notag
\end{align}
Each stage compares a fresh copy of $(\sigma_0, q)$ against its hardwired
pair, so by \cref{def:naive_extension} its naive probabilistic extension is, field-wise, the
mixture
$\hat{\mathbf s}_j=\lambda_j\,w_j+(1-\lambda_j)\,\hat{\mathbf s}_{j-1}$, with
$\hat{\mathbf s}_0=e_X$ the naive probabilistic extension of the $X$-constant gates.

\begin{figure}[H]
\centering
\resizebox{0.6\textwidth}{!}{%
\begin{tikzpicture}[font=\scriptsize, x=1cm, y=1cm, >=Stealth,
   gate/.style={draw, fill=white, minimum width=3.6cm, minimum height=1.0cm, align=center, rounded corners=0.35cm},
   copygate/.style={circle, fill, inner sep=0pt, minimum size=4.5pt},
   copyb/.style={draw, fill=white, rounded corners=1pt, inner sep=1.5pt, font=\scriptsize},
   bundle/.style={line width=1.2pt}]
  \def\xxa{7.6}\def\xxb{8.8}\def\xxc{10.0}
  \def\xsig{0.6}\def\xq{1.5}\def\xW{2.6}
  \node[gate] (o1) at (8.8,1.6)  {$\mathrm{ovr}_1$};
  \node[gate] (oj) at (8.8,4.0)  {$\mathrm{ovr}_j$};
  \node[gate] (oN) at (8.8,6.4)  {$\mathrm{ovr}_N$};
  \node[gate] (rx) at (8.8,8.8)  {$\mathrm{res}_X$};
  \foreach \x in {\xxa,\xxb,\xxc}{
    \node[draw, circle, fill=white, minimum size=1.6em, font=\scriptsize] at (\x,-0.05) {$X$};
    \draw[marrow] (\x,0.3) -- (\x,1.1);
    \draw[marrow] (\x,2.1)  -- (\x,2.6);
    \draw[marrow] (\x,3.0)  -- (\x,3.5);
    \draw[marrow] (\x,4.5)  -- (\x,5.0);
    \draw[marrow] (\x,5.4)  -- (\x,5.9);
    \draw[marrow] (\x,6.9)  -- (\x,8.3);
    \draw[marrow] (\x,9.3)  -- (\x,10.5);
  }
  \node[above] at (\xxa,10.5) {$\sigma'$};
  \node[above] at (\xxb,10.5) {$q'$};
  \node[above] at (\xxc,10.5) {$d$};
  \node at (8.8,2.8) {$\vdots$};
  \node at (8.8,5.2) {$\vdots$};
  \node[copygate] (ds1) at (\xsig,1.9) {};
  \node[copygate] (dsj) at (\xsig,4.3) {};
  \node[copygate] (dsN) at (\xsig,6.7) {};
  \draw[marrow] (\xsig,-0.7) -- (ds1);
  \draw (ds1) -- (\xsig,2.6);
  \node at (\xsig,2.8) {$\vdots$};
  \draw[marrow] (\xsig,3.0) -- (dsj);
  \draw (dsj) -- (\xsig,5.0);
  \node at (\xsig,5.2) {$\vdots$};
  \draw[marrow] (\xsig,5.4) -- (dsN);
  \draw (dsN) -- (\xsig,9.0);
  \draw[marrow] (\xsig,9.0) -- (7.0,9.0);
  \draw[marrow] (dsN) -- (7.0,6.7);
  \draw[marrow] (ds1) -- (7.0,1.9);
  \draw[marrow] (dsj) -- (7.0,4.3);
  \node[copygate] (dq1) at (\xq,1.6) {};
  \node[copygate] (dqj) at (\xq,4.0) {};
  \node[copygate] (dqN) at (\xq,6.4) {};
  \draw[marrow] (\xq,-0.7) -- (dq1);
  \draw (dq1) -- (\xq,2.6);
  \node at (\xq,2.8) {$\vdots$};
  \draw[marrow] (\xq,3.0) -- (dqj);
  \draw (dqj) -- (\xq,5.0);
  \node at (\xq,5.2) {$\vdots$};
  \draw[marrow] (\xq,5.4) -- (dqN);
  \draw (dqN) -- (\xq,8.6);
  \draw[marrow] (\xq,8.6) -- (7.0,8.6);
  \draw[marrow] (dqN) -- (7.0,6.4);
  \draw[marrow] (dq1) -- (7.0,1.6);
  \draw[marrow] (dqj) -- (7.0,4.0);
  \draw[bundle,marrow] (-0.4,1.3) -- (7.0,1.3);
  \draw[bundle,marrow] (-0.4,3.7) -- (7.0,3.7);
  \draw[bundle,marrow] (-0.4,6.1) -- (7.0,6.1);
  \node[left] at (-0.4,1.3) {$w_1$};
  \node[left] at (-0.4,3.7) {$w_j$};
  \node[left] at (-0.4,6.1) {$w_N$};
  \node at (3.4,2.5) {$\vdots$};
  \node at (3.4,4.9) {$\vdots$};
  \node[below] at (\xsig,-0.75) {$\sigma_0$};
  \node[below] at (\xq,-0.75)   {$q$};
\end{tikzpicture}%
}
\caption{How the staged pseudo-UTM computes, from the description tape, the transition it will apply.  The staging triple starts unresolved at the $X$-constant gates and is threaded up through the stages.  Each $\mathrm{ovr}_j$ overwrites it with the code entry $w_j$ when $(\sigma_0, q)$ matches its hardwired pair $(\sigma_j, q_j)$, and $\mathrm{res}_X$ resolves a surviving $X$ to the no-op triple.  The head symbol $\sigma_0$ and the state $q$ are duplicated along chains of diagonals (\cref{ex:copy}), one copy per stage and the final copies into $\mathrm{res}_X$.  The code enters as its tuple sub-bundles.  A later match overwrites an earlier one (contrast \cref{fig:resolve_lk}).}
\label{fig:resolve_st}
\end{figure}
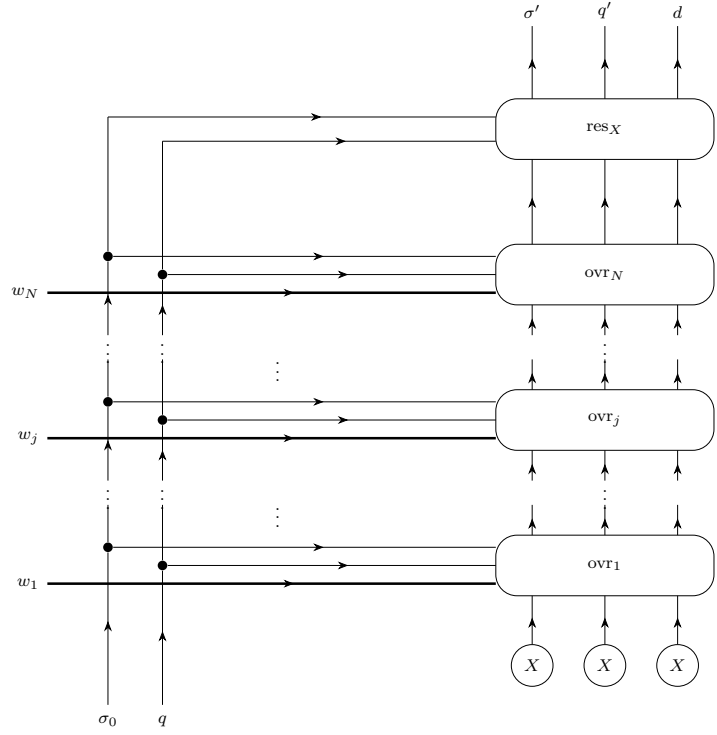

The lookup pseudo-UTM realises resolve as the scan of
\cref{fig:resolve_lk}.  At a tuple boundary the machine is in a state
$\mathrm{compSymbol}_{\kappa}$ before the match, recording the read pair
$\kappa$, and in a state $\mathrm{postCompSymbol}_{t}$ after it, recording
the copied triple $t$ (\cref{sec:lookup_utm}).  Abbreviating these families
as $\mathrm{comp}_\kappa$ and $\mathrm{post}_t$, the UTM-state wire has type
\[
  Q_{\mathcal U}
  := \{\mathrm{comp}_\kappa\}_{\kappa \in K}
  \sqcup \{\mathrm{post}_t\}_{t \in \Sigma \times Q \times D},
\]
a subset of $Q_{\mathrm{UTM}}$, and is threaded through one stage per tuple
by the gates
\begin{align}
  &\mathrm{compInit}\colon \Sigma\times Q \lto Q_{\mathcal U},
  &&\mathrm{compInit}(\sigma, q) = \mathrm{comp}_{(\sigma,q)},
  \notag\\
  &\mathrm{cmp}\colon Q_{\mathcal U}\times\Sigma\times Q\times(\Sigma\times Q\times D) \lto Q_{\mathcal U},
  &&\mathrm{cmp}(\rho, \sigma, q, v)=
  \begin{cases}
    \mathrm{post}_{v}, & \rho=\mathrm{comp}_{(\sigma,q)},\\
    \rho, & \text{else},
  \end{cases}
  \label{eq:lookup_gates}\\
  &\mathrm{export}\colon Q_{\mathcal U} \lto \Sigma\times Q\times D,
  &&\mathrm{export}(\rho) =
  \begin{cases}
    t, & \rho=\mathrm{post}_{t},\\
    (\sigma, q, S), & \rho=\mathrm{comp}_{(\sigma,q)},
  \end{cases}
  \notag
\end{align}
where the $j$-th $\mathrm{cmp}$ receives the input pair $(\sigma_j, q_j)$
of the $j$-th description block as constants, and its entry $w_j$ on wires.
On the unmatched branch $\mathrm{export}$ writes the no-op triple: write the
recorded symbol back, keep the recorded state, stay.  After each stage the UTM-state wire is exported, through a copy at all but
the last stage.  Every exported triple but the final one is discarded by a
counit.  Write $\Delta\resolve_{\mathrm{lookup}}$ for
the naive probabilistic extension of this circuit.

\begin{figure}[H]
\centering
\resizebox{0.5\textwidth}{!}{%
\begin{tikzpicture}[font=\scriptsize, x=1cm, y=1cm, >=Stealth,
   gate/.style={draw, fill=white, minimum width=1.9cm, minimum height=0.9cm, align=center, rounded corners=0.3cm},
   copygate/.style={circle, fill, inner sep=0pt, minimum size=4.5pt},
   counit/.style={circle, draw, fill=white, inner sep=0pt, minimum size=6pt},
   bundle/.style={line width=1.2pt}]
  \node[gate] (ci) at (3.4,0.6)  {$\mathrm{compInit}$};
  \node[gate] (c1) at (3.4,2.4)  {$\mathrm{cmp}$};
  \node[gate] (cj) at (3.4,5.0)  {$\mathrm{cmp}$};
  \node[gate] (cN) at (3.4,7.6)  {$\mathrm{cmp}$};
  \draw[marrow] (3.0,-0.8) -- (3.0,0.15);
  \draw[marrow] (3.8,-0.8) -- (3.8,0.15);
  \node[below] at (3.0,-0.8) {$\sigma_0$};
  \node[below] at (3.8,-0.8) {$q$};
  \draw[marrow] (3.4,1.05) -- (3.4,1.95);
  \node[left] at (3.4,1.5) {$Q_{\mathcal U}$};
  \node[copygate] (d1) at (3.4,3.3) {};
  \draw[marrow] (3.4,2.85) -- (d1);
  \draw (d1) -- (3.4,3.9);
  \node at (3.4,4.1) {$\vdots$};
  \draw[marrow] (3.4,4.3) -- (3.4,4.55);
  \node[copygate] (dj) at (3.4,5.9) {};
  \draw[marrow] (3.4,5.45) -- (dj);
  \draw (dj) -- (3.4,6.5);
  \node at (3.4,6.7) {$\vdots$};
  \draw[marrow] (3.4,6.9) -- (3.4,7.15);
  \draw[marrow] (3.4,8.05) to[out=90,in=180] (6.35,8.75);
  \draw[bundle,marrow] (0.6,2.4) -- (2.45,2.4);
  \draw[bundle,marrow] (0.6,5.0) -- (2.45,5.0);
  \draw[bundle,marrow] (0.6,7.6) -- (2.45,7.6);
  \node[left] at (0.6,2.4) {$w_1$};
  \node[left] at (0.6,5.0) {$w_j$};
  \node[left] at (0.6,7.6) {$w_N$};
  \node at (1.5,3.7) {$\vdots$};
  \node at (1.5,6.3) {$\vdots$};
  \node[gate] (e1) at (7.3,3.3)  {$\mathrm{export}$};
  \node[gate] (ej) at (7.3,5.9)  {$\mathrm{export}$};
  \node[gate] (eN) at (7.3,8.75) {$\mathrm{export}$};
  \draw[marrow] (d1) -- (6.35,3.3);
  \draw[marrow] (dj) -- (6.35,5.9);
  \node[counit] (u1a) at (6.85,4.45) {};
  \node[counit] (u1b) at (7.3,4.45)  {};
  \node[counit] (u1c) at (7.75,4.45) {};
  \draw[marrow] (6.85,3.75) -- (u1a);
  \draw[marrow] (7.3,3.75)  -- (u1b);
  \draw[marrow] (7.75,3.75) -- (u1c);
  \node[counit] (uja) at (6.85,7.05) {};
  \node[counit] (ujb) at (7.3,7.05)  {};
  \node[counit] (ujc) at (7.75,7.05) {};
  \draw[marrow] (6.85,6.35) -- (uja);
  \draw[marrow] (7.3,6.35)  -- (ujb);
  \draw[marrow] (7.75,6.35) -- (ujc);
  \node at (7.3,5.05) {$\vdots$};
  \node at (7.3,7.7) {$\vdots$};
  \draw[marrow] (6.85,9.2) -- (6.85,10.0);
  \draw[marrow] (7.3,9.2)  -- (7.3,10.0);
  \draw[marrow] (7.75,9.2) -- (7.75,10.0);
  \node[above] at (6.85,10.0) {$\sigma'$};
  \node[above] at (7.3,10.0)  {$q'$};
  \node[above] at (7.75,10.0) {$d$};
\end{tikzpicture}%
}
\caption{How the lookup pseudo-UTM computes, from the description tape, the transition it will apply.  The gate $\mathrm{compInit}$ records the read pair $(\sigma_0, q)$ on the UTM-state wire, the $j$-th $\mathrm{cmp}$ replaces it by the entry $w_j$ on a match, and each $\mathrm{export}$ emits the triple the UTM state determines.  The input pairs $(\sigma_j, q_j)$ enter as constants and are not drawn, so only the entry $w_j$ of the $j$-th description block enters stage $j$.  Hollow circles are counits.  Unlike the staged resolve (\cref{fig:resolve_st}), the exported triples are recomputed from the UTM state at every stage and never read back, so deleting the discarded exports leaves an equivalent circuit.}
\label{fig:resolve_lk}
\end{figure}
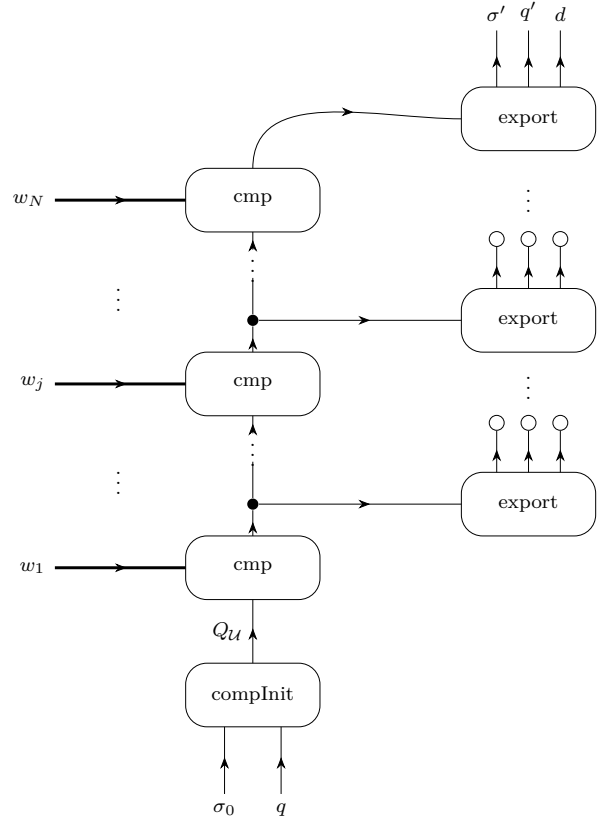

The eval resolve (\cref{fig:resolve_eval}) copies $\sigma_0$ and $q$
(\cref{ex:copy}) into one field gate per output, each taking the sub-bundle
$w_f := (w_{j,f})_j$ of its field:
\begin{equation}\label{eq:eval_gate}
  \mathrm{eval}_f\colon \Sigma\times Q\times f^{K} \lto f,
  \qquad
  \mathrm{eval}_f(\sigma, q, v)=v_{(\sigma,q)},
  \qquad f\in\{\Sigma, Q, D\},
\end{equation}
the address multiplexer $\mathrm{mux}_K$ of \cref{ex:addrmux} precomposed
with the pairing of $(\sigma, q)$ into an address in $K$.  By
\cref{def:naive_extension},
$\Delta\,\mathrm{eval}_f(\boldsymbol\sigma_0, \mathbf q, w_f)
 = \sum_j \lambda_j w_{j,f}$, so
\begin{equation}\label{eq:resolve_eval}
  \Delta\resolve_{\mathrm{eval}}(w,\boldsymbol\sigma_0,\mathbf q)
  =
  \Bigl(
    \sum_j\lambda_j w_{j,\Sigma},
    \sum_j\lambda_j w_{j,Q},
    \sum_j\lambda_j w_{j,D}
  \Bigr).
\end{equation}

\begin{figure}[H]
\centering
\resizebox{0.45\textwidth}{!}{%
\begin{tikzpicture}[font=\small, x=1cm, y=1cm, >=Stealth, bundle/.style={line width=1.2pt}]
  \node[branch] (Ds) at (2,2) {};
  \node[branch] (Dq) at (5,2) {};
  \node[draw, circle, fill=white, minimum size=4em] (Es) at (2,5.5) {$\mathrm{eval}_{\Sigma}$};
  \node[draw, circle, fill=white, minimum size=4em] (Et) at (5,5.5) {$\mathrm{eval}_{Q}$};
  \node[draw, circle, fill=white, minimum size=4em] (Ed) at (8,5.5) {$\mathrm{eval}_{D}$};
  \draw[marrow] (2,0.3) -- (Ds);
  \draw[marrow] (5,0.3) -- (Dq);
  \draw[marrow] (Ds.140) to[out=90,in=-90] (Es.220);
  \draw[marrow] (Ds.90)  to[out=90,in=-90] (Et.220);
  \draw[marrow] (Ds.40)  to[out=90,in=-90] (Ed.220);
  \draw[marrow] (Dq.140) to[out=90,in=-90] (Es.270);
  \draw[marrow] (Dq.90)  to[out=90,in=-90] (Et.270);
  \draw[marrow] (Dq.40)  to[out=90,in=-90] (Ed.270);
  \draw[bundle,marrow] (7.2,0.3) to[out=90,in=-90] (Es.320);
  \draw[bundle,marrow] (8.0,0.3) to[out=90,in=-90] (Et.320);
  \draw[bundle,marrow] (8.8,0.3) to[out=90,in=-90] (Ed.320);
  \draw[marrow] (Es) -- (2,7.6);
  \draw[marrow] (Et) -- (5,7.6);
  \draw[marrow] (Ed) -- (8,7.6);
  \node[below,font=\scriptsize] at (2,0.3) {$\sigma_0$};
  \node[below,font=\scriptsize] at (5,0.3) {$q$};
  \node[below,font=\scriptsize] at (7.2,0.3) {$w_\Sigma$};
  \node[below,font=\scriptsize] at (8.0,0.3) {$w_Q$};
  \node[below,font=\scriptsize] at (8.8,0.3) {$w_D$};
  \node[above,font=\scriptsize] at (2,7.6) {$\sigma'$};
  \node[above,font=\scriptsize] at (5,7.6) {$q'$};
  \node[above,font=\scriptsize] at (8,7.6) {$d$};
\end{tikzpicture}%
}
\caption{The transition-table lookup applied in a single step, with field gates \eqref{eq:eval_gate}.  Each $\mathrm{eval}_f$ selects the $f$-entry of the code tuple addressed by $(\sigma_0, q)$.  The head symbol $\sigma_0$ and the state $q$ are copied by diagonals $\Delta_3$.  The code is not copied: it enters as its field sub-bundles $w_\Sigma$, $w_Q$, $w_D$, since regrouping a bundle needs no gate.}
\label{fig:resolve_eval}
\end{figure}
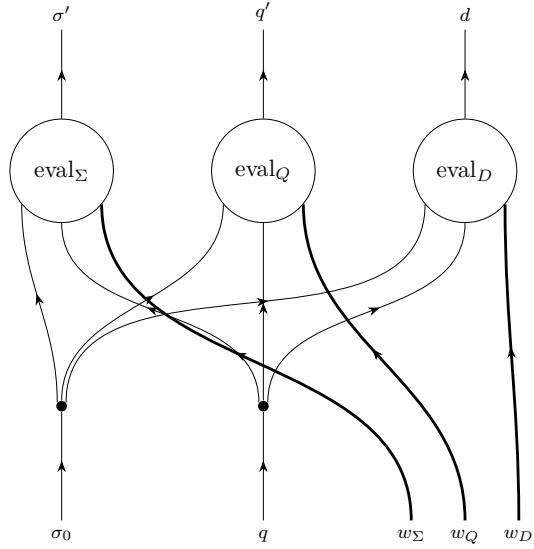

\subsubsection{Update}\label{sec:update}

For a tape $\tau\in\Sigma^{\mathbb Z,\blank}$, write symbol $\sigma'$, and
direction $d\in D$, set
\[
  \widetilde\tau(i)=
  \begin{cases}
    \sigma',& i=0,\\
    \tau(i),& i\neq 0,
  \end{cases}
  \qquad
  \texttt{update}(\tau,\sigma',d)
  =\bigl(i\mapsto\widetilde\tau(i+d)\bigr).
\]
This is the tape part of the configuration map of \cref{def:step}.  In a
finite window it is an array of identical local gates, one tape multiplexer
per output cell, each reading that cell and its two neighbours, so the
output window is one square wider on each side.  Cells outside the window are supplied by the blank
constant (\cref{fig:update_circuit}, with the multiplexers drawn out cell by
cell in \cref{fig:shift_circuit}).

\begin{definition}\label{def:shift}
The tape-shift subcircuit is the family of multiplexers, one per output cell,
the $i$-th computing
\[
  \mathrm{mux}_D\bigl(\widetilde\tau(i-1),\widetilde\tau(i),\widetilde\tau(i+1),d\bigr)
  =\widetilde\tau(i+d),
\]
with blank constants on the boundary of any finite window, and with each
$\widetilde\tau(i)$ and the direction $d$ duplicated by diagonals
(\cref{ex:copy}), one copy per multiplexer.  Together with the overwrite it is
the tape part of the update function above.
\end{definition}

Through the copies, a perturbation of one input square reaches three output
squares, as in \cref{ex:pnotp}.

\begin{definition}\label{def:update}
The update circuit is the overwrite of the head square by $\sigma'$ followed by
the tape-shift subcircuit of \cref{def:shift}.  It computes the function
$\texttt{update}$ above.  Its naive
probabilistic extension sends the write-symbol and direction distributions
$(\boldsymbol\sigma',\mathbf d)$ and the tape distributions $(\boldsymbol\sigma_i)_i$ to
\begin{equation}\label{eq:update_formula}
  \boldsymbol\sigma'_i(\sigma)
  =
  \sum_{d\in D}\mathbf d(d)
  \bigl\{\one[i\neq -d]\,\boldsymbol\sigma_{i+d}(\sigma)
  +\one[i=-d]\,\boldsymbol\sigma'(\sigma)\bigr\}.
\end{equation}
\end{definition}

\begin{remark}\label{rem:window_compat}
Write $\texttt{cycle}_\bullet^{h}$ for the cycle circuit with tape inputs
$\sigma_{-h},\ldots,\sigma_{h}$ (\cref{fig:cycle_circuit}).  Feeding the two outermost tape inputs of
$\texttt{cycle}_\bullet^{h+1}$ from blank constants yields
$\texttt{cycle}_\bullet^{h}$ together with blank constants at the two
outermost outputs, since a multiplexer whose three tape inputs are blank
outputs blank at every direction.  The naive probabilistic extensions at successive windows
therefore agree on tape distributions supported in the smaller window, and
assemble to a single map on all of
$W \times (\Delta\Sigma)^{\mathbb Z,\blank} \times \Delta Q$.
\end{remark}

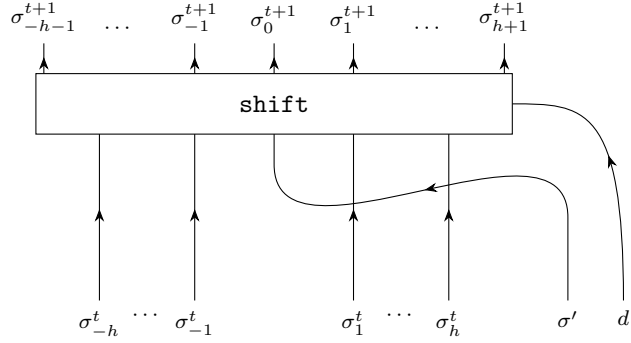
\begin{figure}[H]
\centering
\begin{tikzpicture}[font=\small, x=1.05cm, y=1.0cm, >=Stealth]
  \draw (-3.0,1.3) rectangle (3.0,2.1);
  \node at (0,1.7) {$\texttt{shift}$};
  \draw[marrow] (-2.9,2.1) -- (-2.9,2.5);
  \draw[marrow] (-1.0,2.1) -- (-1.0,2.5);
  \draw[marrow] (0,2.1)    -- (0,2.5);
  \draw[marrow] (1.0,2.1)  -- (1.0,2.5);
  \draw[marrow] (2.9,2.1)  -- (2.9,2.5);
  \node[above,font=\scriptsize] at (-2.9,2.5)  {$\sigma^{t+1}_{-h-1}$};
  \node[above,font=\scriptsize] at (-1.95,2.5) {$\cdots$};
  \node[above,font=\scriptsize] at (-1.0,2.5)  {$\sigma^{t+1}_{-1}$};
  \node[above,font=\scriptsize] at (0,2.5)     {$\sigma^{t+1}_{0}$};
  \node[above,font=\scriptsize] at (1.0,2.5)   {$\sigma^{t+1}_{1}$};
  \node[above,font=\scriptsize] at (1.95,2.5)  {$\cdots$};
  \node[above,font=\scriptsize] at (2.9,2.5)   {$\sigma^{t+1}_{h+1}$};
  \draw[marrow] (-2.2,-0.9) -- (-2.2,1.3);
  \draw[marrow] (-1.0,-0.9) -- (-1.0,1.3);
  \draw[marrow] (1.0,-0.9)  -- (1.0,1.3);
  \draw[marrow] (2.2,-0.9)  -- (2.2,1.3);
  \node[below,font=\scriptsize] at (-2.2,-0.9) {$\sigma^{t}_{-h}$};
  \node[below,font=\scriptsize] at (-1.6,-0.9) {$\cdots$};
  \node[below,font=\scriptsize] at (-1.0,-0.9) {$\sigma^{t}_{-1}$};
  \node[below,font=\scriptsize] at (1.0,-0.9)  {$\sigma^{t}_{1}$};
  \node[below,font=\scriptsize] at (1.6,-0.9)  {$\cdots$};
  \node[below,font=\scriptsize] at (2.2,-0.9)  {$\sigma^{t}_{h}$};
  \draw[marrow] (3.7,-0.9) -- (3.7,0.2) to[out=90,in=-90] (0,0.9) -- (0,1.3);
  \node[below,font=\scriptsize] at (3.7,-0.9) {$\sigma'$};
  \draw[marrow] (4.4,-0.9) .. controls (4.4,1.7) and (3.9,1.7) .. (3.0,1.7);
  \node[below,font=\scriptsize] at (4.4,-0.9) {$d$};
\end{tikzpicture}
\caption{The update circuit (\cref{def:update}).  The resolved write symbol $\sigma'$ takes the place of the head square, which is not itself an input ($\widetilde\tau(0):=\sigma'$).  The non-head squares pass through unchanged, and the overwritten register is shifted by the direction $d$ through the $\texttt{shift}$ subcircuit (\cref{def:shift}), drawn here as a black box and unwrapped in \cref{fig:shift_circuit}.  Its naive probabilistic extension is \eqref{eq:update_formula}.}
\label{fig:update_circuit}
\end{figure}

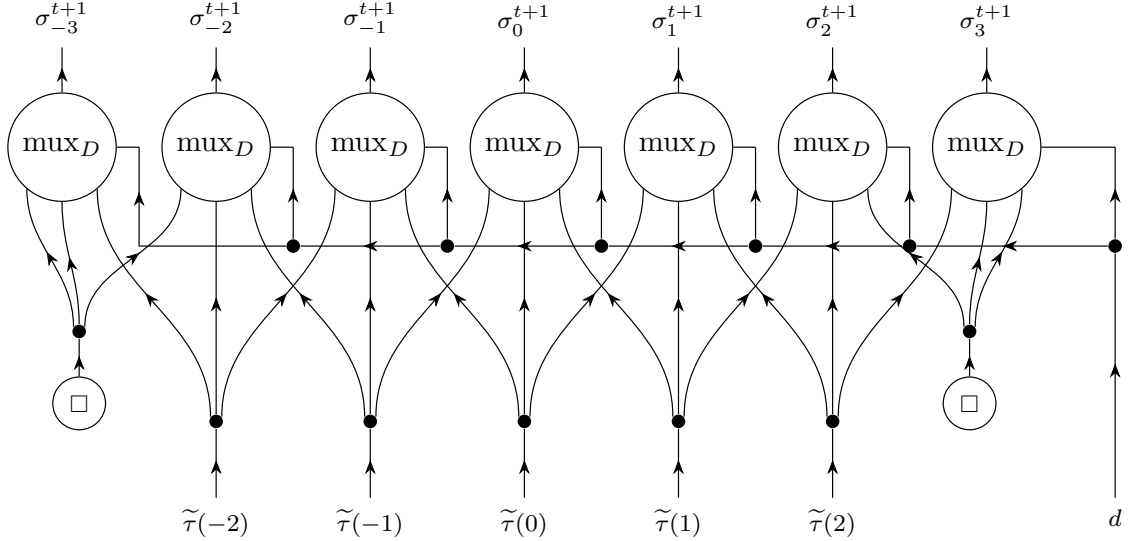
\begin{figure}[H]
\centering
\resizebox{0.95\textwidth}{!}{%
\begin{tikzpicture}[font=\small, x=1.0cm, y=1.05cm, >=Stealth,
   copyg/.style={circle, fill, inner sep=0pt, minimum size=4.5pt},
   blankc/.style={draw, circle, fill=white, minimum size=1.6em, inner sep=1pt, font=\scriptsize}]
  \node[draw, circle, fill=white, minimum size=3.0em] (cm3) at (-5.4,3.6) {$\mathrm{mux}_D$};
  \node[draw, circle, fill=white, minimum size=3.0em] (cm2) at (-3.6,3.6) {$\mathrm{mux}_D$};
  \node[draw, circle, fill=white, minimum size=3.0em] (cm1) at (-1.8,3.6) {$\mathrm{mux}_D$};
  \node[draw, circle, fill=white, minimum size=3.0em] (c0)  at (0,3.6)   {$\mathrm{mux}_D$};
  \node[draw, circle, fill=white, minimum size=3.0em] (cp1) at (1.8,3.6) {$\mathrm{mux}_D$};
  \node[draw, circle, fill=white, minimum size=3.0em] (cp2) at (3.6,3.6) {$\mathrm{mux}_D$};
  \node[draw, circle, fill=white, minimum size=3.0em] (cp3) at (5.4,3.6) {$\mathrm{mux}_D$};
  \node[copyg] (tm2) at (-3.6,0.55) {};
  \node[copyg] (tm1) at (-1.8,0.55) {};
  \node[copyg] (t0)  at (0,0.55)    {};
  \node[copyg] (tp1) at (1.8,0.55)  {};
  \node[copyg] (tp2) at (3.6,0.55)  {};
  \draw[marrow] (-3.6,-0.3) -- (tm2);
  \draw[marrow] (-1.8,-0.3) -- (tm1);
  \draw[marrow] (0,-0.3)    -- (t0);
  \draw[marrow] (1.8,-0.3)  -- (tp1);
  \draw[marrow] (3.6,-0.3)  -- (tp2);
  \node[below,font=\scriptsize] at (-3.6,-0.3) {$\widetilde\tau(-2)$};
  \node[below,font=\scriptsize] at (-1.8,-0.3) {$\widetilde\tau(-1)$};
  \node[below,font=\scriptsize] at (0,-0.3)    {$\widetilde\tau(0)$};
  \node[below,font=\scriptsize] at (1.8,-0.3)  {$\widetilde\tau(1)$};
  \node[below,font=\scriptsize] at (3.6,-0.3)  {$\widetilde\tau(2)$};
  \node[blankc] (bL) at (-5.2,0.75) {$\blank$};
  \node[copyg]  (cbL) at (-5.2,1.55) {};
  \draw[marrow] (bL) -- (cbL);
  \node[blankc] (bR) at (5.2,0.75) {$\blank$};
  \node[copyg]  (cbR) at (5.2,1.55) {};
  \draw[marrow] (bR) -- (cbR);
  \draw[marrow] (tm2.140) to[out=90,in=-90] (cm3.310);
  \draw[marrow] (tm2.90)  to[out=90,in=-90] (cm2.270);
  \draw[marrow] (tm2.40)  to[out=90,in=-90] (cm1.230);
  \draw[marrow] (tm1.140) to[out=90,in=-90] (cm2.310);
  \draw[marrow] (tm1.90)  to[out=90,in=-90] (cm1.270);
  \draw[marrow] (tm1.40)  to[out=90,in=-90] (c0.230);
  \draw[marrow] (t0.140)  to[out=90,in=-90] (cm1.310);
  \draw[marrow] (t0.90)   to[out=90,in=-90] (c0.270);
  \draw[marrow] (t0.40)   to[out=90,in=-90] (cp1.230);
  \draw[marrow] (tp1.140) to[out=90,in=-90] (c0.310);
  \draw[marrow] (tp1.90)  to[out=90,in=-90] (cp1.270);
  \draw[marrow] (tp1.40)  to[out=90,in=-90] (cp2.230);
  \draw[marrow] (tp2.140) to[out=90,in=-90] (cp1.310);
  \draw[marrow] (tp2.90)  to[out=90,in=-90] (cp2.270);
  \draw[marrow] (tp2.40)  to[out=90,in=-90] (cp3.230);
  \draw[marrow] (cbL.140) to[out=90,in=-90] (cm3.230);
  \draw[marrow] (cbL.90)  to[out=90,in=-90] (cm3.270);
  \draw[marrow] (cbL.40)  to[out=90,in=-90] (cm2.230);
  \draw[marrow] (cbR.140) to[out=90,in=-90] (cp2.310);
  \draw[marrow] (cbR.90)  to[out=90,in=-90] (cp3.270);
  \draw[marrow] (cbR.40)  to[out=90,in=-90] (cp3.310);
  \node[copyg] (dd7) at (6.9,2.5)  {};
  \node[copyg] (dd6) at (4.5,2.5)  {};
  \node[copyg] (dd5) at (2.7,2.5)  {};
  \node[copyg] (dd4) at (0.9,2.5)  {};
  \node[copyg] (dd3) at (-0.9,2.5) {};
  \node[copyg] (dd2) at (-2.7,2.5) {};
  \draw[marrow] (6.9,-0.3) -- (dd7);
  \node[below,font=\scriptsize] at (6.9,-0.3) {$d$};
  \draw[marrow] (dd7.west) -- (dd6.east);
  \draw[marrow] (dd6.west) -- (dd5.east);
  \draw[marrow] (dd5.west) -- (dd4.east);
  \draw[marrow] (dd4.west) -- (dd3.east);
  \draw[marrow] (dd3.west) -- (dd2.east);
  \foreach \ddot/\cell in {dd7/cp3, dd6/cp2, dd5/cp1, dd4/c0, dd3/cm1, dd2/cm2}{
    \draw[marrow] (\ddot.north) -- (\ddot.north |- \cell.east);
    \draw (\ddot.north |- \cell.east) -- (\cell.east);
  }
  \coordinate (kL) at (-4.5,2.5);
  \draw (dd2.west) -- (kL);
  \draw[marrow] (kL) -- (kL |- cm3.east);
  \draw (kL |- cm3.east) -- (cm3.east);
  \draw[marrow] (cm3.north) -- ++(0,0.5);
  \draw[marrow] (cm2.north) -- ++(0,0.5);
  \draw[marrow] (cm1.north) -- ++(0,0.5);
  \draw[marrow] (c0.north)  -- ++(0,0.5);
  \draw[marrow] (cp1.north) -- ++(0,0.5);
  \draw[marrow] (cp2.north) -- ++(0,0.5);
  \draw[marrow] (cp3.north) -- ++(0,0.5);
  \node[above,font=\scriptsize] at ($(cm3.north)+(0,0.5)$) {$\sigma^{t+1}_{-3}$};
  \node[above,font=\scriptsize] at ($(cm2.north)+(0,0.5)$) {$\sigma^{t+1}_{-2}$};
  \node[above,font=\scriptsize] at ($(cm1.north)+(0,0.5)$) {$\sigma^{t+1}_{-1}$};
  \node[above,font=\scriptsize] at ($(c0.north)+(0,0.5)$)  {$\sigma^{t+1}_{0}$};
  \node[above,font=\scriptsize] at ($(cp1.north)+(0,0.5)$) {$\sigma^{t+1}_{1}$};
  \node[above,font=\scriptsize] at ($(cp2.north)+(0,0.5)$) {$\sigma^{t+1}_{2}$};
  \node[above,font=\scriptsize] at ($(cp3.north)+(0,0.5)$) {$\sigma^{t+1}_{3}$};
\end{tikzpicture}%
}
\caption{The $\texttt{shift}$ subcircuit drawn out multiplexer by multiplexer, shown for window $h=2$: the $2h+1$ overwritten squares $\widetilde\tau(-2),\ldots,\widetilde\tau(2)$ (with $\widetilde\tau(0)=\sigma'$) feed $2h+3$ output squares, one $\mathrm{mux}_D$ per output cell (\cref{def:shift}).  Each square is duplicated by a diagonal $\Delta_3$ (\cref{ex:copy}) to the three cells that read it, and the direction $d$ is copied along a chain of diagonals, one copy entering each cell's east port.  Blank constants, likewise duplicated, supply the neighbours beyond the window, so the circuit is finite.}
\label{fig:shift_circuit}
\end{figure}

\subsection{The cycle maps in closed form}\label{app:cycle_maps}

Combining resolve and update, we define the cycle maps of the two machines
directly as circuits.

\begin{definition}\label{def:circuit_cycle_map}
For $\bullet \in \{\mathrm{staged}, \mathrm{lookup}\}$ the \emph{cycle map} is
\[
  F^{\bullet} := \Delta\texttt{cycle}_{\bullet}
  \colon W \times \Delta\Cfg \longrightarrow \Delta\Cfg,
\]
the single map assembled over window sizes in \cref{rem:window_compat}.
Either map can serve as the $F$ of \cref{def:model_utm}: each is
polynomial, hence smooth, and each restricts to one classical period of the
simulated machine by \cref{def:resolve} and \eqref{eq:update_formula}.
\end{definition}

\begin{remark}\label{rem:canonical_step}
The object one would ideally relax is the step function $\step_{\mathcal U}$
itself (\cref{def:utm_relaxation}), taking $F$ to be the cycle map induced
by one period of $P = 10N+5$ steps
(\cref{def:periodic_relaxation};
\cite{clift2021geometryprogramsynthesis}).
There is a canonical one-step circuit, in which every square and the UTM
state depend on all squares and the UTM state at the previous step, but
we do not draw it or prove that the naive probabilistic extension of its $P$-fold composite is
$F^{\bullet}$: it
has too many paths, and its equations are unhelpful to reason about.
Instead, the cycle circuits of \cref{sec:one_period_circuits} render only
the dependencies that survive across a period.  It can be checked, mechanically but at some length, that the cycle
circuits are faithful to the dependency structure of the machines, and
hence that $F^{\bullet}$ is the cycle map of the canonical relaxation;
we have verified the identification numerically.  The same identification, for the staged machine, underlies the direct
simulation of \cite[Appendix~I]{clift2021geometryprogramsynthesis}.
\end{remark}

\begin{remark}\label{rem:cd_extension}
At a classical code $w = [M]$ the entries $w_{j,f} = e_{[M]_{j,f}}$ are point masses and the resolve mixture \eqref{eq:resolve_eval} reduces, field-wise, to $\sum_{(\sigma, q)} \boldsymbol\sigma_0(\sigma)\,\mathbf q(q)\, e_{[M]_{(\sigma, q),f}}$, so $\Delta\texttt{cycle}_{\mathrm{eval}}([M], -)$ is the naive probabilistic extension of $\step_{[M]}$ in the sense of \cite{clift2020derivatives}.\footnote{This is no coincidence: if one expresses $\texttt{cycle}_{\mathrm{eval}}$ as a proof in linear logic in the obvious way, one can show that its cut with a classical code is cut-equivalent to \underline{relstep} from \cite{clift2020derivatives}.}  The family $\Delta\texttt{cycle}_{\mathrm{eval}}(w, -)$ over $w \in W$ extends that construction from a machine to the parameter space of noisy codes, and the canonical step circuit of \cref{rem:canonical_step} is the same construction applied to the classical machine $\mathcal U$ itself.
\end{remark}

The closed form of the cycle map is the only property of the machines used
in the main text.

\begin{proposition}\label{prop:cycle_factor}
For $\bullet\in\{\mathrm{staged},\mathrm{lookup}\}$ the cycle map sends
$z=((\boldsymbol\sigma_i)_{i\in\mathbb Z},\mathbf q)$ to
$F^{\bullet}_w(z)=((\boldsymbol\sigma'_i)_{i\in\mathbb Z},\mathbf q')$ as
follows.  With the match weights
$\lambda_j=\boldsymbol\sigma_0(\sigma_j)\,\mathbf q(q_j)$ of
\eqref{eq:match_weights}, the \emph{decision weights} of the machine are
\begin{equation}\label{eq:decision_weights}
  \nu^{\mathrm{lookup}}_j=\lambda_j,
  \qquad
  \nu^{\mathrm{staged}}_j=\lambda_j\prod_{l>j}(1-\lambda_l),
  \qquad
  \nu_X:=1-\sum_j\nu_j.
\end{equation}
The output state $\mathbf q'$, the resolved write symbol
$\boldsymbol\sigma'$, and the direction $\mathbf d$ are the mixtures
\begin{equation}\label{eq:resolved_mixtures}
  \boldsymbol\sigma'=\sum_j\nu_j w_{j,\Sigma}+\nu_X\boldsymbol\sigma_0,
  \qquad
  \mathbf q'=\sum_j\nu_j w_{j,Q}+\nu_X\mathbf q,
  \qquad
  \mathbf d=\sum_j\nu_j w_{j,D}+\nu_X e_0,
\end{equation}
with $\nu=\nu^{\bullet}$, and the output tape is
\begin{equation}\label{eq:cycle_map_formulas}
  \boldsymbol\sigma'_i(\sigma)
    =\sum_{d\in D}\mathbf d(d)
      \bigl\{\one[i\neq -d]\,\boldsymbol\sigma_{i+d}(\sigma)
      +\one[i=-d]\,\boldsymbol\sigma'(\sigma)\bigr\}
      \qquad(i\in\mathbb Z,\ \sigma\in\Sigma).
\end{equation}
\end{proposition}
\begin{proof}
$\texttt{cycle}_{\bullet}$ is $\resolve_{\bullet}$ followed by
$\texttt{update}$ (\cref{fig:cycle_circuit}), naive probabilistic
extensions compose, and the naive probabilistic extension of $\texttt{update}$ is the tape
multiplexer \eqref{eq:update_formula}, which is
\eqref{eq:cycle_map_formulas}.  It remains to compute the naive probabilistic extension of
$\resolve_{\bullet}$.

\emph{Staged.}  Write
$\hat{\mathbf s}_j \in \Delta\Sigma_X \times \Delta Q_X \times \Delta D_X$
for the distributions on the three carried wires after stage $j$
(\cref{fig:resolve_st}), so
$\hat{\mathbf s}_0 = (e_X, e_X, e_X)$, the naive probabilistic
extensions of the $X$-constant gates, and
\[
  \hat{\mathbf s}_j := \Delta\,\mathrm{ovr}_j\bigl(
    \boldsymbol\sigma_0,\, \mathbf q,\,
    w_{j,\Sigma} \otimes w_{j,Q} \otimes w_{j,D},\,
    \hat{\mathbf s}_{j-1}\bigr),
\]
the fresh copies of $\boldsymbol\sigma_0$ and $\mathbf q$ entering as
independent wires.  The gate $\mathrm{ovr}_j$ has three output types,
with $f$-component $v_f$ on a match and $s_f$ otherwise
\eqref{eq:staged_gates}, so \cref{def:naive_extension} splits over the
match event as in the $\mathrm{cmp}$ computation below: the matched
fibre $(\sigma, r) = (\sigma_j, q_j)$ carries weight $\lambda_j$ and
contributes $\lambda_j\, w_{j,f}$, the unmatched fibre carries
$1 - \lambda_j$ and contributes
$(1 - \lambda_j)\, \hat{\mathbf s}_{j-1,f}$, and the remaining wires
marginalise to $1$ in both cases:
\[
  \hat{\mathbf s}_{j,f}
  = \lambda_j\, w_{j,f} + (1 - \lambda_j)\, \hat{\mathbf s}_{j-1,f}.
\]
Unrolling from $\hat{\mathbf s}_0$ by induction, and using that the code
entries carry no $X$ mass,
\[
  \hat{\mathbf s}_{N,f}
  = \sum_j \lambda_j \prod_{l > j}(1 - \lambda_l)\, w_{j,f}
    \;+\; \prod_l (1 - \lambda_l)\, e_X
  = \sum_j \nu_j^{\mathrm{staged}}\, w_{j,f}
    + \nu_X^{\mathrm{staged}}\, e_X.
\]
Finally $\mathrm{res}_X$ has three output types, with $f$-component
$s_f$ when $s_f \neq X$ and the no-op value ($\sigma$, $r$, or $S$) when
$s_f = X$ \eqref{eq:staged_gates}.  It receives fresh copies of
$\boldsymbol\sigma_0$ and $\mathbf q$ and the carried wires
$\hat{\mathbf s}_N$, so for $f = \Sigma$, \cref{def:naive_extension}
splits over the value of $s_\Sigma$: either $s_\Sigma = \sigma' \neq X$
(weight $\hat{\mathbf s}_{N,\Sigma}(\sigma')$, the other wires
marginalising to $1$), or $s_\Sigma = X$ and the fresh copy reads
$\sigma = \sigma'$ (weight
$\hat{\mathbf s}_{N,\Sigma}(X)\, \boldsymbol\sigma_0(\sigma')
 = \nu_X\, \boldsymbol\sigma_0(\sigma')$).  By the unrolled form above,
$\hat{\mathbf s}_{N,\Sigma}(\sigma') = \sum_j \nu_j\, w_{j,\Sigma}(\sigma')$
for $\sigma' \neq X$, so
\[
  \Delta\,\mathrm{res}_X(\boldsymbol\sigma_0, \mathbf q, \hat{\mathbf s}_N)_\Sigma
  = \sum_j \nu_j\, w_{j,\Sigma} + \nu_X\, \boldsymbol\sigma_0,
\]
and likewise the $Q$- and $D$-components are
$\sum_j \nu_j w_{j,Q} + \nu_X \mathbf q$ and
$\sum_j \nu_j w_{j,D} + \nu_X e_0$: this is
\eqref{eq:resolved_mixtures} with $\nu = \nu^{\mathrm{staged}}$.

\emph{Lookup.}  The naive probabilistic extension of a counit is the unique map to a point, so
deleting the discarded exports, and the copies that feed them, does not
change the naive probabilistic extension of the circuit.  What remains is the chain of gates
\eqref{eq:lookup_gates}: $\mathrm{compInit}$, one $\mathrm{cmp}$ per
tuple, and the final $\mathrm{export}$.  Let
$\mu_j \in \Delta Q_{\mathcal U}$ be the distribution on the
$Q_{\mathcal U}$-wire after stage $j$,
\[
  \mu_0 := \Delta\,\mathrm{compInit}(\boldsymbol\sigma_0, \mathbf q),
  \qquad
  \mu_j := \Delta\,\mathrm{cmp}\bigl(\mu_{j-1},\, e_{\sigma_j},\, e_{q_j},\,
    w_{j,\Sigma} \otimes w_{j,Q} \otimes w_{j,D}\bigr),
\]
with the hardwired pair entering as point masses and the entry $w_j$ as
the product of its three field wires.  Since $\mathrm{compInit}$ is
injective, \cref{def:naive_extension} gives
\[
  \mu_0(\mathrm{comp}_\kappa)
  = \sum_{\mathrm{compInit}(\sigma, q) = \mathrm{comp}_\kappa}
      \boldsymbol\sigma_0(\sigma)\, \mathbf q(q)
  = (\boldsymbol\sigma_0 \otimes \mathbf q)(\kappa),
  \qquad
  \mu_0(\mathrm{post}_t) = 0,
\]
For $\mathrm{cmp}$, the point masses collapse the sums over the second
and third inputs to $\sigma_j$ and $q_j$, so \cref{def:naive_extension}
reads
\[
  \mu_j(\rho)
  = \sum_{\mathrm{cmp}(\rho', \sigma_j, q_j, v) = \rho}
      \mu_{j-1}(\rho')\,
      w_{j,\Sigma}(v_\Sigma)\, w_{j,Q}(v_Q)\, w_{j,D}(v_D).
\]
By \eqref{eq:lookup_gates}, for $\rho = \mathrm{comp}_\kappa$ the fibre
is $\rho' = \mathrm{comp}_\kappa$ with $\kappa \neq (\sigma_j, q_j)$ and
$v$ arbitrary (the sum over $v$ is $1$); for $\rho = \mathrm{post}_t$ it
is $\rho' = \mathrm{post}_t$ with $v$ arbitrary, together with
$\rho' = \mathrm{comp}_{(\sigma_j, q_j)}$ with $v = t$.  Hence
\begin{align*}
  \mu_j(\mathrm{comp}_\kappa)
    &= \one\bigl[\kappa \neq (\sigma_j, q_j)\bigr]\,
       \mu_{j-1}(\mathrm{comp}_\kappa),\\
  \mu_j(\mathrm{post}_t)
    &= \mu_{j-1}(\mathrm{post}_t)
       + \mu_{j-1}\bigl(\mathrm{comp}_{(\sigma_j, q_j)}\bigr)\,
         w_{j,\Sigma}(t_\Sigma)\, w_{j,Q}(t_Q)\, w_{j,D}(t_D),
\end{align*}
and by the first line the mass of $\mathrm{comp}_{(\sigma_j, q_j)}$ is
untouched by the earlier stages, so
\[
  \mu_{j-1}(\mathrm{comp}_{(\sigma_j, q_j)})
  = \mu_0(\mathrm{comp}_{(\sigma_j, q_j)})
  = \lambda_j:
\]
stage $j$ moves
exactly the match weight $\lambda_j$ onto post-states distributed as the
entry $w_j$.
Iterating the two lines from $\mu_0$: the inputs exhaust $K$, so every
$\kappa \in K$ equals $(\sigma_j, q_j)$ at exactly one stage and its
indicator kills the $\mathrm{comp}_\kappa$ mass there, while the
$\mathrm{post}$ masses accumulate one summand per stage,
\[
  \mu_N(\mathrm{comp}_\kappa) = 0,
  \qquad
  \mu_N(\mathrm{post}_t)
  = \sum_j \lambda_j\, w_{j,\Sigma}(t_\Sigma)\, w_{j,Q}(t_Q)\,
      w_{j,D}(t_D).
\]
Naive probabilistic extensions compose, so the naive probabilistic
extension of the remaining chain is $\Delta\,\mathrm{export}(\mu_N)$.
The gate $\mathrm{export}$ has three output types
(\cref{def:circuit}), $\mathrm{export} = (\mathrm{export}_\Sigma,
\mathrm{export}_Q, \mathrm{export}_D)$ with
$\mathrm{export}_f(\mathrm{post}_t) = t_f$, so
\cref{def:naive_extension} gives one output distribution per field,
$\Delta\,\mathrm{export}_f(\mu_N)(r)
 = \sum_{\mathrm{export}_f(\rho) = r} \mu_N(\rho)$, in which only the
$\mathrm{post}$ states contribute ($\mu_N$ has no $\mathrm{comp}$
mass).  For $f = \Sigma$,
\begin{align*}
  \Delta\,\mathrm{export}_\Sigma(\mu_N)(s)
  &= \sum_{t \,:\, t_\Sigma = s} \mu_N(\mathrm{post}_t)\\
  &= \sum_j \lambda_j\, w_{j,\Sigma}(s)
      \sum_{t_Q} w_{j,Q}(t_Q) \sum_{t_D} w_{j,D}(t_D)\\
  &= \sum_j \lambda_j\, w_{j,\Sigma}(s),
\end{align*}
each $w_{j,f}$ summing to $1$, and likewise for $Q$ and $D$.  The output
triple is therefore
\[\bigl(\sum_j \lambda_j w_{j,\Sigma},\ \sum_j \lambda_j w_{j,Q},\
 \sum_j \lambda_j w_{j,D}\bigr):\]
this is \eqref{eq:resolved_mixtures} with $\nu=\nu^{\mathrm{lookup}}$ and
$\nu_X=0$.
\end{proof}

\begin{corollary}\label{cor:lookup_srp}
We have
$F^{\mathrm{lookup}}_w=\Delta\texttt{cycle}_{\mathrm{eval}}(w,-)$: the cycle map of
the lookup machine is the naive probabilistic extension of the eval cycle
circuit.\footnote{On deterministic codes $w=[M]$, this is the property Xu calls \emph{smooth relaxation
preserving} \cite{xu2021smoothrelaxationpreservingturing}: simulating the
machine and then relaxing agrees with relaxing the machine's own step. So this corollary implies in particular that the lookup UTM is smooth relaxation preserving.}
\end{corollary}

\begin{proof}
By \cref{prop:cycle_factor}, $\nu^{\mathrm{lookup}}_j=\lambda_j$ and
$\nu_X=0$, so \eqref{eq:resolved_mixtures} is \eqref{eq:resolve_eval}
and $\Delta\resolve_{\mathrm{lookup}}=\Delta\resolve_{\mathrm{eval}}$.
The two cycle circuits share $\texttt{update}$.
\end{proof}

\subsection{Invariance and equivariance of the cycle maps}\label{app:cycle_symmetries}

\begin{lemma}\label{lem:staged_lookup_unmatched_component_invariance}
The smooth relaxations of
the staged pseudo-UTM \cite[\S 5.1]{murfet2025pas} and of the lookup
pseudo-UTM of \cref{sec:lookup_utm} are unmatched-component invariant
(\cref{def:unmatched_component_invariance}).
\end{lemma}

\begin{proof}
Let $C$ be a component of the tuple at position $j_C$, and evaluate one period
at a deterministic simulated configuration with read pair
$(\sigma,q)\neq(\sigma_{j_C},q_{j_C})$.  The match weights
\eqref{eq:match_weights} depend only on the read pair, not on the code, so
they are unchanged by the replacement $[M]_{C \leftarrow v}$: there is a unique
position $j_*$ with input $(\sigma,q)$, and $\lambda_{j_*}=1$, $\lambda_j=0$
for $j\neq j_*$.  Both decision-weight rules in \eqref{eq:decision_weights}
then give $\nu_{j_*}=1$ and all other weights, including $\nu_X$, equal to
zero.  The replaced entry enters the mixtures
\eqref{eq:resolved_mixtures} only through terms with coefficient
$\nu_{j_C}=0$, so $F_{[M]_{C \leftarrow v}}=F_{[M]}$ at this configuration.
\end{proof}

\begin{corollary}\label{cor:lookup_recoding_equivariance}
The lookup pseudo-UTM is recoding equivariant
(\cref{def:recoding_equivariance}).
\end{corollary}

\begin{proof}
For the lookup machine, $\nu_j=\lambda_j$ and
$\nu_X=0$, so by \cref{prop:cycle_factor} one period is determined by the
mixtures
\[
  \boldsymbol\sigma'_{w,z}=\sum_{k\in K}\lambda_k\,w_{k,\Sigma},\qquad
  \mathbf q'_{w,z}=\sum_{k\in K}\lambda_k\,w_{k,Q},\qquad
  \mathbf d_{w,z}=\sum_{k\in K}\lambda_k\,w_{k,D},
\]
where $z=((\boldsymbol\sigma_i)_i,\mathbf q)$, tuples are indexed by their input pairs
$k\in K$, and $\lambda_{(\sigma,r)}=\boldsymbol\sigma_0(\sigma)\,\mathbf q(r)$.  Fix a
recoding $(a,s)$.  By \cref{def:recoding_action_ext},
$\bigl((a, s) \cdot w\bigr)_{(\sigma,r),\Sigma}=a_*\,w_{(a^{-1}(\sigma),s^{-1}(r)),\Sigma}$.
The match weights of the recoded configuration $(a, s) \cdot z$ are
$\lambda_{(\sigma,r)}((a, s) \cdot z)=(a_*\boldsymbol\sigma_0)(\sigma)\,(s_*\mathbf q)(r)
=\boldsymbol\sigma_0(a^{-1}(\sigma))\,\mathbf q(s^{-1}(r))$, so,
reindexing $K$ by the bijection $(\sigma,r)\mapsto(a(\sigma),s(r))$,
\begin{align*}
  \boldsymbol\sigma'_{(a, s) \cdot w,\,(a, s) \cdot z}
  &=\sum_{(\sigma,r)\in K}\boldsymbol\sigma_0(a^{-1}(\sigma))\,\mathbf q(s^{-1}(r))\,
     a_*\,w_{(a^{-1}(\sigma),s^{-1}(r)),\Sigma}\\
  &=\sum_{(\sigma,r)\in K}\boldsymbol\sigma_0(\sigma)\,\mathbf q(r)\,
     a_*\,w_{(\sigma,r),\Sigma}
  =a_*\,\boldsymbol\sigma'_{w,z},
\end{align*}
for arbitrary bijections $a$ and $s$: the inverses in
\cref{def:recoding_action_ext} and in the pushforwards cancel against the reindexing. The same computation in the other two fields gives $\mathbf q'_{(a, s) \cdot w,\,(a, s) \cdot z}=s_*\,\mathbf q'_{w,z}$ and, the
direction values being untouched by the action,
$\mathbf d_{(a, s) \cdot w,\,(a, s) \cdot z}=\mathbf d_{w,z}$.  Now apply
\eqref{eq:cycle_map_formulas}: each output cell of $F^{\mathrm{lookup}}_{(a, s) \cdot w}((a, s) \cdot z)$ is the $\mathbf d$-mixture of the input
cells of $(a, s) \cdot z$ and
$a_*\,\boldsymbol\sigma'_{w,z}$, and the output state is
$s_*\,\mathbf q'_{w,z}$.  Since pushforward commutes with mixtures
with the same weights, these are exactly the $a$-images of the cells and the
$s$-image of the state of $F^{\mathrm{lookup}}_w(z)$:
\[
  (a, s) \cdot F^{\mathrm{lookup}}_{w}(z)
  =
  F^{\mathrm{lookup}}_{(a, s) \cdot w}\bigl((a, s) \cdot z\bigr),
\]
which is \cref{def:recoding_equivariance}.
\end{proof}

\section{Details of the experiments}\label{app:experiment_details}

This appendix expands \cref{sec:empirical_sus} and the experiments of
\cref{sec:experiments}.  It first checks that the rank bound and the
equivariance survive replacing the squared error by the log-loss.  It
justifies the Dirichlet localiser and gives it a Bayesian reading, isolates the
divergence of its slice evaluations at $\alpha < 1$, and shows that the
localised susceptibilities, with their equivariance and rank bound,
survive renormalisation and standardisation.  It then describes the GRLD sampler and its calibration,
the embedding parameters and compute, and the input kernels of the
running examples.  It closes with the parities of the eigenmatrices, the
cluster validation of \cref{fig:umap_main}, the weights of the binary
classifier, the temperature sweep, and the halting-time organisation of
the solution set.

\subsection{Reparametrising the loss preserves the rank bound and the equivariance}\label{app:reparam}

The proofs of \cref{thm:rank_bound,thm:sus_symmetry} use the squared-error data $h_x^2$ and $H$ only through properties preserved under post-composing the error probability $h_x$ with a fixed function.  Both conclusions therefore pass to any such reparametrisation of the loss, the log-loss included.

Fix $\omega \colon [0, 1] \to \mathbb R$ smooth.  Put $\omega_x := \omega \circ h_x$, $\Omega := \E_{x \sim q}[\omega_x]$, and $p_{\beta}^{\Omega} \propto \exp\{-\beta \Omega\}\,\varphi$.  Let $\phi_C^{\Omega}(w) := \delta(u - u^*)\,(\Omega(w) - \Omega([M]))$ be the associated component observable and define, in the covariance form of \cref{rem:fdt},
\[
  \chi_x^{C, \Omega}([M]) := -\Cov_{p_{\beta}^{\Omega}}\bigl[\phi_C^{\Omega},\, \omega_x - \Omega\bigr].
\]
Since $h_x([M]) = 0$, we have $\Omega([M]) = \omega(0)$.  Taking $\omega(t) = t^2$ recovers $\chi_x^{C}$.  Taking $\omega = \kappa_\mu$ (\cref{rem:mu_shift}) with $\mu > 0$ gives $\omega_x = \ell_x$ and $\Omega = L$.

\begin{proposition}\label{prop:signatures_general}
The rank bound of \cref{thm:rank_bound} holds for $\chi_x^{C, \Omega}$ for every $\omega$, and the equivariance of \cref{thm:sus_symmetry} holds for $\chi_x^{C, \Omega}$ for every input-independent $\omega$.
\end{proposition}

\begin{proof}
\emph{Rank bound.}  Write $\langle - \rangle$ for the expectation under $p_{\beta}^{\Omega}$.  On an off-diagonal slice $\{u = u^*\}$, \cref{lem:mismatched_slice} gives $h_x \equiv 0$, so $\omega_x \equiv \omega(0)$ there.  Since $\phi_C^{\Omega}$ is supported on the slice (\cref{def:phi_C}), $\langle \phi_C^{\Omega} \omega_x \rangle = \omega(0)\,\langle \phi_C^{\Omega} \rangle$.  This is the only use of $h_x^2 \equiv 0$ on the slice in the proof of \cref{thm:rank_bound}, which now gives each off-diagonal column as
\[
  \chi_x^{C, \Omega}([M]) = \langle \phi_C^{\Omega} \rangle\,\langle \omega_x \rangle + \Cov_{p_{\beta}^{\Omega}}[\phi_C^{\Omega}, \Omega] - \omega(0)\,\langle \phi_C^{\Omega} \rangle,
\]
a linear combination of the two vectors $(\langle \omega_x \rangle)_{x}$ and $(1)_{x}$ (the constant $\omega(0)$ contributing to the latter), so rank $\le 2$.

\emph{Equivariance.}  By \cref{lem:model_equivariance}, $h_{a(x)}\bigl((a, s) \cdot w\bigr) = h_x(w)$, and since $\omega$ is input-independent, $\omega_{a(x)}\bigl((a, s) \cdot w\bigr) = \omega\bigl(h_{a(x)}((a, s) \cdot w)\bigr) = \omega(h_x(w)) = \omega_x(w)$.  The reindexing of \cref{prop:H_invariance} then gives $\Omega\bigl((a, s) \cdot w\bigr) = \Omega(w)$.  These are the only facts about $h_x^2$ and $H$ used in the proof of \cref{thm:sus_symmetry}, which now applies verbatim.
\end{proof}

\begin{remark}\label{rem:log_loss_mu_zero}
The experiments use the unshifted log-loss, $\omega = \kappa_0$ with $\kappa_0(t) = -\log(1 - t)$, which is smooth only on $[0, 1)$ and infinite at $1$.  Both conclusions of \cref{prop:signatures_general} still hold, with $e^{-\beta L^0} := 0$ where $L^0 = \infty$.  The set $\{L^0 = \infty\} = \bigcup_x \{h_x = 1\}$ is a finite union of zero sets of polynomials, none of which vanishes identically on $W$ or on any slice, since every slice contains $[M]$ where $h_x = 0$.  Hence $\{L^0 = \infty\}$ has measure zero, and since $\kappa_0 \ge 0$ every expectation above remains finite.  The proofs apply unchanged.
\end{remark}

The Gibbs distributions at $\mu = 0$ are moreover the limits of their $\mu$-shifted counterparts.

\begin{proposition}\label{prop:mu_zero_limit}
Write $p_\beta^\mu := e^{-\beta L^\mu}\varphi / Z^\mu$ for the Gibbs distribution of the $\mu$-shifted loss, and $p_C^\mu$ for the restricted Gibbs distribution on the slice $W_C$ through $[M]$ (\cref{sec:empirical_sus}).  As $\mu \to 0$,
\[
  \int_W \bigl| p_\beta^\mu - p_\beta^0 \bigr|\, dw \longrightarrow 0
  \quad\text{and}\quad
  \int_{W_C} \bigl| p_C^\mu - p_C^0 \bigr|\, dv \longrightarrow 0
  \quad \text{for every component } C.
\]
Furthermore, the expectations of $\ell_x$, $L$, $L\,\ell_x$ and $L^2$ under these distributions, which are those appearing in \eqref{eq:renorm_sus}, converge to their values at $\mu = 0$.  In particular, as $\mu \to 0$ the renormalised susceptibility $\tilde\chi_x^C$ \eqref{eq:renorm_sus} converges to its value at $\mu = 0$.
\end{proposition}

\begin{proof}
Recall from \cref{rem:mu_shift} that $\ell_x^\mu = \kappa_\mu \circ h_x$, so that $L^\mu(w) = \sum_x q(x)\, \kappa_\mu(h_x(w))$, where
\[
  \kappa_\mu(h) = -\bigl(1 - \tfrac{\mu}{2}\bigr)\log\bigl(1 - \tfrac{\mu}{2} - (1 - \mu)\,h\bigr) - \tfrac{\mu}{2}\log\bigl(\tfrac{\mu}{2} + (1 - \mu)\,h\bigr).
\]
For $h \in [0, 1]$ both arguments of the logarithms lie in $[\mu/2,\, 1 - \mu/2]$.  Hence $\kappa_\mu \ge 0$.  Now fix $h \in [0, 1)$ and let $\mu \to 0$.  The first term tends to $-\log(1 - h) = \kappa_0(h)$.  The second has absolute value at most $\tfrac{\mu}{2}\,\lvert \log \tfrac{\mu}{2} \rvert$, which tends to $0$.  At $h = 1$, $\kappa_\mu(1) \to +\infty = \kappa_0(1)$.  Hence
\[
  0 \le e^{-\beta L^\mu} \le 1
  \quad\text{and}\quad
  e^{-\beta L^\mu} \to e^{-\beta L^0} \quad \text{pointwise on } W.
\]
The prior $\varphi$ dominates $e^{-\beta L^\mu}\varphi$.  By dominated convergence, $Z^\mu \to Z^0$, and $Z^0 > 0$ since $\{L^0 = \infty\}$ has measure zero (\cref{rem:log_loss_mu_zero}).  Take $\mu$ small enough that $Z^\mu \ge Z^0/2$.  Then the integrand of
\[
  \int_W \bigl| p_\beta^\mu - p_\beta^0 \bigr|\, dw
  = \int_W \Bigl| \frac{e^{-\beta L^\mu}}{Z^\mu} - \frac{e^{-\beta L^0}}{Z^0} \Bigr|\, \varphi\, dw
\]
is dominated by $(3/Z^0)\varphi$ and tends to zero pointwise.  So the integral tends to zero, which is the first claim.  The same argument applies to $p_C^\mu$ on $W_C$, with the density $e^{-\beta L^\mu(u^*, v)}\,\varphi_C(v)$ of \cref{rem:slice_infinity} in place of $e^{-\beta L^\mu}\varphi$.  Its normalising constant at $\mu = 0$, $\int_{W_C} e^{-\beta L^0(u^*, v)}\,\varphi_C(v)\, dv$, is positive, since $W_C$ contains $[M]$ (\cref{rem:log_loss_mu_zero}).

It remains to see that the expectations in \eqref{eq:renorm_sus} converge.  These are expectations of $\ell_x$, $L$, $L\,\ell_x$ and $L^2$ under $p_\beta^\mu$ or its slice analogue.  Each has the form $\frac{1}{Z^\mu}\int_W f^\mu\, e^{-\beta L^\mu}\varphi\, dw$, where $f^\mu$ is one of these four functions.  The losses themselves are unbounded, so we need the Gibbs factor to control them.  Every term of $L^\mu = \sum_x q(x)\,\ell_x^\mu$ is nonnegative, so $L^\mu \ge q(x)\,\ell_x^\mu$.  Since $q(x) > 0$, the elementary bound $t\, e^{-at} \le 1/(ea)$ for $t \ge 0$, $a > 0$ gives
\[
  \ell_x^\mu\, e^{-\beta L^\mu} \le \ell_x^\mu\, e^{-\beta q(x)\, \ell_x^\mu} \le \frac{1}{e \beta q(x)},
  \qquad
  L^\mu\, \ell_x^\mu\, e^{-\beta L^\mu}
  = \bigl(L^\mu e^{-\beta L^\mu / 2}\bigr)\bigl(\ell_x^\mu e^{-\beta L^\mu / 2}\bigr)
  \le \frac{4}{e^2 \beta^2\, q(x)}.
\]
The same argument bounds $L^\mu\, e^{-\beta L^\mu}$ and $(L^\mu)^2\, e^{-\beta L^\mu}$.  Thus $f^\mu e^{-\beta L^\mu}$ is bounded uniformly in $\mu$.  Off the measure-zero set $\{L^0 = \infty\}$ it converges pointwise to $f^0 e^{-\beta L^0}$.  Dominated convergence now gives convergence of each integral.  Dividing by $Z^\mu \to Z^0 > 0$ gives convergence of each expectation.  The renormalised susceptibility \eqref{eq:renorm_sus} is a polynomial combination of these expectations, so it converges with them.
\end{proof}

\subsection{Properties of the localiser and the estimator}\label{app:estimator_properties}

The remarks of this subsection concern the localiser and the slice objects built from it in \cref{sec:empirical_sus}.  Fix a classical solution $M$ and a component $C$.  Write $W = U_C \times W_C$, $w = (u, v)$, and $[M] = (u^*, v^*)$ as in \cref{def:component}.  The localiser \eqref{eq:dirichlet_localiser} is a product of one Dirichlet factor per component, $\varphi_{C'}(p) := \operatorname{Dir}\bigl(p;\, \gamma\, e_{[M]_{C'}} + \alpha\,\mathbf 1\bigr)$ on $W_{C'}$, so it factorises as
\[
  \varphi_{[M], \gamma, \alpha}(u, v) = \varphi_{-C}(u)\,\varphi_C(v),
  \qquad
  \varphi_{-C}(u) := \prod_{C' \neq C}\varphi_{C'}(u_{C'}).
\]
The slice objects of \cref{sec:empirical_sus} factorise with it.  Writing
\[
  \zeta_C := \int_{W_C} e^{-\beta L(u^*, v)}\, \varphi_C(v)\, dv \;\in\; (0, 1],
\]
the component partition function is $Z_C = \varphi_{-C}(u^*)\,\zeta_C$, and the restricted Gibbs distribution cancels the factor $\varphi_{-C}(u^*)$:
\[
  p_C(v)
  = \frac{e^{-\beta L(u^*, v)}\,\varphi_{-C}(u^*)\,\varphi_C(v)}{\varphi_{-C}(u^*)\,\zeta_C}
  = \frac{e^{-\beta L(u^*, v)}\,\varphi_C(v)}{\zeta_C}.
\]
The remarks below need the value of the constant $\varphi_{-C}(u^*)$, so we compute it.  Written out,
\[
  \varphi_{C'}(p)
  = \frac{\Gamma(\gamma + \alpha k_{C'})}{\Gamma(\gamma + \alpha)\,\Gamma(\alpha)^{k_{C'} - 1}}\;
    p([M]_{C'})^{\gamma + \alpha - 1} \prod_{z \neq [M]_{C'}} p(z)^{\alpha - 1},
\]
where $k_{C'}$ is the number of vertices of $W_{C'}$ and $p(z)$ is the probability that the point $p \in W_{C'}$ assigns to the value $z$.  At the vertex $p = e_{[M]_{C'}}$ on which it concentrates, the coordinate $p([M]_{C'}) = 1$ contributes $1^{\gamma + \alpha - 1} = 1$, and each of the other $k_{C'} - 1$ coordinates contributes $0^{\alpha - 1}$.  The value there is therefore $+\infty$ for $\alpha < 1$, the normalising constant $\Gamma(\gamma + k_{C'})/\Gamma(\gamma + 1)$ at $\alpha = 1$, and $0$ for $\alpha > 1$.  Since $u^*$ has coordinate $e_{[M]_{C'}}$ at every $C' \neq C$, the constant $\varphi_{-C}(u^*)$ is a product of these vertex values.

\begin{remark}\label{rem:dirichlet_bayes}
For the susceptibility $\chi([M])$ to be \emph{about} $M$, and not a property of the entire model class $W$, we need to choose a \emph{localising prior} $\varphi(w)$. By \emph{localising}, we mean that the probability mass concentrates on points near $[M]$, which are noisy Turing machines that primarily act like $M$ with small probabilities of error. In practice this means the prior's log-gradient term in a gradient-based sampler should pull its updates towards $[M]$. The natural choice in \cite{baker2025studyingsmalllanguagemodels}, where the parameter space is $\mathbb R^d$, is a Gaussian centred at the trained model $w^* \in \mathbb R^d$, which has exactly this pull.  Our models are points on the boundary of $W$, a product of simplices, and a natural choice for distributions over a simplex is the Dirichlet distribution, and it is the prior family our sampler's update rule assumes (\cref{sec:empirical_sus}, \cref{app:grld}). In this remark, we justify that our $(\gamma, \alpha)$ parametrisation of these Dirichlet distributions produces localising priors when $\alpha \le 1$ and $\gamma + \alpha \ge 1$, and give a Bayesian interpretation of the prior at a given $(\alpha, \gamma)$ pair.

For $\alpha \le 1$ and $\gamma + \alpha \ge 1$, the gradient of the log density of each Dirichlet factor
has positive inner product with the direction to the vertex at every
interior point of the simplex:
\[
  \Bigl\langle \nabla_v \log \operatorname{Dir}\bigl(v;\, \gamma\, e_{[M]_C} + \alpha \mathbf 1\bigr),\;
  e_{[M]_C} - v \Bigr\rangle
  = \frac{\gamma + \alpha - 1}{v([M]_C)} - \bigl(\gamma + \lvert Z_C \rvert\,(\alpha - 1)\bigr)
  \ge (\lvert Z_C \rvert - 1)(1 - \alpha)
  \ge 0.
\]
The prior therefore pulls the gradient flow toward the classical value from
anywhere on the simplex.  For $\alpha > 1$ the pull is toward the interior
mode instead.  The log density is then strictly concave, so its gradient
has positive inner product with the direction to the mode from every interior
point, and the inner product above turns negative once $v([M]_C)$ exceeds
its value at the mode.

The localising prior \eqref{eq:dirichlet_localiser} is itself a Bayesian posterior.
A point $v \in \Delta Z_C$ is a categorical distribution, from which the
relaxed machine draws its samples at the description square $C$, and the
Dirichlet family is its conjugate prior: updating $\operatorname{Dir}(\mathbf c)$
on observed counts $\mathbf x$ gives $\operatorname{Dir}(\mathbf c + \mathbf x)$.
The factor $\varphi_C$ is therefore the posterior
of the base prior $\operatorname{Dir}(\alpha \mathbf 1_{Z_C})$ after observing
the classical value $[M]_C$ with multiplicity $\gamma$.  In other words, the
localiser observes the description tape directly, $\gamma$ times per square.
Under this posterior, the expected probability that the square takes its
classical value is $(\gamma + \alpha)/(\gamma + \alpha\,\lvert Z_C \rvert)$,
rising from $1/\lvert Z_C \rvert$ at $\gamma = 0$ towards $1$ as $\gamma$
grows.  This is the sense in which $\gamma$ sets the strength of the
localisation.

The base prior encodes a
belief about the squares themselves.  For $\alpha > 1$ it vanishes at the
vertices and peaks in the interior, the belief that every square is genuinely
stochastic.  At $\alpha = 1$ it is uniform and believes nothing.  For
$\alpha < 1$ it concentrates at the vertices, the belief that every square is
deterministic with no preference among its values.  The machines under study are deterministic, and stochastic codes enter only
through the smooth relaxation, so we work at $\alpha < 1$.

The experiments take $(\gamma, \alpha) = (1, 0.01)$,
so the localiser is the posterior after observing each description square
exactly once, starting from a base prior strongly concentrated near deterministic squares.
At $\gamma = 1$ the small $\alpha$ carries most of the
localisation: with $\lvert Z_C \rvert = 5$, the square takes its classical
value with probability $1/3$ at $(\gamma, \alpha) = (1, 1)$ but $0.96$ at
$(1, 0.01)$.  The limit
$\alpha \to 0$ is improper, with non-integrable factors $v(z)^{-1}$ at the
boundary, but $\alpha = 0.01$ is empirically indistinguishable from it.
\end{remark}

\begin{remark}[The localiser on a slice]\label{rem:slice_infinity}
By the vertex values above, the constant $\varphi_{-C}(u^*)$ is infinite for $\alpha < 1$.  At $\alpha = 1$ it is finite but enormous: each of the $|\mathcal C| - 1 = 14$ mismatched components contributes its normalising constant $\Gamma(\gamma + \lvert Q \rvert)/\Gamma(\gamma + 1)$, which at $\gamma = 1$ gives $\varphi_{-C}(u^*) = 120^{14} \approx 10^{29}$.  Note that $\varphi_{-C}(u^*)$ depends on neither the slice variable $v$ nor the input $x$.  In $p_C$ it cancels, as displayed above, and the cancelled form is a proper distribution for every $\alpha > 0$.  For $\alpha < 1$, where the uncancelled expression reads $\infty/\infty$, we take the cancelled form as the definition of $p_C$.  The update in the GRLD sampler (\cref{app:grld}) uses only $\nabla_v \log p_C(v)$, and the factor enters $\log p_C(v)$ as the additive constant $\log \varphi_{-C}(u^*)$, whose gradient is zero.  In the susceptibility the factor survives:
\[
  \chi_x^{C, L}
  = \frac{Z_C}{Z}\,\tilde\chi_x^C
  = \frac{\varphi_{-C}(u^*)\,\zeta_C}{Z}\,\tilde\chi_x^C.
\]
For $\alpha < 1$ this makes $Z_C$ infinite and $\chi_x^{C, L}$ ill-defined.  Everything we compute, however, lives on the cancelled side: the renormalised $\tilde\chi$ is the second factor above, and the estimator \eqref{eq:two_chain_estimator} replaces its $p_C$- and $p$-expectations by sample means.  Standardisation stays there too, dividing each centred column by its standard deviation, which removes any constant along the column.  The conclusions of \cref{thm:rank_bound,thm:sus_symmetry} survive as well, because they are statements about columns of the susceptibility matrix that are stable under positive column scalings (\cref{lem:standardisation_equivariance,prop:standardisation_rank}).  None of this care is needed for a prior that is finite and positive on slices, such as the localiser at $\alpha = 1$: there $Z_C$ and $\chi_x^{C, L}$ are finite, the cancelled form of $p_C$ is an ordinary division, and the equations above hold between real numbers.
\end{remark}

\begin{lemma}\label{lem:sol_closure}
Let $g = (a, s)$ be a recoding satisfying $y(g \cdot x) = g \cdot y(x)$ for all $x \in I$.  Then $[g \cdot M] \in W^{\mathrm{sol}}$ for every $[M] \in W^{\mathrm{sol}}$ \eqref{def:canonical_solutions}.
\end{lemma}

\begin{proof}
By induction on $t$, the run of $g \cdot M$ on $a(x)$ visits the pair $(a(x_t), s(q_t(x)))$ whenever the run of $M$ on $x$ visits $(x_t, q_t(x))$.  At $t = T$, the run of $g \cdot M$ on $a(x)$ therefore ends in the state $s(q_T(x)) = s(y(x)) = y(a(x))$, and $a$ permutes $I$, so $[g \cdot M] \in W^{\mathrm{sol}}_{\mathrm{all}}$ \eqref{def:classical_solutions_L_A0}.  The same induction gives $U(g \cdot M) = g \cdot U(M)$ for the used pairs of \cref{sec:expt_setup}.  Hence $(\sigma, q) \notin U(g \cdot M)$ if and only if $(a^{-1}(\sigma), s^{-1}(q)) \notin U(M)$, where \eqref{def:canonical_solutions} for $M$ gives $(g \cdot M)_Q(\sigma, q) = s\bigl(M_Q(a^{-1}(\sigma), s^{-1}(q))\bigr) = s(s^{-1}(q)) = q$.
\end{proof}

\begin{lemma}\label{lem:localised_equivariance}
Let $g$ be as in \cref{lem:sol_closure}, suppose in addition $q(g \cdot x) = q(x)$ for all $x \in I$, and assume the cycle map of the pseudo-UTM is recoding equivariant (\cref{def:recoding_equivariance}).  Then the susceptibility map is equivariant: for every classical solution $M$,
\[
  \chi\bigl([g \cdot M]\bigr) = P_g\, \chi\bigl([M]\bigr),
\]
with $P_g$ the recoding operator of \cref{sec:expt_pointwise}.  Together with \cref{lem:sol_closure}, the covariance operator $\Gamma$ of \cref{sec:expt_setup} therefore satisfies $P_g \Gamma P_g^{-1} = \Gamma$.  \Cref{thm:sus_symmetry} is the case $g \cdot M = M$, where $\chi([M])$ is itself $P_g$-invariant.
\end{lemma}

\begin{proof}
By \cref{lem:model_equivariance}, the target hypothesis gives $h_x(g \cdot w) = h_{a^{-1}(x)}(w)$ for every $w \in W$.  Hence $\ell_x(g \cdot w) = \ell_{a^{-1}(x)}(w)$, and $L(g \cdot w) = L(w)$ by the invariance of $q$, as in \cref{prop:H_invariance}.  For the localiser, \cref{rem:action_compat} identifies the factor of $g \cdot w$ at $C$ with the pushforward of $w_{g^{-1} \cdot C}$ along the value map, a bijection of vertices carrying $e_{[M]_{g^{-1} \cdot C}}$ to $e_{[g \cdot M]_C}$.  Since the symmetric part $\alpha \mathbf 1$ of the concentration vector is fixed by any permutation,
\[
  \operatorname{Dir}\bigl((g \cdot w)_C;\; \gamma\, e_{[g \cdot M]_C} + \alpha \mathbf 1\bigr)
  = \operatorname{Dir}\bigl(w_{g^{-1} \cdot C};\; \gamma\, e_{[M]_{g^{-1} \cdot C}} + \alpha \mathbf 1\bigr),
\]
and taking the product over $\mathcal C$ gives $\varphi_{[g \cdot M], \gamma, \alpha}(g \cdot w) = \varphi_{[M], \gamma, \alpha}(w)$.  The change of variables $w \mapsto g \cdot w$ therefore carries the local Gibbs distribution at $[M]$ to the one at $[g \cdot M]$ and matches the observables, so the argument in the proof of \cref{thm:sus_symmetry} applies verbatim and gives $\chi_{g \cdot x}^{g \cdot C, L}([g \cdot M]) = \chi_x^{C, L}([M])$ for all $x$ and $C$, the entrywise form of the displayed identity.  Finally, $M \mapsto g \cdot M$ is a bijection of $W^{\mathrm{sol}}$ (\cref{lem:sol_closure} and injectivity of the action on the finite $W^{\mathrm{sol}}$), so reindexing the sums defining $\bar\chi$ and $\Gamma$ in \cref{sec:expt_setup} gives $P_g \bar\chi = \bar\chi$ and $P_g \Gamma P_g^{-1} = \Gamma$.
\end{proof}

By \cref{cor:target_compat}, the target hypothesis of \cref{lem:sol_closure,lem:localised_equivariance} holds whenever $g$ is a symmetry of some classical solution.  Both recodings of \cref{sec:recoding_symmetries} are symmetries of $M_5$, so the lemmas apply to them.

\begin{lemma}\label{lem:standardisation_equivariance}
Under the hypotheses of \cref{thm:sus_symmetry}, the component partition functions of \cref{sec:empirical_sus} satisfy $Z_C = Z_{g \cdot C}$, and columnwise standardisation \eqref{eq:standardised_sus} commutes with the recoding operator $P_g$ of \cref{sec:expt_pointwise}.  The equivariance of \cref{thm:sus_symmetry} therefore passes from $\chi$ to the renormalised $\tilde\chi$ and the standardised $\psi$.
\end{lemma}

\begin{proof}
The action of $g$ maps the slice along $C$ onto the slice along $g \cdot C$, preserving $L$, the localiser (\cref{lem:localised_equivariance}), and the measure.  The two slice integrals therefore agree, and $Z_C = Z_{g \cdot C}$.  Hence the renormalisation factor $Z/Z_C$ is constant on orbits of the component action, and $\tilde\chi$ satisfies the same equivariance as $\chi$.  By \cref{thm:sus_symmetry} the column of $g \cdot C$ is the $a$-reindexing of the column of $C$, so the two columns have the same mean and standard deviation.  Columnwise standardisation \eqref{eq:standardised_sus} therefore commutes with $P_g$, given by $(P_g\chi)^C_x := \chi^{g^{-1} \cdot C}_{g^{-1} \cdot x}$, and the standardised matrix satisfies the equivariance if and only if the raw one does.
\end{proof}

\begin{proposition}\label{prop:standardisation_rank}
The rank bound of \cref{thm:rank_bound} also holds for the renormalised and standardised matrices $\tilde\chi$ and $\psi$, and, up to floating-point rounding, for the estimates $\widehat{\tilde\chi}$ and $\widehat{\psi}$.
\end{proposition}

\begin{proof}
Consider $X_{\acc, R}$ (the argument for $X_{\rej, A}$ is identical).  By the proof of \cref{thm:rank_bound}, the columns $\chi^C|_{I_{\acc}}$ of the block lie in the two-dimensional space $V := \operatorname{span}\{(\langle h_x^2 \rangle)_{x \in I_{\acc}},\, \one_{I_{\acc}}\}$.  Renormalisation \eqref{eq:renorm_sus} scales the column of $C$ by $Z/Z_C$.  Standardisation \eqref{eq:standardised_sus} subtracts from $\tilde\chi^C$ its mean $m_C$ over $I$ and divides by its standard deviation $s_C$ over $I$.  On the rows $I_{\acc}$, then,
\[
  \tilde\chi^C|_{I_{\acc}}
  = \frac{Z}{Z_C}\,\chi^C|_{I_{\acc}},
  \qquad
  \psi^C|_{I_{\acc}}
  = \frac{1}{s_C}\,\tilde\chi^C|_{I_{\acc}} - \frac{m_C}{s_C}\,\one_{I_{\acc}}.
\]
Both right-hand sides are linear combinations of $\chi^C|_{I_{\acc}}$ and $\one_{I_{\acc}}$, so the columns of both blocks lie in $V$ and their ranks are $\le 2$.  Now consider the estimated block $\widehat{\tilde\chi}|_{I_{\acc} \times \mathcal C_R}$.  It is close to the true block, so its numerical rank (\cref{sec:expt_pointwise}) satisfies the bound whenever the tolerance exceeds the sampling error.  More is true: the bound holds exactly, at any number of samples.  For $x \in I_{\acc}$ and $C \in \mathcal C_R$ the estimator \eqref{eq:two_chain_estimator} reads
\[
  \widehat{\tilde\chi}^{\,C}_x
  = -\frac{1}{r} \sum_{t=1}^{r} L(u^*, v_t)\, \bigl[ \ell_x(u^*, v_t) - L(u^*, v_t) \bigr]
  + \Bigl( \frac{1}{r} \sum_{t=1}^{r} L(u^*, v_t) \Bigr) \cdot \Bigl( \frac{1}{r'} \sum_{s=1}^{r'} \bigl[ \ell_x(w_s) - L(w_s) \bigr] \Bigr).
\]
By \cref{lem:mismatched_slice}(i), $h_x(u^*, v_t) = 0$ at every restricted sample, so $\ell_x(u^*, v_t)$ takes the same value for every $x \in I_{\acc}$.  The first term is therefore a number $c_C$ independent of $x$.  The second is the product $a_C\, b_x$ of $a_C := \frac{1}{r} \sum_{t} L(u^*, v_t)$ and $b_x := \frac{1}{r'} \sum_{s} \bigl[ \ell_x(w_s) - L(w_s) \bigr]$.  So $\widehat{\tilde\chi}^{\,C}_x = c_C + a_C\, b_x$, and the columns of the estimated block lie in the two-dimensional span of $(b_x)_{x \in I_{\acc}}$ and $\one_{I_{\acc}}$, up to floating-point rounding.  The columns of $\widehat{\psi}$ lie there too by the computation above.
\end{proof}

\subsection{The GRLD sampler}\label{app:grld}

The chains of \cref{sec:empirical_sus} are driven by \emph{gradient
Riemannian Langevin dynamics} (GRLD), a full-gradient variant of the
stochastic gradient Riemannian Langevin dynamics (SGRLD) of Patterson and
Teh \cite{patterson2013sgrld}, targeting the local Gibbs distribution
\eqref{eq:local_gibbs} on the product of simplices $W$.

Each factor $W_C = \Delta Z_C$ is parametrised by the \emph{expanded-mean}
coordinates of \cite{patterson2013sgrld}: an unnormalised vector
$\theta_C \in \mathbb R_{>0}^{Z_C}$ with
$w_C = \theta_C / \lVert \theta_C \rVert_1$.  Write
$\mathbf c_C := \gamma\, e_{[M]_C} + \alpha\, \mathbf 1_{Z_C}$ for the
concentration vector of the localiser factor at $C$
\eqref{eq:dirichlet_localiser}.  One GRLD step of step size
$\varepsilon$ updates every free factor by
\begin{equation}\label{eq:grld_update}
  \theta_C \longleftarrow
  \operatorname{clip}\Bigl(\,\bigl\lvert\,
  \theta_C
  + \tfrac{\varepsilon}{2}\bigl[(\mathbf c_C - \theta_C)
  + \beta\, g_C(w)\bigr]
  + \sqrt{\varepsilon}\,\sqrt{\theta_C} \odot \eta_C
  \,\bigr\rvert,\; \theta_{\min},\, \theta_{\max}\Bigr),
  \qquad \eta_C \sim \mathcal N(0, I),
\end{equation}
where $\odot$, $\sqrt{\phantom{\theta}}$, and
$\lvert\,\cdot\,\rvert$ act entrywise, and the \emph{data drift} $g_C$ is
the natural gradient of the loss, projected onto the tangent space of the
simplex:
\begin{equation}\label{eq:grld_natgrad}
  g_C(w) := w_C \odot \bigl(-\nabla_{w_C} L\bigr)
  - \bigl\langle w_C,\, -\nabla_{w_C} L \bigr\rangle\, w_C,
  \qquad \textstyle\sum_{z \in Z_C} g_C(w)_z = 0.
\end{equation}
The state of the chain is the unnormalised $\theta_C$, which is never
rescaled.  The simplex point $w_C = \theta_C / \lVert \theta_C \rVert_1$
is recomputed from it at each step, both for the gradient evaluation
and for the recorded samples.
This is the expanded-mean SGRLD update of \cite{patterson2013sgrld}: the
term $\mathbf c_C - \theta_C$ is the stationary drift of the Riemannian
diffusion $d\theta = (\mathbf c - \theta)\,dt + \sqrt{2\theta}\, dW$ for
the Dirichlet factor (the score of the density plus the It\^o correction
of the multiplicative noise), the mirroring $\lvert\,\cdot\,\rvert$ is
their reflection at the boundary of the positive orthant, and, since $L$
depends on $\theta_C$ only through $w_C$, the data drift
\eqref{eq:grld_natgrad} coincides with their Riemannian gradient
$\operatorname{diag}(\theta_C)\, \nabla_{\theta_C}(-L)$.

GRLD differs from \cite{patterson2013sgrld} in four ways.  (i) The
gradient is exact: the input set $I$ is finite and $L$ is evaluated in
closed form (the model probabilities are polynomial in $w$ and
$\ell_x = -\log p$ is computed exactly), so there is no minibatching and no
stochastic gradient, hence the name.  (ii) The target is the tempered, localised
Gibbs distribution \eqref{eq:local_gibbs} rather than a conjugate
Bayesian posterior, though the prior is Dirichlet in both.  (iii) The step size is constant rather than following a
decreasing schedule.  As in \cite{patterson2013sgrld} there is no
Metropolis correction, so the chain carries an $O(\varepsilon)$
discretisation bias, quantified by the NUTS calibration below.  (iv) The
clip to $[\theta_{\min}, \theta_{\max}] = [10^{-12}, 10^{30}]$ is a
numerical guard for single-precision arithmetic near the boundary of the
simplex, where the $\alpha < 1$ localiser diverges.

The sampler is applied componentwise, realising the estimator
of \cref{sec:empirical_sus}: for each component $C$, chains targeting
$p_C$ apply \eqref{eq:grld_update} to $\theta_C$ alone, holding every
other factor fixed at the corresponding vertex of $[M]$, and chains
targeting $p$ update all factors.  The susceptibility matrix is
then assembled from the sample means via
\eqref{eq:two_chain_estimator}.

\paragraph{Susceptibility runs.}
All susceptibility runs use $4$ chains of $3000$ sampling steps after $500$
burn-in at step size $\varepsilon = 0.01$ (seed $42$), at the base
parameters $\beta = 30$ and $(\gamma, \alpha) = (1, 0.01)$ of
\cref{sec:expt_setup}.

\paragraph{Sampler calibration.}
The sampler was calibrated on a single DFA.  Convergence is measured by
the CKA between the running kernel estimate and the final $3000$-sample
estimate, judged against the scale of between-machine similarities: the largest
between-machine CKA in a $2{,}216$-machine reference set is $0.995$,
so self-similarity $0.999$ identifies a machine with margin.  At $\alpha = 1$,
NUTS \cite{hoffman2014nuts} reaches CKA $0.999$ in $200$--$800$ samples where
GRLD needs the full $3000$, but costs $236$--$373$s per machine
against $31$s.  At $\alpha < 1$ the boundary singularity of the Dirichlet
factor forces NUTS to thousands of leapfrog steps per iteration ($3503$s at
$\alpha = 0.01$), while the reflecting discretisation
of GRLD is insensitive to the localiser.  At $\alpha = 1$ the two
samplers' final kernel estimates agree to CKA $\ge 0.988$, across
$\gamma \in \{0, 1, 2, 10\}$.  At $\alpha < 1$ the agreement degrades
($0.82$ at $\alpha = 0.1$, $0.49$ at $\alpha = 0.01$), so the
cross-check is informative only at $\alpha = 1$.

\subsection{Embedding and compute}\label{app:embedding_compute}

\paragraph{Embedding parameters.}
The input-kernel embeddings of
\cref{sec:expt_population,sec:expt_reading} use UMAP with
\texttt{n\_neighbors} $= 15$ and \texttt{min\_dist} $= 0.1$, with the
cosine metric on the flattened centred, normalised input kernels, whose
cosine similarity is exactly the CKA, so the effective dissimilarity is
$1 - \mathrm{CKA}$.

\paragraph{Compute.}
Each temperature slice of \cref{sec:expt_beta_sweep} is one
sampling job over the full solution set ($38{,}019$ machines at the sampler
settings of \cref{app:grld}).  The $35$ jobs took a median of $3.3$ hours each
on a single RTX~4090 ($3.2$ to $9.7$ hours, the slowest at the
divergence-prone low-$\beta$ end), summing to about $145$ GPU-hours for the sweep.
The base-parameter susceptibility runs of \cref{sec:expt_setup} are jobs of
the same kind.  Each aligned UMAP fit over the $35$ slices took one to
two hours on a CPU workstation.

\subsection{Input kernels of the running examples}\label{app:examples_kernels}

Two machines are compared through the input kernels
\eqref{eq:input_kernel} of their susceptibility matrices
(\cref{fig:examples_kernels}).  The centred kernel alignment
$\mathrm{CKA}(X, Y)$ \eqref{eq:cka} is the cosine similarity of $K(X)$ and
$K(Y)$, near $1$ when the same pairs of inputs respond alike and near $0$
when they do not.  \Cref{fig:examples_cka} shows it on the five running
examples, where $M_1$ aligns with $M_2$ and $M_3$ with $M_4$ (both
$\approx 0.9$) while the two pairs are unlike ($0.4$--$0.6$), separating the
path-separable $M_1, M_2$ from the non-separable $M_3, M_4$.  Scaling this
to the whole solution set gives the input-kernel embeddings of
\cref{sec:expt_population}: all $38{,}019$ machines are embedded in the
plane by UMAP \cite{mcinnes2018umap} on
$1 - \mathrm{CKA}(\widehat{\tilde\chi}_M, \widehat{\tilde\chi}_N)$.

\begin{figure}[H]
\centering
\includegraphics[width=\linewidth]{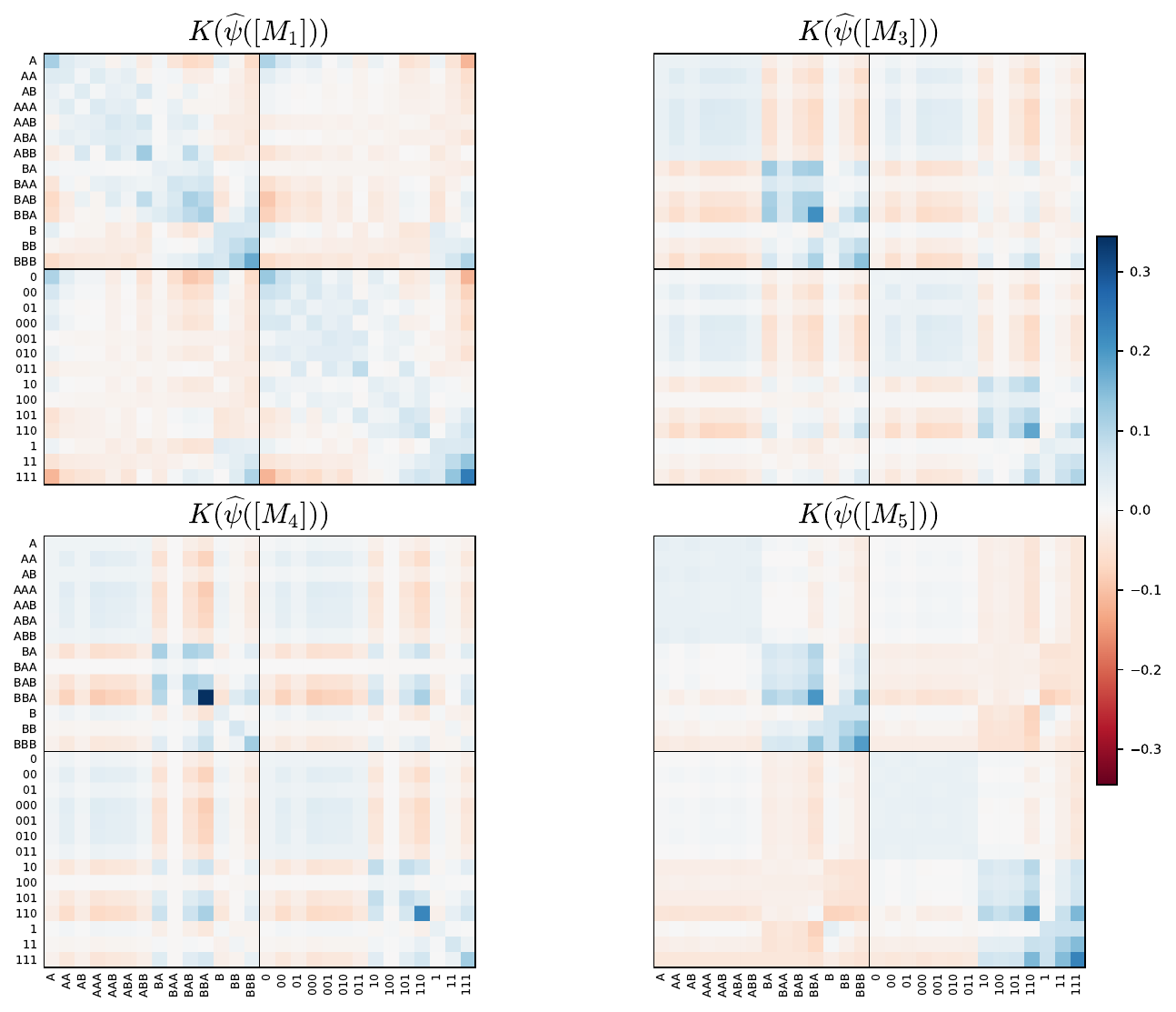}
\caption{\textbf{Input kernels of $M_1, M_3, M_4, M_5$}. For inputs $x, y \in \Sigma^*$, the value at $K(\widehat{\psi}([M_i]))_{x,y}$ is proportional to the cosine-similarity between rows $\widehat{\psi}([M_i])_x$ and $\widehat{\psi}([M_i])_y$, with a strong positive value indicating similar response-vectors. Visual similarity between two kernels indicates greater centred-kernel-alignment between susceptibilities as representations of $I$.}
\label{fig:examples_kernels}
\end{figure}

\begin{figure}[H]
\centering
\includegraphics[width=0.62\linewidth]{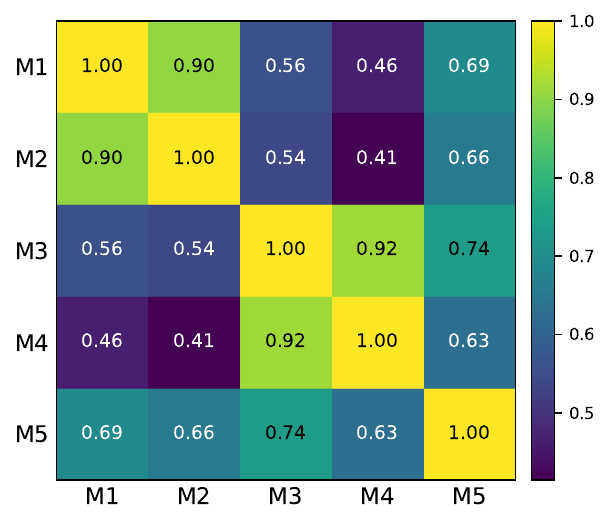}
\caption{\textbf{Pairwise CKA of $\widehat{\psi}$ on the DFAs
$M_1$--$M_5$}.}
\label{fig:examples_cka}
\end{figure}

\subsection{Eigenmatrix parities}\label{app:sd_spectrum}

\Cref{fig:sd_spectrum} shows the symmetry defects of the leading $32$
of the $420$ eigenmatrices of $\Gamma$.  Each eigenmatrix $v_k$ has
two dots: its defect at $(\theta, \id_Q)$ and its defect at
$(\theta, s_1{\leftrightarrow}s_2)$.  A defect of $0$ means the recoding
fixes $v_k$, and a defect of $1$ means it negates $v_k$.  Both dots at $0$
means $v_k$ is fixed by both recodings, as for $v_1$.  Both dots at $1$
means $v_k$ is negated by both, as for $v_7$.  One dot at $0$ and one
at $1$ means $v_k$ is negated by exactly one of the two recodings, as
for $v_2$ and $v_3$.  A dot strictly between $0$ and $1$ means the axis
has no definite parity.  This occurs from the ninth eigenmatrix on.  The
bottom panel gives the fraction of the variance carried by each axis.

\begin{figure}[H]
\centering
\includegraphics[width=\linewidth]{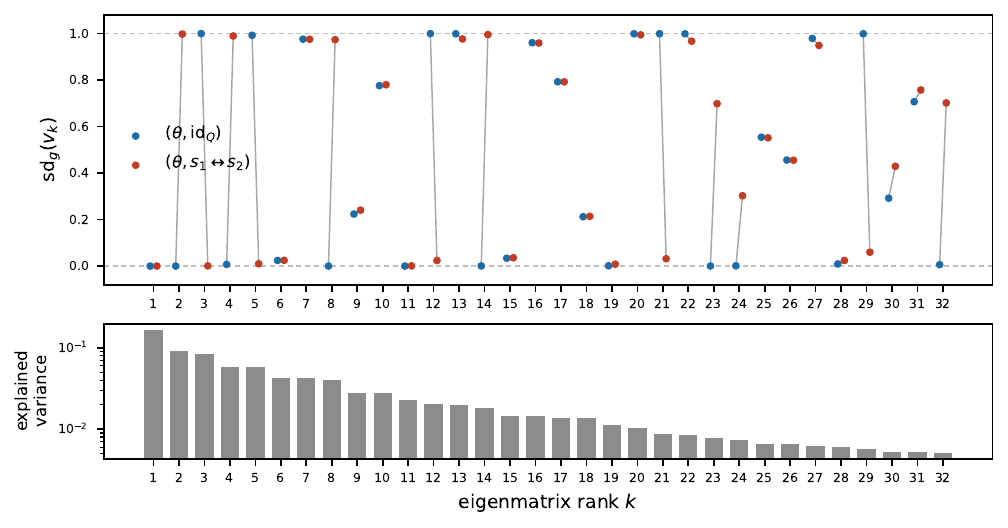}
\caption{\textbf{Parities of the leading eigenmatrices.}  Top: the
symmetry defect \eqref{eq:sym_defect} of the $k$-th eigenmatrix of
$\Gamma$ at both recodings.  A defect of $0$ means the recoding fixes
the eigenmatrix (even), and $1$ that it negates it (odd).  Bottom: the
fraction of the variance carried by each axis (log scale).  The first
eight axes have definite parity.}
\label{fig:sd_spectrum}
\end{figure}

\subsection{Cluster validation}\label{sec:expt_cluster_validation}

Planar embedding algorithms can manufacture artificial structure, so we
test whether the $\PSV_{\min}$ label, binary ($= 0$ versus $> 0$)
or six-class (which of its six values on this solution set), can be
read off the susceptibility data alone.  Each machine is represented as
$\widehat{\tilde\chi}$ or $\widehat{\psi}$, flattened or through
its input kernel $K(\cdot)$, and a linear SVM (class-balanced,
five-fold cross-validation) is scored by class-balanced accuracy on
held-out machines.  \Cref{fig:cluster_svm} shows the result: the
binary label is recovered essentially perfectly in every representation
(on the standardised matrices all $38{,}019$ held-out predictions are
correct, against $0.50$ for label-shuffled controls), and the six-class
label at $0.99$ on the standardised matrices, every error mistaking the
$\PSV_{\min}$ value for an adjacent one.  This gives us confidence in the
clustering by colour of \cref{fig:umap_main}: the labels can be read
off the susceptibility data directly, without the embedding.

\begin{figure}[H]
\centering
\includegraphics[width=\linewidth]{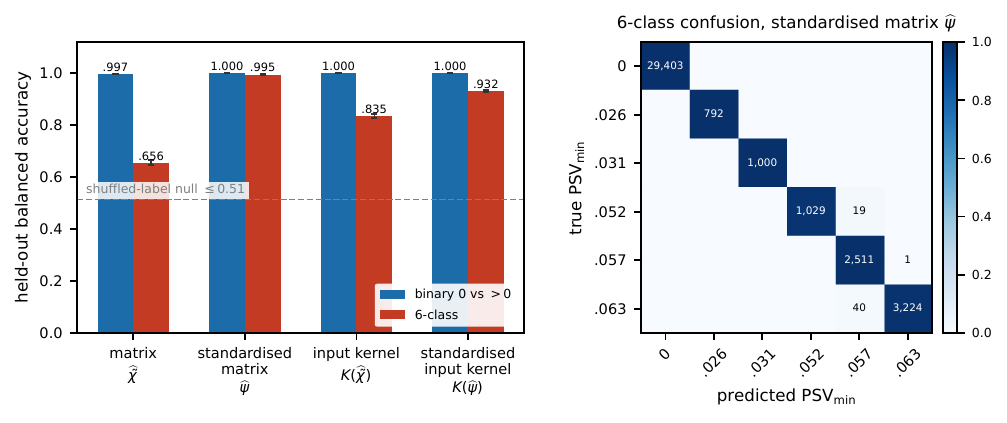}
\caption{\textbf{The $\PSV_{\min}$ classes are linear readouts of the
susceptibility matrix.}
Left: class-balanced accuracy, on held-out machines, of the linear
classifiers of \cref{sec:expt_cluster_validation} recovering the binary and
six-class $\PSV_{\min}$ labels in the four representations.  Error bars are
$95\%$ bootstrap intervals over machines, and the dashed line is the best of
ten label-shuffled controls.  Right: the errors of the six-class model
on the standardised matrices: entry $(i, j)$ is the fraction of the
machines with true label $i$ that are predicted as $j$, so a perfect
model is the identity matrix.}
\label{fig:cluster_svm}
\end{figure}

\subsection{The weights of the binary classifier}\label{app:svm_weights}
\Cref{fig:svm_weights} shows the weights of the binary
($\PSV_{\min} = 0$ versus $> 0$) classifier of
\cref{sec:expt_cluster_validation} on the standardised matrices, refitted
once on the full solution set (the per-fold fits agree with this one to cosine
similarity $\ge 0.98$) and reshaped to the input $\times$ component grid of
$\widehat{\tilde\chi}$.  Measured by squared weight, the classifier
concentrates on the rejecting inputs ($43\%$ of the weight on $21\%$ of the
entries), and above all on the block
$I_{\rej} \times \mathcal C_{\{s_1, s_2\}}$ ($37\%$ of the weight on $14\%$
of the entries), the off-diagonal block $X_{\rej, A}$ of
\cref{thm:rank_bound} at the partition $A = \{\qacc, s_1, s_2\}$
that organises the solution set (\cref{sec:expt_population}).  The
$\{\texttt{A},\texttt{B}\}$-class and $\{\texttt{0},\texttt{1}\}$-class rows show the same pattern: the
solution set is closed under the alphabet involution (\cref{lem:sol_closure}) and the labels are
invariant under it, so the learned readout is approximately equivariant.

\begin{figure}[H]
\centering
\includegraphics[width=0.6\linewidth]{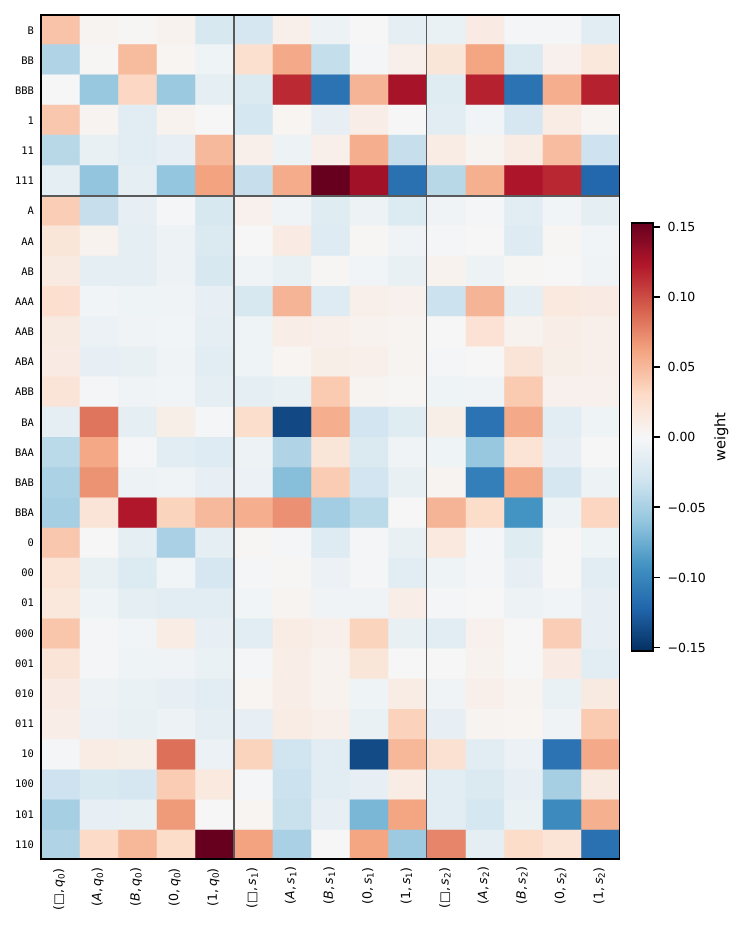}
\caption{Weights of the binary $\PSV_{\min}$ classifier on the standardised
matrices, refitted on the full solution set and reshaped to the input $\times$
component grid (rows grouped rejecting/accepting, columns by state).}
\label{fig:svm_weights}
\end{figure}

\subsection{The temperature sweep}\label{app:beta_sweep}\label{sec:expt_beta_sweep}

We recompute the susceptibilities (renormalised values $\widehat{\tilde \chi}$) of the full solution set at $35$ values
of the inverse temperature $\beta \in [0, 1000]$, one sampling run per
value.  The $35$
checkpoints are embedded jointly by aligned UMAP \cite{mcinnes2018umap}, so that
a machine's position is comparable across $\beta$. \Cref{fig:beta_sweep} shows nine of the $35$ slices, coloured by $\PSV_{\min}$.
Early on, we notice there are many blue points ($\PSV_{\min} > 0$, but still low) mixed in with the cluster of purple points ($\PSV_{\min}= 0$), and across $\beta \in [0, 17]$ these points slowly separate out from the purple cluster. As $\beta$ increases beyond $17$, we see the two blue clusters separate from the rest of the points on the right. Finally, as we sweep $\beta$ between $10$ and $100$, we see the lower half of the purple cluster separate out into a tightly packed disk, while the upper half becomes much more dispersed.

\paragraph{Aligned UMAP.}
The temperature sweeps of \cref{sec:expt_beta_sweep,sec:expt_reading}
are embedded with \texttt{AlignedUMAP} \cite{mcinnes2018umap}: the $35$
slices are optimised jointly, at the parameters above, with a relation
regularisation $\lambda$ penalising the displacement of each machine's
position between consecutive slices (window size $2$). The strength $\lambda = 5 \times 10^{-3}$
was selected by fitting the full sweep at
$\lambda \in \{10^{-3}, 5 \times 10^{-3}, 10^{-2}\}$, and noting that $\lambda = 5 \times 10^{-3}$ was sufficiently weak that the clustering was nearly unaffected in comparison to an unaligned UMAP at $\beta=30$, while still preventing large-scale rotations.

\begin{figure}[H]
\centering
\umappanel{$\beta = 0$}{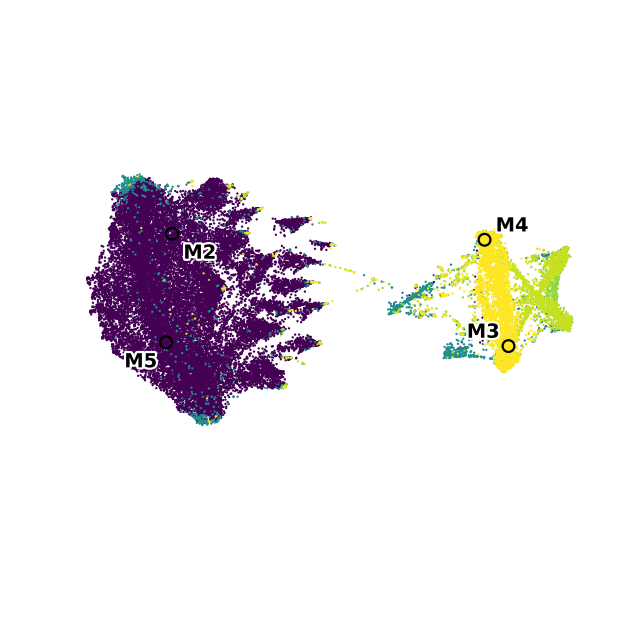}\hfill
\umappanel{$\beta = 1$}{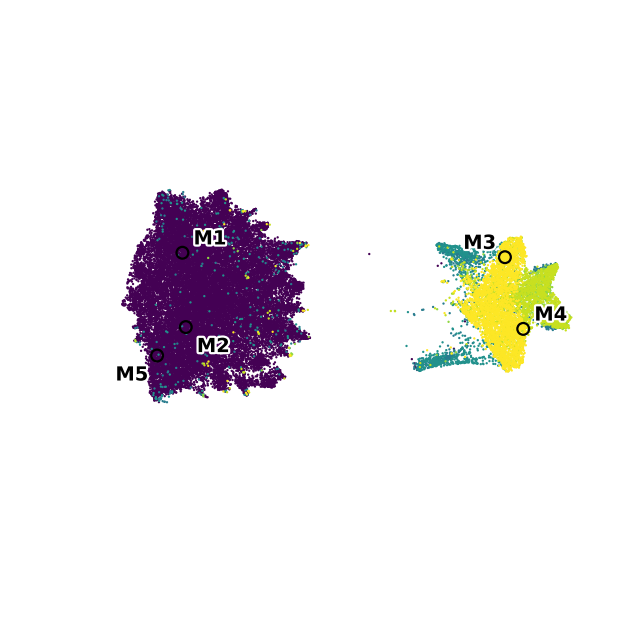}\hfill
\umappanel{$\beta = 3.2$}{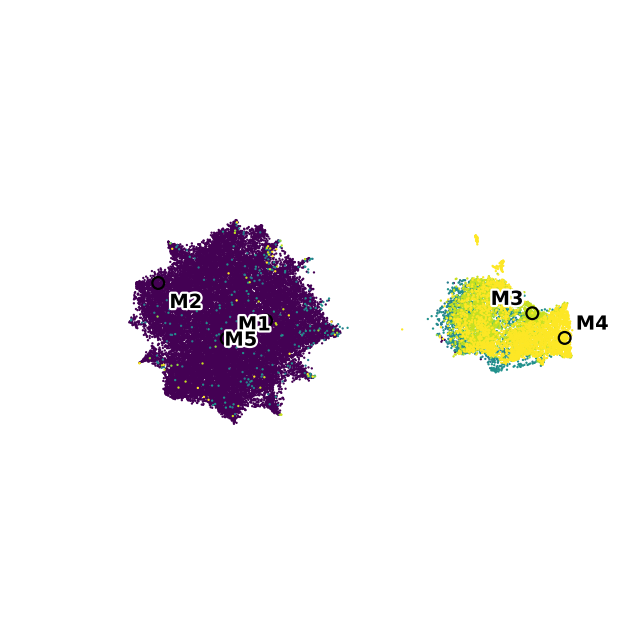}\\[10pt]
\umappanel{$\beta = 5.6$}{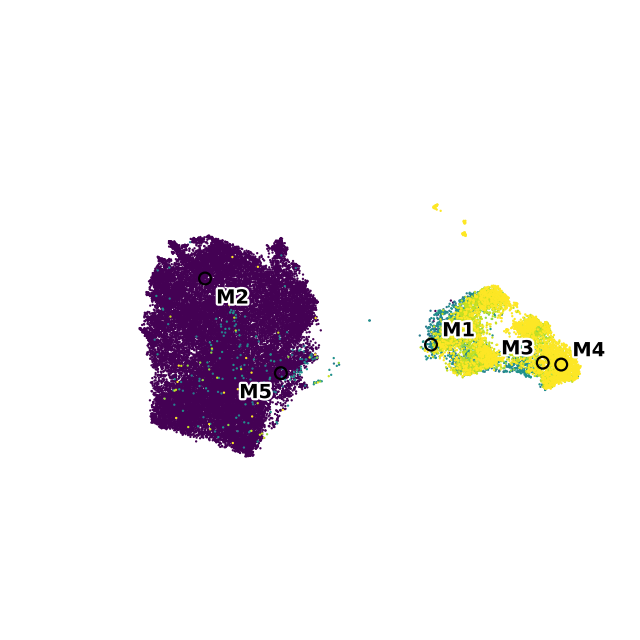}\hfill
\umappanel{$\beta = 10$}{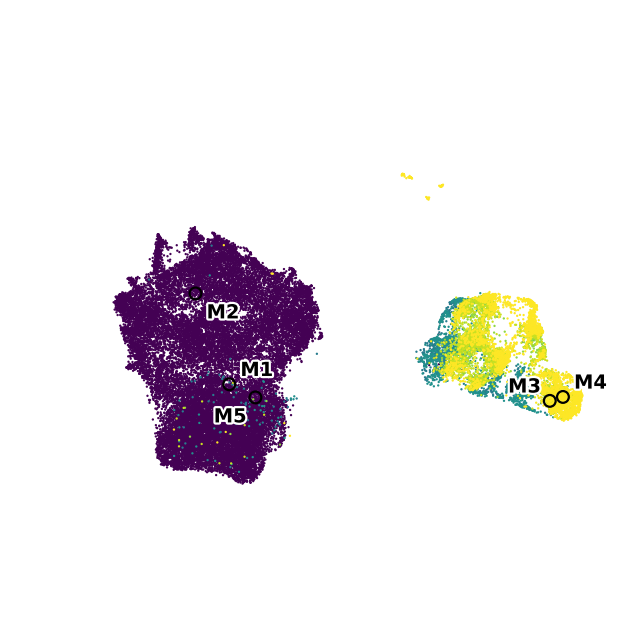}\hfill
\umappanel{$\beta = 17$}{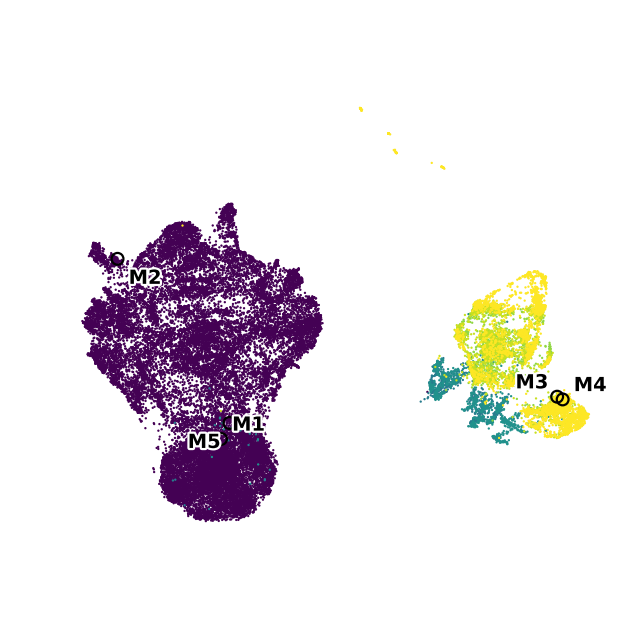}\\[10pt]
\umappanel{$\beta = 30$}{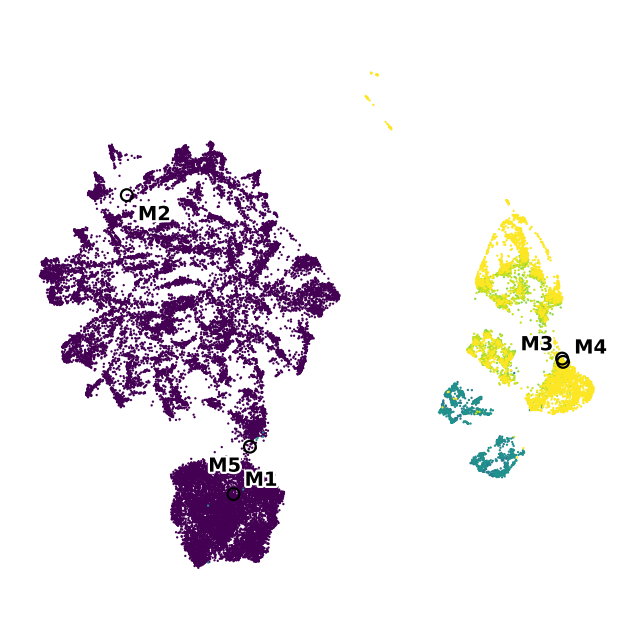}\hfill
\umappanel{$\beta = 100$}{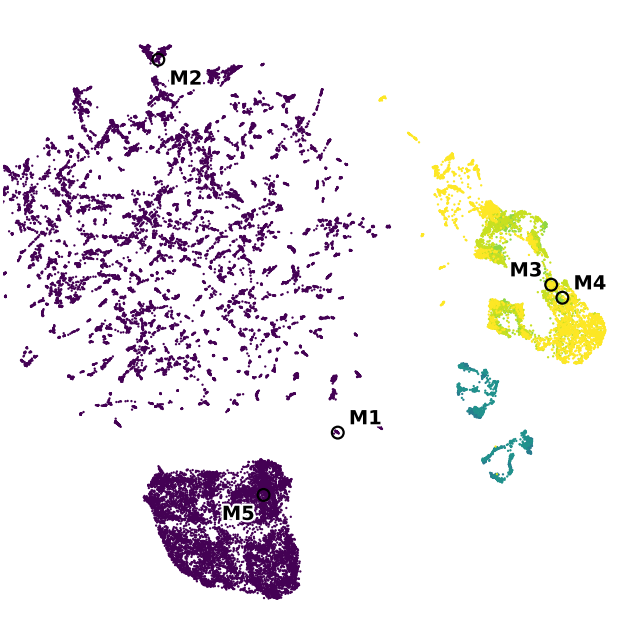}\hfill
\umappanel{$\beta = 1000$}{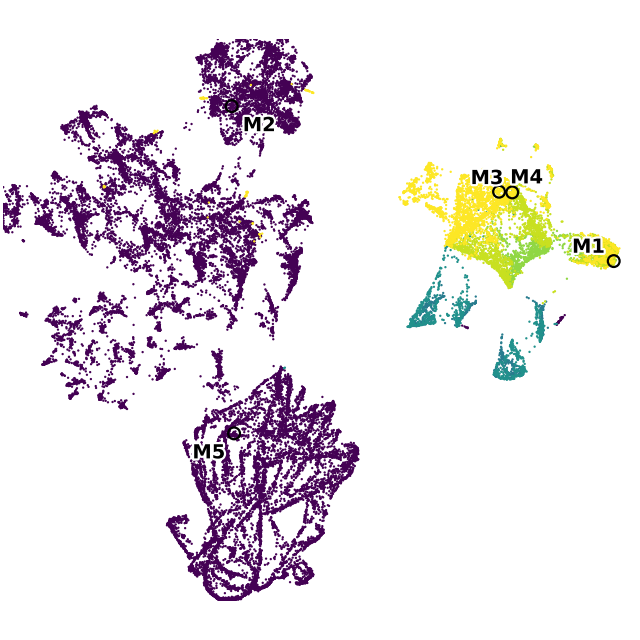}
\caption{\textbf{Aligned UMAP embedding of DFA susceptibilities over an inverse temperature sweep $\beta \in [0, 1000]$ coloured by $\PSV_{\min}$.}
}
\label{fig:beta_sweep}
\end{figure}

\subsection{Halting-time organisation}\label{app:halting_time}

For a classical solution $M$ and input $x$, the \emph{halting time}
$t_M(x)$ is the first step at which the run of $M$ on $x$ reaches the
target state $y(x)$.  The terminal states absorb
(\cref{sec:expt_task}), so the run remains there.  Write
$t_{\texttt{0}} := t_M(\texttt{0})$ and $t_{\texttt{A}} := t_M(\texttt{A})$ for the halting times of the two
single-symbol accepting inputs.  Both lie in $\{1, 2, 3\}$. Since these are immediately
recognisable as accepted strings from the first time step, their halting time is a proxy for how long the machine delays reaching the accept state by traversing intermediate states.
\Cref{fig:read_lat} shows an aligned embedding, from the CKA of
the standardised matrices $\widehat{\psi}$, coloured
by the pair $(t_{\texttt{0}}, t_{\texttt{A}})$.  This organisation is present from the outset
and stable: the fraction of a machine's $20$ nearest neighbours sharing
its $(t_{\texttt{0}}, t_{\texttt{A}})$ class is
$0.89$ at $\beta = 1$, $0.88$ at $\beta = 67$, and $0.90$ at
$\beta = 610$, against a largest-class share of $0.44$.  The class
counts are symmetric in $(t_{\texttt{0}}, t_{\texttt{A}})$ (legend): the alphabet involution
exchanges the $(i, j)$ and $(j, i)$ cells.

\begin{figure}[H]
\centering
\umappanel{$\beta = 1$}{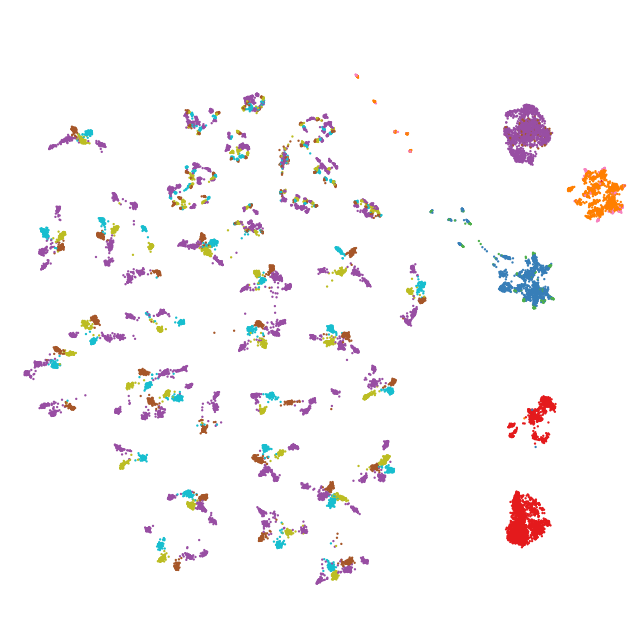}\hfill
\umappanel{$\beta = 67$}{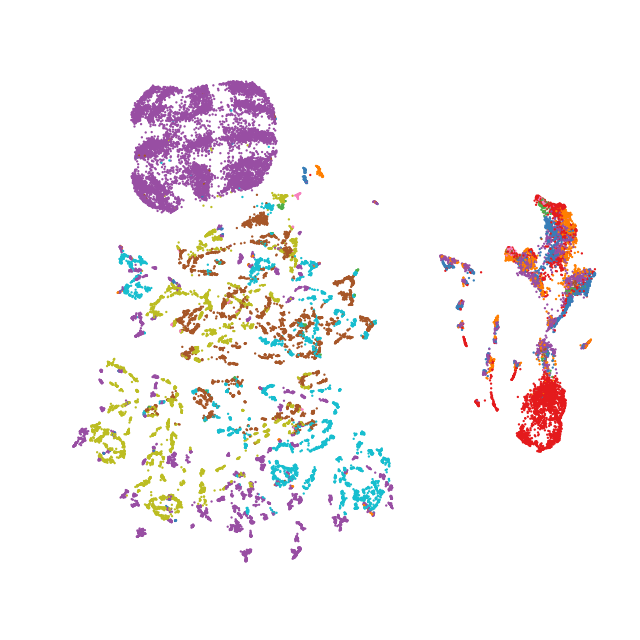}\hfill
\umappanel{$\beta = 610$}{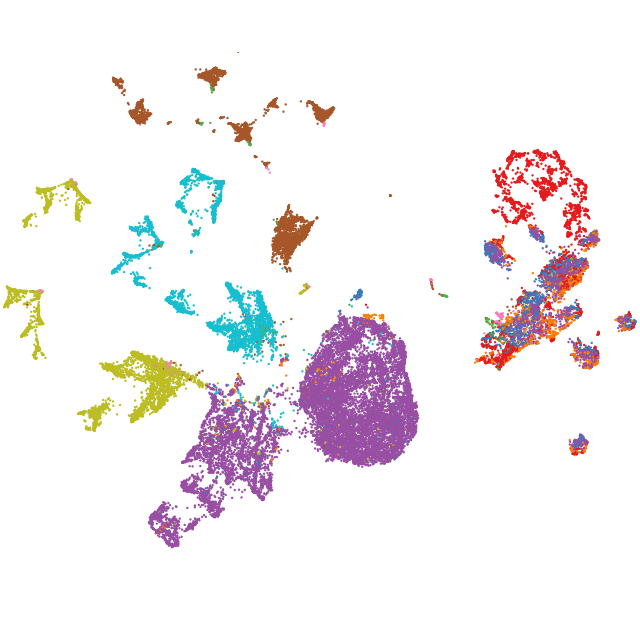}\\[10pt]
{\small\setlength{\tabcolsep}{5pt}\renewcommand{\arraystretch}{1.12}
\begin{tabular}{@{}c ccc@{}}
 & $t_{\texttt{A}} = 1$ & $t_{\texttt{A}} = 2$ & $t_{\texttt{A}} = 3$ \\
$t_{\texttt{0}} = 1$ & \textcolor[HTML]{E41A1C}{\rule{1.5ex}{1.5ex}} $3{,}243$ & \textcolor[HTML]{377EB8}{\rule{1.5ex}{1.5ex}} $1{,}572$ & \textcolor[HTML]{4DAF4A}{\rule{1.5ex}{1.5ex}} $258$ \\
$t_{\texttt{0}} = 2$ & \textcolor[HTML]{FF7F00}{\rule{1.5ex}{1.5ex}} $1{,}572$ & \textcolor[HTML]{984EA3}{\rule{1.5ex}{1.5ex}} $16{,}582$ & \textcolor[HTML]{17BECF}{\rule{1.5ex}{1.5ex}} $4{,}802$ \\
$t_{\texttt{0}} = 3$ & \textcolor[HTML]{F781BF}{\rule{1.5ex}{1.5ex}} $258$ & \textcolor[HTML]{BCBD22}{\rule{1.5ex}{1.5ex}} $4{,}802$ & \textcolor[HTML]{A65628}{\rule{1.5ex}{1.5ex}} $4{,}930$ \\
\end{tabular}}

\caption{\textbf{Halting-time clusters are present in susceptibility space across $\beta$.}  The embedding at three
values of $\beta$, coloured by the halting-time pair $(t_{\texttt{0}}, t_{\texttt{A}})$.  The
legend gives the class colours and machine counts.}
\label{fig:read_lat}
\end{figure}

\section{Minimal theoretical example: What is the simplest thing susceptibilities see that the Hessian does not?}\label{app:absorbing_example}

This appendix works out a minimal example in which susceptibilities see structure that the Hessian does not: the reduction of the smooth cycle map for noisy DFAs, the exact error probabilities and Hessian blindness, the asymptotics of the Gibbs moments underlying the susceptibility values, and an extension with noisy writes illustrating flat directions.

Throughout this appendix we drop the component superscript from $\chi^{C_w}_{\texttt{A}^n}$: because there is only one noisy parameter, the single component $C_w$ coincides with the full parameter space, and we write $\chi_{\texttt{A}^n}$ for the resulting susceptibility.

\subsection{The noisy absorbing DFA}\label{subsection:absorbing_model}

Fix a number of time steps $T \ge 2$, set of states $Q=\{q_0,q_1\}$ with $q_0$ the initial/reject state and $q_1$ the accept state, alphabet $\Sigma = \{\blank, \texttt{A}\}$, and input set
\[
I = \{\texttt{A}^1, \ldots, \texttt{A}^T\}.
\]
The noisy transition function $w_Q : \Sigma \times Q \longrightarrow \Delta Q$ is given in \cref{fig:absorbing_dfa}.
For an input $\texttt{A}^n$ we run the machine for exactly $T$ steps, padding the remaining $T - n$ steps with blanks $\blank$. The start state is $q_0$, and the output is the final state after $T$ steps.

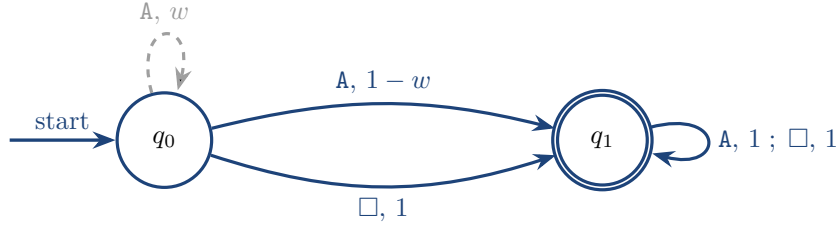
\begin{figure}[H]
\centering
\begin{tikzpicture}[
    >=Stealth,
    node distance=4.5cm,
    every node/.style={font=\small},
    state/.style={circle, draw=Accent!90!black, very thick, minimum size=1.25cm, inner sep=0pt},
    noisy/.style={draw=gray!75, dashed, line width=1.2pt}
  ]
  \node[state] (q0) {$q_0$};
  \node[state, double, right=of q0] (q1) {$q_1$};

  \draw[->, Accent!90!black, line width=1.25pt] ([xshift=-1.4cm]q0.west) -- (q0.west)
      node[midway, above] {start};

  \draw[->, Accent!90!black, line width=1.2pt] (q0) to[bend left=14]
      node[midway, above] {$\texttt{A},\,1-w$} (q1);
  \draw[->, Accent!90!black, line width=1.2pt] (q0) to[bend right=18]
      node[midway, below] {$\blank,\,1$} (q1);
  \draw[->, Accent!90!black, line width=1.2pt] (q1) edge[loop right]
      node {$\texttt{A},\,1\; ;\; \blank,\,1$} (q1);
  \draw[->, noisy] (q0) edge[loop above]
      node[text=gray!80] {$\texttt{A},\,w$} (q0);
\end{tikzpicture}
\caption{The single-parameter noisy absorbing DFA. The only noisy transition is the dashed self-loop on $q_0$: when the machine reads $\texttt{A}$ in state $q_0$, it fails to absorb into $q_1$ with probability $w$. The first blank acts as a repair boundary: it sends either state to $q_1$, after which the run is locked into the correct sink.}
\label{fig:absorbing_dfa}
\end{figure}

In this appendix we use the lookup pseudo-UTM \cref{sec:lookup_utm} to simulate machines. We remark that since its cycle map is the naive probabilistic extension of the eval cycle circuit (\cref{cor:lookup_srp}), it suffices to work only with $\Delta\texttt{cycle}_\mathrm{eval}$ (whose equations are given by \cref{prop:cycle_factor} with $\nu=\nu^{\mathrm{lookup}}$) and drop consideration of the UTM in the following analysis.\footnote{Indeed this was one reason for choosing the lookup UTM in the first place.}

We first simplify the general update equations for $\Delta\texttt{cycle}_\mathrm{eval}$ for noisy DFAs as follows. From \cref{prop:cycle_factor}, $\Delta\texttt{cycle}_\mathrm{eval}: W \times (\Delta \Sigma)^{\mathbb Z,\blank} \times \Delta Q \longrightarrow (\Delta\Sigma)^{\mathbb Z,\blank} \times \Delta Q $ is given by \[\Delta\texttt{cycle}_\mathrm{eval}(w, (\mathbf{y}_i)_{i\in\mathbb Z},\mathbf q) = ((\mathbf{y}'_i)_{i\in\mathbb Z},\mathbf q'),\] where
\begin{align}
  \mathbf q'(q)
    &=\sum_{\sigma^s, q^s} \mathbf{y}_0(\sigma^s)\mathbf{q}(q^s) w_{(\sigma^s,q^s),Q}(q)
      &&(q\in Q), \nonumber\\
  \boldsymbol\sigma'(\sigma)
    &=\sum_{\sigma^w, q^w} \mathbf{y}_0(\sigma^w)\mathbf{q}(q^w) w_{(\sigma^w,q^w),\Sigma}(\sigma)
      &&(\sigma\in\Sigma),\\
  \mathbf d(d)
    &=\sum_{\sigma^d, q^d} \mathbf{y}_0(\sigma^d)\mathbf{q}(q^d) w_{(\sigma^d,q^d),D}(d)
      &&(d\in D), \nonumber\\
  \mathbf{y}'_i(\sigma)
    &=\sum_{d\in D}\mathbf d(d)
      \bigl\{\one[i\neq -d]\,\mathbf{y}_{i+d}(\sigma)
      +\one[i=-d]\,\boldsymbol\sigma'(\sigma)\bigr\}
      &&(i\in\mathbb Z,\ \sigma\in\Sigma). \nonumber
\end{align}

A \emph{noisy DFA} (cf.\ \cref{def:DFA}) is a noisy Turing machine $w = (w_\Sigma, w_Q, w_D) : \Sigma \times Q \longrightarrow \Delta \Sigma \times \Delta Q \times \Delta D$ for which $w_{(\sigma, q), \Sigma}= e_\sigma$ (write what you read) and $w_{(\sigma,q),D} = e_R$ (always move right). So its noisy transition function is the data purely given by $w_Q \colon \Sigma \times Q \longrightarrow \Delta Q$.

\begin{proposition}\label{prop:dfa_matrix}
Let $w$ be a noisy DFA and suppose the tape holds a deterministic string, $\mathbf{y}_k = e_{\sigma_{k+1}}$ for $k \ge 0$, with an arbitrary state distribution $\mathbf q \in \Delta Q$.  Then $\Delta\texttt{cycle}_\mathrm{eval}(w, -)$ shifts the tape, $\mathbf{y}'_i = \mathbf{y}_{i+1}$, and acts linearly on the state,
\[
\mathbf{q}' = T_{\sigma_1}\, \mathbf{q},
\qquad
(T_\sigma)_{q, q^s} := w_{(\sigma,q^s),Q}(q),
\]
with $T_\sigma$ column-stochastic.  Iterating,
\[
(\Delta\texttt{cycle}_\mathrm{eval}(w,-))^n\bigl((\mathbf{y}_i)_{i \in \mathbb{Z}},\, \mathbf{q}\bigr) = \bigl((\mathbf{y}_{i + n})_{i \in \mathbb{Z}},\, T_{\sigma_n} \cdots T_{\sigma_1}\, \mathbf{q}\bigr).
\]
\end{proposition}

\begin{proof}
Since $d = 1$ in every summand of the tape update, and the normalised sums
\[
  \sum_{\sigma^d, q^d}\mathbf{y}_0(\sigma^d)\,\mathbf{q}(q^d) = 1,
  \qquad
  \sum_{q^w}\mathbf{q}(q^w) = 1
\]
collapse, the tape update reduces to $\mathbf{y}'_i = \mathbf{y}_{i+1}$: the tape shifts left in head-relative coordinates and remains deterministic.  Using $\mathbf{y}_0 = e_{\sigma_1}$ in the state update gives
\[
\mathbf{q}'(q) = \sum_{q^s \in Q} \mathbf{q}(q^s)\, w_{(\sigma_1,q^s),Q}(q) = (T_{\sigma_1}\, \mathbf{q})(q),
\]
and $T_\sigma$ is column-stochastic because each $w_{(\sigma,q^s),Q}$ is a probability distribution.  The iterated formula follows by induction: the shifted tape again satisfies the hypotheses, with the string $\sigma_2 \sigma_3 \cdots$ under the head.
\end{proof}

In the example of the noisy absorbing DFA, its noisy transition function is given by
\[
w_{(\texttt{A}, q_0),Q} = w \cdot q_0 + (1 - w) \cdot q_1,
\qquad
w_{(\texttt{A}, q_1),Q} = q_1,
\qquad
w_{(\blank, q),Q} = q_1 \text{ for all } q,
\]
with the abuse of notation denoting $w\in[0,1]$ as the probability of staying in $q_0$ when reading $\texttt{A}$ in state $q_0$,
which yields the transition matrices
\[
T_A(w) = \begin{pmatrix} w & 0 \\ 1-w & 1 \end{pmatrix},
\qquad
T_{\blank} = \begin{pmatrix} 0 & 0 \\ 1 & 1 \end{pmatrix}.
\]
Thus on symbol $\texttt{A}$, the machine stays in $q_0$ with probability $w$ and otherwise moves to $q_1$.  On a blank, the machine moves to $q_1$ with probability $1$.  Once the machine reaches $q_1$, it remains there forever. We think of small $w$ as perturbing away from the (classical) DFA given by $w=0$, so the ``true'' machine absorbs into $q_1$ immediately on every step. Hence the correct output is $q_1$ for every input in $I$.

\subsection{Exact error probabilities and Hessian blindness}\label{subsection:absorbing_hessian}

Recall from \cref{sec:bayes_setup} and \cref{def:model_utm} that the statistical model is given by
\begin{align*}
p(y|x,w)=\Delta\step^T(x,w)_y
&= (\pi_{\Delta Q} ((\Delta \texttt{cycle}_\mathrm{eval}(w,-))^T(e_{\tau_x}, e_{q_0})))_y\\
&= \langle e_y, T_{\blank}^{\,T-n}\, T_{\sigma_n}\dots T_{\sigma_1} e_{q_0} \rangle,
\end{align*}
where $\tau_x$ is the tape with deterministic input $x=\sigma_1\sigma_2\dots\sigma_n$ and the last equality is \cref{prop:dfa_matrix}. The true distribution is given by $q(y|x) = \one[y=y(x)] = p(y|x,0)$ with target map $y:I \to Q$ defined by $y(\texttt{A}^n) = q_1$ for all $1\leq n \leq T$. The error probability (\cref{def:error_probability}) is thus
\[
h_n(w) := h_{\texttt{A}^n}(w) = p(q_0|\texttt{A}^n, w),
\]
the probability that the machine is still in the wrong state $q_0$ after running on $\texttt{A}^n$ for $T$ steps given the possibility of errors $(\texttt{A},q_0) \to q_0$ occurring with probability $w$ during execution.

\begin{proposition}\label{prop:absorbing_error}
For the one-parameter absorbing DFA,
\[
h_n(w) = 0 \quad (1 \le n < T),
\qquad
h_T(w) = w^T.
\]
\end{proposition}

\begin{proof}
Rather than compute with the transition matrices of \cref{prop:dfa_matrix}, we argue directly due to the simplicity of the setup.

If $n < T$, then after the $n$ many $\texttt{A}$-steps there is at least one blank step. Regardless of the state at that moment, the first blank sends the machine to $q_1$ deterministically, so the final state cannot be $q_0$. Hence $h_n(w) = 0$.

If $n = T$, there are no blank steps. The machine ends in $q_0$ exactly when it takes the noisy self-loop on every one of the $T$ consecutive $\texttt{A}$-steps. Those events are conditionally independent along the unique $q_0$-path, so the error probability is $w^T$.
\end{proof}

Let $q_n := q(\texttt{A}^n)$ with $\sum_{n=1}^T q_n = 1$. Specialising the loss to the one-parameter slice,
\[
H(w) := \sum_{n=1}^T q_n h_n(w)^2,
\]
\cref{prop:absorbing_error} yields the exact formula
\begin{equation}\label{eq:H_absorbing}
H(w) = q_T\, w^{2T}.
\end{equation}
Thus only the terminal input $\texttt{A}^T$ contributes any loss at all in this reduced model.

Differentiating \eqref{eq:H_absorbing} yields
\[
H'(w) = 2T q_T w^{2T-1},
\qquad
H''(w) = 2T(2T-1) q_T w^{2T-2}.
\]
For $T = 1$ the model is regular, with $H''(0) = 2 q_1$. For $T \ge 2$, however,
\[
H''(0) = 0.
\]
In the singular regime $T \ge 2$ and $q_T > 0$, the zero set is just $\{0\}$, yet the Hessian still vanishes. So the Hessian completely misses an error direction.

\begin{remark}
The parameter $w$ measures a \emph{persistent failure to absorb} on symbol $\texttt{A}$. A single error is not enough to change the output, and this is all the Hessian can see. There needs to be $T$ errors in this transition to result in an incorrect final state of the machine. The error probability therefore starts at order $w^T$, and the squared loss starts at order $w^{2T}$. For $T \ge 2$ there is no quadratic term for the Hessian to see.
\end{remark}

\subsection{The Gibbs distribution and monomial asymptotics}\label{subsection:absorbing_gibbs}

For tractability we take the prior $\varphi(w)$ on the one-parameter slice $w \in [0,1]$ to be uniform. The choice is inessential: for any prior positive and continuous at $w = 0$, the factor $\varphi(0)$ cancels between the partition function and every moment of the Gibbs distribution to leading order, so the coefficients below, and hence the signs and relative magnitudes of $\chi_{\texttt{A}^n}$, are unchanged.

The Gibbs distribution is
\begin{equation}\label{eq:gibbs_absorbing}
p^{\beta}(w) = \frac{1}{Z_{\beta}}\, \exp\bigl(-\beta\, q_T\, w^{2T}\bigr)\, \mathbf{1}_{[0,1]}(w)\, dw,
\qquad
Z_{\beta} = \int_0^1 e^{-\beta q_T w^{2T}}\, dw.
\end{equation}
The asymptotic moments of the squared-error loss $H$ under this distribution are controlled by the lower incomplete gamma function
\[
\gamma(s, x) := \int_0^x u^{s-1} e^{-u}\, du.
\]

\begin{lemma}\label{lemma:absorbing_gamma}
For the Gibbs distribution~\eqref{eq:gibbs_absorbing} with uniform prior,
\begin{equation}\label{eq:EH_absorbing}
\langle H \rangle = \frac{1}{2T\,\beta} + o(\beta^{-1}),
\qquad
\Var(H) = \frac{1}{2T \beta^2} + o(\beta^{-2}).
\end{equation}
\end{lemma}

\begin{proof}
Write $p := 2T$ for brevity. Recall that $\gamma(s, x) = \Gamma(s) + O(x^{s-1} e^{-x})$ as $x \to \infty$, where $\Gamma(s) := \int_0^\infty u^{s-1} e^{-u}\, du$: the complementary upper incomplete gamma satisfies $\Gamma(s, x) = x^{s-1} e^{-x}(1 + O(x^{-1}))$, so in Laplace integrals over the finite interval $[0, 1]$, replacing the upper limit by $\infty$ changes the leading term only by an exponentially small tail,
\begin{equation}\label{eq:lowertail_absorbing}
\gamma(s, x) = \Gamma(s) + O\bigl(x^{s-1} e^{-x}\bigr) \qquad (x \to \infty).
\end{equation}

For the partition function, substitute $u = \beta q_T w^p$, so that $w = (u/(\beta q_T))^{1/p}$ and $dw = \tfrac{1}{p}\, (\beta q_T)^{-1/p}\, u^{1/p - 1}\, du$, giving
\[
Z_{\beta} = \int_0^1 e^{-\beta q_T w^p}\, dw = \frac{1}{p}\, (\beta q_T)^{-1/p}\, \gamma\!\left(\tfrac{1}{p},\, \beta q_T\right).
\]
The same substitution yields, for each integer $k \ge 0$,
\[
\int_0^1 w^{kp} e^{-\beta q_T w^p}\, dw = \frac{1}{p}\, (\beta q_T)^{-k - 1/p}\, \gamma\!\left(k + \tfrac{1}{p},\, \beta q_T\right),
\]
and hence
\[
\langle w^{kp} \rangle = (\beta q_T)^{-k}\, \frac{\gamma\!\left(k + \tfrac{1}{p},\, \beta q_T\right)}{\gamma\!\left(\tfrac{1}{p},\, \beta q_T\right)}.
\]
Multiplying by suitable powers of $q_T$ gives the exact moments of $H = q_T w^p$:
\[
\langle H \rangle = \frac{1}{\beta}\, \frac{\gamma\!\left(1 + \tfrac{1}{p},\, \beta q_T\right)}{\gamma\!\left(\tfrac{1}{p},\, \beta q_T\right)},
\qquad
\langle H^2 \rangle = \frac{1}{\beta^2}\, \frac{\gamma\!\left(2 + \tfrac{1}{p},\, \beta q_T\right)}{\gamma\!\left(\tfrac{1}{p},\, \beta q_T\right)}.
\]
Applying~\eqref{eq:lowertail_absorbing} and the gamma recurrence $\Gamma(s+1) = s\, \Gamma(s)$, and using
\[
\frac{\Gamma(1 + 1/p)}{\Gamma(1/p)} = \frac{1}{p},
\qquad
\frac{\Gamma(2 + 1/p)}{\Gamma(1/p)} = \frac{1 + 1/p}{p},
\]
we obtain
\[
\langle H \rangle = \frac{1}{p\,\beta} + o(\beta^{-1}),
\qquad
\langle H^2 \rangle = \frac{1 + 1/p}{p \beta^2} + o(\beta^{-2}).
\]
Subtracting the square of the mean yields
\[
\Var(H) = \frac{1 + 1/p}{p \beta^2} - \frac{1}{p^2 \beta^2} + o(\beta^{-2}) = \frac{1}{p \beta^2} + o(\beta^{-2}).
\]
Setting $p = 2T$ gives~\eqref{eq:EH_absorbing}.
\end{proof}

The proof also gives the local learning coefficient of this one-parameter model,
\[
\lambda = \frac{1}{2T},
\]
since $Z_{\beta} \asymp \beta^{-1/(2T)}$, though this can also be read off $H$ directly.

\begin{remark}
This value of $\lambda$ agrees with, and in fact saturates, the general bound
\[
\lambda([M], q) \le \frac{d}{2(C + 1)}
\]
proved in \cite{murfet2025pas} for a Turing machine $M$ that corrects every error syndrome of weight $\le C$. In the one-parameter slice we have $d = 1$, and the DFA produces the correct output on every input in $I$ under any error syndrome of weight strictly less than $T$: only when all $T$ consecutive $A$-reads of input $\texttt{A}^T$ fail does the fault reach the output. So $C = T - 1$, and the bound reduces to $\lambda \le 1/(2T)$, which is attained with equality here.
\end{remark}

\subsection{Susceptibilities}\label{subsection:absorbing_chi}

For input $\texttt{A}^n$ define the centred per-input squared loss
\[
\Delta_n(w) := h_n(w)^2 - H(w).
\]
Using \cref{prop:absorbing_error} and \eqref{eq:H_absorbing},
\begin{equation}\label{eq:Delta_absorbing}
\Delta_n(w) =
\begin{cases}
-H(w), & 1 \le n < T,\\[1mm]
\dfrac{1 - q_T}{q_T}\, H(w), & n = T.
\end{cases}
\end{equation}
The per-input susceptibility is $\chi_{\texttt{A}^n} := -\Cov[H, \Delta_n]$, with the covariance taken under \eqref{eq:gibbs_absorbing}. (This matches \cref{def:susceptibility}, which uses the per-input squared error.  Were the negative log-likelihood used instead, \cref{rem:mu_shift} gives $\ell_x - \kappa_\mu(0) = \tfrac12 \kappa''_\mu(0)\, h_x^2 + O(h_x^3)$ near $[M]$, so the two observables agree to leading order up to the input-independent constant $\tfrac12 \kappa''_\mu(0) > 0$, and the signs and relative magnitudes below transfer to the negative log-likelihood susceptibilities unchanged.)  Since $\Delta_n$ is a scalar multiple of $H$, \eqref{eq:EH_absorbing} immediately gives the asymptotics.

\begin{proposition}\label{prop:chi_absorbing}
For the one-parameter noisy absorbing DFA,
\begin{equation}\label{eq:chi_absorbing}
\chi_{\texttt{A}^n} =
\begin{cases}
\dfrac{1}{2T \beta^2} + o(\beta^{-2}), & 1 \le n < T, \\[3mm]
-\dfrac{1 - q_T}{2 T q_T\, \beta^2} + o(\beta^{-2}), & n = T.
\end{cases}
\end{equation}
\end{proposition}

\begin{proof}
If $n < T$, then $\Delta_n = -H$, so by \cref{lemma:absorbing_gamma}
\[
\chi_{\texttt{A}^n} = -\Cov[H, -H] = \Var(H) = \frac{1}{2T \beta^2} + o(\beta^{-2}).
\]
If $n = T$, then $\Delta_T = \tfrac{1 - q_T}{q_T}\, H$, and the same lemma gives
\[
\chi_{\texttt{A}^T} = -\tfrac{1 - q_T}{q_T}\, \Var(H) = -\frac{1 - q_T}{2 T q_T\, \beta^2} + o(\beta^{-2}).
\]
\end{proof}

\subsection{Computational interpretation}\label{subsection:absorbing_interpretation}

The machine has a single failure mode: it can fail to leave $q_0$ while reading $\texttt{A}$'s, and $w$ is the probability of this failure on a single $\texttt{A}$-step.

For an input $\texttt{A}^n$ padded to length $T$, the run splits into two phases:
\[
\underbrace{A \cdots A}_{n\text{ steps}}\; \underbrace{\blank \cdots \blank}_{T - n\text{ steps}}.
\]
During the $\texttt{A}$-processing phase, a faulty trajectory can remain in $q_0$. Once the first blank is read, both states map deterministically to $q_1$. So the first blank is a \emph{repair boundary}: after that point, the faulty and correct trajectories merge and the future evolution is identical. This is why \cref{prop:absorbing_error} has the dichotomy $h_n(w) = 0$ for $n < T$ and $h_T(w) = w^T$: if $n < T$, the run crosses the repair boundary before the output is read.  If $n = T$, there is no blank-reset phase before halting, so the failure remains visible.

The sign of $\chi_{\texttt{A}^n}$ records which side of the repair boundary the output is read on. For $n < T$ the computation corrects every error pattern, so the likelihood of $\texttt{A}^n$ carries no information about $w$, and only $\texttt{A}^T$ constrains the Gibbs distribution in the $w$-direction. Upweighting $\texttt{A}^T$ increases the coefficient of the only error term $w^{2T}$ in the potential, so the Gibbs distribution narrows in the $w$-direction and the expected loss decreases: $\chi_{\texttt{A}^T} < 0$. Upweighting a shorter input downweights $\texttt{A}^T$, so the Gibbs distribution broadens and $\chi_{\texttt{A}^n} > 0$. In summary:
\begin{center}
\begin{tabular}{@{}lll@{}}
\toprule
Input & Sign & Interpretation \\
\midrule
$\texttt{A}^T$ & $\chi_{\texttt{A}^T} < 0$ & exposes the fault, and the Gibbs distribution concentrates \\
$\texttt{A}^n$, $n < T$ & $\chi_{\texttt{A}^n} > 0$ & corrects the fault, and the Gibbs distribution broadens \\
\bottomrule
\end{tabular}
\end{center}

As for the magnitudes in~\eqref{eq:chi_absorbing}: since each $\Delta_n$ is a scalar multiple of $H$ \eqref{eq:Delta_absorbing}, each $\chi_{\texttt{A}^n}$ is the corresponding multiple of $-\Var(H)$, whose leading order is the common factor $1/(2T \beta^2)$ (\cref{lemma:absorbing_gamma}). The common factor decreases when $\beta$ grows (the Gibbs distribution concentrates) and when $T$ grows (a longer uninterrupted run of noisy $\texttt{A}$-steps is needed for the fault to reach the output). Every input with $n < T$ has the same limit $2T \beta^2 \chi_{\texttt{A}^n} \to 1$ as $\beta \to \infty$: once a blank appears, the fault is erased, and the number of remaining blanks is irrelevant. For $\texttt{A}^T$, $2T \beta^2 \chi_{\texttt{A}^T} \to -(1 - q_T)/q_T$, which is large in magnitude when $q_T$ is small.

\subsection{Extension: noisy writes and flat directions}\label{subsection:absorbing_writes}

We extend the example to a model that is no longer strictly a noisy DFA, keeping the machine as before but allowing the write distribution $w_\Sigma : \Sigma \times Q \to \simplex\Sigma$ to be noisy as well.  Rename the transition parameter of the previous subsections from $w$ to $a$, freeing $w$ for the full parameter $w = (a, b)$.  The classical machine writes what it reads, $w_\Sigma(\sigma, q) = e_\sigma$, and we parameterise deviations by
\[
b_{(\sigma, q)} := w_{(\sigma, q),\Sigma}(\bar\sigma) \;\in\; [0, 1] \qquad \text{for } (\sigma, q) \in \Sigma \times Q,
\]
where $\bar\sigma$ is the other symbol: $b_{(\sigma,q)}$ is the probability of writing the wrong symbol, the classical code has all four $b_{(\sigma,q)} = 0$, and the parameter space has dimension $d = 5$.

Since the head moves right at every step, a noisy write lands on the square behind the head and is never re-read, while the tape ahead of the head remains a deterministic shift.  The state evolution therefore depends only on the input string and on $a$, so
\[
h_n(a, b) = h_n(a),
\qquad
H(w) = q_T\, a^{2T},
\]
unchanged from \cref{prop:absorbing_error} and \eqref{eq:H_absorbing}: the four write parameters parametrise a four-dimensional family of distinct noisy machines computing the same input--output map.

The write components accordingly have zero susceptibility.  On the slice through $[M]$ along a write component $C_b$ only $b_{(\sigma,q)}$ varies and $a = 0$, so $H \equiv 0$ there, $\phi_{C_b} \equiv 0$, and $\chi^{C_b}_x = 0$ for every input $x \in I$.  The susceptibilities along $a$ and the local learning coefficient $\lambda = 1/(2T)$ are unchanged, the integration over $b \in [0,1]^4$ contributing only a constant factor to the partition function.  Geometrically, the loss factors through the projection $(a, b) \mapsto a$, so $W_0 = \{a = 0\} \times [0,1]^4$ is a four-dimensional flat fibre over the singular point of the one-parameter problem, and the $b$-parameters are free directions tangent to $W_0$ at $[M]$.  The example thus separates two ways a parameter direction can fail to be regular: along a curved direction ($a$, where the Hessian vanishes but the loss is nonzero at higher order) the per-input susceptibilities are nonzero with informative signs, while along a free direction (the writes) they vanish for every input.

\end{document}